\documentclass{article}
\usepackage{ijcai21}

\usepackage{times}

\usepackage{soul}
\usepackage{url}
\usepackage[hidelinks]{hyperref}
\usepackage[utf8]{inputenc}
\usepackage[small]{caption}
\usepackage{graphicx}
\usepackage{booktabs}
\title{Addressing the Long-term Impact of ML Decisions via Policy Regret}

\author{
David Lindner$^1$ \and
Hoda Heidari$^2$ \and
Andreas Krause$^1$
\affiliations
$^1$ETH Zurich\\
$^2$Carnegie Mellon University\\
\emails
\{lindnerd, krausea\}@ethz.ch,
hheidari@cmu.edu
}

\usepackage[numbers,sort&compress]{natbib}

\renewcommand{\citet}[1]{\citeauthor{#1},~\citeyear{#1}~[\citenum{#1}]}

\usepackage{amsmath}
\usepackage{amsthm}
\usepackage{amsfonts,amssymb}
\usepackage[nameinlink]{cleveref}
\usepackage{braket}
\usepackage{algorithm}
\usepackage{algorithmicx}
\usepackage{algpseudocode}

\newcommand{\legendHeight}{0.5ex}
\def\legendRaiseby{0.5ex}
\newcommand{\legendDUCB}{\raisebox{\legendRaiseby}{\includegraphics[height=\legendHeight]{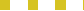}}\xspace}
\newcommand{\legendSWUCB}{\raisebox{\legendRaiseby}{\includegraphics[height=\legendHeight]{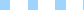}}\xspace}
\newcommand{\legendEXP}{\raisebox{\legendRaiseby}{\includegraphics[height=\legendHeight]{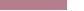}}\xspace}
\newcommand{\legendREXP}{\raisebox{\legendRaiseby}{\includegraphics[height=\legendHeight]{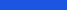}}\xspace}
\newcommand{\legendOSO}{\raisebox{\legendRaiseby}{\includegraphics[height=\legendHeight]{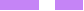}}\xspace}
\newcommand{\legendGREEDY}{\raisebox{\legendRaiseby}{\includegraphics[height=\legendHeight]{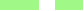}}\xspace}
\newcommand{\legendUCB}{\raisebox{\legendRaiseby}{\includegraphics[height=\legendHeight]{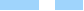}}\xspace}
\newcommand{\legendSPO}{\raisebox{\legendRaiseby}{\includegraphics[height=\legendHeight]{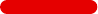}}\xspace}
\newcommand{\legendOPTIMAL}{\raisebox{\legendRaiseby}{\includegraphics[height=\legendHeight]{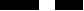}}\xspace}
\newcommand{\legendINCREASING}{\raisebox{\legendRaiseby}{\includegraphics[height=\legendHeight]{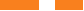}}\xspace}

\usepackage{thmtools}
\usepackage{thm-restate}
\usepackage{balance}

\theoremstyle{definition}
\newtheorem{definition}{Definition}

\theoremstyle{plain}
\newtheorem{lemma}{Lemma}
\newtheorem*{lemma*}{Lemma}

\newtheorem*{theorem*}{Theorem}

\newtheorem*{corollary*}{Corollary}

\newcommand{\bbE}{\mathbb{E}}
\newcommand{\cA}{\mathcal{A}}

\newcommand{\cO}{\mathcal{O}}

\newcommand{\N}{\mathbb{N}}
\newcommand{\e}{\mathrm e}

\newcommand\abs[1]{\left| #1 \right|}
\newcommand\cbr[1]{\left\{ #1 \right\}}
\newcommand\rbr[1]{\left( #1 \right)}
\newcommand\sbr[1]{\left[ #1 \right]}

\newcommand\OPT{\mathtt{OPT}}

\newcommand{\argmax}{\mathop{\mathrm{argmax}}}

\usepackage{xspace}   %
\newcommand{\banditname}{single-peaked bandit\xspace}
\newcommand{\banditnamePl}{single-peaked bandits\xspace}
\newcommand{\alglong}{Single-Peaked Optimism\xspace}
\newcommand{\algshort}{\texttt{SPO}\xspace}

\newcommand{\codeurl}{\mbox{\url{https://github.com/david-lindner/single-peaked-bandits}}}
\newcommand{\extendedurl}{\mbox{\url{https://las.inf.ethz.ch/files/lindner21addressing.pdf}}}
\newboolean{cameraready}
\newboolean{noappendix}

\setboolean{cameraready}{false}
\setboolean{noappendix}{false}
\ifthenelse{\boolean{noappendix}}{\nofiles}{}

\usepackage{subcaption}
\usepackage[dvipsnames]{xcolor}

\begin{document}

\maketitle

\begin{abstract}
Machine Learning (ML) increasingly informs the allocation of opportunities to individuals and communities in areas such as lending, education, employment, and beyond. 
Such decisions often impact their subjects' future characteristics and capabilities in an a priori unknown fashion. The decision-maker, therefore, faces exploration-exploitation dilemmas akin to those in multi-armed bandits.
Following prior work, we model communities as arms. To capture the long-term effects of ML-based allocation decisions, we study a setting in which the reward from each arm evolves every time the decision-maker pulls that arm. We focus on reward functions that are initially increasing in the number of pulls but may become (and remain) decreasing after a certain point. We argue that an acceptable sequential allocation of opportunities must take an arm's potential for growth into account. We capture these considerations through the notion of \emph{policy} regret, a much stronger notion than the often-studied \emph{external} regret, and present an algorithm with provably sub-linear policy regret for sufficiently long time horizons. We empirically compare our algorithm with several baselines and find that it consistently outperforms them, in particular for long time horizons.
\end{abstract}

\section{Introduction}\label{sec:introduction}

\begin{figure*}[t]\centering
   \includegraphics[width=0.32\linewidth]{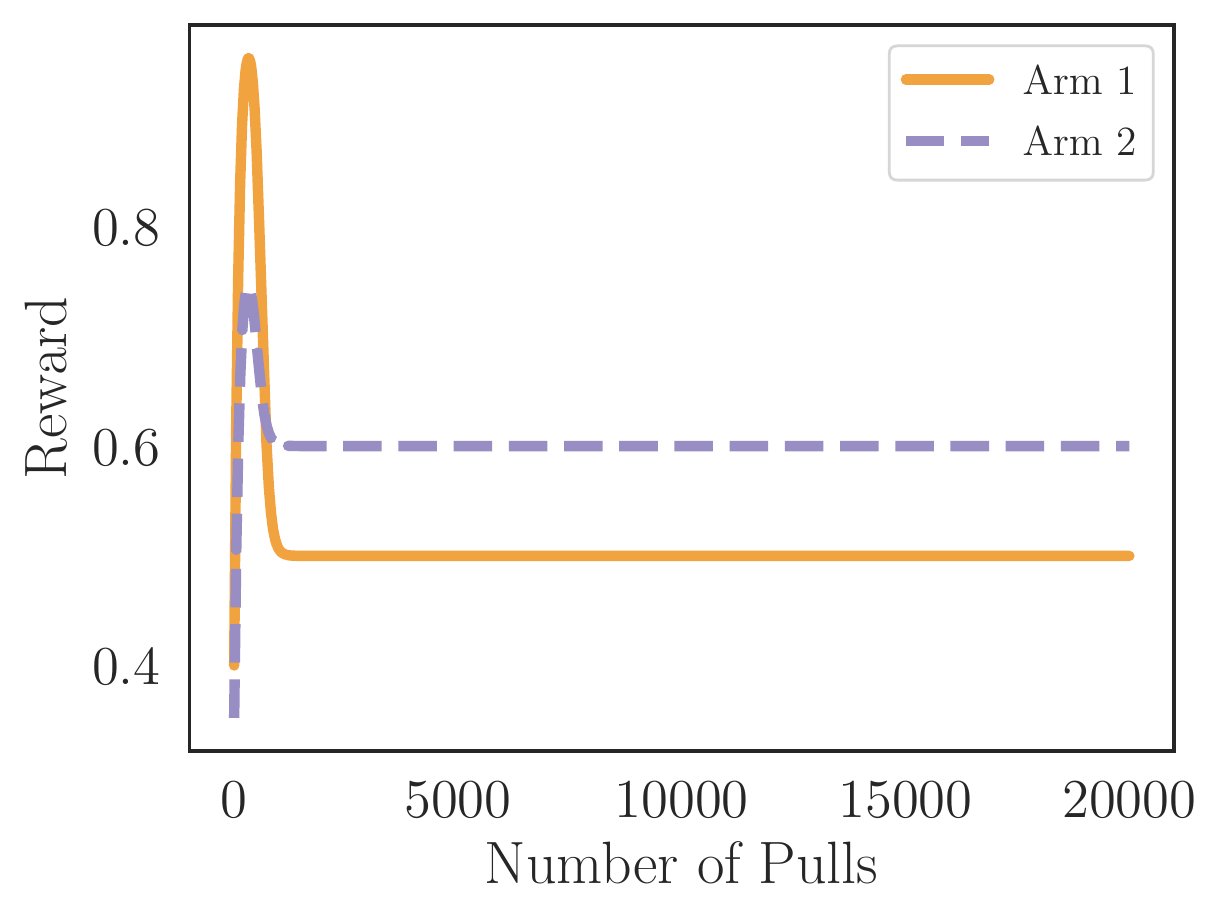}\hfill
   \includegraphics[width=0.32\linewidth]{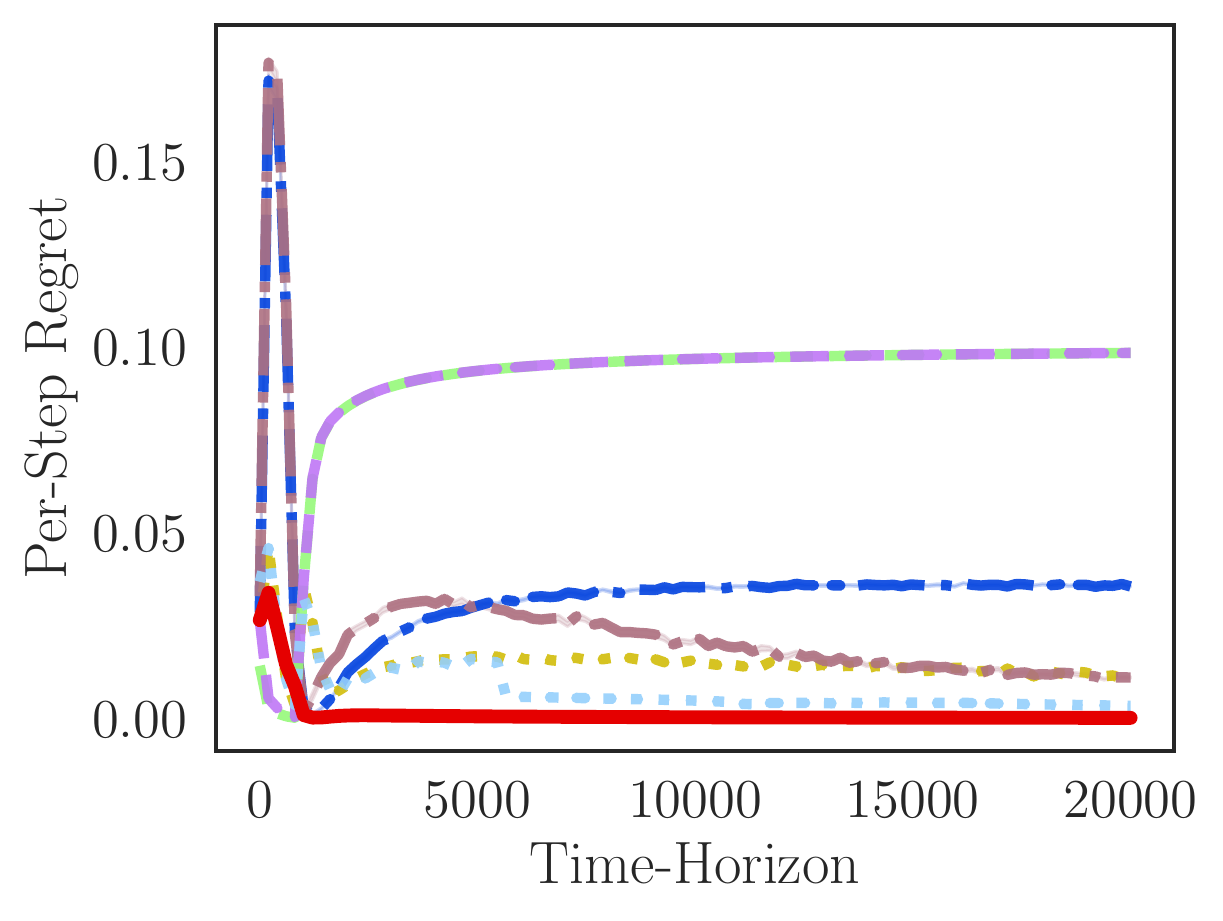}\hfill
   \includegraphics[width=0.32\linewidth]{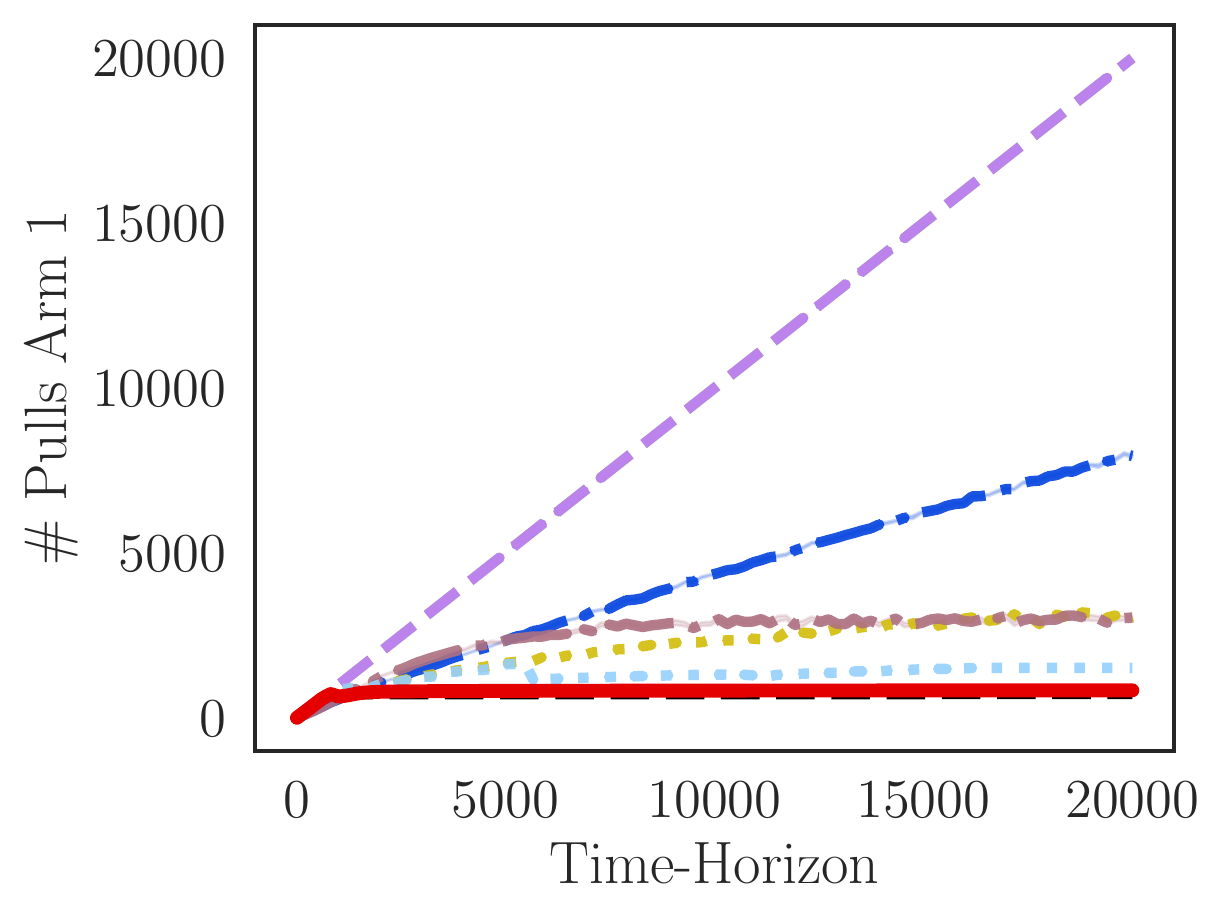}
   \caption{The left plot shows a \banditname with two reward functions, modeling the evolution of rewards in the number of times each arm is pulled. For long time horizons, the optimal strategy is to play Arm 2 because it has a higher asymptotic reward. However, bandit algorithms that maximize external regret fail to recognize this because the initial reward of Arm 2 is smaller than the initial and asymptotic reward of Arm 1. The middle plot shows the regret of a greedy-selection strategy (\legendGREEDY), EXP3~\cite{auer2002nonstochastic}~(\legendEXP), which minimizes external regret, as well as D-UCB~\cite{garivier2011upper}~(\legendDUCB), SW-UCB~\cite{garivier2011upper}~(\legendSWUCB), and R-EXP3~\cite{auer2002nonstochastic}~(\legendREXP), three bandit algorithms designed for nonstationary bandits. All of these algorithms fail on the \banditname. We propose \algshort (\legendSPO) which achieves sub-linear policy regret in \banditnamePl settings. The right plot shows how often each algorithm pulls the first arm. The plot shows that \algshort stops pulling Arm 1 much earlier than the other algorithms, and is much closer to the optimal policy (\legendOPTIMAL). For more details on our experiments, see \Cref{sec:experiments}.}
   \label{fig:synthetic}
\end{figure*}

Machine learning (ML) systems increasingly inform or make high-stakes decisions about people, in areas such as credit lending~\cite{dobbie2018measuring}, education~\cite{marcinkowski2020implications}, criminal justice~\cite{berk2015machine}, employment~\cite{sanchez2020does}, and beyond. These ML-based decisions can negatively impact already-disadvantaged individuals and communities~\cite{sweeney2013discrimination,buolamwini2018gender,propublica}. This realization has spawned an active area of research into quantifying and mitigating the disparate effects of ML~\cite{dwork2012fairness,kleinberg2016inherent,hardt2016equality}.
Much of this work has focused on the \emph{immediate predictive disparities} that arise when supervised learning techniques are applied to batches of training data sampled from a \emph{fixed} underlying population~\cite{dwork2012fairness,zafar2017dmt,hardt2016equality,kleinberg2016inherent}. While such approaches capture important types of disparity, they fail to account for the \emph{long-term} effects of present decisions on individuals and communities. Recent work has advocated for shifting the focus to societal-level implications of ML in the long run~\cite{liu2018delayed,hu2018disparate,heidari2019long,dong2018strategic,milli2018social}.

In many real-world domains, decisions made today correspond to the allocation of opportunities and resources that impact the recipients' future characteristics and capabilities. In such settings, we argue that a socially and ethically acceptable allocation of opportunities must account for the recipients' \emph{long-term potential} for turning resources into social utility. As an example, consider the following stylized scenario:
Suppose a decision-maker must allocate funds to several communities, all residing in one city, at the beginning of every fiscal period. The communities have distinct racial and wealth compositions, and for historical reasons, they initially have different capabilities to turn their allocated funds into economic prosperity and welfare for members of the community and the city. The decision-maker does not know ahead of time how the economic capabilities of each community will evolve in response to the funds allocated to it. Moreover, he/she can only observe the return on each possible allocation strategy \emph{after} employing it. While the decision-maker does not know the precise return-on-investment or \emph{reward} curves associated with each community in advance, domain knowledge may provide him/her with information about the general shape of such curves. For instance, he/she may be able to reliably assume that reward curves are often initially increasing with diminishing marginal returns; and if investment continues beyond a point of saturation, they exhibit decreasing returns to additional investments.

How should a just-minded decision-maker allocate funds in this hypothetical example? Should he/she always aim for equal allocation of funds in every fiscal period to ensure a form of distributive equality today, or are there cases\footnote{For example, such considerations may come to the fore once all communities have received a reasonable minimum budget.} in which he/she should additionally take each community's potential for growth into account and allocate funds proportionately?
Note that in this example, a myopic decision maker might neglect disadvantaged communities with high long-term potential to turn funds into welfare, and as a result, amplify disparities between advantaged and disadvantaged communities over time.
If the decision-maker aims to \emph{maximize the city's long-term economic welfare and prosperity}, he/she should prioritize communities that produce higher returns on investment over time. Aside from the utilitarian argument for this objective, it can also be justified through the classic \emph{fitness argument} to justice and fairness, which states that \emph{resources and opportunities must be allocated to those who make the best use of them}~\cite{sandel2010justice,moulin2004fair}.\footnote{We emphasize that in many domains, considerations such as \emph{need} and \emph{rights} should take precedence to \emph{fitness} as defined in our stylized example. In certain domains, however, fitness can be one of the key criteria in determining whether an allocation is morally acceptable. For the sake of simplicity and concreteness, we solely focus on this particular factor. It is worth noting that our model and findings are equally applicable to settings in which needs or entitlements change in response to the sequence of allocation decisions.}
This objective motivates the algorithmic question we focus on in this work: how should the decision-maker choose a sequence of allocations to ensure that communities receive funds in proportion to their relative potential for producing high reward for society in the long-run?

Motivated by the above example and numerous other real-world domains in which similar concerns arise,\footnote{Additional real-world examples that fit into our model include: allocating policing resources to neighborhoods to maximize safety; allocating funds to research institutions to maximize scientific discoveries and innovations; allocating loans to students to maximize the rate of graduation/ loan pay-back.}we study a \emph{multi-armed-bandit (MAB) setting} in which communities correspond to arms and the reward from each arm \emph{evolves} every time the decision-maker pulls that arm.
We consider a decision-maker who aims to maximize the overall reward obtained within a set time-horizon, but because he/she does not know how the reward curves evolve, he/she is bound to incur some regret. We formulate the decision-maker's goal in this sequential setting as achieving low \emph{policy} regret~\cite{Arora2012}. As \Cref{fig:synthetic} shows, conventional no \emph{external} regret algorithms ignore the impact of their decisions today on the evolution of rewards, so they are prone to spending many of their initial pulls on arms that exhibit high immediate rewards but lack adequate potential for growth.

\paragraph{Technical findings.} We study \emph{\banditnamePl}, a new MAB setting with reward functions that are initially increasing and concave in the number of pulls but can become decreasing at some point (\Cref{sec:model}). We introduce \emph{\alglong} (\algshort), a novel algorithm that considers potential long-term effects of pulling different arms (\Cref{sec:algorithm}). We prove that \algshort achieves sub-linear policy regret if rewards can be observed free of noise (\Cref{sec:theory}). Further, we present an LP-based heuristic that effectively handles noisy reward observations (\Cref{sec:noise}). We empirically compare \algshort with several standard no-external-regret algorithms and additional baselines, and find that \algshort consistently performs better, in particular, for long time horizons (\Cref{sec:experiments}).

\paragraph{Broader implications.}
Our work takes \emph{conceptual} and \emph{technical} steps toward modeling and analyzing the long-term implications of ML-informed allocations made over time.
From a conceptual point of view, our work showcases the importance of accounting for \emph{domain knowledge} (here, the general shape of reward curves) and \emph{social-scientific insights} (e.g., the dynamic by which communities evolve in response to allocation policies) to formulate ML's long-term impact. Our results draw attention to the necessity of understanding social dynamics of a domain for designing allocation algorithms that improve equity and fairness in the long-run.
From a technical perspective, we believe our work can serve as a stepping stone toward designing and analyzing better no-policy regret algorithms for domain-specific reward curves beyond those considered here. Our work is directly applicable to a specific class of reward functions which generalizes and subsumes those in prior work (e.g.,  \citep{Heidari2016}). Finally, our novel approach to handling noise allows utilizing the proposed algorithm in more practical settings where observed rewards are expected to be noisy.
\subsection{Related Work}\label{sec:related}
\looseness -1 Much of the existing work on the social implications of ML focuses on disparities in a model's predictions~\cite{dwork2012fairness,zafar2017dmt,hardt2016equality,kleinberg2016inherent}. However, these approaches are only suited to evaluate \emph{one-shot} decision scenarios. In contrast, we formalize disparities that arise when making a \emph{sequence} of allocation decisions. Recent work has initiated the study of longer-term consequences and effects of ML-based decisions on people, communities, and society. For example, \citet{liu2018delayed} and \citet{kannan2018downstream} study how a utility-maximizing decision-maker may interpret and use ML-based predictions. \citet{dong2018strategic}, \citet{hu2018disparate}, and \citet{milli2018social} address \emph{strategic classification}, a setting in which decision subjects are assumed to respond \emph{strategically} and \emph{untruthfully} to the choice of the classification model, and the goal is to design classifiers that are robust to strategic manipulation. \citet{hu2018short} study the impact of enforcing statistical parity on hiring decisions made in a temporary labor market that precedes a permanent labor market. 
\citet{mouzannar2019fair} and \citet{heidari2019long} model the dynamics of how members of a population react to a selection rule by changing their qualifications (defined in terms of true labels or feature vectors). However, none of the prior articles investigate the community-level implications of ML-based decision-making policies over \emph{multiple time-steps}.

Another conceptually-relevant line of work studies fairness in online learning~\cite{joseph2016fairness,joseph2017fair,jabbari2017fairness}.
\citet{joseph2016fairness}, for example, study fairness in the MAB setting, where arms correspond to socially salient groups (e.g., racial groups), and pulling an arm is equivalent to choosing that group (e.g., to allocate a loan to). They consider an algorithm fair if it never prefers a worse arm to a better one, that is, the arm chosen by the algorithm never has a lower expected reward than the other arms.
Similar to \citet{joseph2016fairness}, in our running example, each arm corresponds to a community. However, instead of imposing short-term notions of fairness, we focus on longer-term implications and disparities arising from present decisions.

Our model is based on the MAB framework, which has been established as a powerful tool for modeling sequential decision-making, and has been used successfully for many decades and across a wide range of real-world domains~\cite{gittins2011multi,bubeck2012regret}.
From a technical perspective we study a \emph{nonstationary} bandit problem. In nonstationary bandits (with limits to the change of the reward distributions), modified versions of common bandit algorithms have strong theoretical guarantees and good empirical performance. For example, if the reward distributions only change a small number of times, variants of the upper confidence bound algorithm (UCB) such as \emph{discounted} or \emph{sliding-window} UCB perform well~\cite{garivier2011upper}. Similarly, if there is a fixed budget on how much the rewards can change, R-EXP3, a variant of the popular EXP3 algorithm for adversarial bandits~\cite{auer2002nonstochastic}, guarantees low regret~\cite{besbes2019optimal}. In this work, we do not restrict how much the rewards can change, but instead restrict the functional shape of the reward functions to be first increasing and concave before switching to decreasing. This is somewhat similar to \emph{rotting bandits}~\cite{levine2017rotting} and \emph{recharging bandits}~\cite{kleinberg2018recharging}, but distinct from them in crucial ways. In contrast to rotting bandits, where the rewards decrease when pulling an arm more often, our reward functions first increase and only later decrease with the number of pulls. In contrast to recharging bandits, where rewards are increasing and concave in the amount of time an arm has \emph{not} been pulled, we consider a bandit setting with rewards depending on the number of times an arm \emph{has} been pulled.
We consider reward functions that exhibit a ``unimodal'' shape. However, our setting is very different from \emph{unimodal bandits}, which are stationary bandit models with a unimodal structure across arms~\cite{yu2011unimodal,combes2014unimodal}.

The setting by \citet{Heidari2016} is closest to ours. They consider two separate models, one with rewards that are increasing and concave, and another with decreasing rewards in the number of pulls of an arm. While \citep{Heidari2016} provides different algorithms for these two cases, we present a single algorithm that can adapt to both settings and beyond, while matching the respective asymptotic policy regret bounds in \citep{Heidari2016} (cf. \Cref{app:monotone_bandits}\ifthenelse{\boolean{cameraready}}{\footnote{All appendicies can be found in the extended paper at: \extendedurl{}}}{}).  Additionally, in contrast to \citep{Heidari2016} which primarility studies noise-free observations, we provide an effective heuristic for handling noise.

For an extended discussion of prior work, see \Cref{app:more_related_work}.
\section{The Single-Peaked Bandit Setting}\label{sec:model}

We consider a multi-armed bandit (MAB) with arms $\Set{1, \dots, N}$, corresponding, e.g., to the different communities in our introductory example. At each time step $t = 1, 2, \dots T$, the decision-maker pulls one arm and observes its immediate reward, e.g., the short-term outcome of an investment in a community. The decision-maker aims to achieve the highest cumulative reward within the fixed time horizon $T$, e.g., he/she wants to get the best total return-on-investment over $T$ years. Each arm $i$ has an underlying reward function $f_i: \Set{1, \dots, T} \to [0,1]$. When the decision-maker pulls arm $i$ for the $m$-th time ($1 \leq m \leq T$), he/she observes reward $f_i(m)$. Later, in Section~\ref{sec:noise}, we study noisy reward observations of the form $f_i(m) + \epsilon_i$, but for now, let's assume observed rewards are noise-free. We denote the cumulative reward of arm $i$ after $m$ pulls by $F_i(m) = \sum_{t=1}^m f_i(t)$.

A deterministic \emph{policy} $\pi$ is a sequence of mappings $(\pi_1, \dots, \pi_T)$ from observed action-reward histories to arms, where $\pi_t$ maps histories of length $(t-1)$ to the next arm to be pulled:
\[
\pi_t : \Set{1,2,\dots,N}^{t-1} \times [0,1]^{t-1} \to \Set{1,2,\dots,N}.
\]
The cumulative reward of a policy only depends on how often it pulls each arm, so it is determined by a tuple $(n^T_1(\pi), \dots, n^T_N(\pi))$ where $n^T_i(\pi)$ denotes how often $\pi$ pulls arm $i$ within the time horizon $T$. Note that $\sum_{i=1}^N n^T_i(\pi) = T$. We can write the cumulative reward of a policy as:
\[
r_T(\pi) = \sum_{i=1}^N \sum_{t=1}^{n^T_i(\pi)} f_i(t) = \sum_{i=1}^N F_i(n_i^T(\pi)).
\]
Let $\Pi$ denote the space of all possible deterministic policies, and $\OPT \in \argmax_{\pi \in \Pi} r_T(\pi)$ be an optimal policy, that is, a policy achieving the highest possible cumulative reward. The decision-maker does not know the reward functions ($f_i$'s) in advance, so he/she cannot find an optimal policy ahead of time. Instead, he/she can aim to design a (possibly stochastic) policy that minimizes the \emph{policy regret}: $r_T(\OPT)-\bbE r_T(\pi)$.
Given a fixed set of reward functions, we say an algorithm $\cA$ that follows policy $\pi^{\cA, T}$ over time horizon $T$ has \emph{sub-linear} policy regret, if
\[
\lim_{T\to\infty} \frac{r_T(\OPT)-\bbE r_T\left(\pi^{\cA, T}\right)}{T} = 0.
\]
\looseness -1 It is in general impossible to achieve sub-linear policy regret in an adversarial bandit setting~\cite{Arora2012}, and we have to make additional assumptions about the shape of the reward functions $f_i$. In this work, we assume that the underlying reward functions are initially increasing and concave, then decreasing.

\begin{definition}[Single-peaked bandit]
We call $f_i(.)$ a \emph{single-peaked} reward function, if there exists a tipping points $\bar{m}_i$ such that $f_i(m)$ increases monotonically in $m$ and is concave up to $m \leq \bar{m}_i$, and then decreases monotonically for $m > \bar{m}_i$. We call a bandit with single-peaked reward functions a \emph{\banditname}. 
\end{definition}
Note that bandits with monotonically increasing or decreasing reward functions are \banditnamePl with $\bar{m}_i = \infty$ and $\bar{m}_i = 0$, respectively.
\section{SPO: A New No-Policy-Regret Algorithm}\label{sec:algorithm}
Our algorithm operationalizes the principle of \emph{optimism in the face of uncertainty}, which has been successfully applied with different interpretations to a wide range of MAB problems~\cite{bubeck2012regret}. Our interpretation of the principle is as follows: At each time step, pull the arm with the highest \emph{optimistic future reward}.
The reward functions of a \banditname are first increasing and concave, then become decreasing. Therefore, we can define the future optimistic reward in the increasing phase using concavity and in the decreasing phase using monotonicity of the reward function. In the increasing concave phase, we estimate the optimistic future reward as
\[
p_i^T(n_i, t) = \sum_{s=t+1}^T \min\cbr{1, \rbr{f_i(n_i) + \Delta_i(n_i) \cdot (s-t)}},
\]
where $\Delta_i(n_i) = f_i(n_i) - f_i(n_i-1)$. Defined this way, $p_i^T(n_i, t)$ is a linear optimistic approximation of future rewards from arm $i$ after it has been pulled $n_i$ times within the first $t$ pulls. Similarly, for the decreasing phase, we can define
\[
p_i^T(n_i, t) = f_i(n_i) \cdot (T-t).
\]
In the increasing phase, we use the fact that the reward will increase \emph{at most linearly}, and in the decreasing phase we use that it will at best remain constant.

The \emph{\alglong} algorithm (\algshort, \Cref{main_algorithm}) performs two main steps at every round $t$:
\begin{enumerate}
   \item Pull the arm that maximizes $p_i^T(n_i, t)$ where $n_i$ is the number of times the algorithm has pulled arm $i$ so far.
   \item Update the optimistic future rewards $p_i^T(n_i, t)$.
\end{enumerate}
For technical reasons, we add an initial phase in which we pull each arm $\log(T)$ times, which only adds sub-linear policy regret, but simplifies the analysis (see \Cref{app:proofs}).

Our analysis formalizes the observation that while \algshort may initially overestimate the future reward of an arm that grows at a high rate, it will stop pulling that arm as soon as it ceases to live up to the optimistic expectations.

\begin{algorithm}[t]
\caption{The \emph{\alglong} (\algshort) algorithm.}
\label{main_algorithm}
   \begin{algorithmic}
      \Function{\alglong}{}
         \State $N_{\text{init}} \gets \max(\log(T), 2)$  \Comment{initial phase}
         \For{arm $i$ in $1, \dots, N$}
         \State pull it $N_{\text{init}}$ times
         \State observe the rewards $f_i(1),\dots,f_i(N_{\text{init}})$
         \State $n_i \gets N_{\text{init}}$
         \EndFor
         \State $t \gets N_{\text{init}} \cdot N$  \Comment{main phase}
         \While{$t \leq T$}
            \State $p_1^T, \dots p_N^T \gets$ \Call{UpdateOptimisticReward}{}
            \State let $i^* \in \argmax_i p_i^T$ (break ties arbitrarily)
            \State pull arm $i^*$ and observe $f_{i^*}(n_{i^*}+1)$
            \State $n_{i^*} \gets n_{i^*} + 1$
            \State $t \gets t + 1$
         \EndWhile
      \EndFunction
      \State \vspace{-0.5em}
      \Function{UpdateOptimisticReward}{}
         \For{arm $i$ in $1, \dots, N$}
            \If{$f_i(n_i) \geq f_i(n_i-1)$}
               \State $p_i^T \gets \sum_{s=t+1}^T \min\{1, (f_i(n_i) +$\\
               \hspace{2.9cm} $(f_i(n_i) - f_i(n_i-1)) \cdot (s-t))\}$
            \Else
               \State $p_i^T \gets f_i(n_i) \cdot (T-t)$
            \EndIf
         \EndFor
         \State \Return $p_1^T, \dots p_N^T$
      \EndFunction
   \end{algorithmic}
\end{algorithm}
\begin{figure*}\centering
   \includegraphics[width=0.37\linewidth]{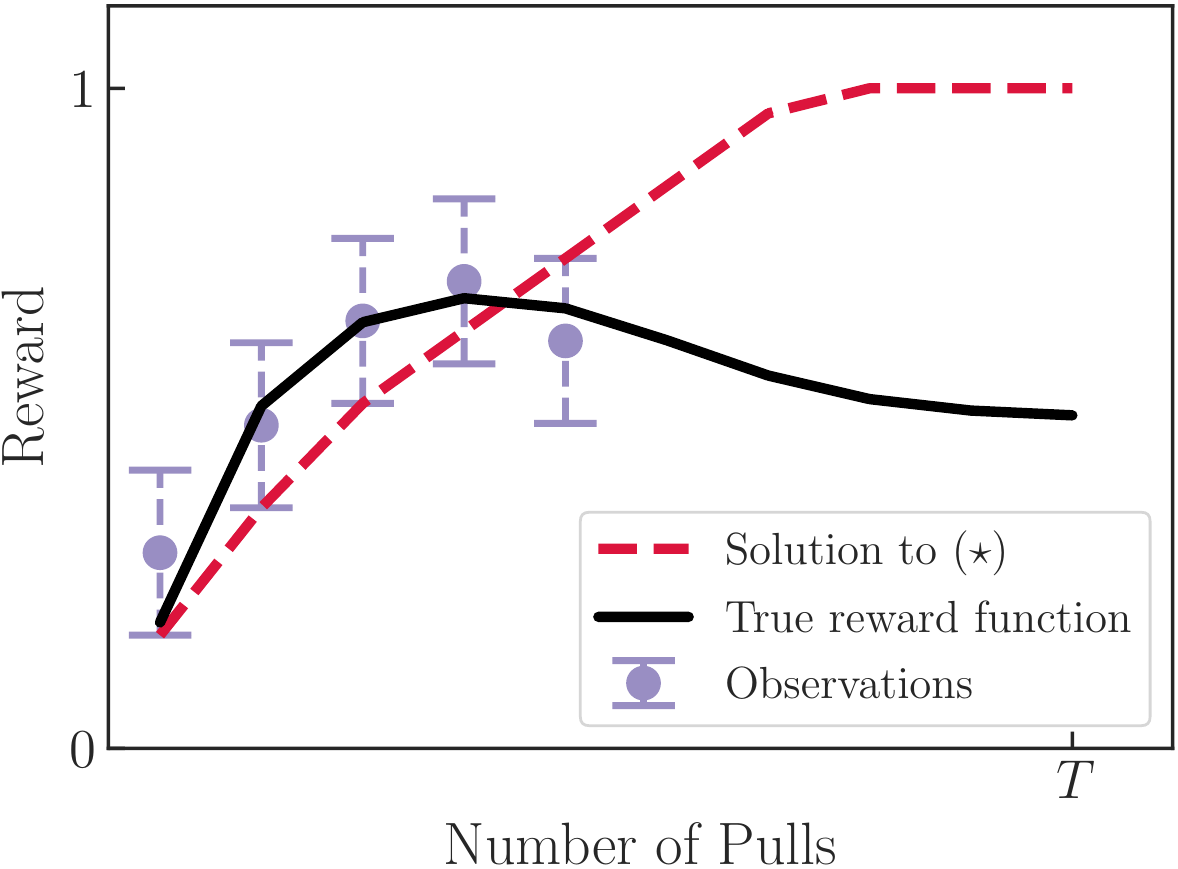}\hspace{5em}
   \includegraphics[width=0.37\linewidth]{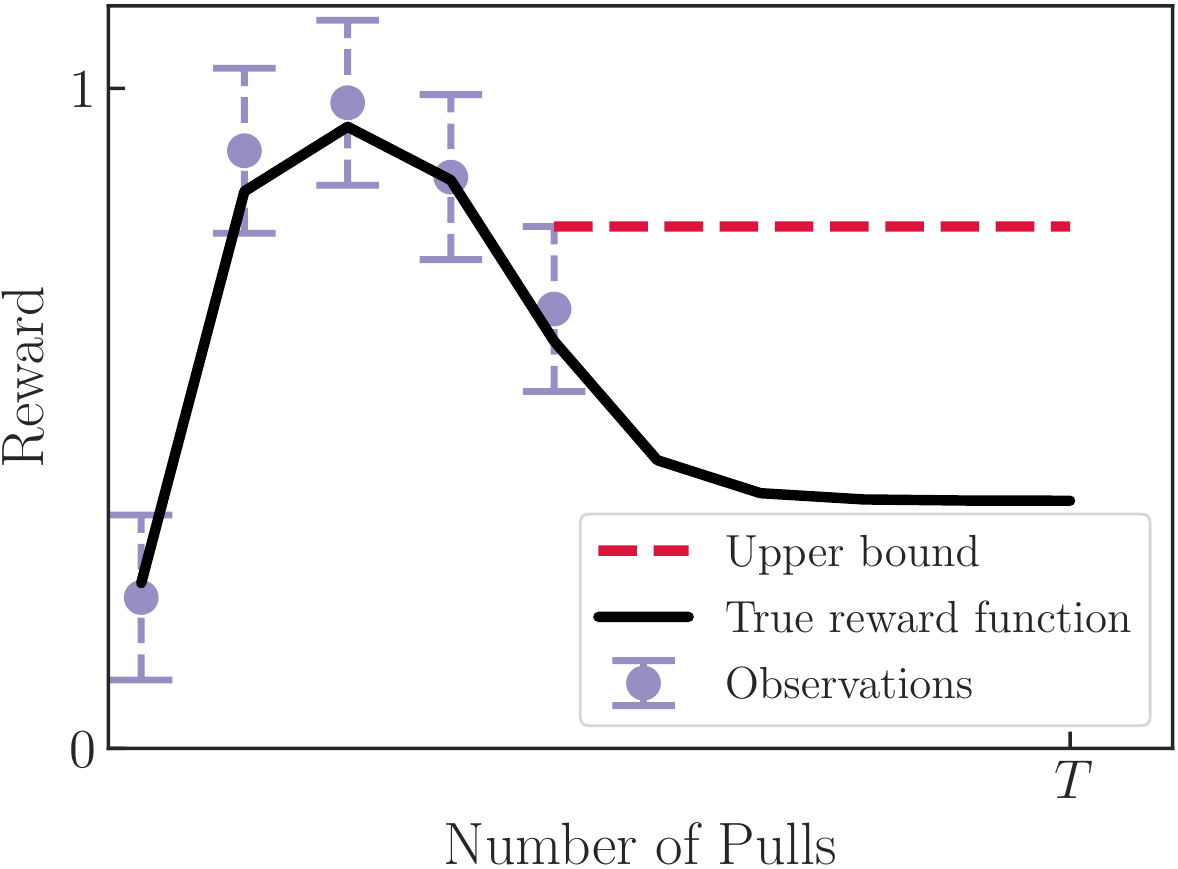}
   \caption{An illustration of 5 noisy reward observations from an arm, along with the true reward values which lie within the depicted confidence intervals. The dashed red curves specify our upper bound on cumulative future rewards obtained by solving ($\star$). The left plot shows an instance in which ($\star$) has a feasible concave and increasing solution.   Note that the upper-bound estimate can be lower than the true reward function for \emph{past} observation, but it is indeed an upper bound for \emph{future} observations. The right plot shows an instance in which ($\star$) is not feasible because the reward curve has entered its decreasing phase. In this case the last observation provide an upper bound on the cumulative future reward.}
   \label{fig:noise_handling_illustration}
\end{figure*}

\subsection{Regret Analysis}\label{sec:theory}
Next, we present our main theoretical result, which establishes the sub-linear policy regret of \algshort. All omitted proofs and technical material can be found in \Cref{app:proofs}.
\begin{restatable}{theorem}{ThmIncreasingDecreasing}[informal statement]\label{thm_increasing_decreasing}
For any (noise-free) \banditname, \algshort achieves sub-linear policy regret.
\end{restatable}

The proof consists of several steps: we first observe that all single-peaked reward functions have finite asymptotes (\Cref{lemma_inc_dec_arms_have_asymptotes}, \Cref{app:proofs}), which follows from the monotone convergence theorem. Then, we show that for sufficiently large time horizons, always pulling the single arm with the highest asymptote would lead to sub-linear policy regret (\Cref{lemma_quasi_optimal_policy}, \Cref{app:proofs}). Finally, the key step of the proof is to show that \algshort pulls all arms with suboptimal asymptotes less than linear in $T$. Together these steps imply the sub-linear policy regret of \algshort.

\section{An LP-based Heuristics to Handle Noise}\label{sec:noise}
So far we have assumed the decision-maker can observe rewards free of noise. In this section, we describe how to find an upper bound on the future reward from noisy observations. This allows us to extend \algshort to noisy observation.

Assume that when pulling arm $i$ for the $n$-th time, we observe $\hat{f}_i(n) = f_i(n) + \varepsilon_i(n)$ where $\varepsilon_i$ is a random noise term. We start by assuming that the magnitude of the noise is bounded $|\varepsilon_i(n)| \leq \bar{\varepsilon}_i$, and $\bar{\varepsilon}_i$ is known for each arm.
We can then define $L_i^j = \hat{f}_i(n)-\bar{\varepsilon}_i$ and $U_i^j = \hat{f}_i(n)+\bar{\varepsilon}_i$ to obtain confidence intervals for the true reward $f_i(n)$ such that $f_i(n) \in [L_i^j, U_i^j]$ with probability $1$. 

We first extend our algorithm to this case of bounded noise, and then relax this assumption to confidence intervals that contain the true value with probability less than $1$.

\paragraph{Decreasing phase.} For arms in their decreasing phase, we define the optimistic future return as $p_i^T(n_i, t) = U_i^j \cdot (T-t)$ using the confidence interval $[L_i^j, U_i^j]$.

\paragraph{Increasing phase.} For arms in their increasing phase, we combine our noise confidence intervals with our knowledge that the function is concave. Concretely, we find the monotone concave function with the highest cumulative future reward that can explain past observations. We can phrase this as solving the following linear program (LP) for each arm $i$:
\begin{equation*}\tag{$\star$}\hspace{-0.4em}
{
\begin{array}{ll@{}ll}
\text{maximize}_{v}  & \displaystyle\sum\limits_{j=n+1}^{n+T-t} v_{j}& &\\
\text{subject to}& 0 \leq v_j \leq 1, &  &j=1 , \dots, T\\
                 & L_i^j \leq v_j \leq U_i^j, &  &j=1 , \dots, n\\
                 & v_j \leq v_{j+1}, &  &j=1, \dots, T-1\\
                 & v_j \leq 2 v_{j-1} - v_{j-2}, &  &j=3 , \dots, T
\end{array}
}
\end{equation*}
where $n$ is the number of times arm $i$ has been pulled up to time $t$. The optimization variables $v_j$ correspond to the values of the reward function $f_i$ after $j$ pulls of arm $i$.
The constraints encode that the true reward function is bounded, consistent with past observations, increasing, and concave, in that order. Hence, a feasible solution to the LP corresponds to a possible reward function and an optimal solution provides a tight upper bound on future rewards.
\begin{restatable}{theorem}{LemmaNoiseLP}\label{lemma_noise_lp}
Let $f_i: \N^+ \to [0, 1]$ be a concave, increasing function with confidence bounds $L_i^1, U_i^1, \dots, L_i^n, U_i^n \in [0, 1]$ such that $f_i(j) \in [L_i^j, U_i^j]$ for $1 \leq j \leq n$. Let $V^* = \sum_{j=n+1}^{n+T-t} v_j$ be the solution to ($\star$). Then, $\sum_{j=n+1}^{n+T-t} f_i(j) \leq V^*$. Furthermore, there exists a concave, increasing function, $f^*_i: \N^+ \to [0, 1]$, such that $\sum_{j=n+1}^{n+T-t} f^*_i(j) = V^*$.
\end{restatable}
We can extend \algshort to noisy observations, by solving the LP ($\star$) every time we update the future optimistic reward for a given arm. If the LP does not have a feasible solution, we can infer that the arm is in its decreasing phase, and use a corresponding upper bound. \Cref{fig:noise_handling_illustration} illustrates both cases.

\paragraph{Unbounded noise.}
We can readily extend this approach to unbounded noise with confidence intervals. 
\begin{restatable}{corollary}{LemmaNoiseUnbounded}\label{lemma_noise_unbounded}
Let $f_i: \N^+ \to [0, 1]$ be a concave, increasing function. Suppose that for any $\delta > 0$ and observation $\hat{f}_i(n_i)$ we can find a confidence interval $\sbr{L_i^{n_i}(\delta), U_i^{n_i}(\delta)}$ such that $f_i(n_i) \in \sbr{L_i^{n_i}(\delta), U_i^{n_i}(\delta)}$ with probability at least $1 - \delta$. Let $V^*$ be the solution to ($\star$). Then for any $\epsilon>0$, we can choose $\delta$ such that $\sum_{j=n+1}^{n+T-t} f_i(j) \leq V^*$ with probability at least $1-\epsilon$.
\end{restatable}
The proof sketch goes as follows:
The probability that within the remainder of time horizon $T$, at least one true reward value falls outside of its confidence interval is upper bounded by $1 - (1-\delta)^T$. For the given $\epsilon$, we can choose
$
\delta \leq 1 - \e^{\frac{1}{T} \log(1-\epsilon)}, 
$
so that the probability of any true reward being outside its confidence interval is bounded by $\epsilon$. More precisely, we can write:
\begin{align*}
1 - (1-\delta)^T &\leq 1 - \rbr{ 1 - \rbr{1 - \e^{\frac{1}{T} \log(1-\epsilon)}}}^T \\
&= 1 - \e^{\log(1-\epsilon)} = \epsilon
\end{align*}
With the above choice for $\delta$, the optimistic future reward is estimated correctly with probability at least $1-\epsilon$.

\section{Experiments}\label{sec:experiments}

\begin{figure*}[t!]
\centering
\begin{subfigure}[b]{0.49\linewidth}
   \includegraphics[width=0.49\linewidth]{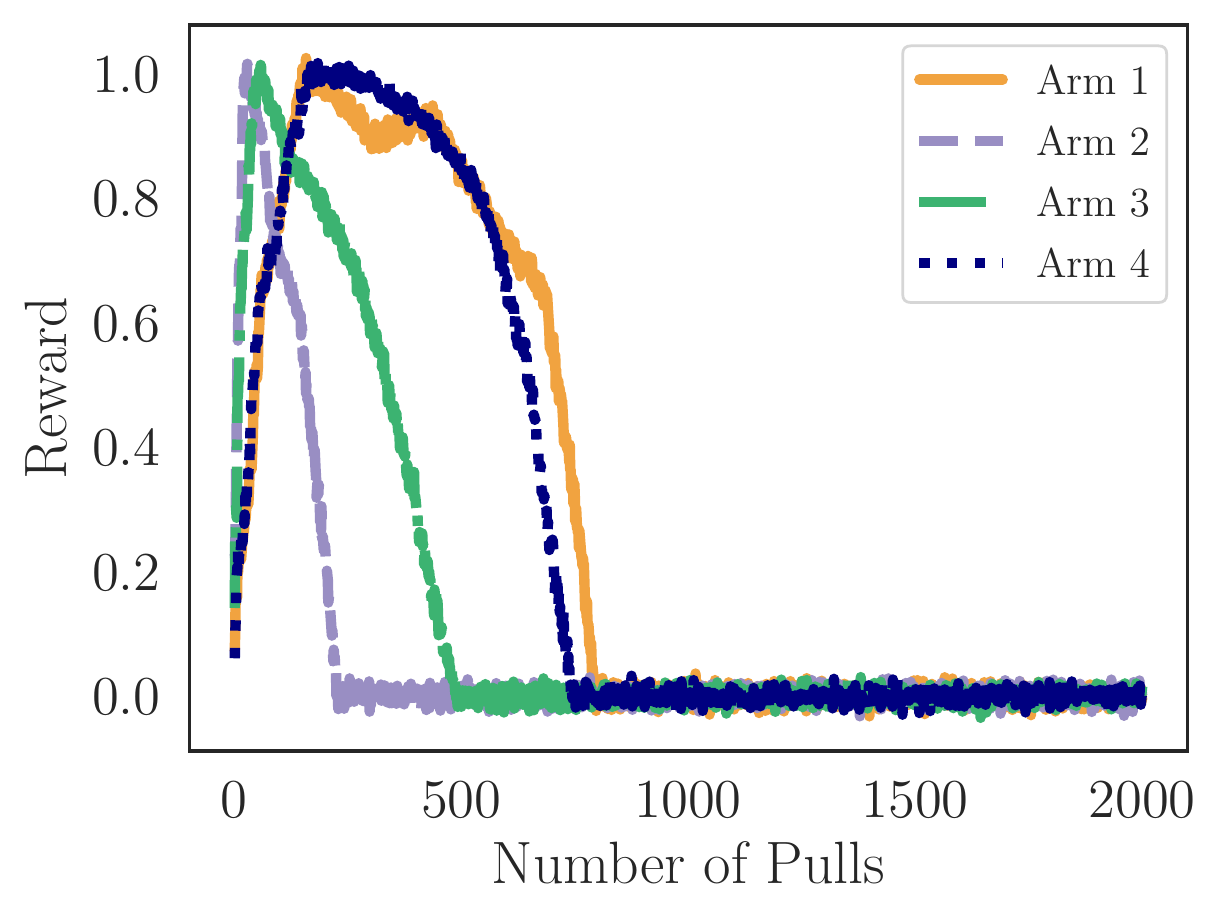}
   \includegraphics[width=0.49\linewidth]{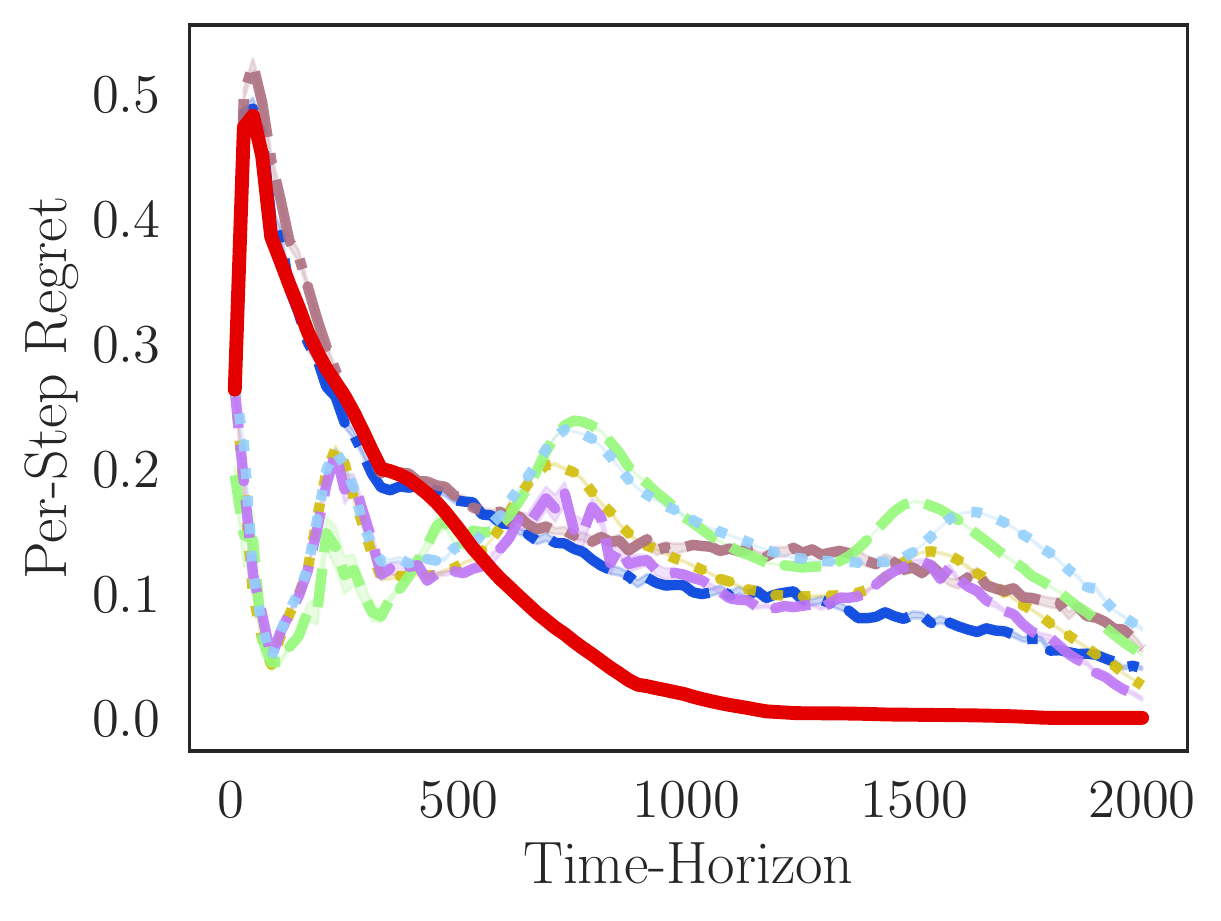}
   \caption{FICO Dataset}\label{fig:fico}%
\end{subfigure}
\begin{subfigure}[b]{0.49\linewidth}
   \includegraphics[width=0.49\linewidth]{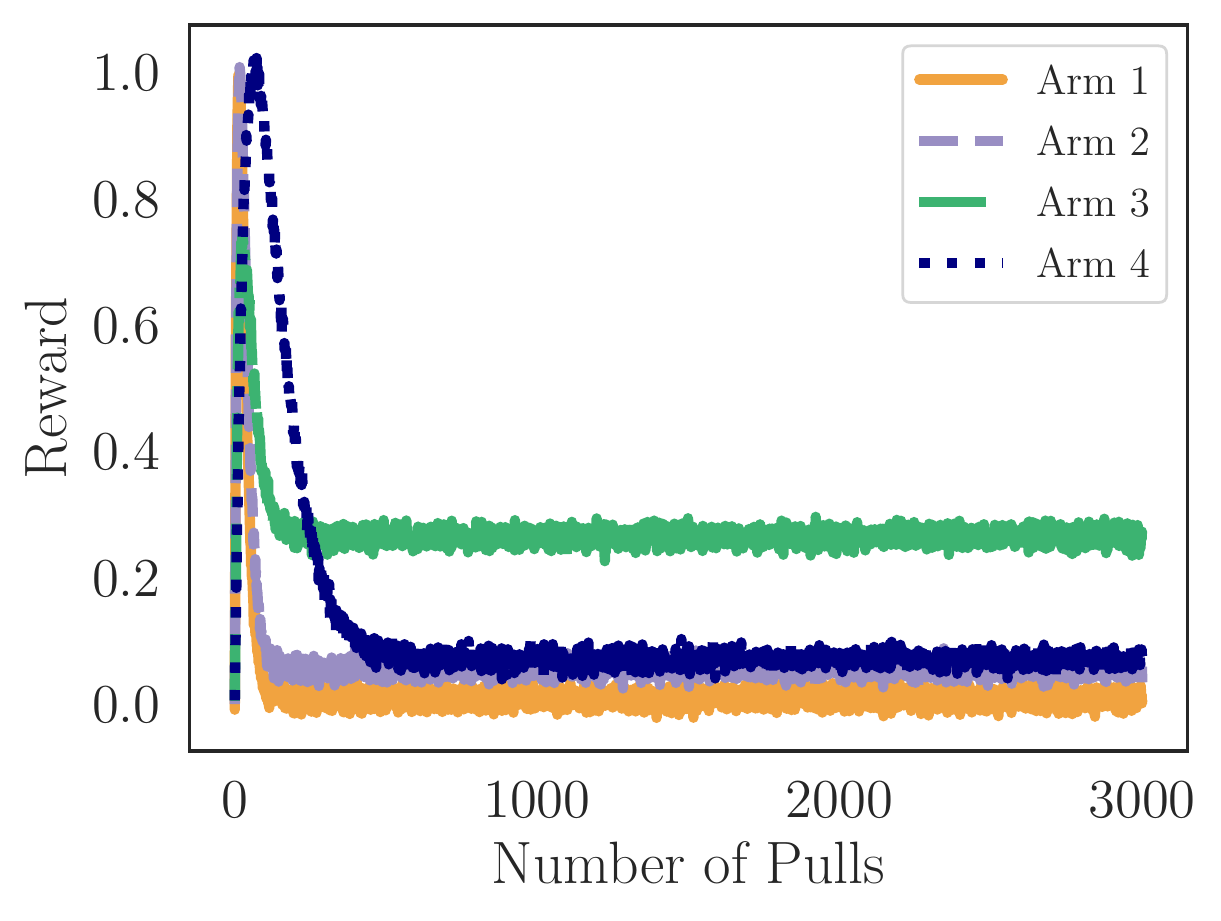}
   \includegraphics[width=0.49\linewidth]{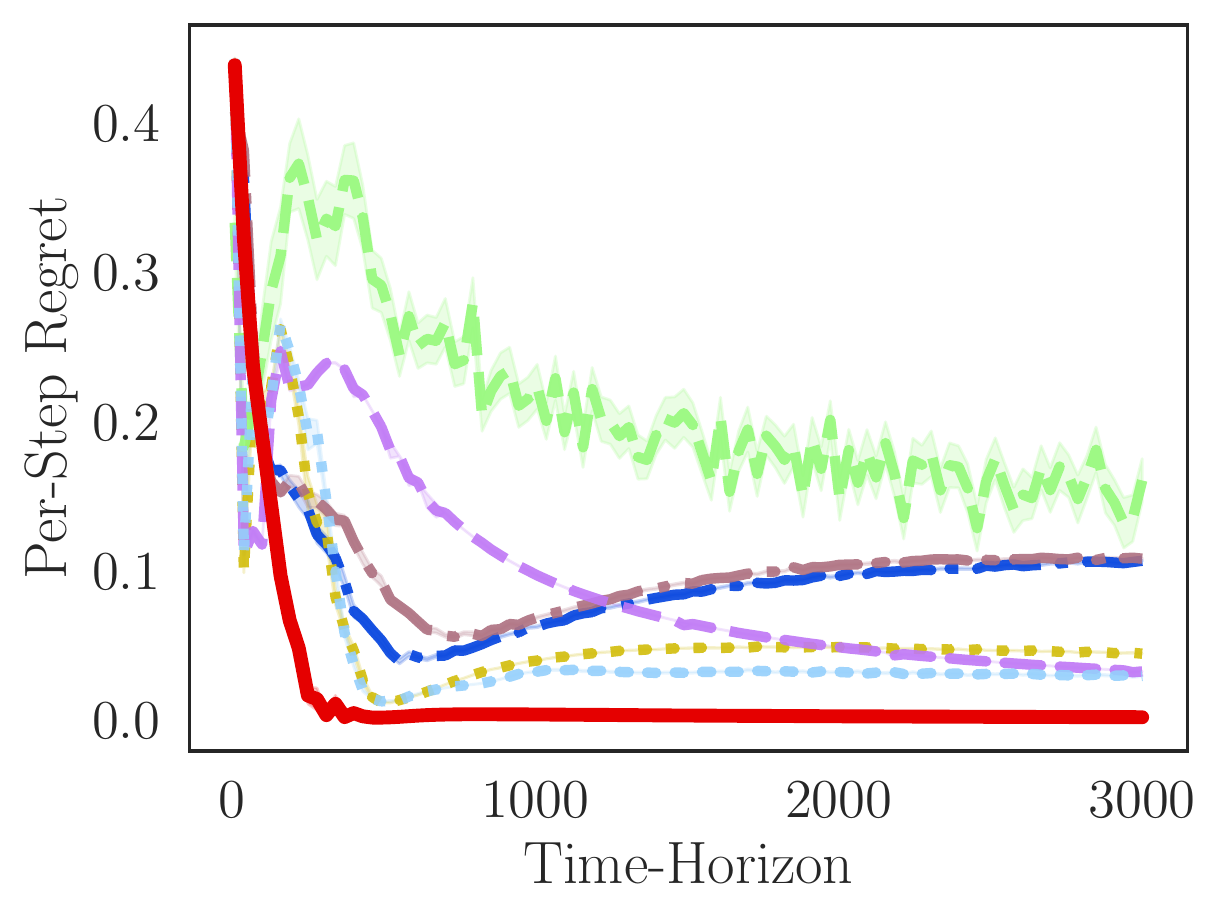}
   \caption{Recommender System Simulations}\label{fig:recommender}
\end{subfigure}\hfill\vspace{1em}
\begin{tabular}{llllllll}
    \legendDUCB & D-UCB~\cite{garivier2011upper} &
    \legendSWUCB & SW-UCB~\cite{garivier2011upper} &
    \legendEXP & EXP3~\cite{auer2002nonstochastic} &
    \legendREXP & R-EXP3~\cite{besbes2019optimal} \\
    \legendOSO & One-Step-Optimistic &
    \legendGREEDY & Greedy &
    \legendSPO & \algshort (ours) & &
\end{tabular}\\
\caption{
   Results of our simulation experiments with (\subref{fig:fico}) the FICO credit scoring dataset, and (\subref{fig:recommender}) synthetic recommender system data. In both cases, the left plot shows the reward functions of the bandit, and the right plot shows the per-step regret, i.e., the policy regret divided by $T$. The x-axes of the regret plots show the time horizon $T$, discretized in 100 points, where each point corresponds to a single experiment. The per-step regret is averaged over 30 random seeds. \algshort outperforms all baselines and is the only algorithm that achieves low policy regret for long time horizons.
}\label{fig:main_results}
\end{figure*}

In this section, we empirically investigate the effectiveness of our noise-handling approach on several datasets.\footnote{Code to reproduce all of our experiments can be found at \codeurl.}

\paragraph{Setup.} We consider three datasets: (1) a set of synthetic reward functions, (2) a simulation of a user interacting with a recommender system, and (3) a dataset constructed from the FICO credit scoring data.
We compare \algshort with six baselines: (1) a \emph{greedy} algorithm that always pulls the arm that provided the highest reward at the last pull, (2) a \emph{one-step-optimistic} variant of \algshort that pulls the arm with the highest upper bound on the reward at the next pull, (3) \emph{EXP3}, a standard no-external-regret algorithm for adversarial bandits~\cite{auer2002nonstochastic}, (4) \emph{R-EXP3}, a modification of EXP3 for non-stationary bandits~\cite{besbes2019optimal}, (5) discounted UCB (D-UCB), and (6) sliding window UCB (SW-UCB), two adaptations of UCB to nonstationary bandits~\cite{garivier2011upper}.

\paragraph{Illustrations on synthetic data.} \label{sec:synthetic_experiment}
We first perform a series of experiments on \banditnamePl with two arms, and synthetic reward functions with Gaussian noise. To this end we define a class of single-peaked functions and combine them into multiple \banditnamePl with two arms each. In \Cref{fig:synthetic} we highlight one experiment in which algorithms that minimize external regret fail. The figure shows how \algshort avoids this kind of failure by minimizing policy regret.
In \Cref{sec:additional_experiments} we provide more detailed results comparing \algshort to the baselines on various synthetic reward functions, including the monotonic functions proposed by \citet{Heidari2016}, and evaluate the effect of varying the observation noise. We find that \algshort matches the performance of the baselines in all cases and significantly outperforms them in some. Further, we show that \algshort can also handle stationary MABs, where arms have fixed reward distributions.

\paragraph{FICO credit lending data.} \label{sec:fico_experiment}
Motivated by our initial example of a budget planner in Section~\ref{sec:introduction}, we simulate a credit lending scenario based on the FICO credit scoring dataset from 2003~\cite{reserve2007report}. We pre-process the data using code provided by previous work~\cite{hardt2016equality,liu2018delayed}.
The dataset contains four ethnic groups: 'Asian', 'Black', 'Hispanic' and 'White'. Each group has a distribution of credit-scores and an estimated mapping from credit-scores to the probability of repaying a loan. We use this group-level data to simulate a population of individuals applying for a loan. Each individual belongs to one ethnic group and has a credit score sampled from the group's distribution and a probability of repaying a loan.

We consider a hypothetical decision-making scenario in which at each round, there is exactly one loan applicants from each group. In each time step, the decision-maker (i.e., a bank) can approve only one loan applicant. We are interested in the long-term impact of decision-maker's choices on the underlying groups.
As discussed in \Cref{sec:introduction}, we argue that a fair decision-maker will allocate loans according to the groups' long-term potential to turn them into welfare/prosperity, which is measured by the per-group reward functions in this simulation. In other words, policy regret is our measure of long-term disparity and achieving low policy regret improves the fairness of resource allocation decisions.

To simulate this situation based on the FICO dataset, we first sample $N$ applicants from each group. We assume the decision-maker always approves the loan of the applicant with the highest credit score within a given group; hence, we order the applicants decreasing by their credit score. Thereby, we reduce the problem to the decision-maker deciding between four arms to pull, each corresponding to one group. We interpret pulling an arm as approving the load on the highest scoring applicant within the corresponding group. However, the credit scores do not directly correspond to the reward of pulling an arm. Rather we want to define a reward function that quantifies the benefit/loss of giving out a loan to a group.\footnote{We also investigated a variant of this setting considering only the utility of a loan to the decision-maker, see \Cref{sec:additional_experiments}.} We follow \citet{liu2018delayed}, and measure the impact on a group as the change in mean credit score for this group. \citeauthor{liu2018delayed}'s model assumes an increase in credit score of $75$ points for a repaid loan and a decrease of $150$ points for a defaulted loan, while the credit score is always being clipped to the range~$[300, 850]$. Finally, we rescale the rewards to $[0,1]$.

The resulting reward functions increase at first because the first individuals in each group are highly credit-worthy and them paying back their loan increases the mean credit score for the group. The reward functions are concave because as the decision-maker gives more loans to a group he/she starts to give loans to less creditworthy individuals. Eventually, the reward functions start to decrease because giving loans to individuals who cannot pay them back decreases the average credit score of the group. Hence, this setup can be approximately modelled as a \banditname.

\Cref{fig:fico} shows results of running \algshort in this setup. Overall, \algshort strictly outperforms the baselines over long time horizons. For short time horizons we find that simple greedy approaches or UCB variants can perform favorably.

\balance
\paragraph{Synthetic recommender system data.} \label{sec:recommender_experiments}
Recent work shows that strategies minimizing external regret can perform poorly in the context of recommender systems, due to negative feedback loops~\cite{jiang2019degenerate,warlop2018fighting,mladenov2020optimizing,chaney2018algorithmic}. Here, we focus on one concrete problem that can arise: many recommender systems exhibit a bias towards recommending particularly engaging or novel content that leads to high instantaneous reward, disregarding the long-term benefit and cost to the user. We argue that this situation is analogous to our example of a budget maker, and that a recommender system should aim to maximize it's users' long-term benefit.

Motivated by this observation, we simulate a system that recommends content, e.g., articles or videos, to a user and receives feedback about how much the user engaged with the content. For simplicity, we assume the user's engagement with a piece of content is driven by two factors only: (i) the user's inherent preferences, and (ii) a novelty factor which makes new content more engaging to the user. We assume that the user's inherent preferences stay constant, but the novelty factor decays when showing an item more often.
Note that an algorithm that minimizes external regret would show content with high novelty and neglect content that is a better match for the user's inherent preferences. An algorithm that minimizes policy regret would select the content that best matches the user's inherent preferences in the long-run.

\looseness -1 We simulate the user's feedback with a reward function~$f_i$ for each item that can be recommended. Each item has an inherent value $v$ to the user, a novelty factor $n$, and decay factors $\gamma$ and $c$. The reward is $f_i(0) = 0$ for never showing an item, and subsequent rewards are defined as 
\[f_i(t) = f_i(t-1) + n \cdot \gamma^t - c \cdot (f_i(t) - v).\]
The second term in the expression models the novelty of an item which decays when showing it more often. The third term models the tendency of the reward to move towards how much the user values the item inherently. The resulting rewards increase at first due to the novelty of an item and decrease later as the novelty factor decays. For simplicity we model all effects that are not captured by this stylized model as Gaussian noise on the observed rewards.

\looseness -1
\Cref{fig:recommender} shows that \algshort significantly outperforms the baselines for long time horizons, at the cost of worse performance for short time horizons. This results indicates that if a decision-maker acts on a short time-horizon classical bandit algorithms perform well. However, if the decision-maker aims to achieve a good long-term impact, \algshort is preferable. We present results on additional instances of the recommender system simulation in \Cref{sec:additional_experiments}.

\section{Conclusion}\label{sec:conclusion}

\looseness -1
Motivated by several real-world domains, we studied \banditnamePl in which the reward from each arm is initially increasing then decreasing in the number of pulls of the arm. We introduced \alglong (\algshort), an algorithm that achieves sub-linear policy regret in \banditnamePl.
Our findings highlight the importance of understanding the long-term implications of ML-based decisions for impacted communities and society at large, and utilizing domain knowledge, e.g., regarding social- and population-level dynamics stemming from decisions today, to design appropriate sequences of allocations that do not amplify historical disparities.

\paragraph{Limitations.}
We argued that \banditnamePl are a useful model to provide insights about allocation decisions in a range of practical domains, e.g., allocating loans to communities, allocating funds to research institutions, or allocating policing resources to districts. However, \banditnamePl can also be too restrictive in domains where the evolution of rewards are more nuanced, e.g., if rewards can later increase again after first decreasing.
We emphasize that single-peaked reward functions are one among many reasonable classes of reward functions that are interesting to study from an algorithmic perspective. We consider our work as an starting point to look into more complex dynamics in future work.

\paragraph{Future work.}
We hope that our work draws the research community's attention to the study of policy regret for typical reward-evolution curves. Additional directions for future work include (1) establishing regret bounds for settings with arbitrary noise distributions, (2) providing instance-specific (and potentially tighter) regret bounds for \alglong, and finally (3) more broadly characterizing the limits of ``optimism in the face of uncertainty'' principle in achieving low policy regret.

\section*{Acknowledgements}
Lindner was partially supported by Microsoft Swiss JRC.
This work was in part done while Heidari was a postdoctoral fellow at ETH Zurich. Heidari acknowledges partial support from NSF IIS2040929.
Any opinions, findings, and conclusions, or recommendations expressed in this material are those of the authors and do not necessarily reflect the views of the National Science Foundation, or other funding agencies.
\balance

\newcommand{\ICML}{ICML}

\newcommand{\NeurIPS}{Advances in Neural Information Processing Systems}

\newcommand{\IJCAI}{IJCAI}

\newcommand{\FAT}{Conference on Fairness, Accountability, and Transparency}

\newcommand{\WWW}{WWW}

\newcommand{\ITCS}{ITCS}

{\small
\setlength{\bibsep}{0pt}
\bibliographystyle{named}
\bibliography{references}
}

\ifthenelse{\boolean{noappendix}}{}{
    \clearpage
    \appendix
    \section{More Details on Related Work}
\label{app:more_related_work}

In this section we give more details on related work.

\paragraph{External vs. policy regret.}
The vast majority of existing algorithms for the well-studied MAB setting aim to minimize the so-called \emph{external regret}. External regret is the difference between the total reward collected by an online algorithm and the maximum reward that could have been collected by pulling the best/optimal arm \emph{on the same sequence of rewards as the one generated by the online algorithm}.
Existing work on MAB often focuses on one of the following two settings: first, the statistical setting in which the reward from each arm is assumed to be sampled from a fixed, but unknown distribution~\cite{robbins1952some,gittins2011multi,auer2002nonstochastic}, and second, the adversarial setting where an adversary---who is capable of reacting to the choices of the decision-maker---determines the sequence of rewards~\cite{auer2002nonstochastic}.
As discussed by \citet{Arora2012}, the notion of external regret fails to capture the actual regret of an online algorithm compared to the optimal sequence of actions it \emph{could} have taken when the adversary is \emph{adaptive}. Policy regret is defined to address this counterfactual setting.
\citeauthor{Arora2012} show that if the adaptive adversary has \textit{bounded memory}, then a variant of traditional online learning algorithms can still guarantee no-policy-regret. If the adversary does not have bounded memory, as is the case in our setting, no algorithm can achieve no-policy-regret in general. Hence, additional assumptions are required, such as the ones we make when defining \banditnamePl.

\paragraph{Restless and rested bandits.}
Nonstationary bandits are called \emph{restless} if the rewards of an arm can change in each timestep, or \emph{rested} if the reward of an arm can only change when the arm is pulled~\cite{tekin2012online}. Restless bandits model changes in the environment due to external effects \cite{whittle1988restless}, whereas rested bandits model changes to the environment caused by interactions of the decision maker with the environment. Such feedback effects are the focus of our work, and hence we consider a rested nonstationary bandit.

\paragraph{Feedback loops as a source of unfairness.}
At a high-level, feedback loops occur when outputs of a system influence its inputs in future rounds, and this cause-and-effect circuit amplifies or inhibits the system~\cite{Ensign2018limited,Ensign2018loops,liu2018delayed}.

\citet{Ensign2018limited}, for example, study feedback loops in the context of recidivism prediction, that is, predicting if an inmate will re-offend within a fixed time after being released, and predictive policing, that is, allocating police patrols to city districts based on historical crime data. In both cases, the decision-maker faces a \textit{partial monitoring problem}: the decision-maker only receives feedback if he/she takes specific actions (e.g., release an inmate or send a patrol).
\citet{Ensign2018loops} show that commonly-used algorithms for predictive policing suffer from a specific type of feedback loop: if police patrols are assigned based on historical crime data, the assignment rate of police officers to neighborhoods can quickly become highly disproportionate to the actual crime rates of the neighborhoods.

Feedback loops can also occur in recommender systems \cite{jiang2019degenerate,warlop2018fighting,mladenov2020optimizing,chaney2018algorithmic}, which motivated some of the simulation experiment in \Cref{sec:recommender_experiments}. \citet{jiang2019degenerate} argue that feedback loops can give rise to "echo chamber" and "filter bubble effects", and \citet{chaney2018algorithmic} argue that they lead to more homegenity and lower utility for the users. \citet{warlop2018fighting} suggest that minimizing external regret is sub-optimal for recommender systems because the user's preferences can change depending on what is recommended to them, and \citet{mladenov2020optimizing} propose to maximize the long-term social welfare instead of short-term reward. These interpretations of feedback loops in recommender systems are in line with the main conceptual argument of our work: when deploying ML-based systems to make decisions that affect humans, one should focus on maximize their long-term well-being instead of sacrificing it for short-term reward.

We argue that these different kinds of feedback loops are a consequence of deploying algorithms that minimize external regret instead of policy regret, and, therefore, do not take into account the long-term effects of their decisions.

    \section{Monotone Bandits}\label{app:monotone_bandits}

In this section, we describe bandits with monotonically increasing or decreasing reward functions, as described by \citet{Heidari2016}. We show that these are special cases of \banditnamePl, and that \algshort achieves the same asymptotic policy regret as the algorithms introduced by \citeauthor{Heidari2016}.

\subsection{Increasing Reward Functions}\label{sec:increasing}

\citeauthor{Heidari2016} consider reward functions that are monotonically increasing and concave which correspond to \banditnamePl with $\bar{t}_i = \infty$ for every arm. With this observation we receive sub-linear policy regret of \algshort in this setting as a corollary of \Cref{thm_increasing_decreasing}.

\begin{restatable}{lemma}{ThmIncreasing}\label{thm_increasing}
For any set of bounded, concave and increasing reward functions $f_1, \dots f_N$,
the policy regret of \algshort is sub-linear.
\end{restatable}

This results shows that our algorithm matches the asymptotic performance of Algorithm 1 by \citeauthor{Heidari2016} in the case of increasing reward functions.

\subsection{Decreasing Reward Functions}

Further, \citeauthor{Heidari2016} study bandits with monotonically decreasing reward functions, which are \banditnamePl with $\bar{t}_i = 0$.
\citeauthor{Heidari2016} show that for decreasing reward functions the optimal policy in hindsight always pulls the arm with the highest instantaneous reward. This immediately gives a constant-policy-regret algorithm that greedily pulls the arms that gave the highest reward at the last pull.

It is straightforward to show that in the case of monotonically decreasing reward functions, \algshort is equivalent to the greedy algorithm after the initial phase of pulling all arms. In particular, this implies that it achieves sub-linear policy regret for this case as well.

\begin{restatable}{lemma}{ThmDecreasing}\label{thm:decreasing}
   For any set of monotonically decreasing reward functions $f_1,\dots,f_N$ and a time horizon $T$,
   \algshort greedily pulls the arms that gave the highest reward at the last pull after its initial phase. Therefore, it achieves sub-linear policy regret.
\end{restatable}

\begin{proof}
After pulling every arm $\log(T)$ times, \algshort will always update the optimistic future reward to be
\[
p_i^T(n_i, t) = f_i(n_i) \cdot (T-t)
\]
after $i$ has been pulled $n_i$ times because the reward functions are decreasing.
At each time step, \algshort then pulls an arm $j \in \argmax_i p_i^T(n_i, t)$.
For fixed $T$ and $t$
\[
\argmax_i p_i^T(n_i, t) = \argmax_i f_i(n_i)
\]
Hence, at each time step \algshort pulls an arm that gave the highest reward at the last pull. \citeauthor{Heidari2016} show that this behavior is optimal. Therefore \algshort only incurs policy regret of order $\cO(\log(T))$ during its initial phase. In particular, the policy regret is sub-linear.
\end{proof}

While \algshort achieves sub-linear regret, it does not match the performance of a greedy algorithm which achieves constant policy regret \citeauthor{Heidari2016}. In situations where the reward functions are known to be decreasing a greedy approach is therefore preferrable. If, however, reward functions might be increasing or decreasing \algshort can still achieve sub-linear regret.

    \section{Proofs of Theoretical Statements}\label{app:proofs}

This section contains detailed proofs for all results presented in the main paper.

\subsection{Noise-free Observations}

\begin{definition}
We call arms with a single-peaked reward function with $\bar{t}_i = \infty$ \emph{asymptotically increasing}, and all other single-peaked arms \emph{asymptotically decreasing}.
\end{definition}

\begin{restatable}{lemma}{LemmaIncDecArmsHaveAsymptotes}\label{lemma_inc_dec_arms_have_asymptotes}
All arms $i$ of a single-peaked bandit have an asymptote, i.e., a finite limit $a_i = \lim_{t\to\infty} f_i(t)$, which is bounded between $0$ and $1$, i.e., $0 \leq a_i \leq 1$.
\end{restatable}

\begin{proof}
This follows from the monotone convergence theorem:
\begin{theorem*}[Monotone Convergence]
If $\{a_n\}$ is a monotone sequence of real numbers ($a_n \leq a_{n+1}$ or $a_{n+1} \geq a_n$ for all $n \geq 1$) then this sequence has a finite limit if and only if it is bounded.
\end{theorem*}
In our case, all $f_i$ are bounded between $0$ and $1$. Asymptotically increasing arms are monotone and asymptotically decreasing arms are monotone for all $t > \bar{t}_i$. Hence the theorem applies to both cases and gives a finite limit $a_i$ between $0$ and $1$ for each arm.
\end{proof}

\begin{restatable}{lemma}{LemmaQuasiOptimalPolicy}\label{lemma_quasi_optimal_policy}
Let $W = \argmax_i a_i$. For any $k \in W$ let $\pi_k^T$ be the policy that pulls the $k$-th arm $T$ times and does not pull any other arm. Then, $\pi_k^T$ achieves sub-linear policy regret.
\end{restatable}

\begin{proof}
Let $n_i^T$ denote the number of times $\OPT$ pulls arm $i$ for time horizon $T$. Then
$$
r_T(\OPT) = \sum_{i=1}^N F_i(n_i^T)
$$

Observe that
$$
\lim_{t\to\infty} f_i(t) \leq a^*
$$
And because $f_i(t) = F_i(t) - F_i(t-1)$ also
$$
\lim_{t\to\infty} \rbr{F_i(t) - F_i(t-1)} \leq a^*
$$
Hence for any $\varepsilon > 0$ there exists a $t_\varepsilon$ such that for any $t > t_\varepsilon$ and all $i$
$$
F_i(t) - F_i(t-1) \leq a^* + \varepsilon
$$
$$
F_i(t) \leq F_i(t-1) + a^* + \varepsilon
$$
Then with $C^i_2 = F_i(t_\varepsilon)$
$$
F_i(t) \leq C^i_2 + (t-t_\varepsilon) (a^* + \varepsilon)
$$
Let $C_2 = \max_i\{C_2^i\}$ and consider a general $t \geq 1$. Then because $F_i$ is increasing
\begin{align*}
F_i(t) &\leq C_2 + \max(0, t-t_\varepsilon) (a^* + \varepsilon) \\
&\leq C_2 + t \cdot (a^* + \varepsilon)
\end{align*}

We can use this to upper bound the reward of the optimal policy
\begin{align*}
r_T(\OPT) &= \sum_{i = 1}^N F_i(n_i^T) \\
&\leq \sum_{i = 1}^N (C_2 + n_i^T \cdot (a^* + \varepsilon)) \\
&= N \cdot C_2 + T \cdot (a^* + \varepsilon)
\end{align*}
where in the last step we use $\sum_{i=1}^N n_i^T = T$. Dividing by $T$ and taking the limit $T \to \infty$ gives us
$$
\lim_{T\to\infty} \frac{r_T(\OPT)}{T} \leq a^* + \varepsilon
$$

Because this has to hold for all $\varepsilon > 0$,
it implies
\begin{align}
\lim_{T\to\infty} \frac{r_T(\OPT)}{T} \leq a^* \label{rewardOPT}
\end{align}

Not consider the policy $\pi_k^T$. It only pulls arm $k$ and therefore achieves reward
$$
r_T(\pi_k^T) = F_k(T)
$$
Because $k \in W$ we have
$$
\lim_{t \to \infty} f_k(t) = a^*
$$
and this means
$$
\lim_{t \to \infty} \rbr{F_k(t)-F_k(t-1)} = a^*
$$
In particular for any $\varepsilon > 0$ there exists a $t_\varepsilon$ such that for $t > t_\varepsilon$
\begin{align*}
F_k(t) &- F_k(t-1) \geq a^* - \varepsilon \\
F_k(t) &\geq F_k(t-1) + a^* - \varepsilon \\
F_k(t) &\geq F_k(t_\varepsilon) + (t - t_\varepsilon) (a^*-\varepsilon) \\
F_k(t) &\geq C_k + t \cdot (a^*-\varepsilon)
\end{align*}
where we defined $C_k = F_k(t_\varepsilon) - t_\varepsilon \cdot (a^*-\varepsilon)$. Therefore we get
$$
r_T(\pi_k^T) = F_k(T) \geq C_k + T \cdot (a^*-\varepsilon)
$$
and
$$
\lim_{T\to\infty} \frac{r_T(\pi_k^T)}{T} \geq a^* - \varepsilon
$$
Because this has to hold for all $\varepsilon > 0$ we end up with
\begin{align}
\lim_{T\to\infty} \frac{r_T(\pi_k^T)}{T} \geq a^* \label{rewardPI}
\end{align}
Combining eqs. \eqref{rewardOPT} and \eqref{rewardPI} gives
$$
\lim_{T\to\infty} \frac{r_T(\OPT)-r_T(\pi_k^T)}{T} = \lim_{T\to\infty} \frac{r_T(\OPT)}{T} - \lim_{T\to\infty} \frac{r_T(\pi_k^T)}{T}
$$
$$
\leq a^*-a^* = 0
$$
which gives the desired result
$$
\lim_{T\to\infty} \frac{r_T(\OPT)-r_T(\pi_k^T)}{T} = 0
$$
because
$$
\frac{r_T(\OPT)-r_T(\pi_k^T)}{T} \geq 0
$$
\end{proof}

\ThmIncreasingDecreasing*

\begin{theorem*}[formal statement of \Cref{thm_increasing_decreasing}]
Let $f_1, \dots, f_N$ be the arms of a \banditname. Denote the \algshort algorithm by $\cA$. Then \algshort achieves sub-linear policy regret, i.e.,
\[
\lim_{T\to\infty} \frac{r_T(\OPT)-r_T(\pi^{\cA, T})}{T} = 0.
\]
\end{theorem*}

\begin{proof}
Let $W = \argmax_i a_i$, let arm $k$ be in $W$ and $\pi_k^T$ be the policy that only pulls arm $k$. We will show that
$$
\lim_{T\to\infty}\frac{r_T(\pi_k^T) - r_T(\pi^{\cA, T})}{T} = 0
$$
which implies the theorem because of lemma \ref{lemma_quasi_optimal_policy} and
$$
\lim_{T\to\infty} \frac{r_T(\OPT)-r_T(\pi^{\cA, T})}{T} =
$$
$$
= \lim_{T\to\infty}\rbr{\frac{r_T(\OPT)-r_T(\pi_k^T)}{T} + \frac{r_T(\pi_k^T)-r_T(\pi^{\cA, T})}{T}} = 0
$$
We now go on to show two things: (i) pulling arms $i \in W$ adds sub-linearly to the reward difference between $\pi_i^T$ and $\cA$ and (ii) the number of times $\cA$ pulls arms $i \notin W$ is sub-linear in $T$.

To show (i) we modify a Lemma by \citet{Heidari2016} to hold for single-peaked reward functions instead of monotonically increasing reward functions.

\begin{lemma}[Adaptation of Lemma 2 in \citet{Heidari2016}]\label{lemma_step_i}
Let $k \in W$ and let $n_i^T$ denote the number of times $\cA$ pulls arm $i$. Then
$$
\lim_{T\to\infty} \frac{F_k(T) - \sum_{i\in W} F_i(n_i^T)}{T} = 0
$$
holds for \banditnamePl.
\end{lemma}

\begin{proof}[Proof of Lemma]
If $i \in W$ is pulled $o(T)$ times, it can not cause the algorithm to suffer linear policy regret. Now consider a subset $W'$ of $W$ consisting of any suboptimal arm $i$ which is pulled $\Theta(T)$ times by the algorithm.

Given that all arms not in $W'$ are pulled less than $T' = o(T)$ times, we have that $T - \sum_{i\in W'} n_i^T = T' = o(T)$.

Let $w \in W'$ be the arm for which $\frac{F_i(n_i^T)}{n_i^T}$ is the smallest, i.e. the arm with optimal asymptote that has the lowest average reward. We have

\resizebox{\linewidth}{!}{
    \begin{minipage}{\linewidth}\begin{align*}
    F_{k}(T) - \sum_{i\in W'} F_i(n_i^T) &= F_{k}(T) - \sum_{i\in W'} n_i^T \cdot \frac{F_i(n_i^T)}{n_i^T} \\
    &< F_{k}(T) - \sum_{i\in W'} n_i^T \cdot \frac{F_w(n_w^T)}{n_w^T} \\
    &= F_{k}(T) - (T-T') \cdot \frac{F_w(n_w^T)}{n_w^T} \\
    &= T \cdot \rbr{\frac{F_{k}(T)}{T} - \frac{F_w(n_w^T)}{n_w^T}} + T' \cdot \frac{F_w(n_w^T)}{n_w^T} \\
    &\leq T \cdot \rbr{\frac{F_{k}(T)}{T} - \frac{F_w(n_w^T)}{n_w^T}} + o(T)
    \end{align*}\end{minipage}
}
where in the last step we used that $\frac{F_w(n_w^T)}{n_w^T}$ is constant and $T' = o(T)$. In total we now have
$$
F_{k}(T) - \sum_{i\in W'} F_i(n_i^T) \leq T \cdot \rbr{\frac{F_{k}(T)}{T} - \frac{F_w(n_w^T)}{n_w^T}} + o(T)
$$
It only remains to show that
$$
T \cdot \rbr{\frac{F_{k}(T)}{T} - \frac{F_w(n_w^T)}{n_w^T}} = o(T)
$$
We have $\lim_{T \to \infty} \frac{F_{k}(T)}{T} = \lim_{T \to \infty} \frac{F_w(n_w^T)}{n_w^T} = a^*$ and therefore
\[
\resizebox{\linewidth}{!}{
    $
    \lim_{T\to\infty} \frac{T \cdot \rbr{\frac{F_{k}(T)}{T} - \frac{F_w(n_w^T)}{n_w^T}}}{T} = \lim_{T\to\infty} \frac{F_{k}(T)}{T} - \frac{F_w(n_w^T)}{n_w^T} = 0
    $
}
\]
which means $T \cdot \rbr{\frac{F_{k}(T)}{T} - \frac{F_w(n_w^T)}{n_w^T}} = o(T)$. Together this gives the desired result.
\end{proof}

\Cref{lemma_step_i} shows that whenever $\cA$ pulls an arm in $W$ it incurs sub-linear regret, which concludes step (i).

It remains to show (ii), that is the number of times $\cA$ pulls arms $i \notin W$ is sub-linear in $T$. For this we use the initial phase in which the algorithm pulls each arm $\log(T)$ times. Importantly, this phase only adds sub-linear regret, and after this phase we can assume $T$ to be large enough such that for every arm $i$
\begin{align}\label{eq_close_to_asymptote}
\abs{f_i(\log(T)) - a_i} \leq \varepsilon
\end{align}
where $\varepsilon = \min_i \frac{a^* - a_i}{4}$

$n_i^T$ is the number of times the algorithm pulls arm $i$ and
$$
n_i^T = n_{i,1}^T + n_{i,2}^T = \log(T) + n_{i,2}^T
$$

Let $i \notin W$. We will show that $\lim_{T\to\infty} \frac{n_i^T}{T} = 0$ by using $\lim_{T\to\infty} \frac{n_i^T}{T} =  \lim_{T\to\infty} \frac{\log(T) + n_{i,2}^T}{T} = \lim_{T\to\infty} \frac{n_{i,2}^T}{T}$. We show that $\limsup_{T \to \infty} n_{i,2}^T = 0$ which implies $\lim_{T\to\infty} \frac{n_{i,2}^T}{T} = 0$.

Assume this was not the case, then we could find an infinite subsequence of $T$'s, which we label $\{\tau_i^j\}_{j=1}^\infty$ such that for every $j\geq 1$
$$
n_{i,2}^{\tau_i^j} \geq 1
$$
Dropping the indices, we just write $\tau$ to mean $\tau_i^j$ for an arbitrary $j$.

Let $\gamma_i^\tau$ be the last timestep at which the algorithm pulls arm $i$ within time horizon $\tau$ and $s_i^\tau = n_k^\tau(\gamma_i^\tau)$ the number of times it pulls the arm $k$ up until the last pull of $i$.

Now, at timestep $\gamma_i^\tau$ the algorithm pulls arm $i$, and because $n_{i,2}^\tau \geq 1$, this implies that $i$ has the highest optimistic future reward. In particular:
$$
p_i^\tau(n_i^\tau-1, \gamma_i^\tau) \geq p_k^\tau(s_i^\tau, \gamma_i^\tau)
$$
No matter weather $k$ is asymptotically increasing or asymptotically decreasing, we can always lower-bound its optimistic future reward by
$$
p_k^\tau(s_i^\tau, \gamma_i^\tau) \geq (\tau - \gamma_i^\tau) \cdot f_k(s_i^\tau)
$$
Now, there are two cases: either arm $i$ is asymptotically increasing or it is asymptotically decreasing.

\paragraph{First case: $i$ is asymptotically increasing.}
We have

\resizebox{\linewidth}{!}{
    \begin{minipage}{\linewidth}\begin{align*}
    &p_i^\tau(n_i^\tau-1, \gamma_i^\tau)\\
    =&\sum_{t=\gamma_i^\tau+1}^\tau \min\cbr{1, \rbr{f_i(n_i^\tau-1) + \Delta_i(n_i^\tau-1) \cdot (t-\gamma_i^\tau)}} \\
    \leq& \sum_{t=\gamma_i^\tau+1}^\tau \rbr{f_i(n_i^\tau-1) + \Delta_i(n_i^\tau-1) \cdot (t-\gamma_i^\tau)} \\
    =& f_i(n_i^\tau-1) \cdot (\tau-\gamma_i^\tau) + \Delta_i(n_i^\tau-1) \cdot \sum_{t=1}^{\tau-\gamma_i^\tau} t \\
    =& f_i(n_i^\tau-1) \cdot (\tau-\gamma_i^\tau) + \frac12 \Delta_i(n_i^\tau-1) \cdot (\tau-\gamma_i^\tau)(\tau-\gamma_i^\tau+1)
    \end{align*}\end{minipage}
}
where $\Delta_i(t) = f_i(t) - f_i(t-1)$.

Combining the inequalities gives
$$f_i(n_i^\tau-1) \cdot (\tau-\gamma_i^\tau) + \frac12 \Delta_i(n_i^\tau-1) \cdot (\tau-\gamma_i^\tau)(\tau-\gamma_i^\tau+1)$$
$$\geq (\tau-\gamma_i^\tau) \cdot f_{k}(s_i^\tau)$$

Dividing by $\tau-\gamma_i^\tau$ gives
\begin{align*}
f_i(n_i^\tau-1) + \frac12 \cdot \Delta_i(n_i^\tau-1) \cdot (\tau-\gamma_i^\tau+1) \geq f_{k}(s_i^\tau)
\end{align*}

and subtracting $f_i(n_i^\tau-1)$ gives
$$
\frac12 \cdot \Delta_i(n_i^\tau-1) \cdot (\tau-\gamma_i^\tau+1) \geq f_{k}(s_i^\tau) - f_i(n_i^\tau-1)
$$
We can upper bound the l.h.s. by $\frac12 \Delta_i(n_i^\tau-1) \cdot \tau$ and lower bound the r.h.s. using eq. \ref{eq_close_to_asymptote} and $f_i(n_u^\tau-1) \leq a_i$ by
$$
f_{k}(s_i^\tau) - f_i(n_i^\tau-1)
\geq a^* - \varepsilon - a_i
\geq \frac{3(a^*-a_i)}{4}
\geq \frac{a^*-a_i}{2}
$$
which together gives
\begin{align*}
\Delta_i(n_i^\tau-1) \cdot \tau \geq a^*-a_i
\end{align*}
Reintroducing the indices of $\tau$, the full statement is that for any $j$ we have
\begin{align}\label{contradiction_ineq}
\Delta_i(n_i^{\tau_i^j}-1) \cdot \tau_i^j \geq a^*-a_i
\end{align}

In the next step, we use a small technical Lemma, that is proven after the main proof.
\begin{restatable}{lemma}{LemmaDeltaT}\label{lemma_delta_T}
Let $\cbr{n_T}_{T=1}^{\infty}$ be a positive sequence with $\lim_{T \to \infty} n_T = \infty$,
$f_i$ a bounded and increasing reward function. Then
$$
\lim_{T\to\infty} \Delta_i(n_T) \cdot T = 0
$$
where $\Delta_i(t) = f_i(t) - f_i(t-1)$.
\end{restatable}

Because $\lim_{T \to \infty} n_i^T = \lim_{T \to \infty} \rbr{\log(T) + n_{i,2}^T} = \infty$. We can apply Lemma \ref{lemma_delta_T} to get
$\lim_{T \to \infty} \Delta_i(n_i^{T}-1) \cdot T = 0$. In particular, this holds for the sub-sequence $\{\tau_i^j\}_{j=1}^\infty$:
$$
\lim_{j\to\infty} \Delta_i(n_i^{\tau_i^j}-1) \cdot \tau_i^j = 0
$$
But together with eq. \ref{contradiction_ineq} this would mean $a^* = a_i$ which is a contradiction to $i \notin W$. Therefore, $\cA$ drops increasing arms with suboptimal asymptotes in sub-linear time. This concludes the discussion of the first case.

\paragraph{Second case: $i$ is asymptotically decreasing.}
We have
$$
p_i^\tau(n_i^\tau-1, \gamma_i^\tau) = (\tau - \gamma_i^\tau) \cdot f_i(n_i^\tau-1)
$$
and the resulting inequality is
$$
(\tau - \gamma_i^\tau) \cdot f_i(n_i^\tau-1) \geq (\tau - \gamma_i^\tau) \cdot f_k(s_i^\tau)
$$
which after dividing by $(\tau - \gamma_i^\tau)$ leaves us just with
\begin{align}\label{ineq_i_k}
f_i(n_i^\tau-1) \geq f_k(s_i^\tau)
\end{align}
However $\abs{f_k(s_i^\tau)-a_k} \leq \varepsilon$ and $\abs{f_i(n_i^\tau-1)-a_i} \leq \varepsilon$ imply
\begin{align*}
f_k(s_i^\tau)-a_k &\leq \varepsilon \\
a_k-f_k(s_i^\tau) &\leq \varepsilon \\
f_i(n_i^\tau-1)-a_i &\leq \varepsilon \\
a_i-f_i(n_i^\tau-1) &\leq \varepsilon
\end{align*}
And combining this with eq. \ref{ineq_i_k} we get
$$
a_i + \varepsilon \geq a_k - \varepsilon
$$
$$
a_i + 2 \cdot \varepsilon \geq a_k
$$
$$
a_i + 2 \cdot \min_j \frac{a^*-a_j}{4} \geq a_k
$$
$$
a_i + 2 \cdot \frac{a^*-a_i}{4} \geq a_k
$$
$$
a_i + \frac{a^*}{2} - \frac{a_i}{2} \geq a_k
$$
$$
\frac{a_i}{2} \geq a_k - \frac{a^*}{2}
$$
$$
\frac{a_i}{2} \geq \frac{a^*}{2}
$$
where in the last step we used $k \in W$ and therefore $a_k = a^*$.
But this implies $a_i \geq a^*$, which is a contradiction to the assumption $i \notin W$. This shows that we also get a contradiction and (ii) holds for asymptotically increasing as well as asymptotically decreasing arms. This concludes the proof.
\end{proof}

\LemmaDeltaT*

\begin{proof}
This result, which is also used by \citet{Heidari2016}, follows from the reward function $f$ being
increasing and bounded between $0$ and $1$. Observe that
$$
\sum_{T=1}^\infty \Delta_i(T) \leq 1
$$
$$
\sum_{T=1}^\infty \Delta_i(T) \cdot t \leq t
$$
$$
\lim_{T\to\infty} \Delta_i(T) \cdot t = 0
$$
$$
\lim_{T\to\infty} \Delta_i(n_T) \cdot t = 0
$$
where in the last step we replaced the $T$ with $n_T$ which is possible because $\lim_{T\to\infty} n_T = \infty$.

Hence, the sequence $\cbr{\Delta_i(n_T) \cdot t}_{T=1}^{\infty}$ converges to $0$ for any
fixed $t$. For the next step, recall Lebesgue's Dominated Convergence Theorem.
\begin{theorem*}[Lebesgue's Dominated Convergence]
   Let $\{f_n\}_{n=1}^\infty$ be a sequence of complex-valued measurable functions on a measure space $(S, \Sigma, \mu)$. Suppose that the sequence converges pointwise to a function $f$ and is dominated by some integrable function $g$ in the sense that $|f_n(x)| \leq g(x)$ for all numbers $n$ in the index set of the sequence and all points $x \in S$. Then $f$ is integrable and
   $$\lim_{n\to\infty} \int_S |f_n-f|\,d\mu = 0$$
\end{theorem*}
We can apply this theorem here by considering the sequence $\{f_T(t)\}_{T=1}^\infty = \{\Delta_i(n_T)\}_{T=1}^\infty$ which is define on $S=\N^+$ and $\mu$ is the counting measure. We just showed that this sequence converges pointwise to the function $f(t) = 0$. Because $\Delta_i(n_T) \leq 1$, the sequence is also dominated by $g(t) = t$ which is integrable.

The theorem then gives
$$
\lim_{T\to\infty} \sum_{t=1}^\infty \Delta_i(n_T) \cdot t = 0
$$
For any $T > 0$
$$
\Delta_i(n_T) \cdot T \leq \sum_{t=1}^\infty \Delta_i(n_T) \cdot t
$$
because $\Delta_i(n_T)\cdot T$ is just one term of the sum and all terms are non-negative.

It follows that
$$
\lim_{T\to\infty} \Delta_i(n_T) \cdot T \leq \lim_{T\to\infty} \sum_{t=1}^\infty \Delta_i(n_T) \cdot t = 0
$$
Because $f$ is increasing we also have
$$
\lim_{T\to\infty} \Delta_i(n_T) \cdot T \geq 0
$$
and therefore
$$
\lim_{T\to\infty} \Delta_i(n_T) \cdot T = 0
$$
\end{proof}

    \subsection{Handling Noise}

\LemmaNoiseLP*

\begin{proof}
For convenience we restate the LP:
\begin{equation*}\tag{$\star$}
{\small
\begin{array}{ll@{}ll}
\text{maximize}_{v}  & \displaystyle\sum\limits_{j=n+1}^{n+T-t} v_{j}& &\\
\text{subject to}& 0 \leq v_j \leq 1, &  &j=1 , \dots, T\\
                 & L_i^j \leq v_j \leq U_i^j, &  &j=1 , \dots, n\\
                 & v_j \leq v_{j+1}, &  &j=1, \dots, T-1\\
                 & v_j \leq 2 v_{j-1} - v_{j-2}, &  &j=3 , \dots, T
\end{array}
}
\end{equation*}

We show that $v_1, \dots, v_T$ is a feasible solution to the LP if and only if $f_i$ defined by $v_1 = f_i(1), v_2 = f_i(2), \dots, v_T = f_i(T)$ is a concave increasing reward function. This implies the claim of the lemma and in particular that the optimal value of the LP provides an upper bound on the future optimistic reward of any reward function $f_i$ consistent with the past observations.

To show the claim, we verify that the constraints are equivalent to $f_i$ being bounded, increasing and concave:
\begin{itemize}
    \item $0 \leq f_i(j) \leq 1$ for $j = 1, \dots, T$ means that $f_i$ is bounded
    \item $L_i^j \leq f_i(j) \leq U_i^j$ for $j = 1, \dots, n$ restricts $f_i$ to functions consistent with the past observations
    \item $f_i(j) \leq f_i(j+1)$ for $j = 1, \dots, T-1$ means that $f_i$ is increasing
    \item $f_i(j) - f_i(j-1) \leq f_i(j-1) - f_i(j-2)$ for $j=3,\dots,T$ is the concavity of $f_i$
\end{itemize}
Hence, every $f_i$ gives a feasible solution to the LP. The objective of this solution is $\sum_{j=n+1}^{n+T-t} v_j = \sum_{j=n+1}^{n+T-t} f(j)$, which is the cumulative future return of $f_i$.
\end{proof}

    \section{Additional Experiments}\label{sec:additional_experiments}

In this section, we provide more extensive empirical results from the simulations discussed in the main text of the paper.

In \Cref{app:additional_increasing}, we discuss increasing and concave reward functions and compare \algshort to an algorithm proposed by \citet{Heidari2016} for this specific setting. In \Cref{app:additional_increasing_decreasing}, we consider more synthetic single-peaked reward functions. In \Cref{app:additional_const_1} we consider constant rewards, which corresponds to stationary MABs with Gaussian observations. In \Cref{app:additional_recommender}, we consider different instances of the recommender system simulation. In \Cref{app:additional_fico}, we provide a variation of the reward functions from the FICO dataset that only considers the bank's utility. 

\subsection{Increasing Reward Functions}\label{app:additional_increasing}

For monotonically increasing and concave reward functions we choose a set of synthetic functions that \citet{Heidari2016} use. We compare their algorithm, which is restricted to increasing, concave reward functions, to ours by plotting the per-step-regret of both algorithms. The results are shown in \Cref{fig:experiment_inc1,fig:experiment_inc2,fig:experiment_inc3}. Additionally we choose one of the reward function to test the robustness of the algorithms to noisy observations, shown in \Cref{fig:experiment_inc_noise}. \Cref{fig:experiment_inc1} shows \algshort perform comparably to \citet{Heidari2016}'s algorithm. In all other experiments \algshort outperforms their algorithm significantly. Their algorithm fails especially when adding noise as in \Cref{fig:experiment_inc_noise}. EXP3 and R-EXP3 perform poorly on all experiments, while D-UCB and SW-UCB work well in some of them, but fail in others.

\subsection{Single-Peaked Reward Functions}\label{app:additional_increasing_decreasing}

Here we consider more synthetic single-peaked reward functions to evaluate \algshort on, similar to \Cref{fig:synthetic} in the main text. The experimental setup is as before with two reward functions $f_1$ and $f_2$. We consider three sets of reward functions with different levels of Gaussian noise.

The results of these experiments are shown in \Cref{fig:experiment_inc_dec_1,fig:experiment_inc_dec_2,fig:experiment_inc_dec_3}.
\algshort compares favorably to the baselines. While the baselines perform comparably in some cases, they fail in others, e.g., in \Cref{fig:experiment_inc_dec_3}.

\subsection{Gaussian Multi-armed Bandits}\label{app:additional_const_1}

One might wonder if \algshort can also be used for classical multi-armed bandits for which the rewards are drawn from a fixed distribution for each arm. Of course, one would expect \algshort to perform worse in this situation than algorithms designed for the simple MAB problem. However, in some situations there might be little a priori information on the shape of the reward functions, and then it is beneficial to have an algorithm that can handle different situations such as increasing, decreasing or constant rewards.

To test whether \algshort can handle simple MABs, we consider constant reward functions with Gaussian noise. This corresponds to a MAB setup where the reward of each arm is drawn from a Gaussian distribution with a fixed mean. We evaluate \algshort on sets of arms with randomly sampled means, and compare it to UCB. \Cref{fig:experiment_const_1} shows that \algshort can still find the optimal arm in this situation, but takes longer than UCB which is d.

\subsection{FICO Dataset}\label{app:additional_fico}

\Cref{fig:experiment_fico_score_change,fig:experiment_fico_utility} show additional results for our experiments using the FICO dataset.

\Cref{fig:experiment_fico_score_change} shows the same results presented in the main paper, but for different simulated noise levels. The general observation that \algshort outperforms all baselines holds across different levels of noise. As expected its advantage shrinks a bit for high noise levels.

In a separate experiment, we consider reward functions defined by the banks's utility of giving a loan. We use the model by \citet{liu2018delayed}, which assumes a utility of $+1$ for a repaid loan and $-4$ for a defaulted loan. In this case, the reward functions are decreasing almost everywhere. For monotonically decreasing reward functions \citet{Heidari2016} show that greedy approaches are optimal. This is also what we see in the results in \Cref{fig:experiment_fico_utility}: both the greedy and the one-step-optimistic algorithm achieve no-policy-regret everywhere. However, our algorithm also achieves good performance and outperforms some of the other baselines significantly.

\subsection{Simulated Recommender System}\label{app:additional_recommender}

We randomly generate 3 instances of the simulated recommender system described in the main text. \Cref{fig:experiment_recommender_A,fig:experiment_recommender_B,fig:experiment_recommender_C} show the reward functions and experimental results for different levels of simulated Gaussian noise. The results across all instances and noise levels are consistent with the results we report in the main paper, and they show that \algshort outperforms all alternatives significantly except for very short time horizons.

\begin{figure*}[p]
\centering
\begin{minipage}[b]{.48\textwidth}
   \centering
   \begin{subfigure}[c]{\linewidth}
      \includegraphics[width=0.32\linewidth]{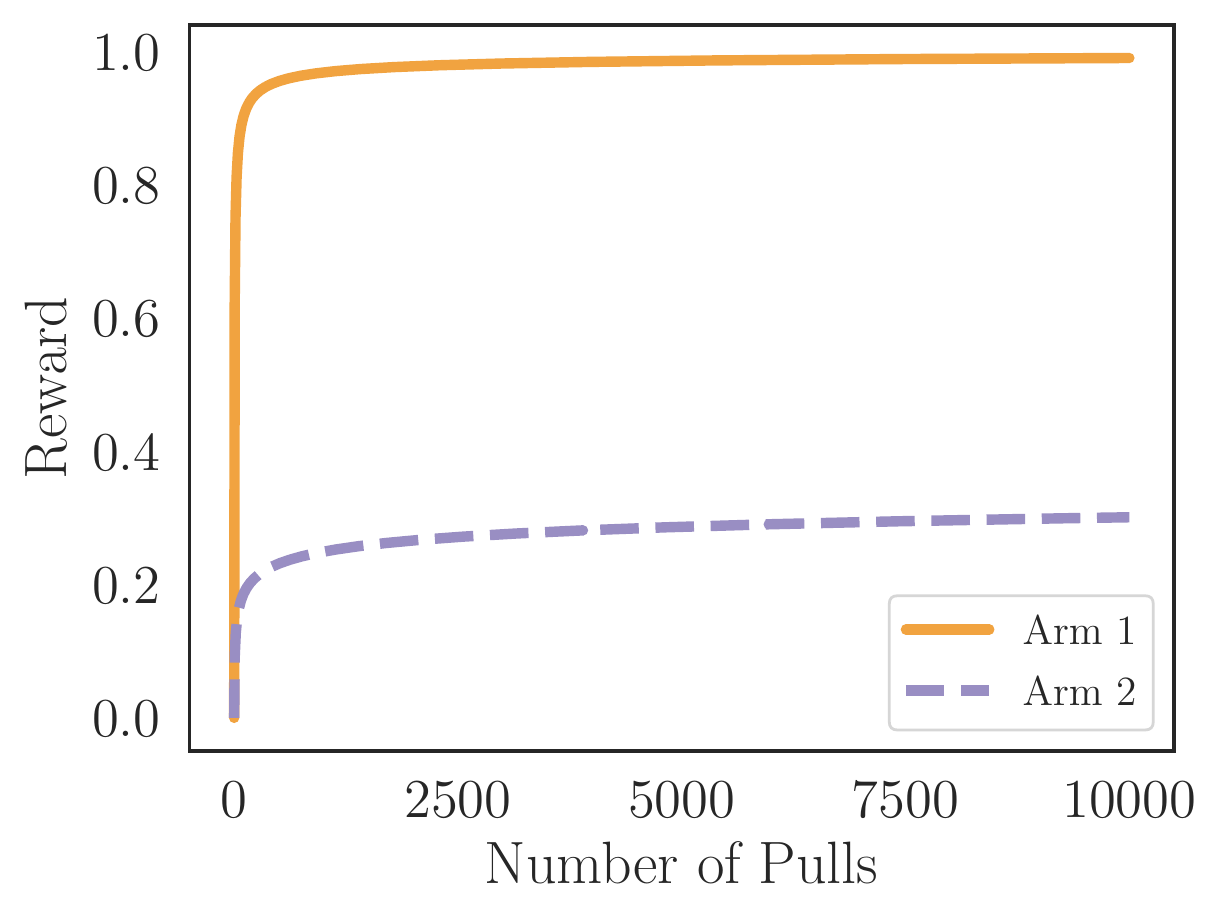}\hfill
      \includegraphics[width=0.32\linewidth]{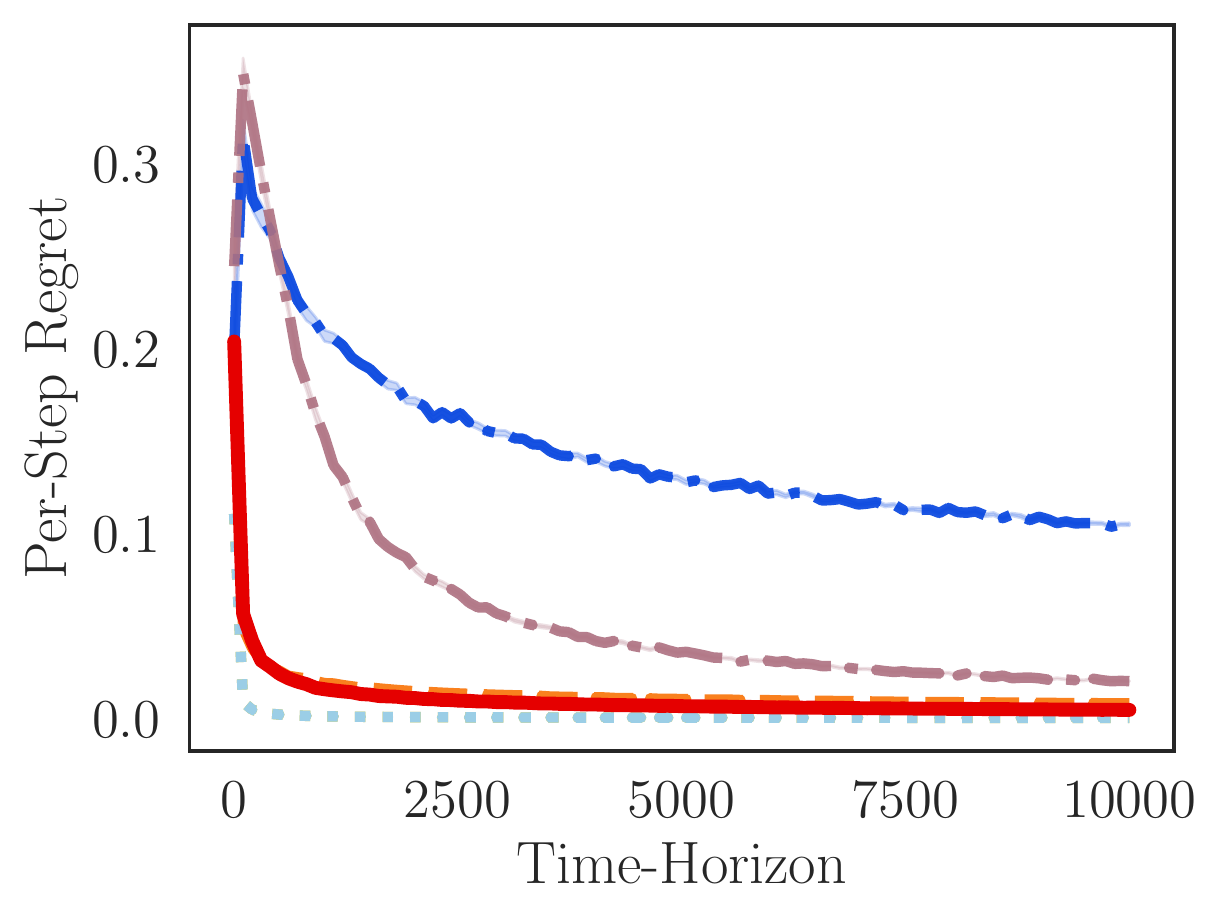}\hfill
      \includegraphics[width=0.32\linewidth]{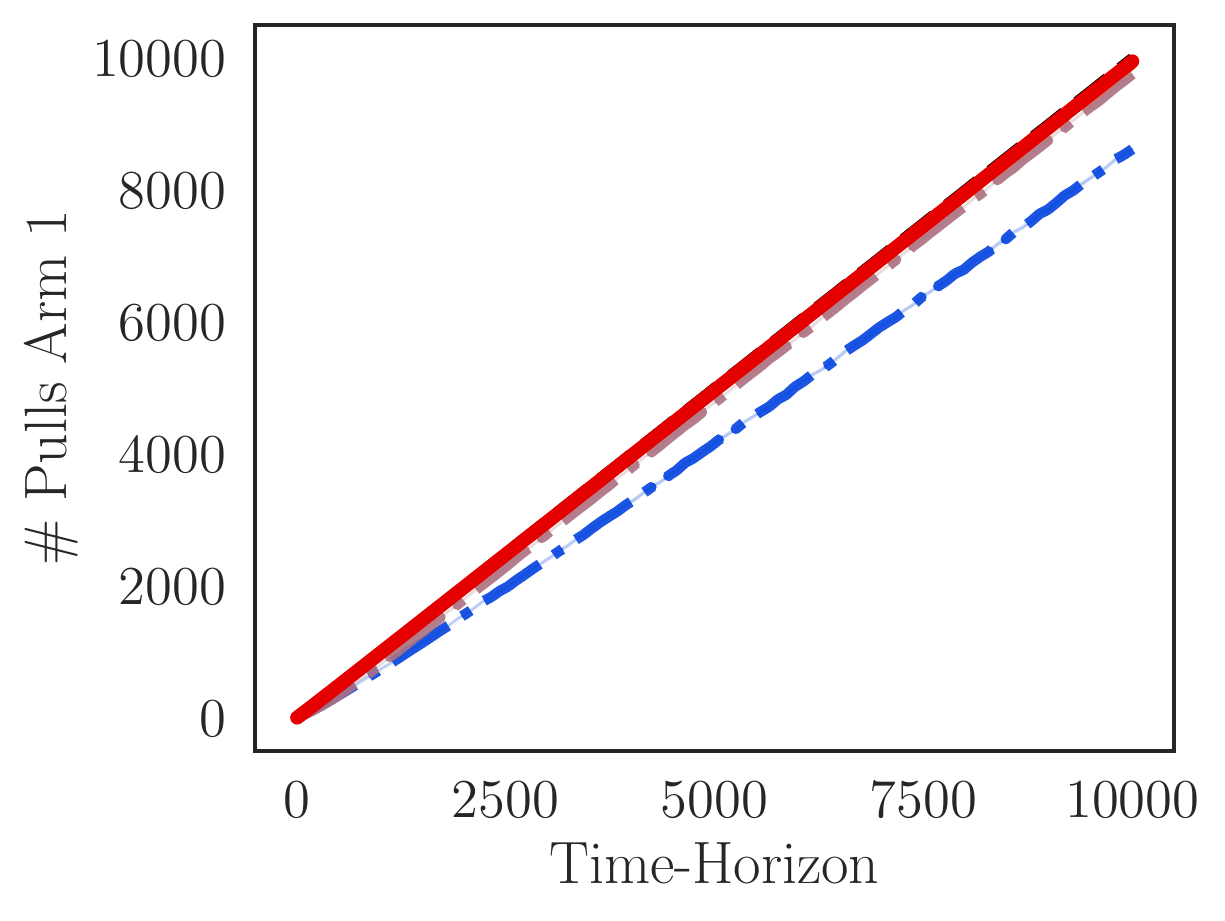}
      \caption{$\alpha = 0.1$}
   \end{subfigure}\vspace{1em}
   \begin{subfigure}[c]{\linewidth}
      \includegraphics[width=0.32\linewidth]{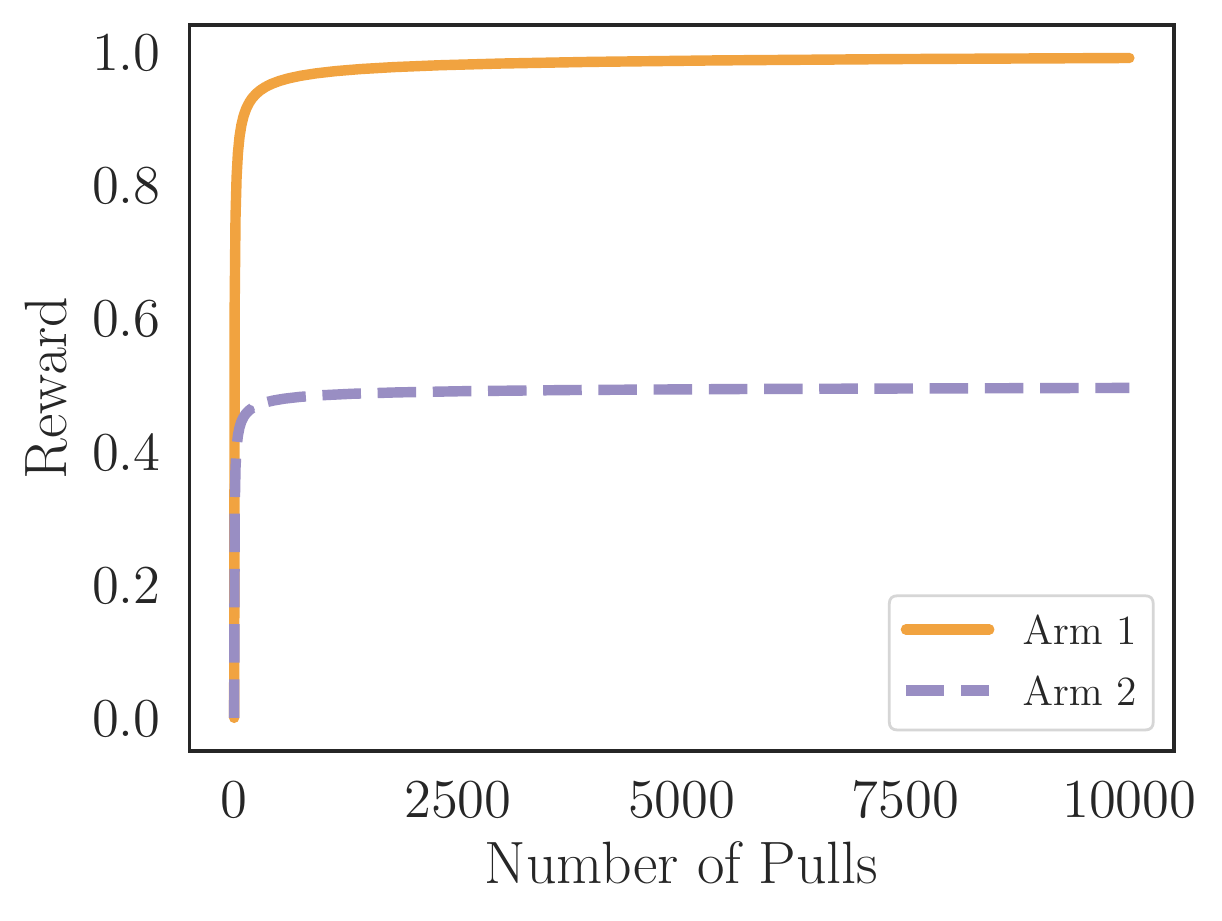}\hfill
      \includegraphics[width=0.32\linewidth]{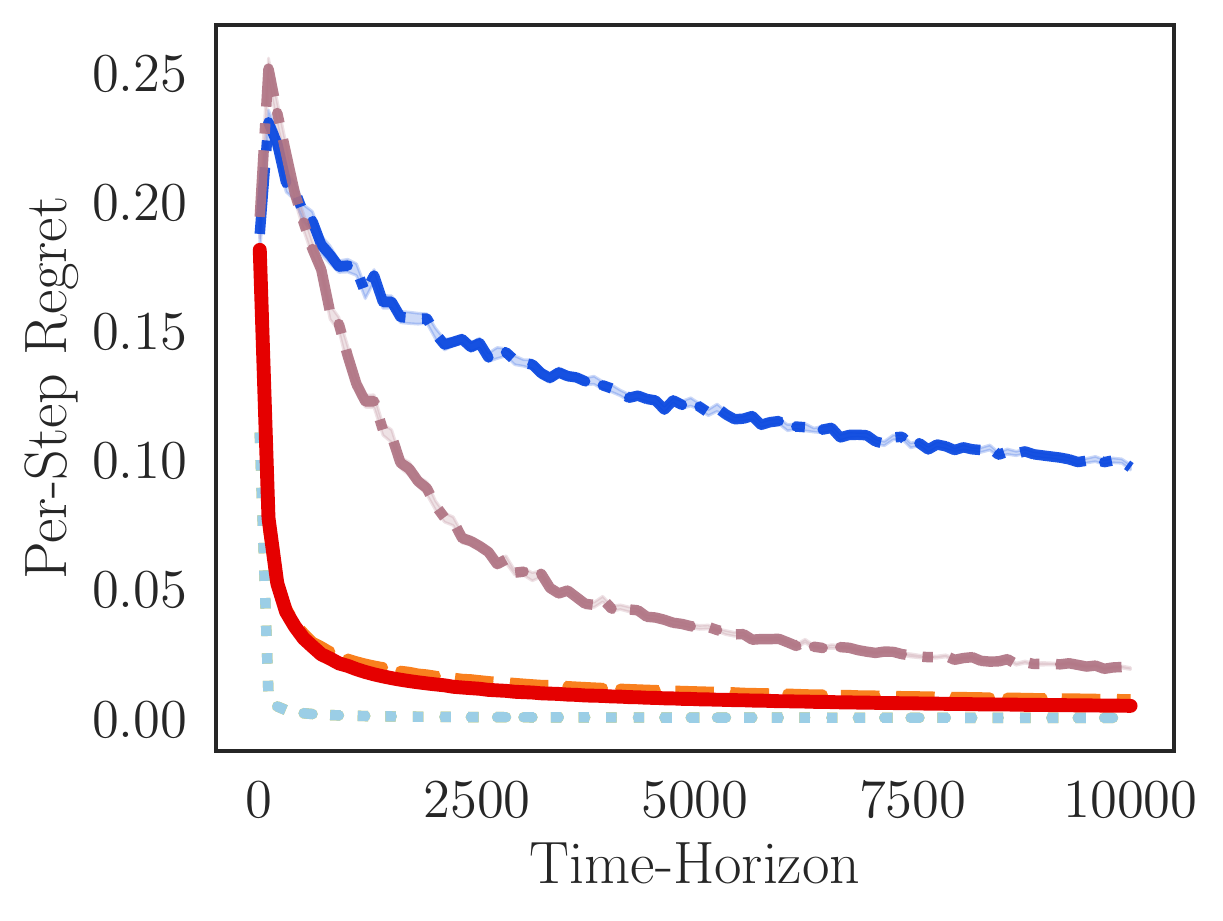}\hfill
      \includegraphics[width=0.32\linewidth]{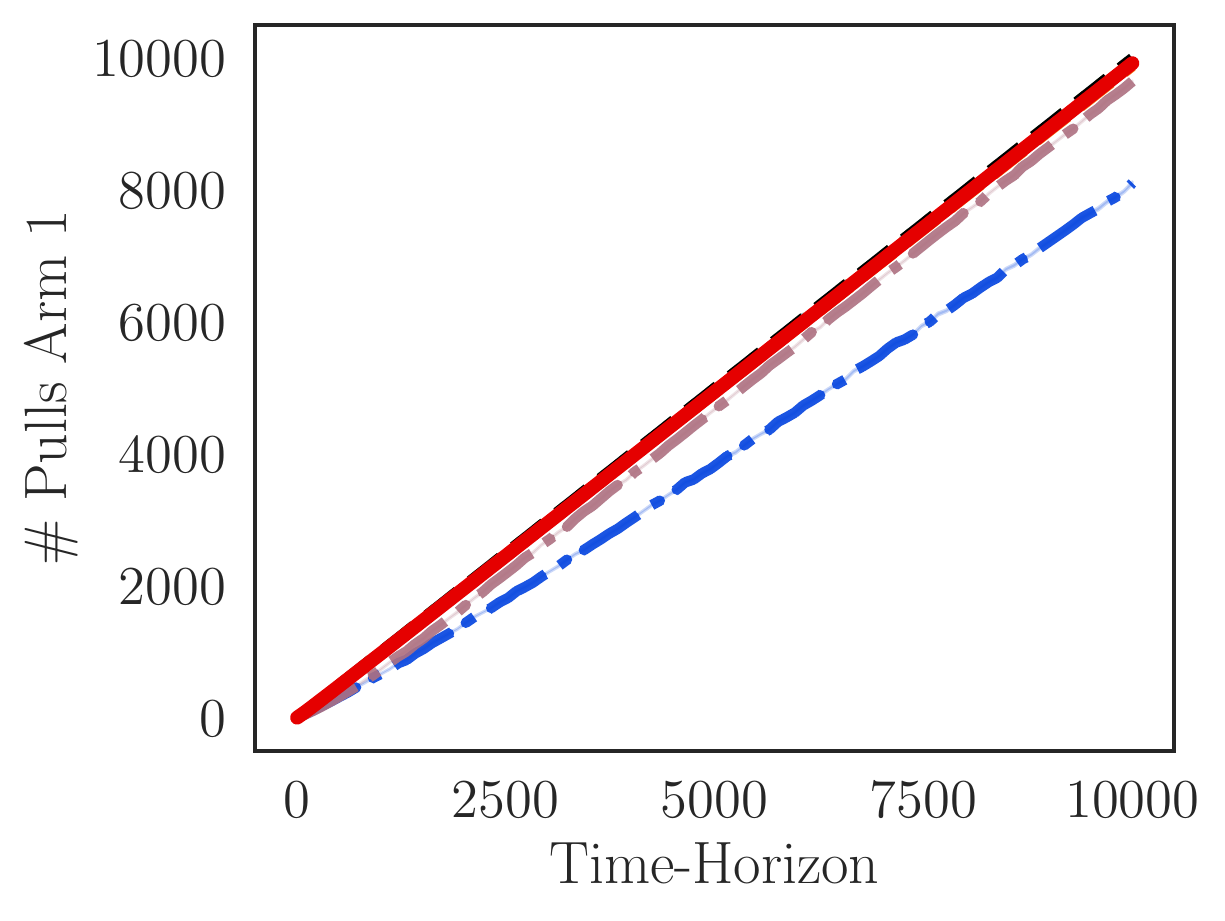}
      \caption{$\alpha = 0.5$}
   \end{subfigure}\vspace{1em}
   \begin{subfigure}[c]{\linewidth}
      \includegraphics[width=0.32\linewidth]{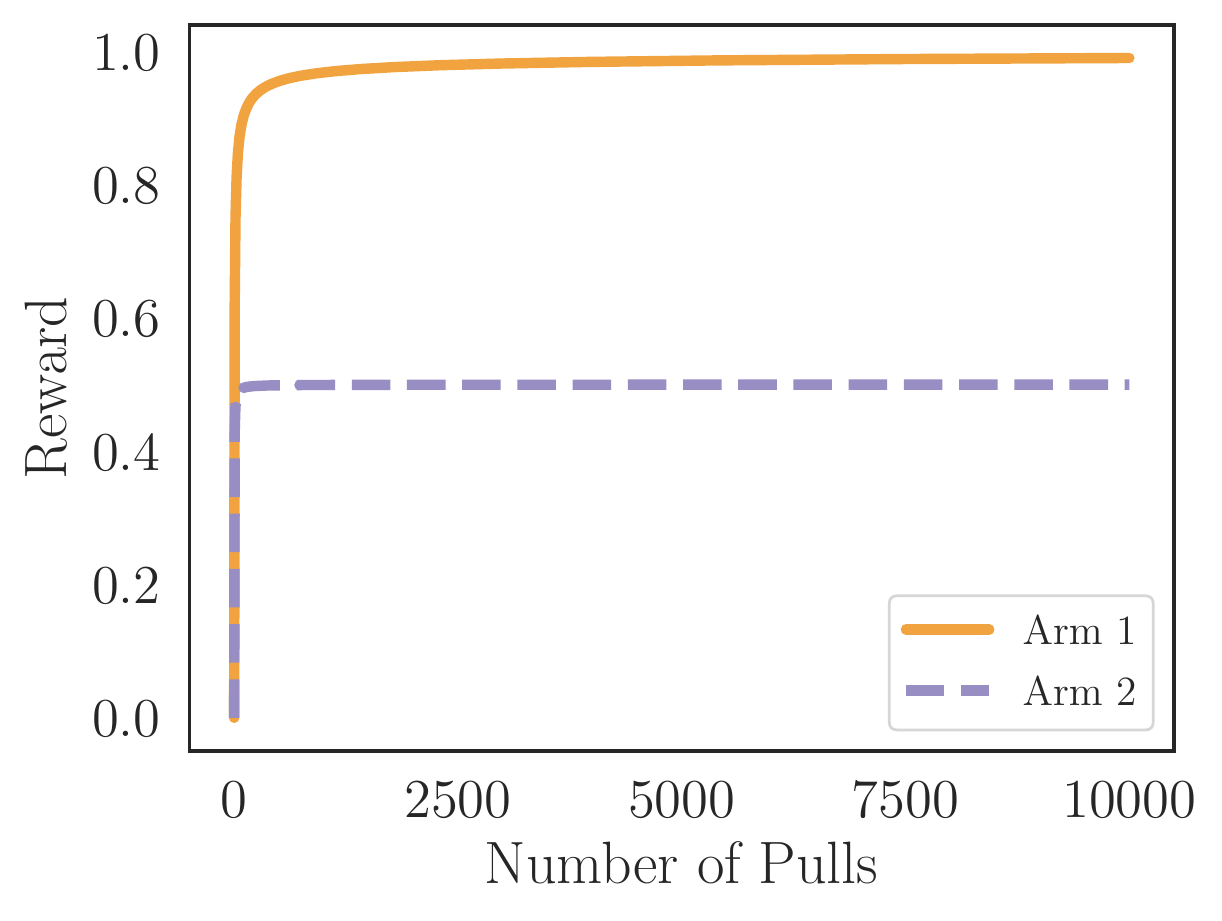}\hfill
      \includegraphics[width=0.32\linewidth]{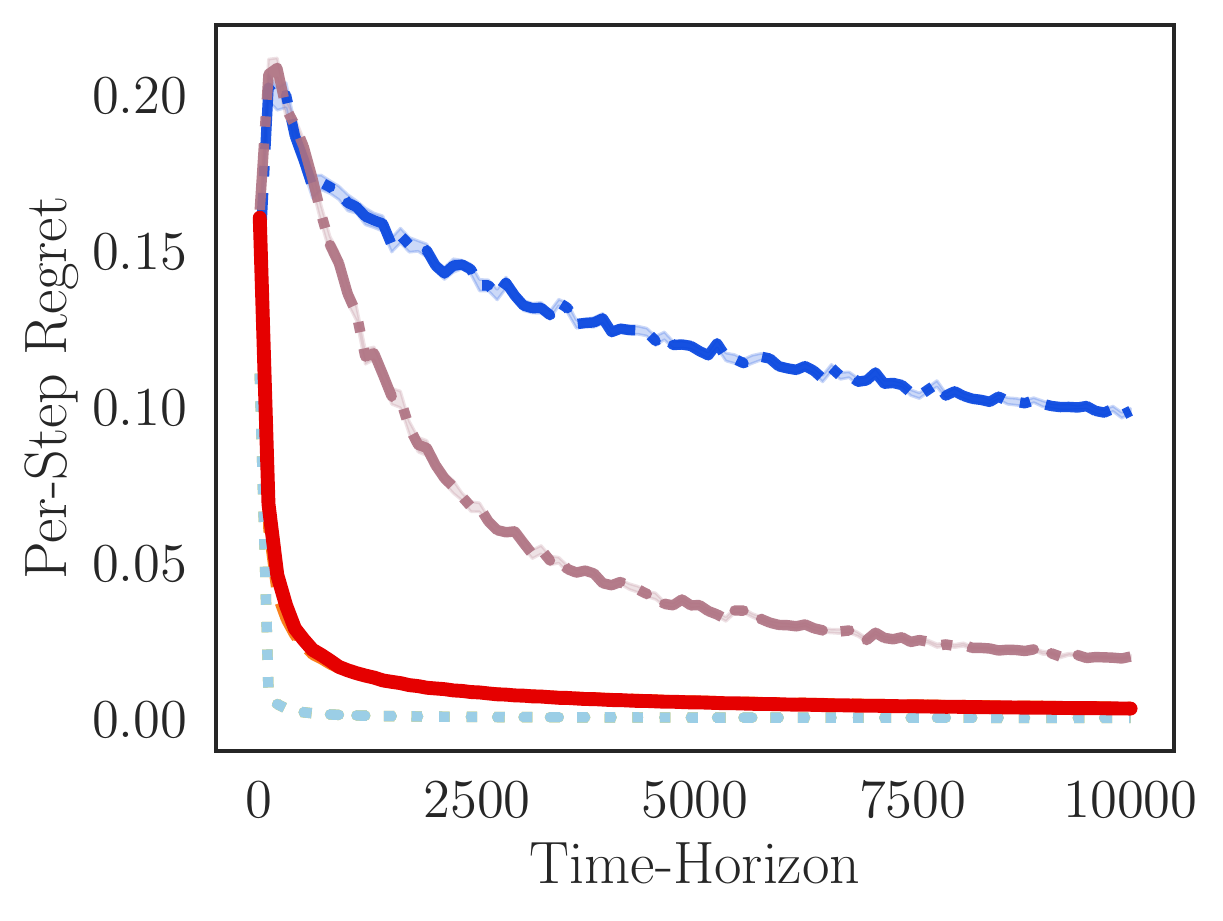}\hfill
      \includegraphics[width=0.32\linewidth]{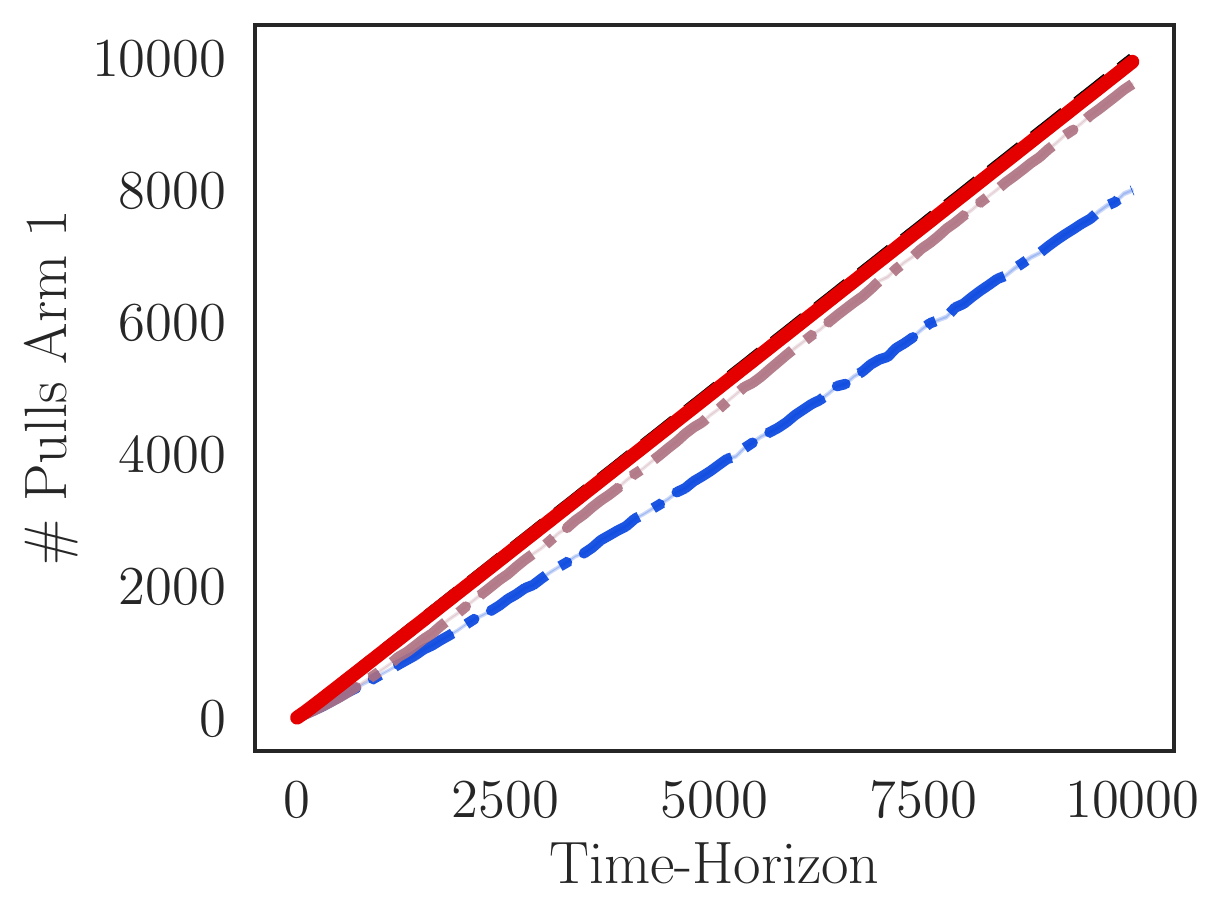}
      \caption{$\alpha = 1$}
   \end{subfigure}\vspace{1em}
   \begin{subfigure}[c]{\linewidth}
      \includegraphics[width=0.32\linewidth]{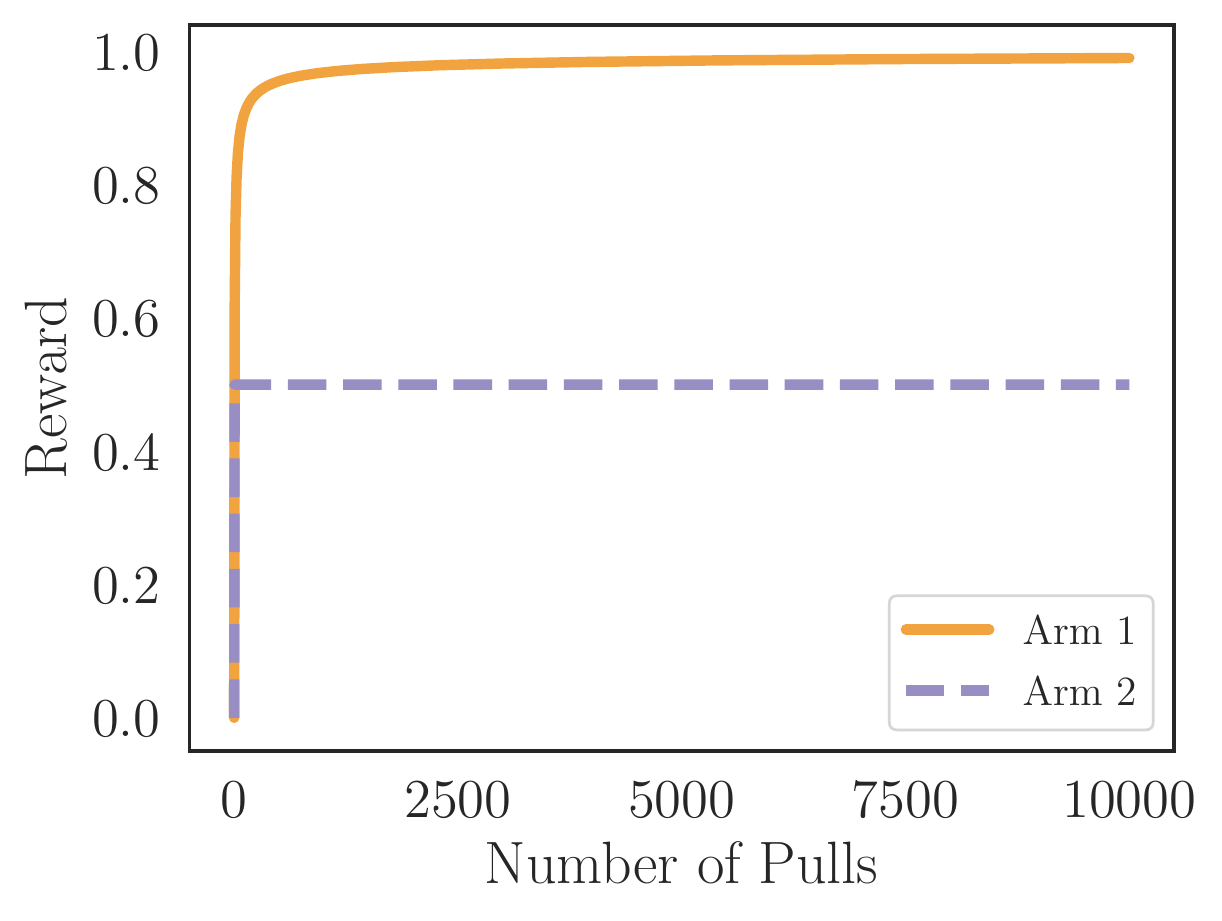}\hfill
      \includegraphics[width=0.32\linewidth]{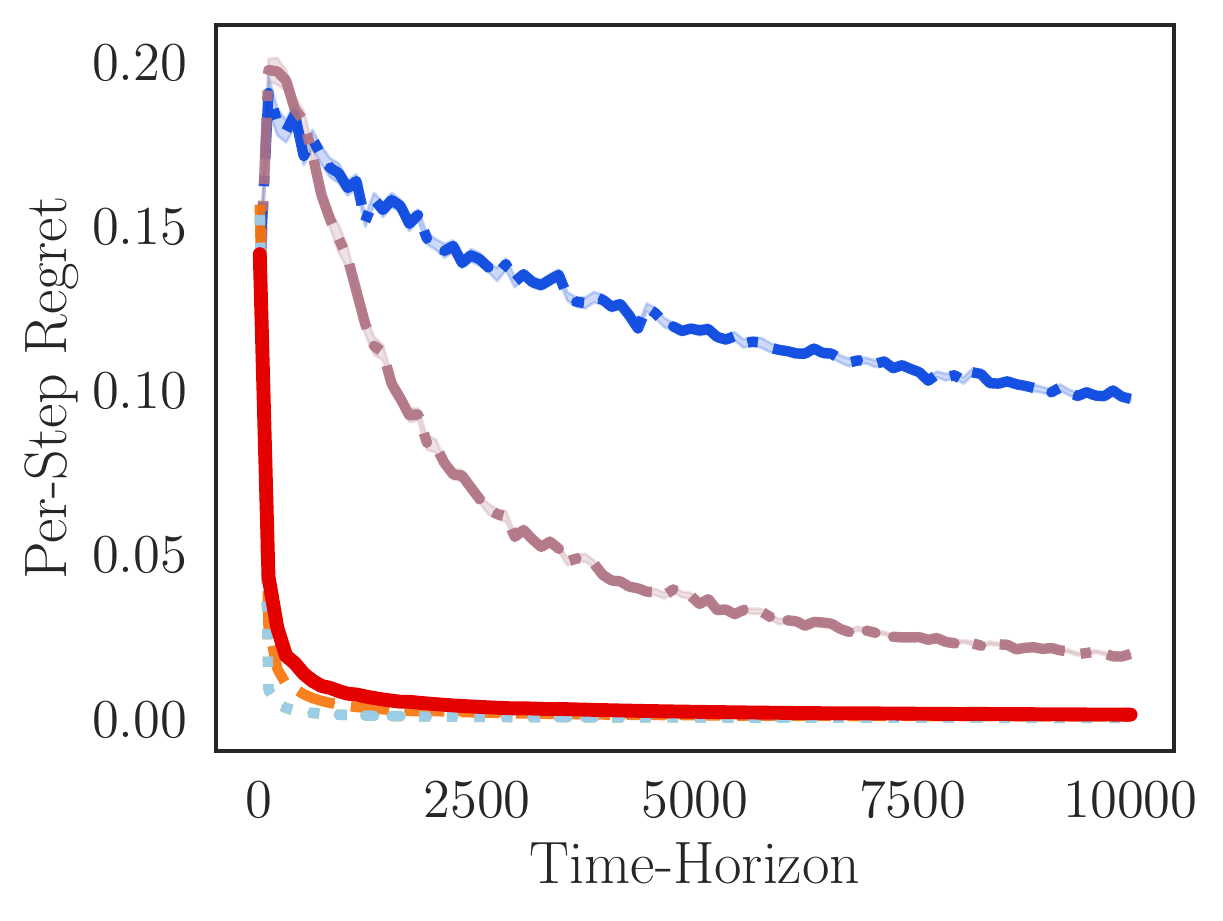}\hfill
      \includegraphics[width=0.32\linewidth]{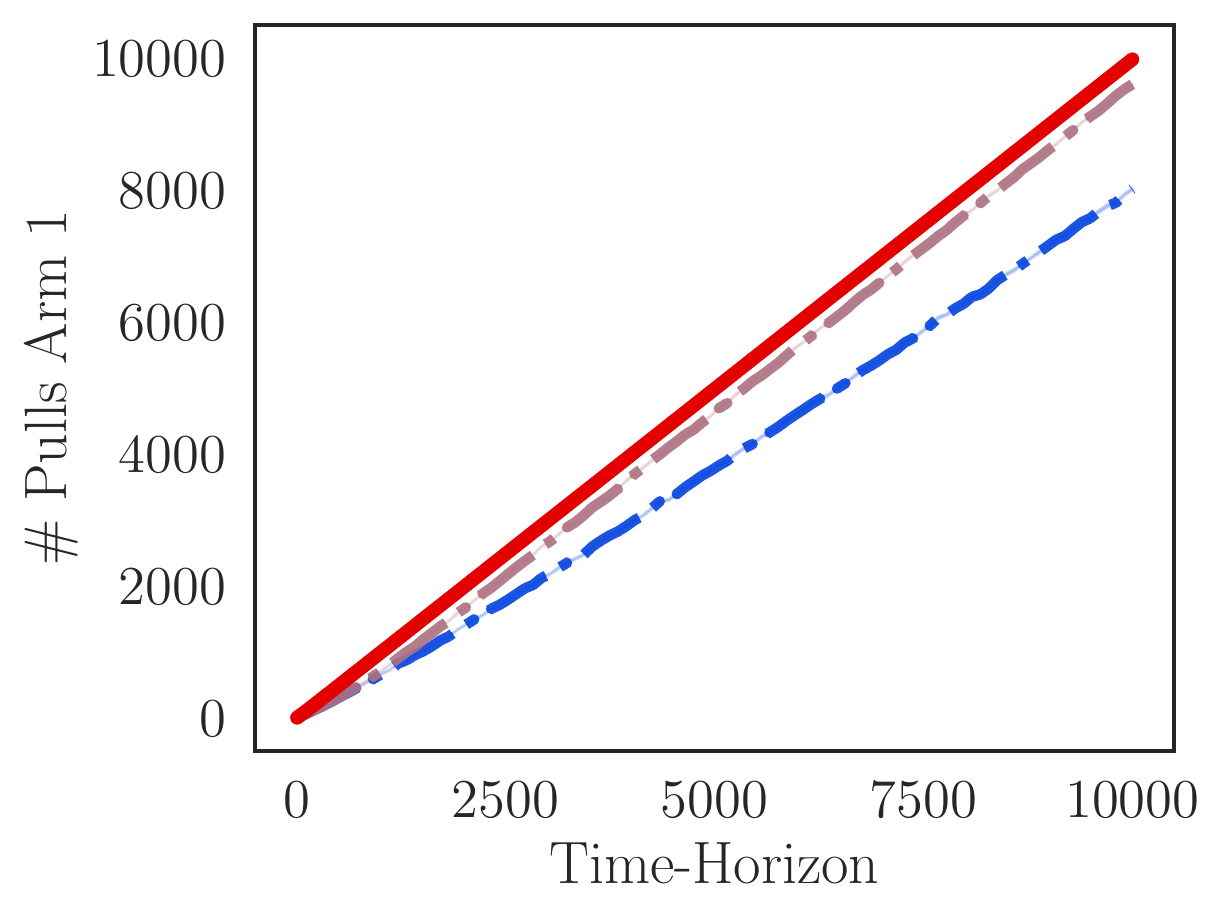}
      \caption{$\alpha = 5$}
   \end{subfigure}\vspace{1em}
   \caption{The left-hand plots show the increasing reward functions defined by $f_1(t) = 1 - t^{-0.5}$, $f_2(t) = 0.5 - 0.5 t^{-\alpha}$ where $\alpha = 0.1, 0.5, 1, 5$ (from top to bottom). The middle plots show the per-step policy regret achieved by \algshort ({\protect\legendSPO}) compared to the algorithm proposed by \citet{Heidari2016} for increasing rewards ({\protect\legendINCREASING}), EXP3 ({\protect\legendEXP}), R-EXP3 ({\protect\legendREXP}), D-UCB ({\protect\legendDUCB}), and SW-UCB ({\protect\legendSWUCB}). The right-hand plots show the policies these algorithms choose in comparison to the optimal policy ({\protect\legendOPTIMAL}).}
   \label{fig:experiment_inc1}
\end{minipage}\hfill
\begin{minipage}[b]{.48\textwidth}
    \begin{subfigure}[c]{\linewidth}
      \includegraphics[width=0.32\linewidth]{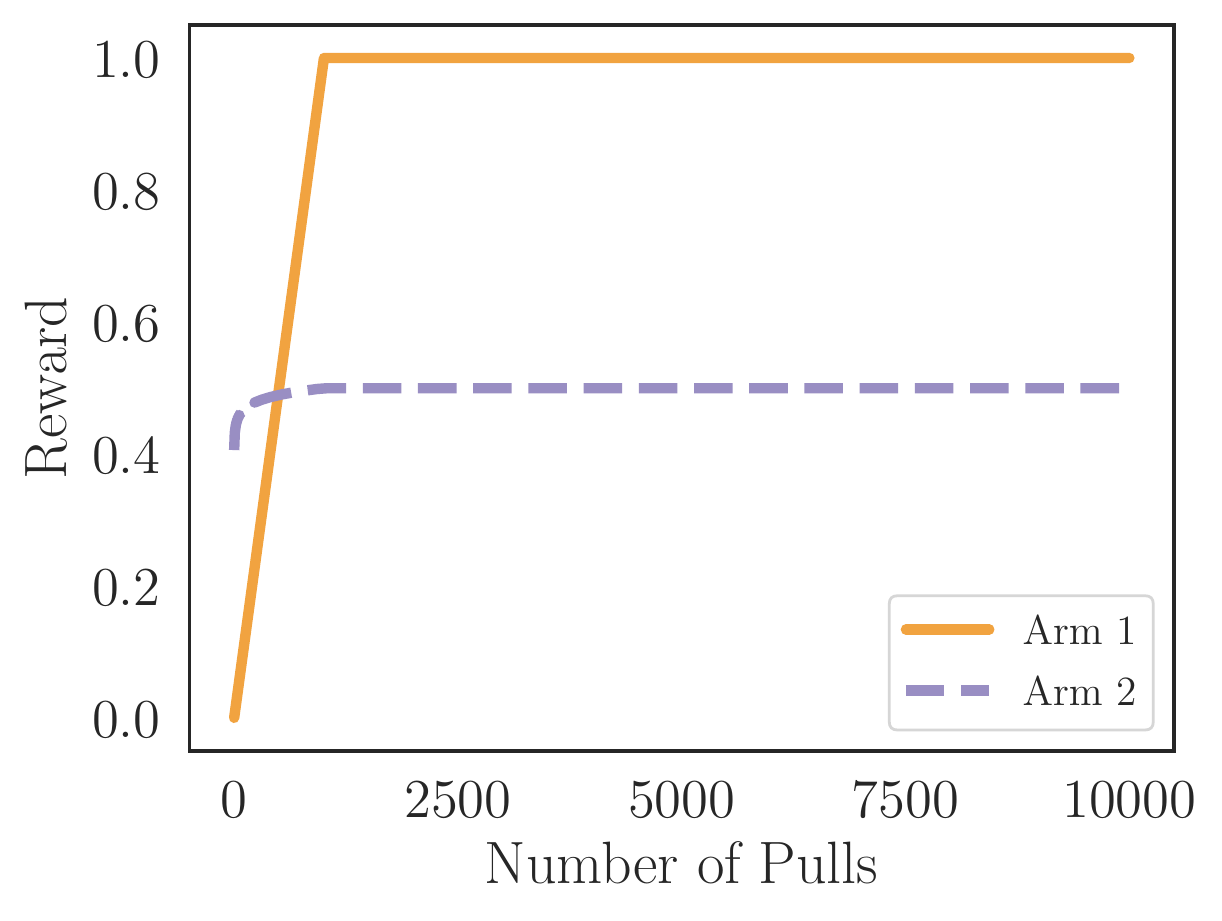}\hfill
      \includegraphics[width=0.32\linewidth]{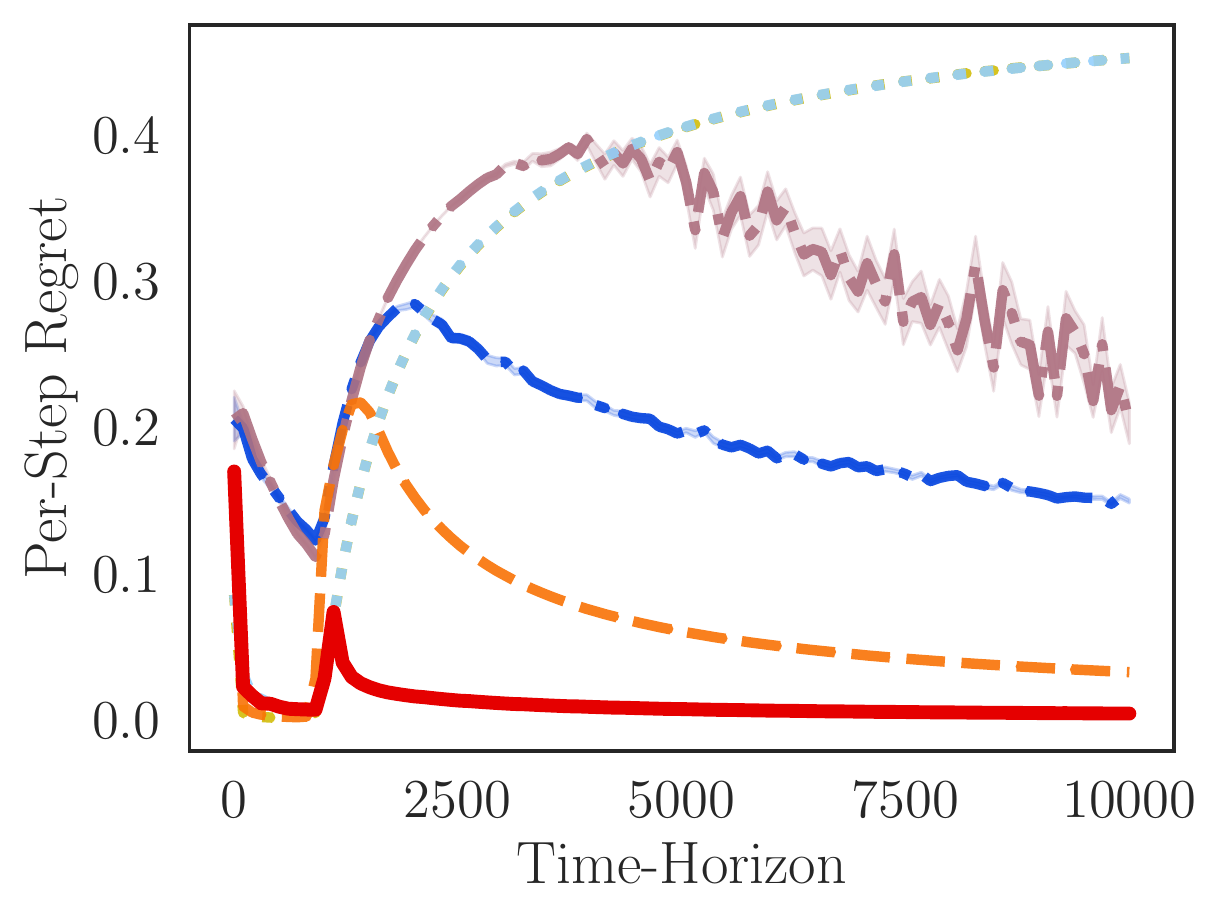}\hfill
      \includegraphics[width=0.32\linewidth]{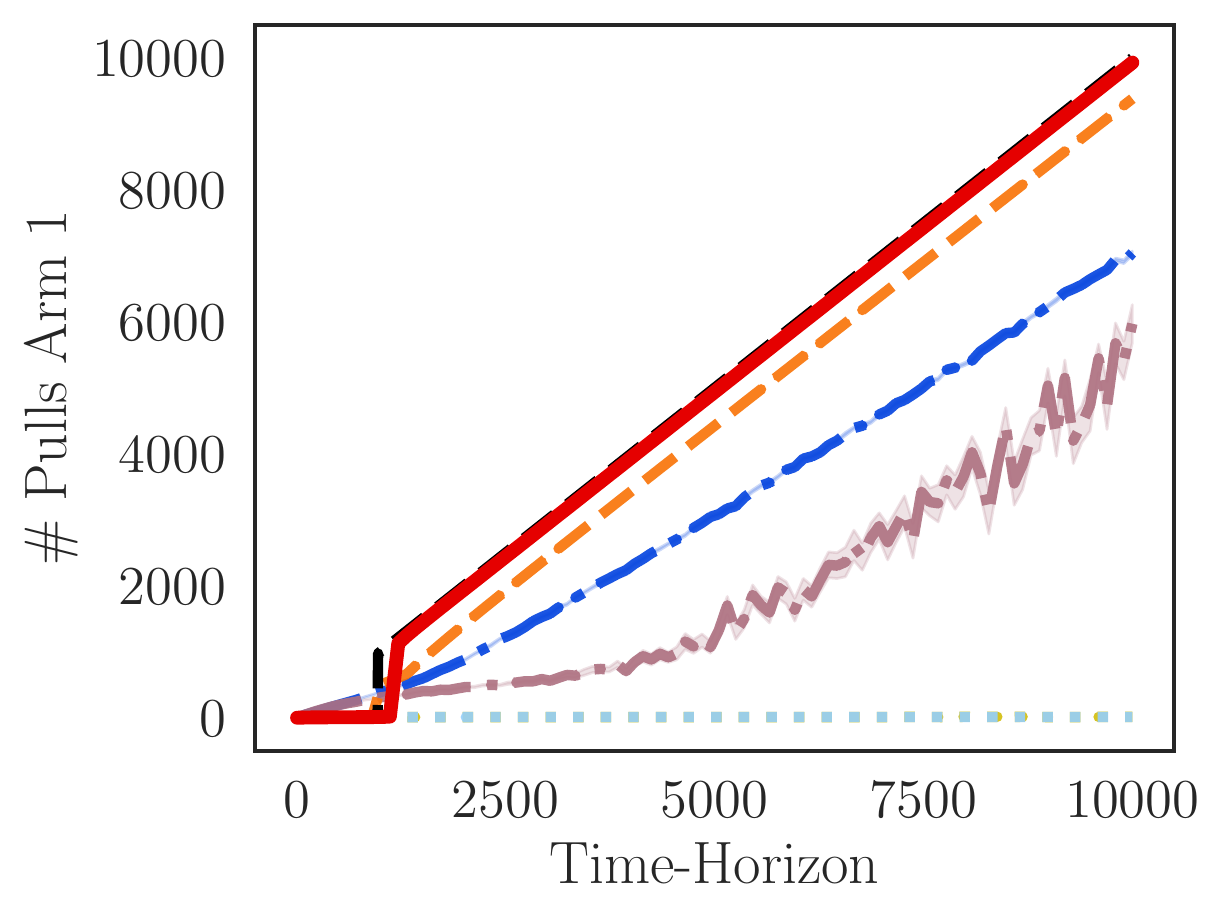}
      \caption{$\alpha = 0.03$}
    \end{subfigure}\vspace{1em}
    \begin{subfigure}[c]{\linewidth}
      \includegraphics[width=0.32\linewidth]{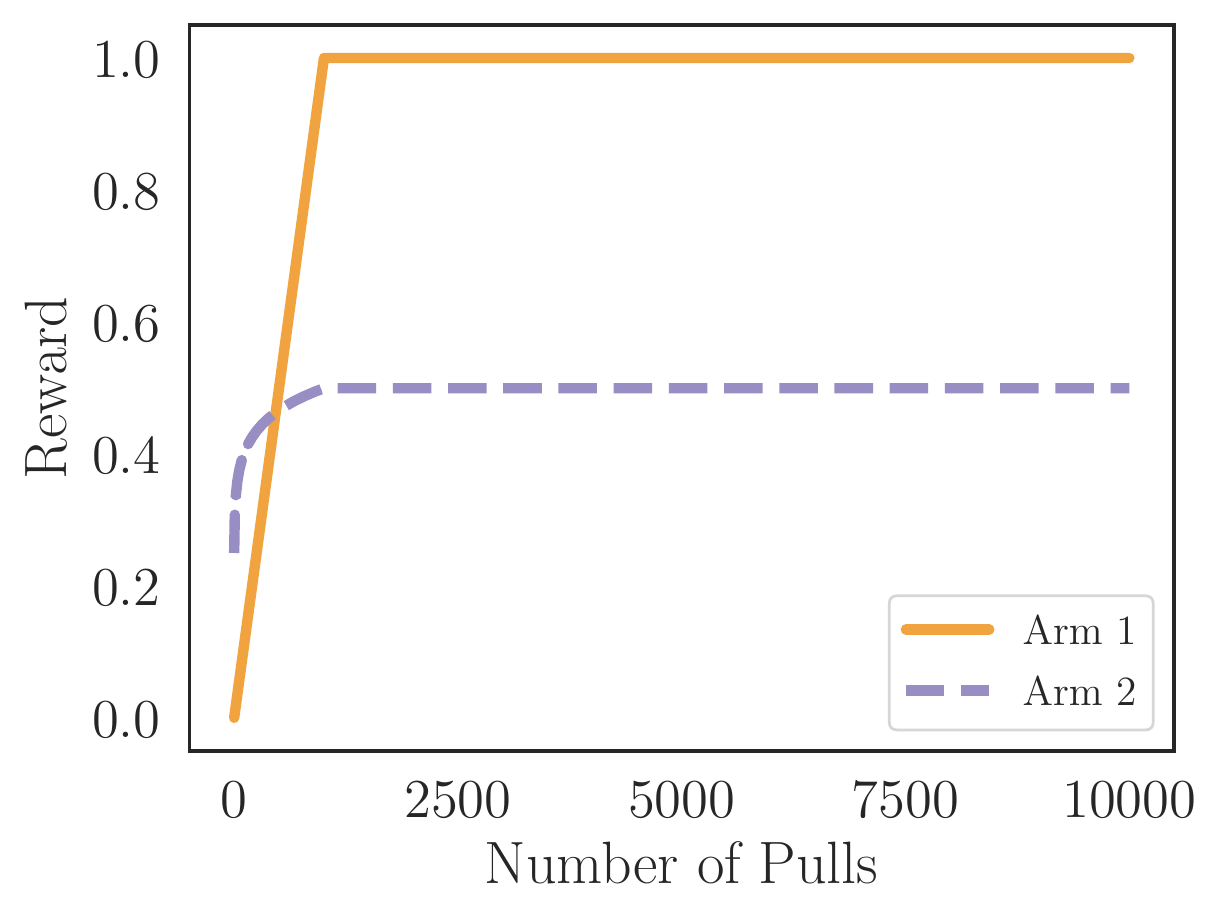}\hfill
      \includegraphics[width=0.32\linewidth]{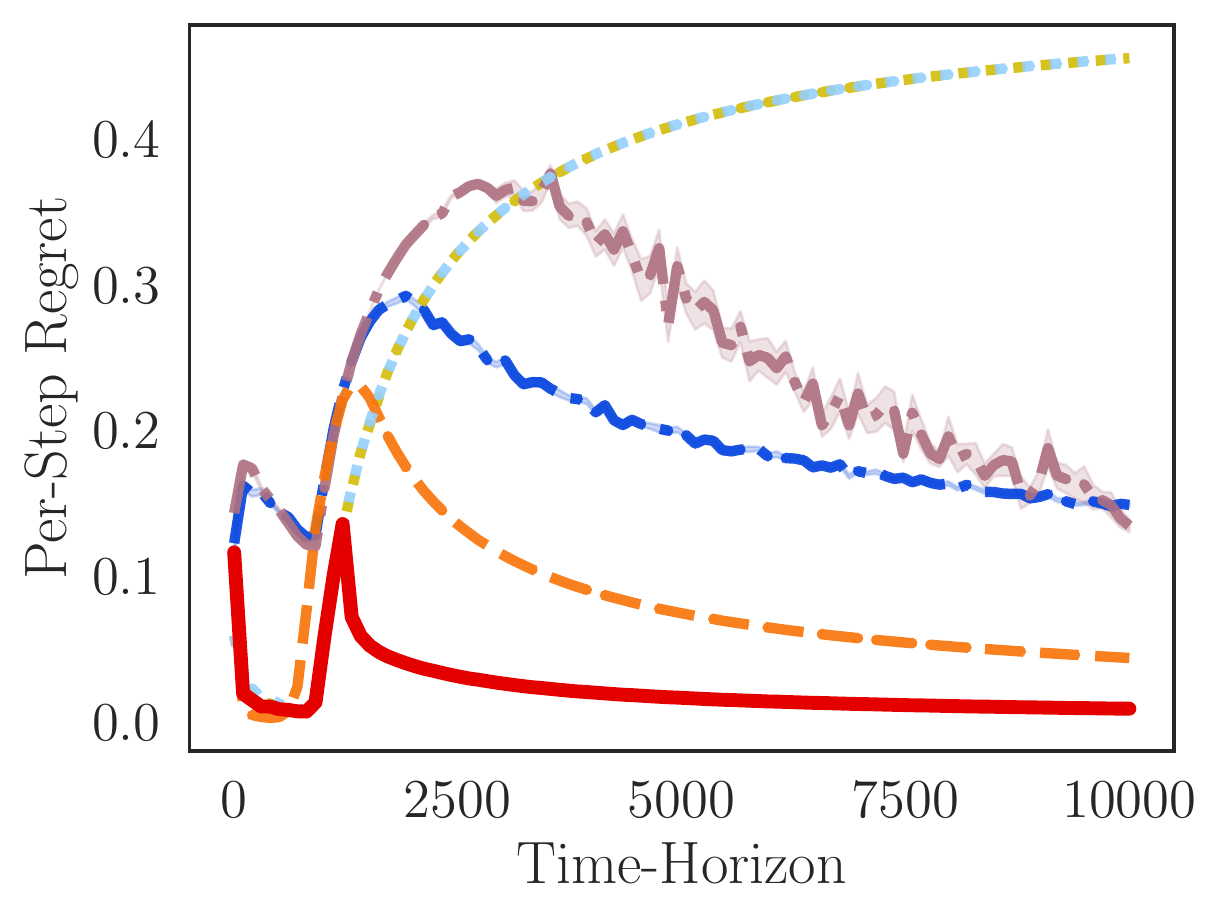}\hfill
      \includegraphics[width=0.32\linewidth]{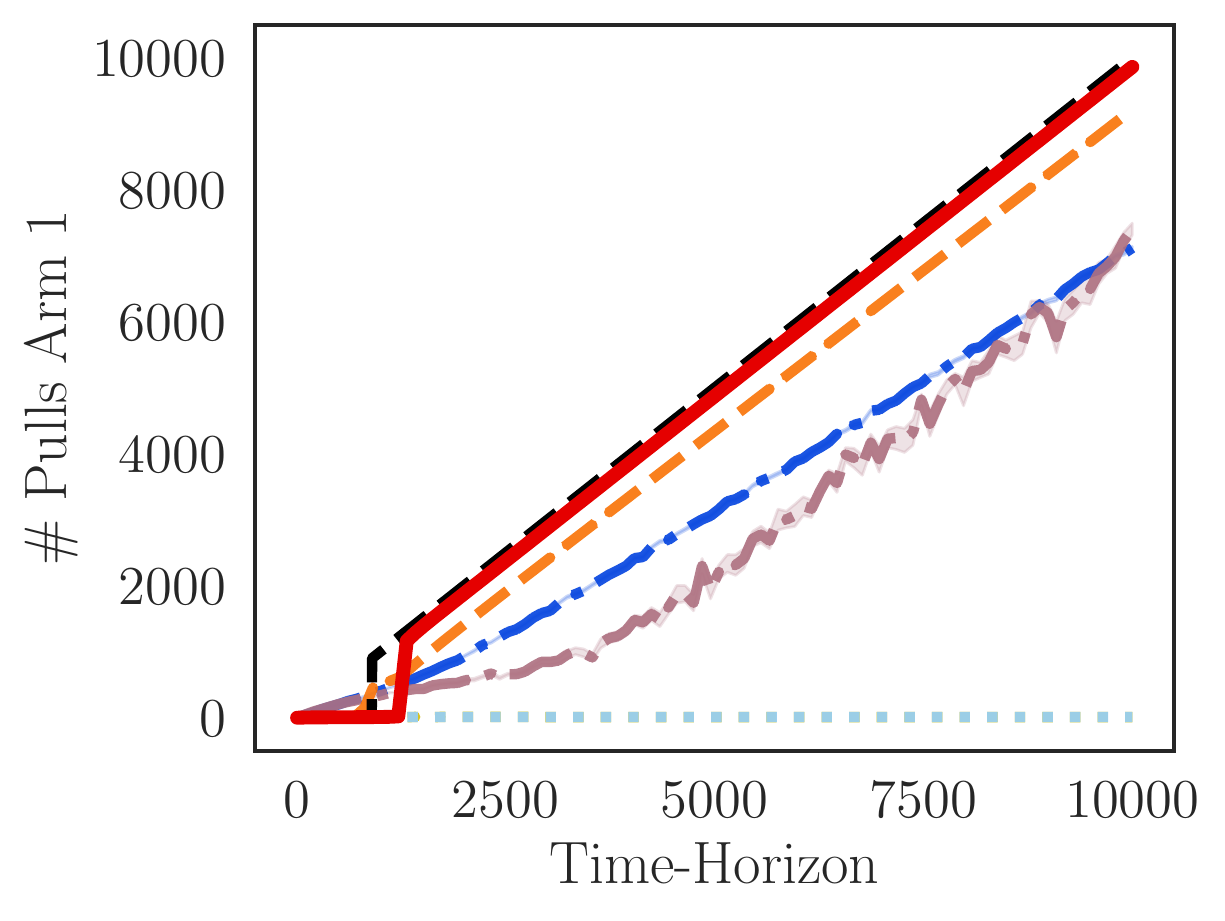}
      \caption{$\alpha = 0.1$}
    \end{subfigure}\vspace{1em}
    \begin{subfigure}[c]{\linewidth}
      \includegraphics[width=0.32\linewidth]{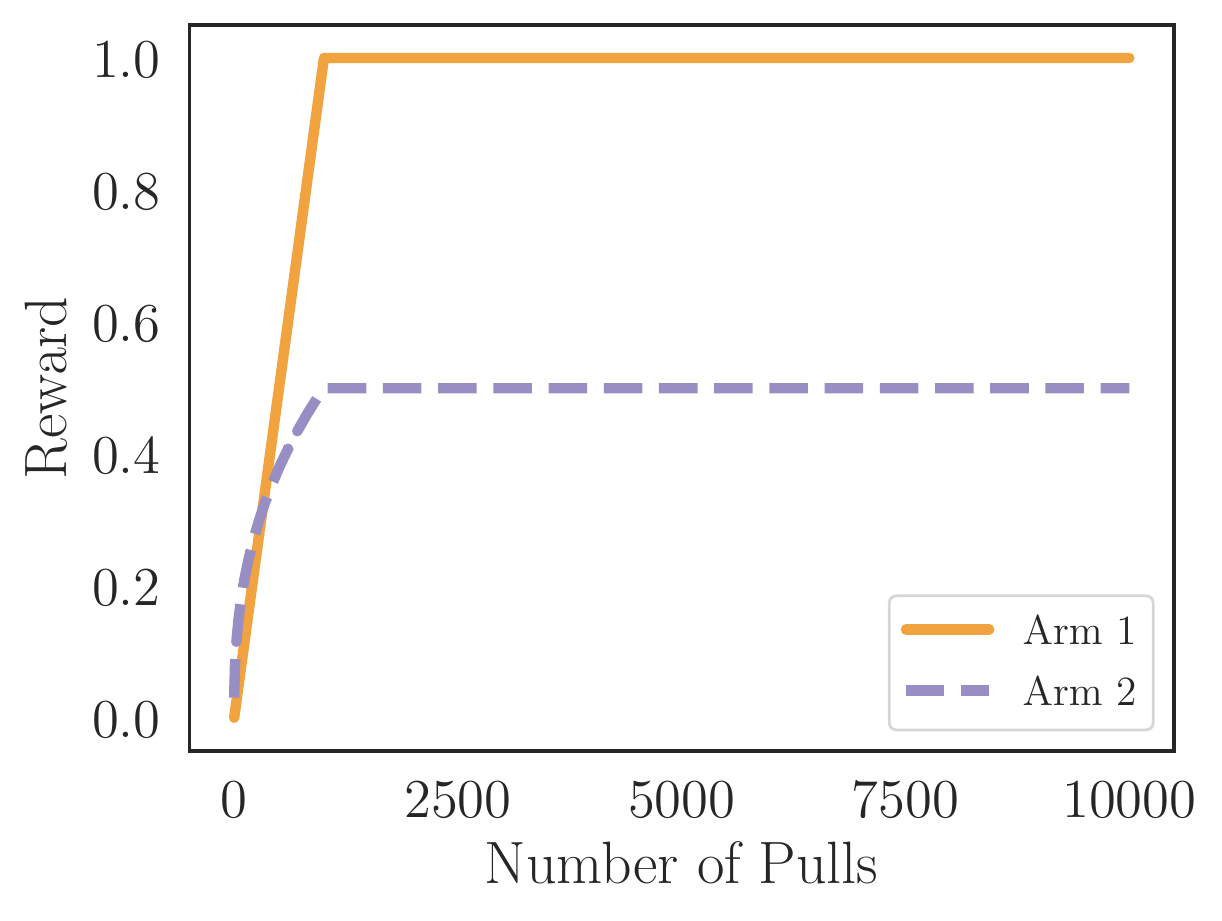}\hfill
      \includegraphics[width=0.32\linewidth]{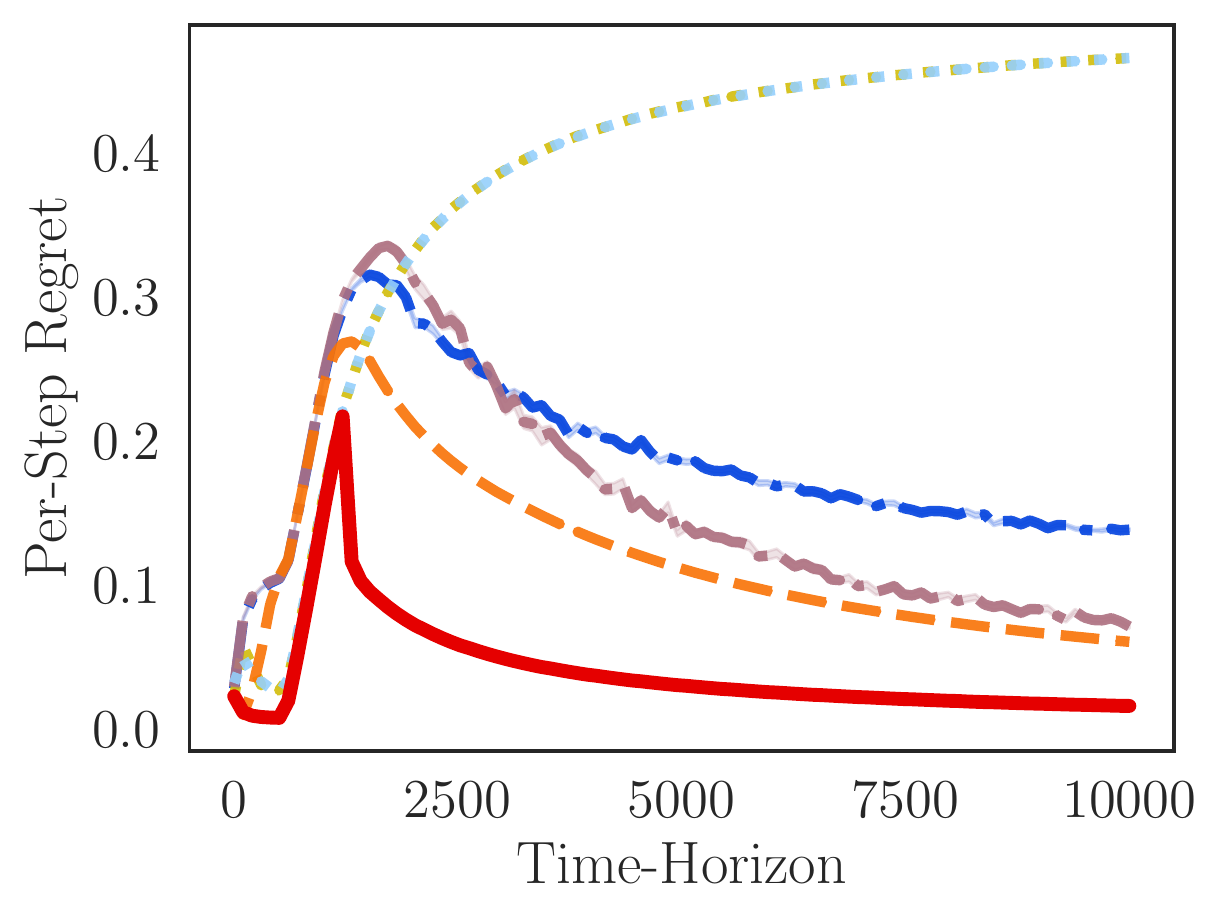}\hfill
      \includegraphics[width=0.32\linewidth]{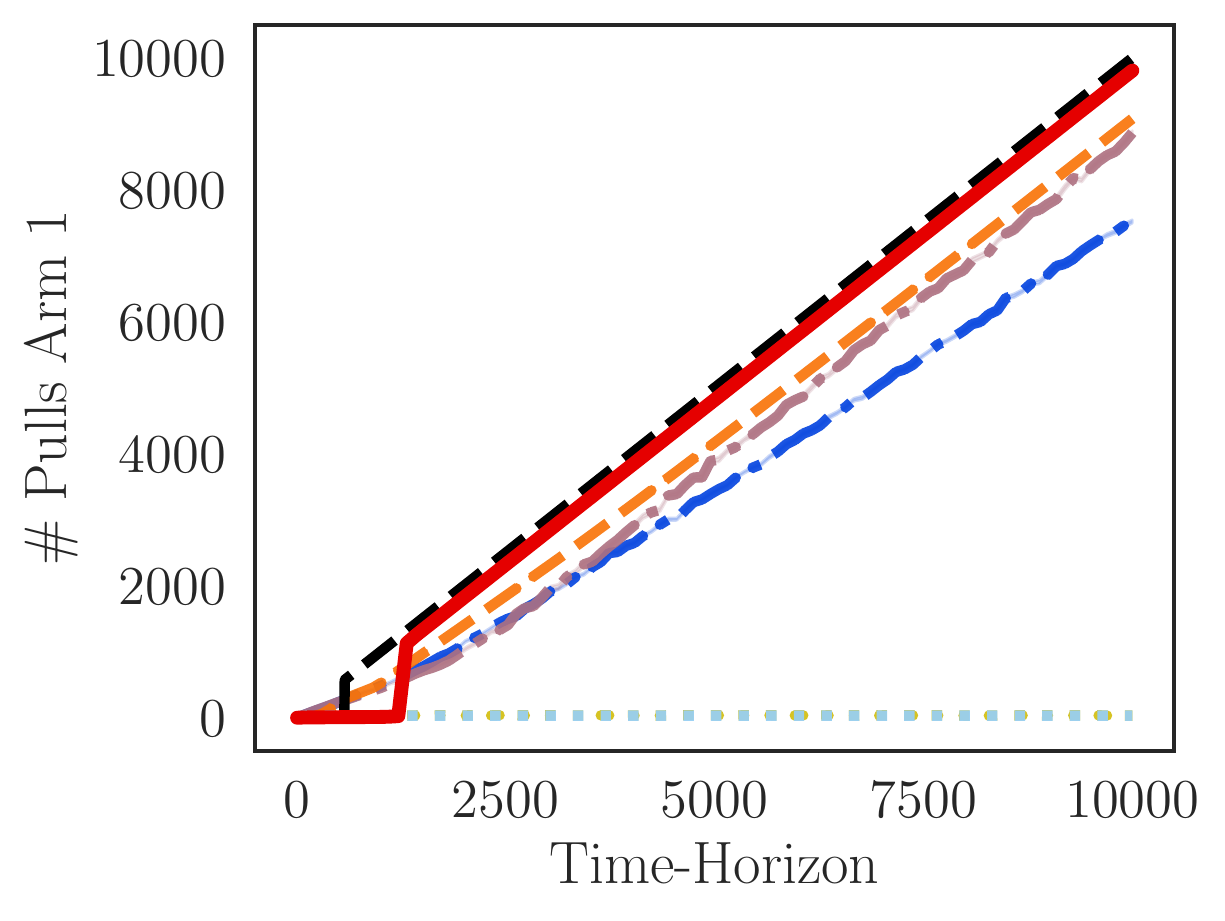}
      \caption{$\alpha = 0.4$}
    \end{subfigure}\vspace{1em}
    \begin{subfigure}[c]{\linewidth}
      \includegraphics[width=0.32\linewidth]{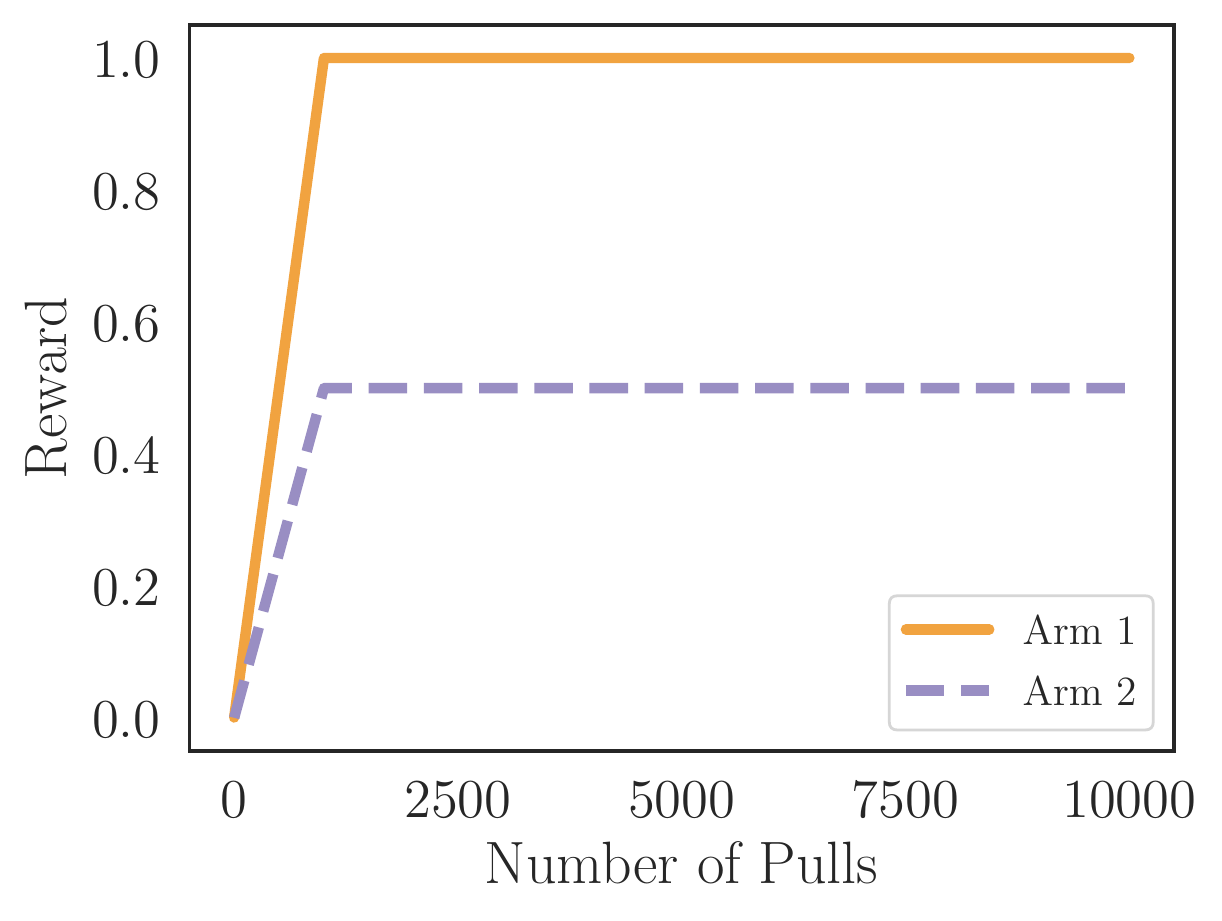}\hfill
      \includegraphics[width=0.32\linewidth]{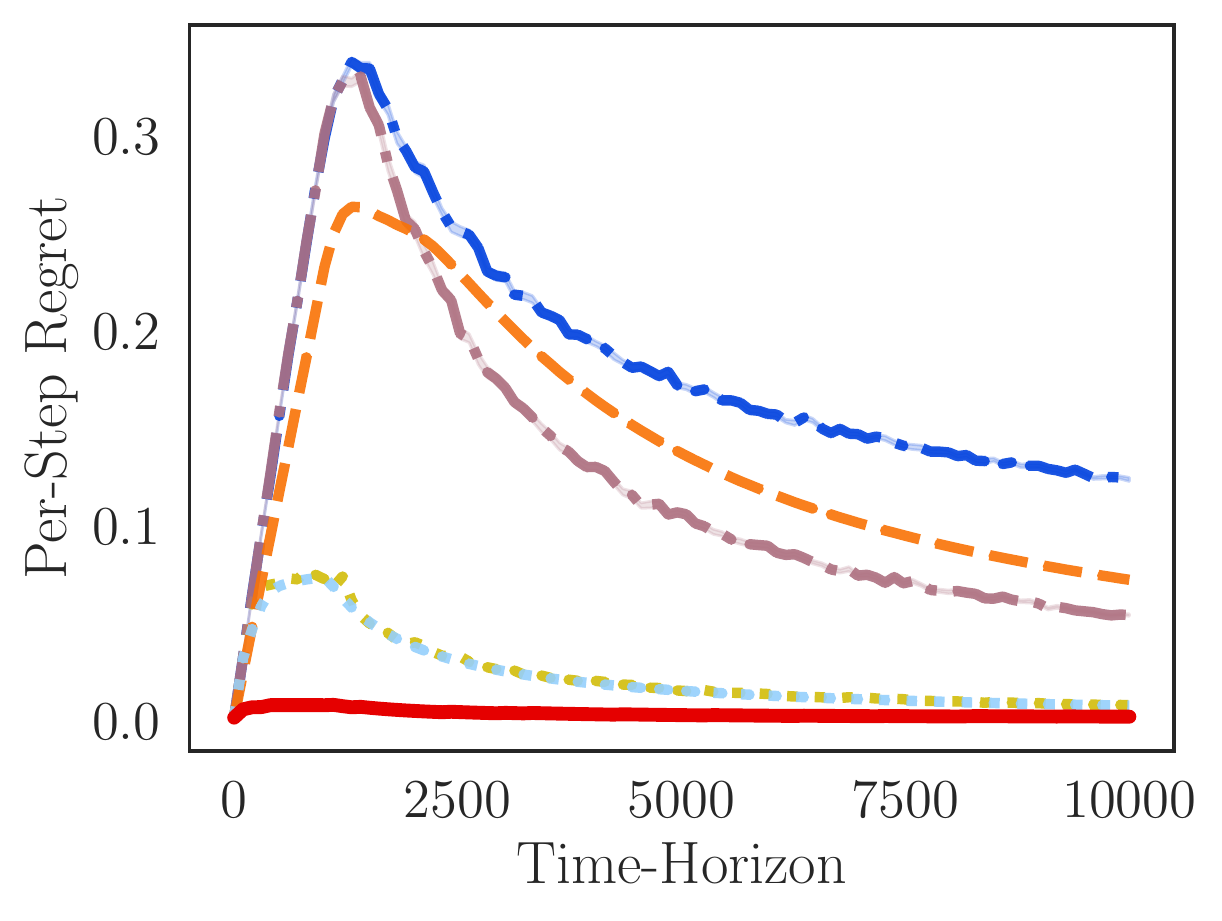}\hfill
      \includegraphics[width=0.32\linewidth]{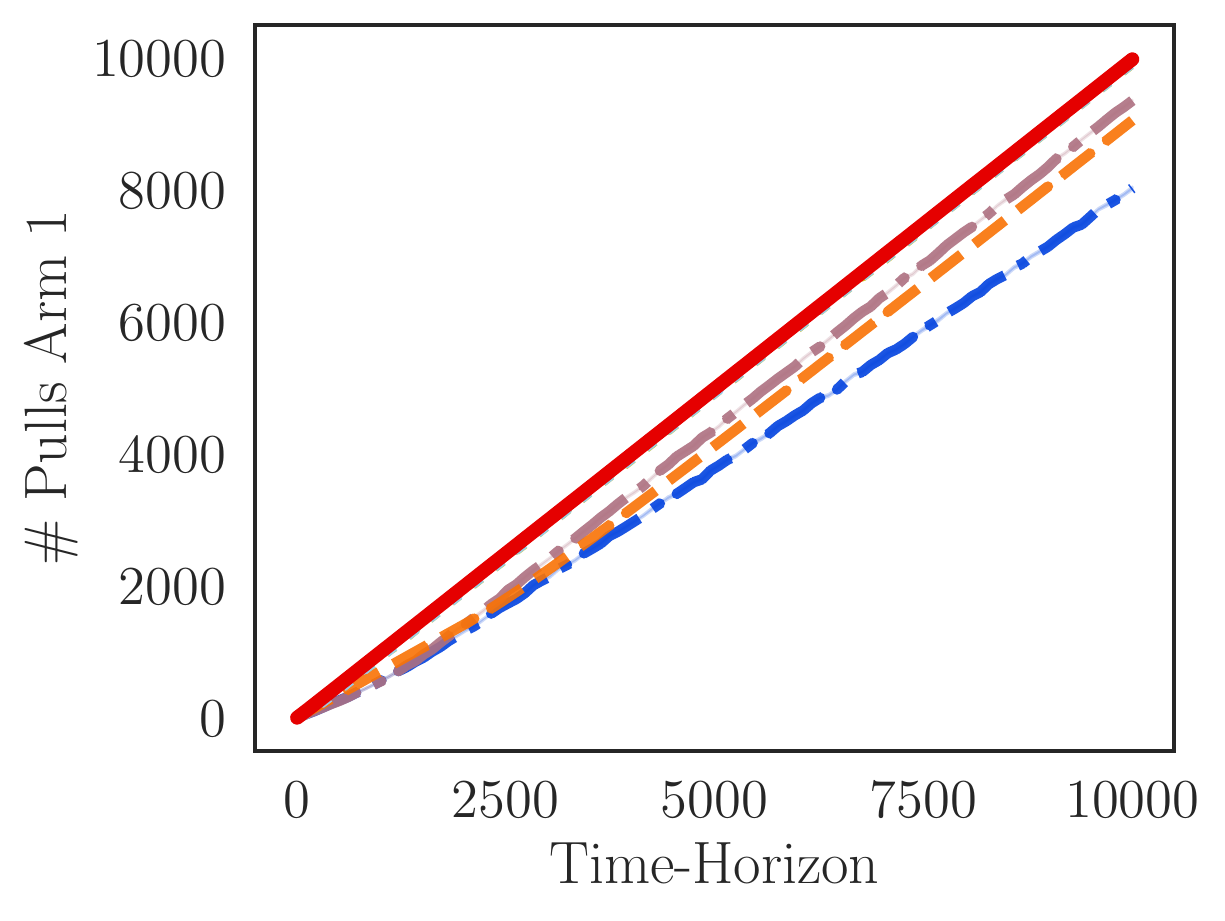}
      \caption{$\alpha = 1$}
    \end{subfigure}\vspace{1em}
    \caption{The left-hand plots show the reward functions given by $f_1(t) = \min\cbr{1, \frac{1}{1000}}$, $f_2(t) = \min\cbr{0.5, 0.5\rbr{\frac{t}{1000}}^\alpha}$ where $\alpha = 0.03, 0.1, 0.4, 1$ (from top to bottom). The middle plots show the per-step policy regret achieved by \algshort ({\protect\legendSPO}) compared to the algorithm proposed by \citet{Heidari2016} for increasing rewards ({\protect\legendINCREASING}), EXP3 ({\protect\legendEXP}), R-EXP3 ({\protect\legendREXP}), D-UCB ({\protect\legendDUCB}), and SW-UCB ({\protect\legendSWUCB}). The right-hand plots show the policies these algorithms choose in comparison to the optimal policy ({\protect\legendOPTIMAL}).}
    \label{fig:experiment_inc2}
\end{minipage}
\end{figure*}

\begin{figure*}[p]
\centering
\begin{minipage}[b]{.48\textwidth}
   \centering
   \begin{subfigure}[c]{\linewidth}
      \includegraphics[width=0.32\linewidth]{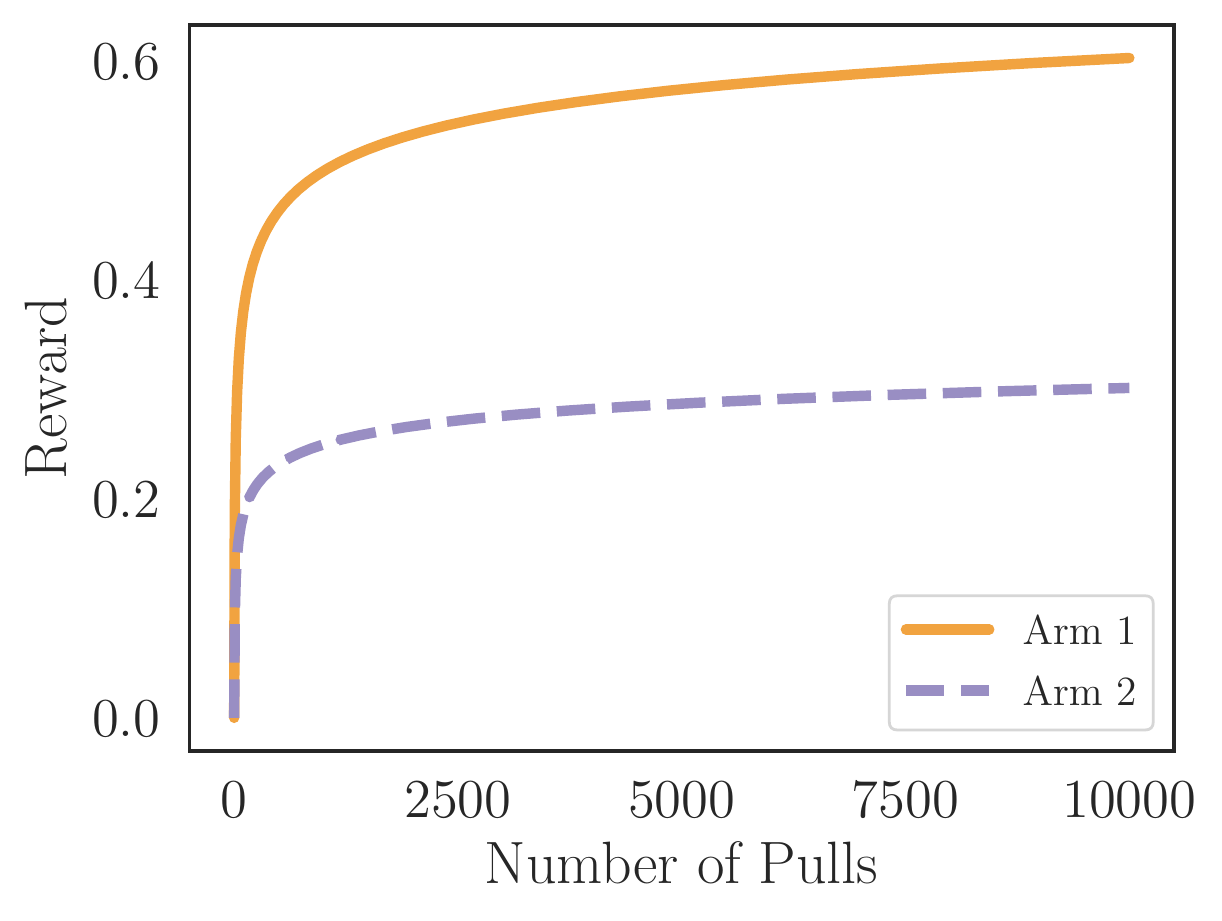}\hfill
      \includegraphics[width=0.32\linewidth]{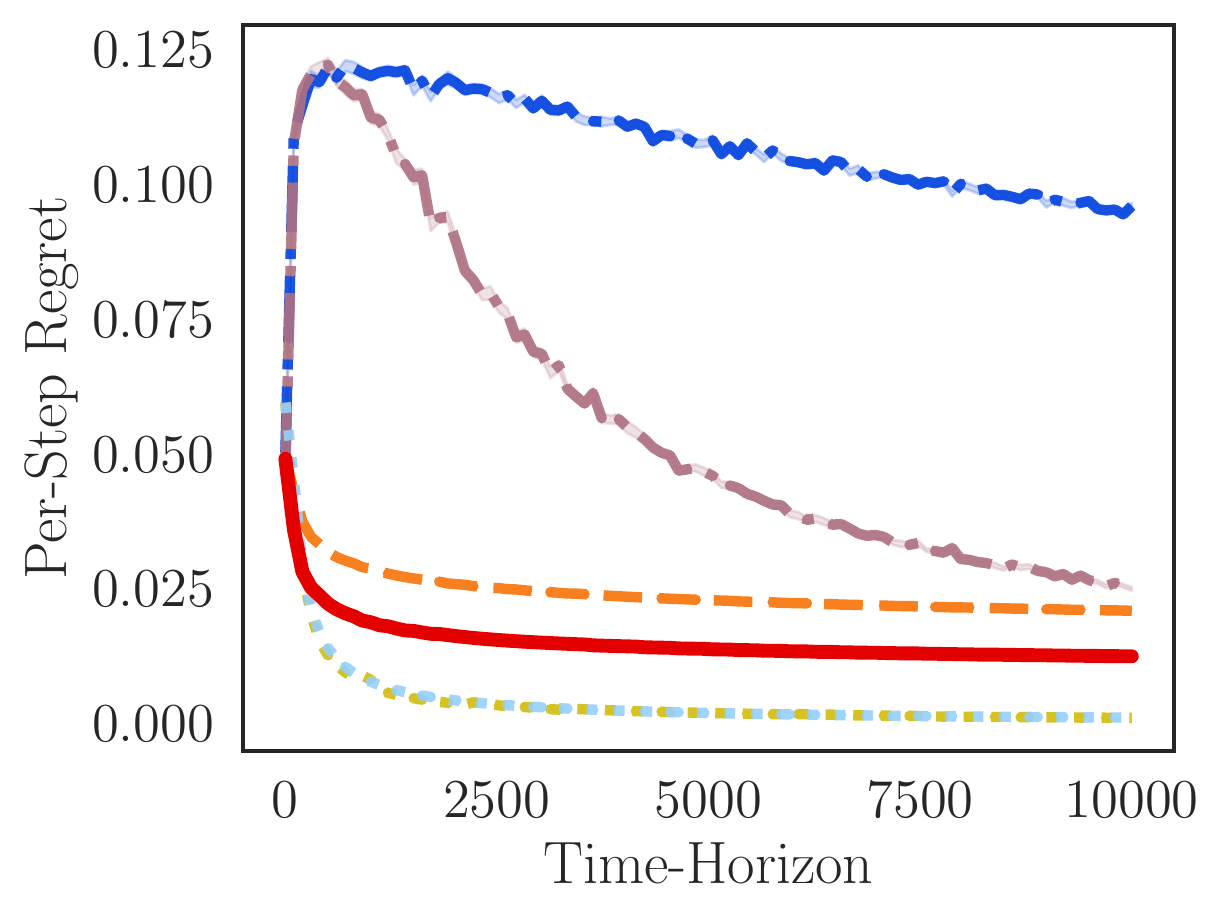}\hfill
      \includegraphics[width=0.32\linewidth]{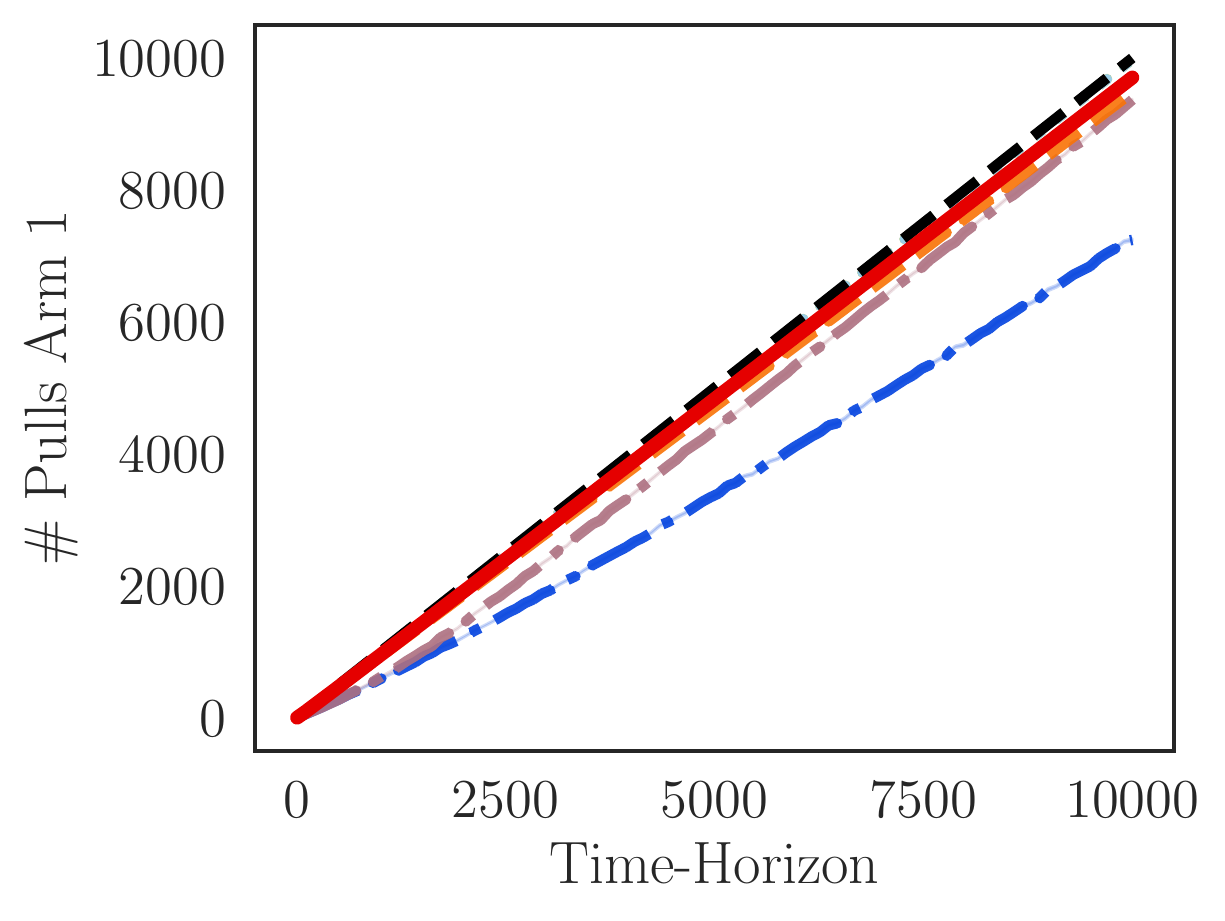}
      \caption{$\alpha = 0.1$}
   \end{subfigure}\vspace{1em}
   \begin{subfigure}[c]{\linewidth}
      \includegraphics[width=0.32\linewidth]{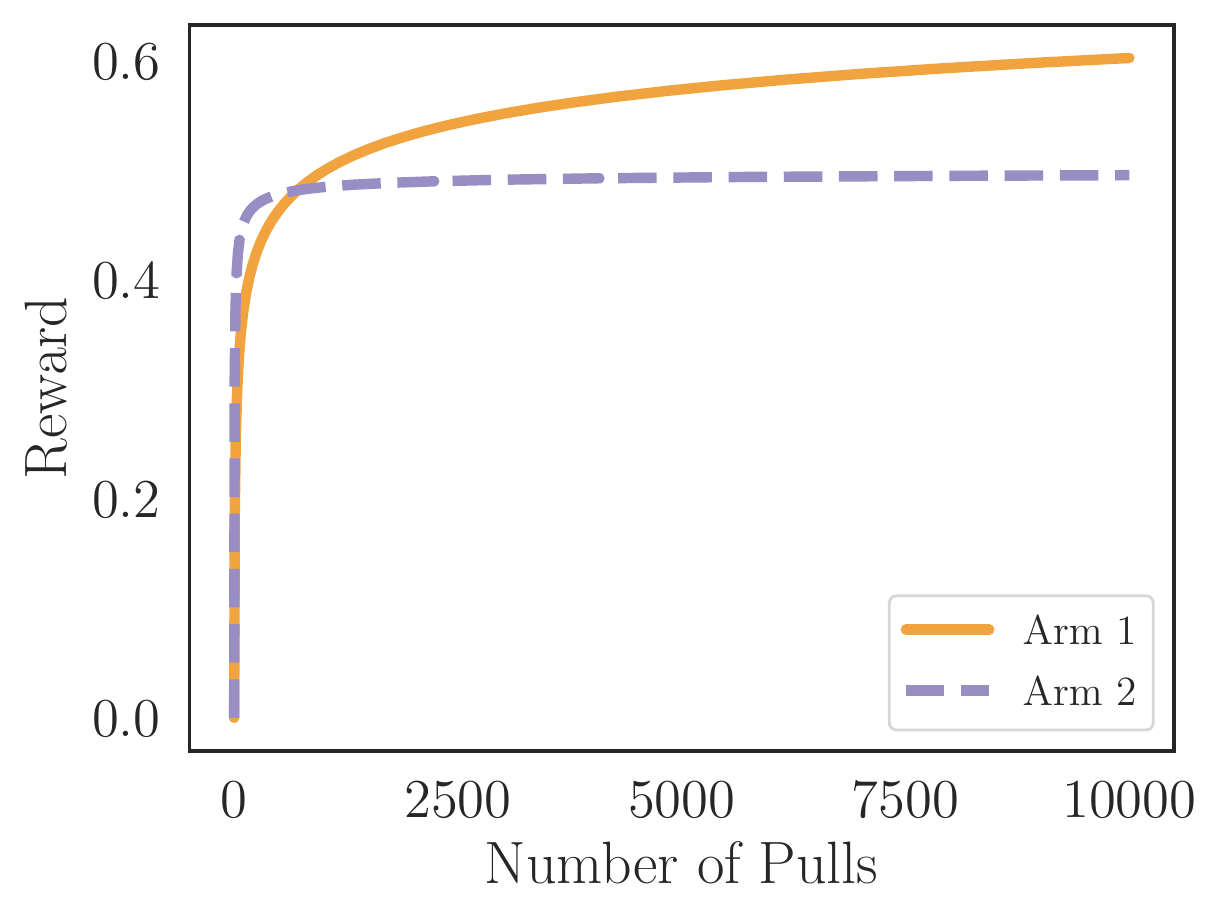}\hfill
      \includegraphics[width=0.32\linewidth]{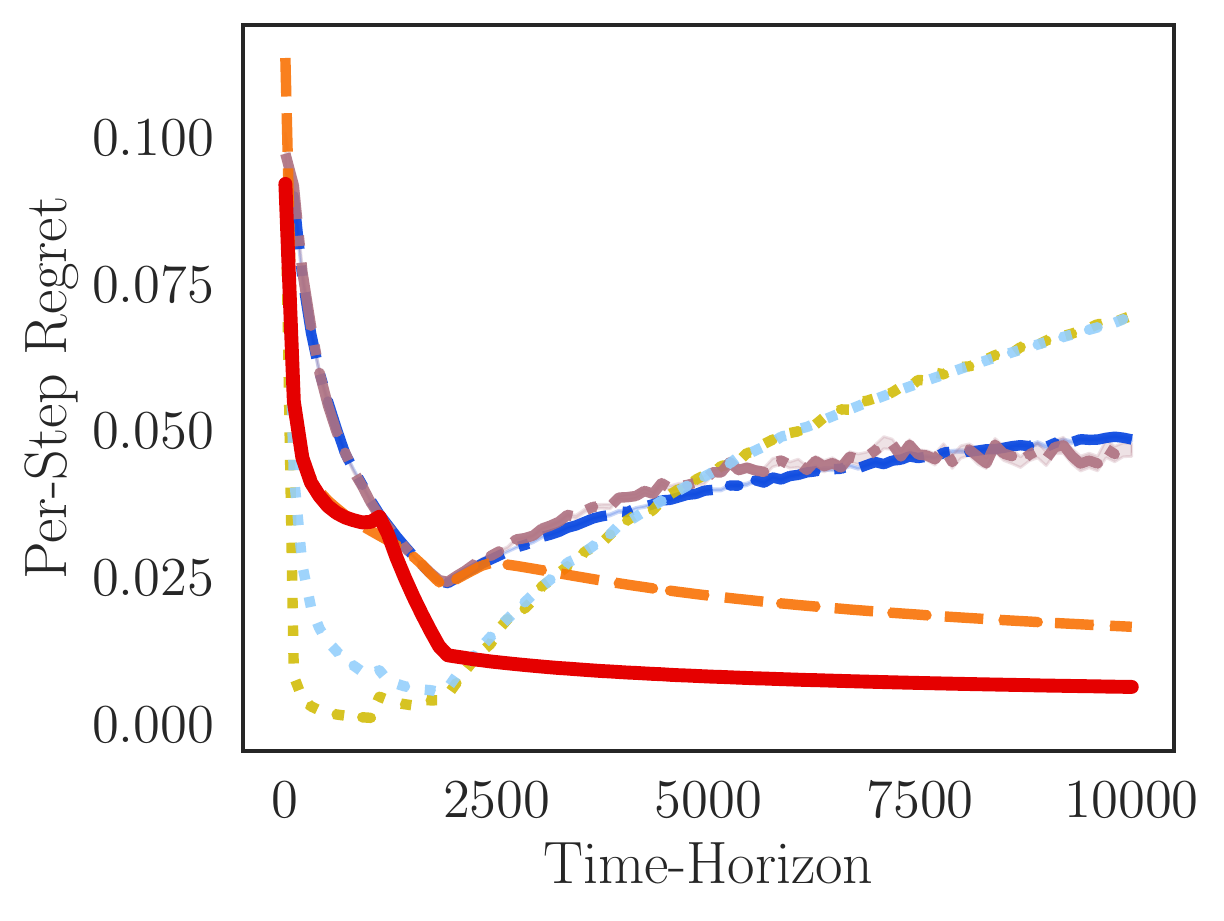}\hfill
      \includegraphics[width=0.32\linewidth]{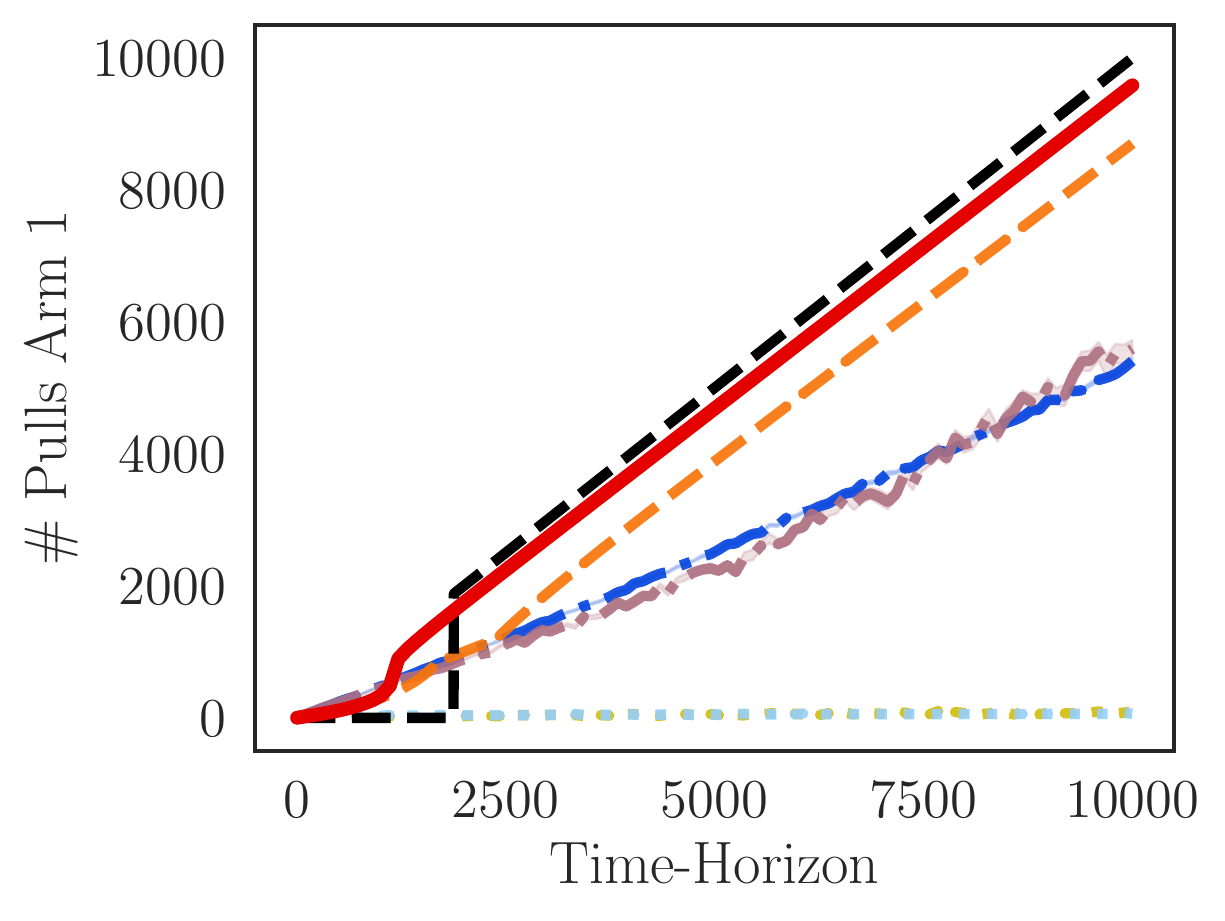}
      \caption{$\alpha = 0.5$}
   \end{subfigure}\vspace{1em}
   \begin{subfigure}[c]{\linewidth} 
      \includegraphics[width=0.32\linewidth]{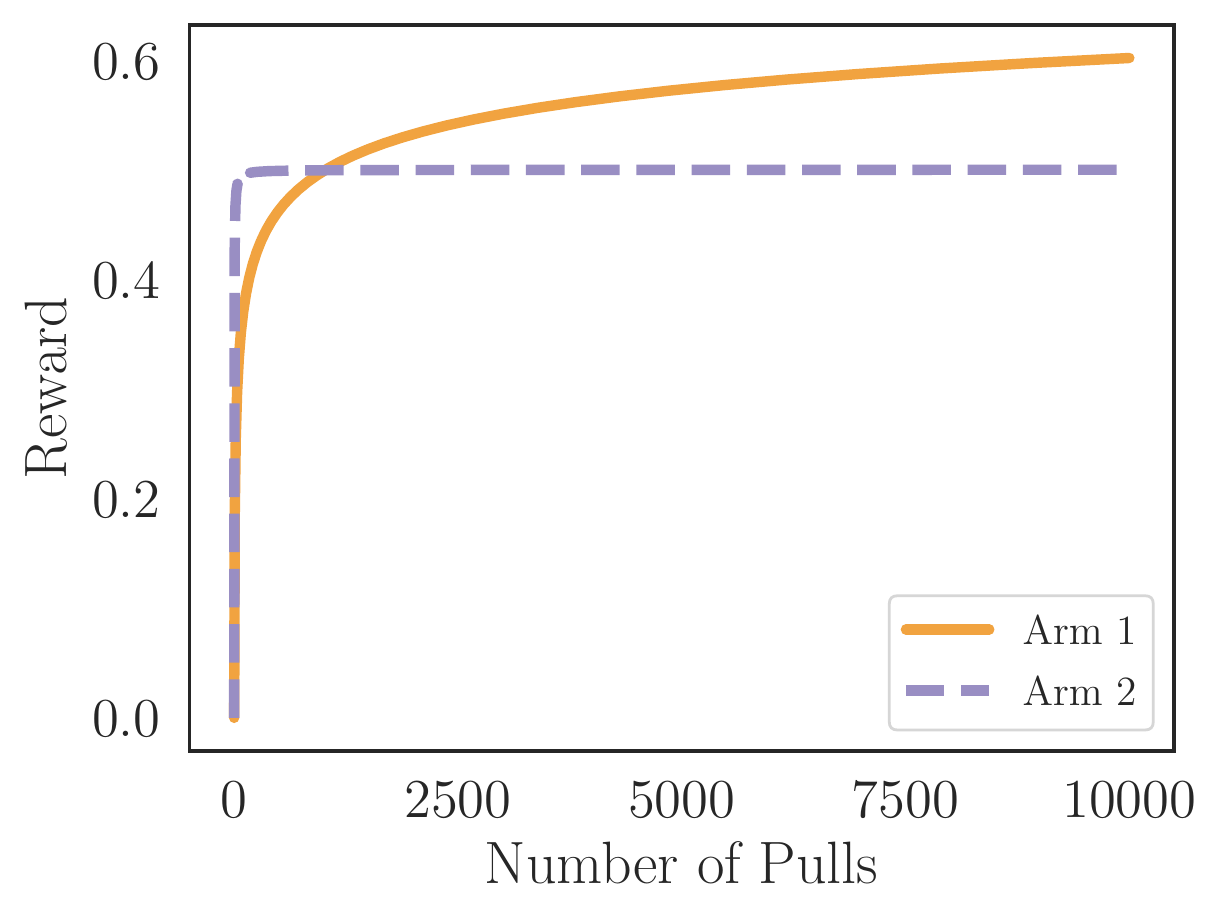}\hfill
      \includegraphics[width=0.32\linewidth]{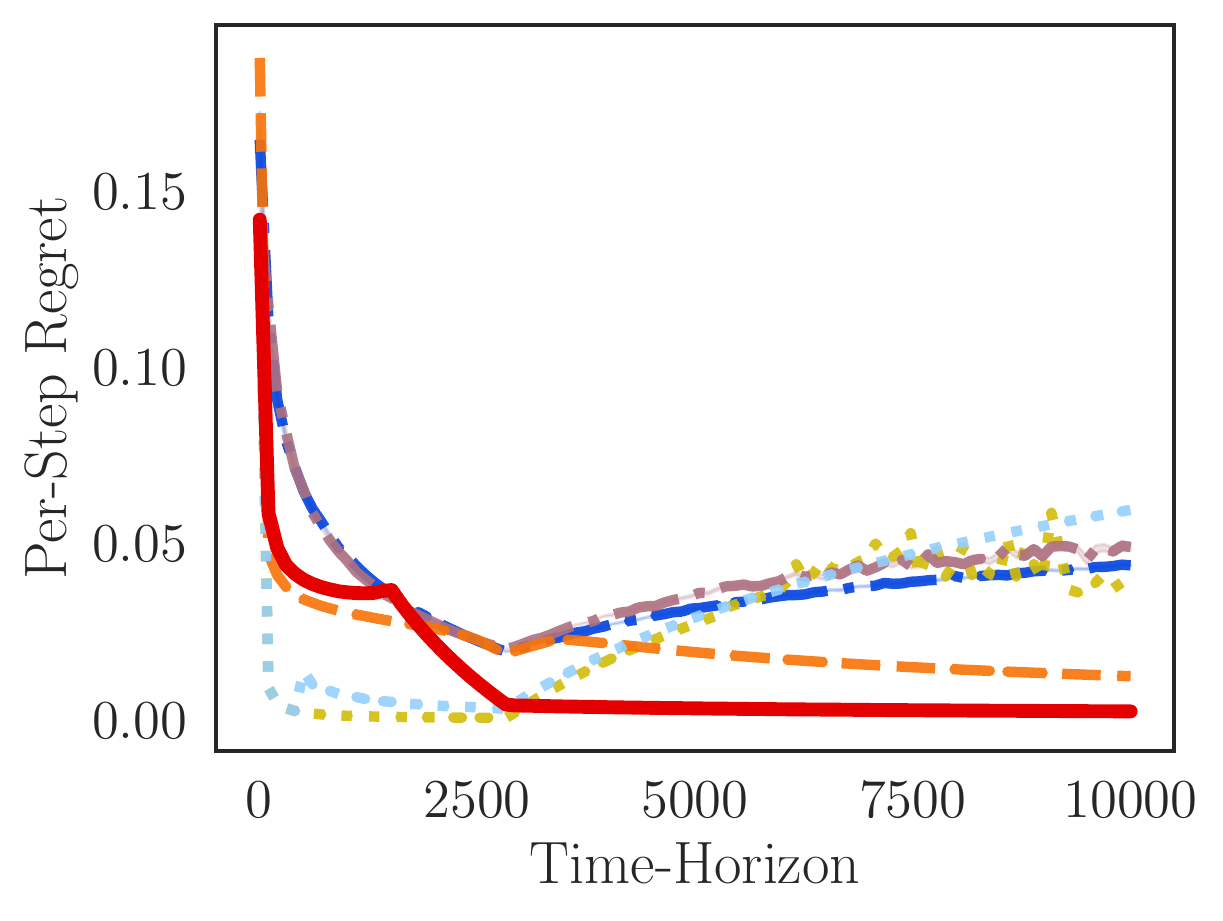}\hfill
      \includegraphics[width=0.32\linewidth]{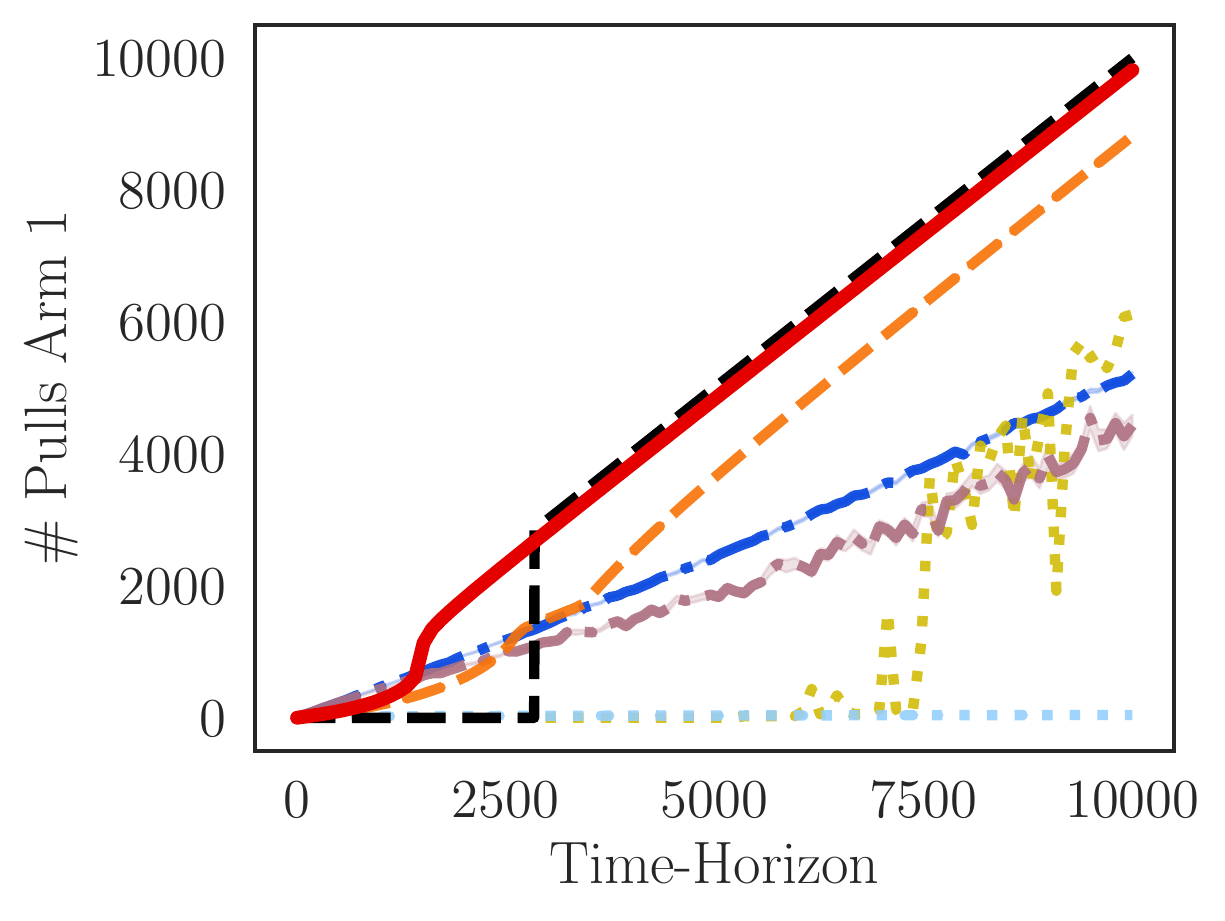}
      \caption{$\alpha = 1$}
   \end{subfigure}\vspace{1em}
   \begin{subfigure}[c]{\linewidth}
      \includegraphics[width=0.32\linewidth]{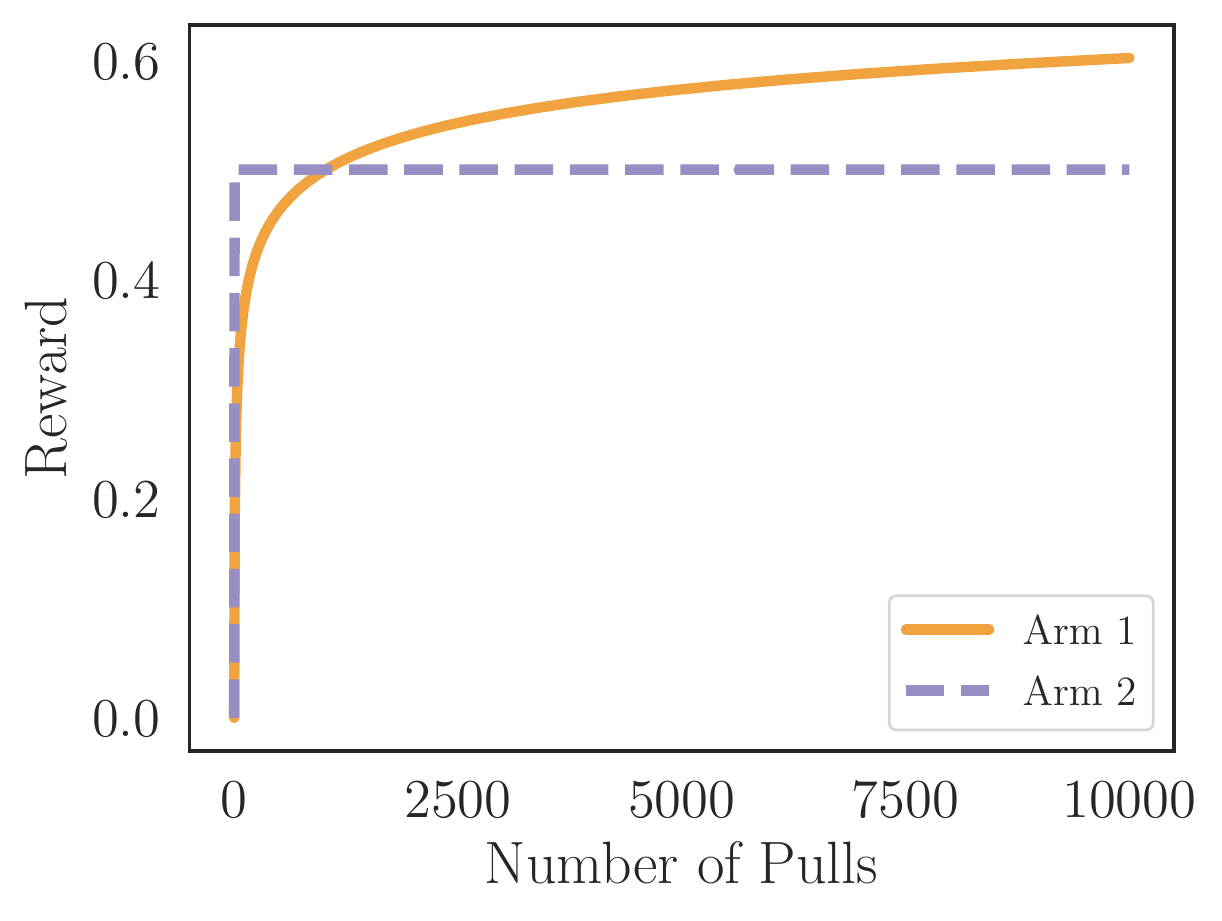}\hfill
      \includegraphics[width=0.32\linewidth]{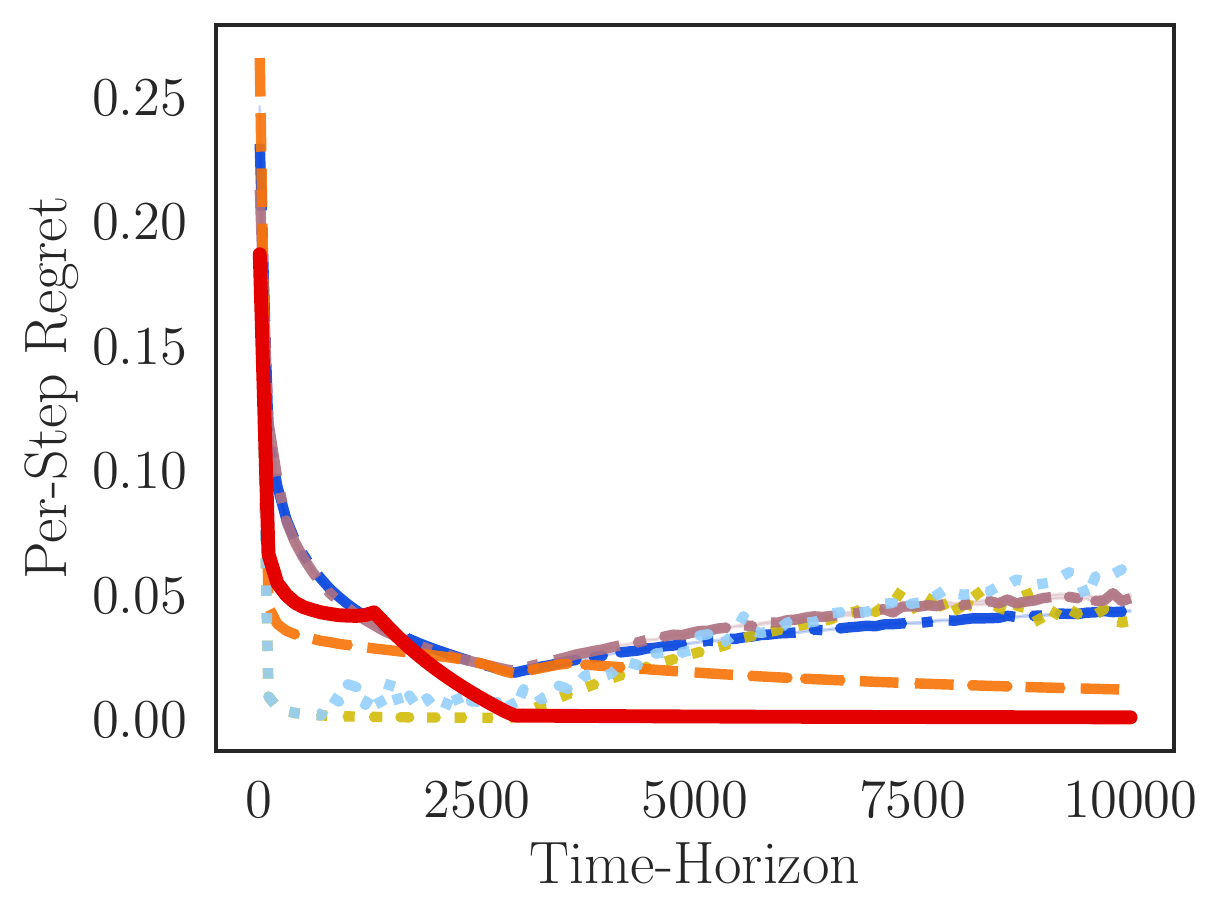}\hfill
      \includegraphics[width=0.32\linewidth]{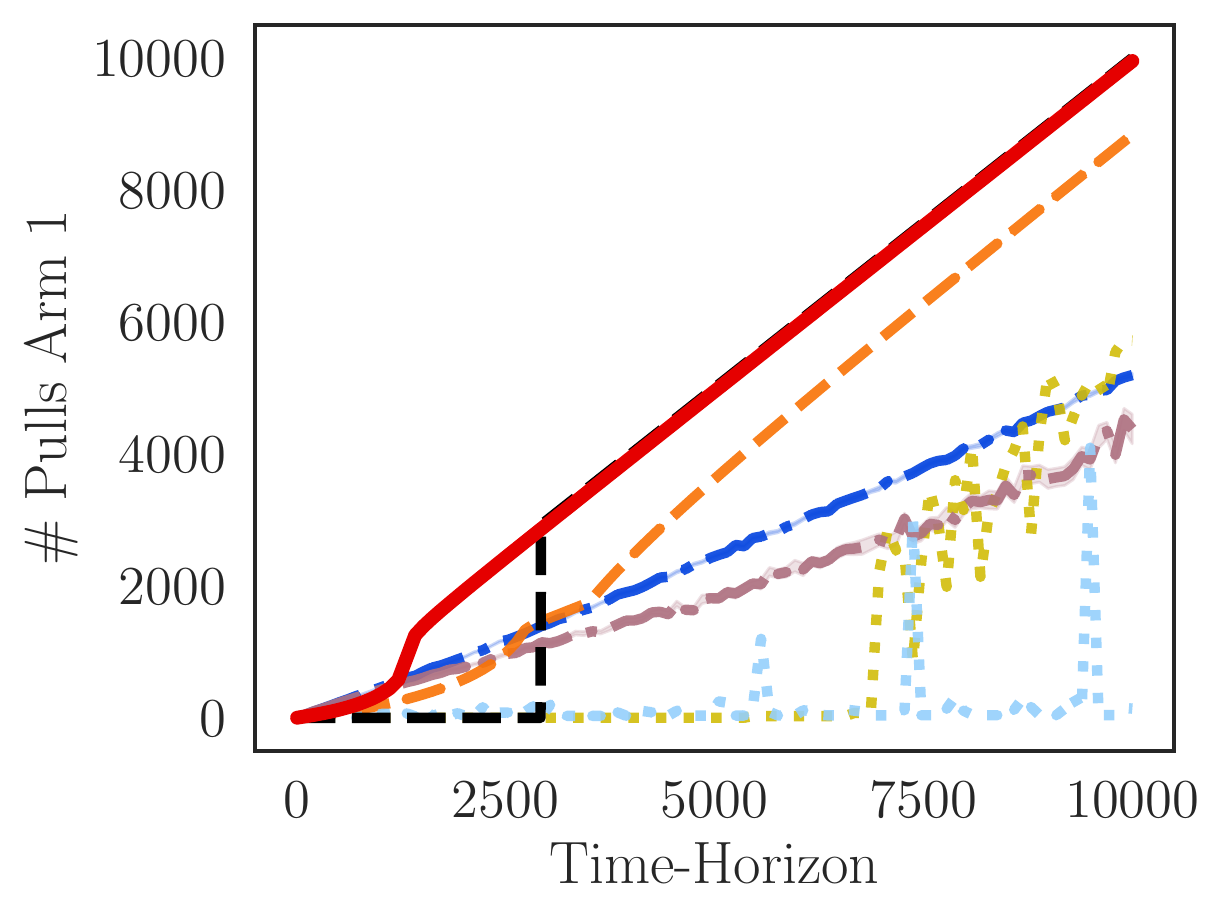}
      \caption{$\alpha = 5$}
   \end{subfigure}\vspace{1em}
   \caption{The left-hand plots show the reward functions given by $f_1(t) = 1 - t^{-0.1}$, $f_2(t) = 0.5 - 0.5 t^{-\alpha}$ where $\alpha = 0.1, 0.5, 1, 5$ (from top to bottom). The middle plots show the per-step policy regret achieved by \algshort ({\protect\legendSPO}) compared to the algorithm proposed by \citet{Heidari2016} for increasing rewards ({\protect\legendINCREASING}), EXP3 ({\protect\legendEXP}), R-EXP3 ({\protect\legendREXP}), D-UCB ({\protect\legendDUCB}), and SW-UCB ({\protect\legendSWUCB}). The right-hand plots show the policies these algorithms choose in comparison to the optimal policy ({\protect\legendOPTIMAL}).}
   \label{fig:experiment_inc3}
\end{minipage}\hfill
\begin{minipage}[b]{.48\textwidth}
   \centering
   \begin{subfigure}[c]{\linewidth}
      \includegraphics[width=0.32\linewidth]{images/experiments/inc_1_alpha_0.1_arms.pdf}\hfill
      \includegraphics[width=0.32\linewidth]{images/experiments/inc_1_alpha_0.1_regret.pdf}\hfill
      \includegraphics[width=0.32\linewidth]{images/experiments/inc_1_alpha_0.1_policy.pdf}
      \caption{Noise-free observations}
   \end{subfigure}\vspace{1em}
   \begin{subfigure}[c]{\linewidth}
      \includegraphics[width=0.32\linewidth]{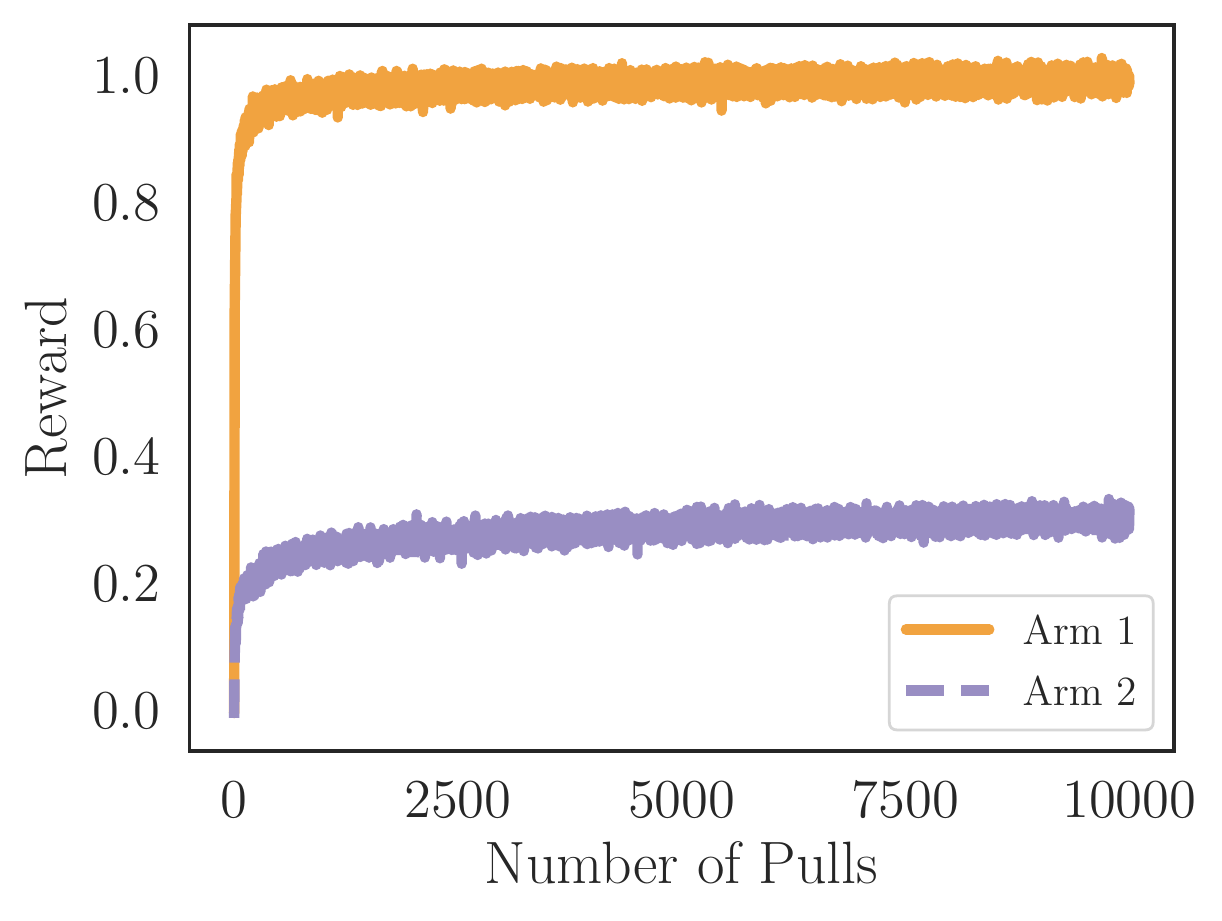}\hfill
      \includegraphics[width=0.32\linewidth]{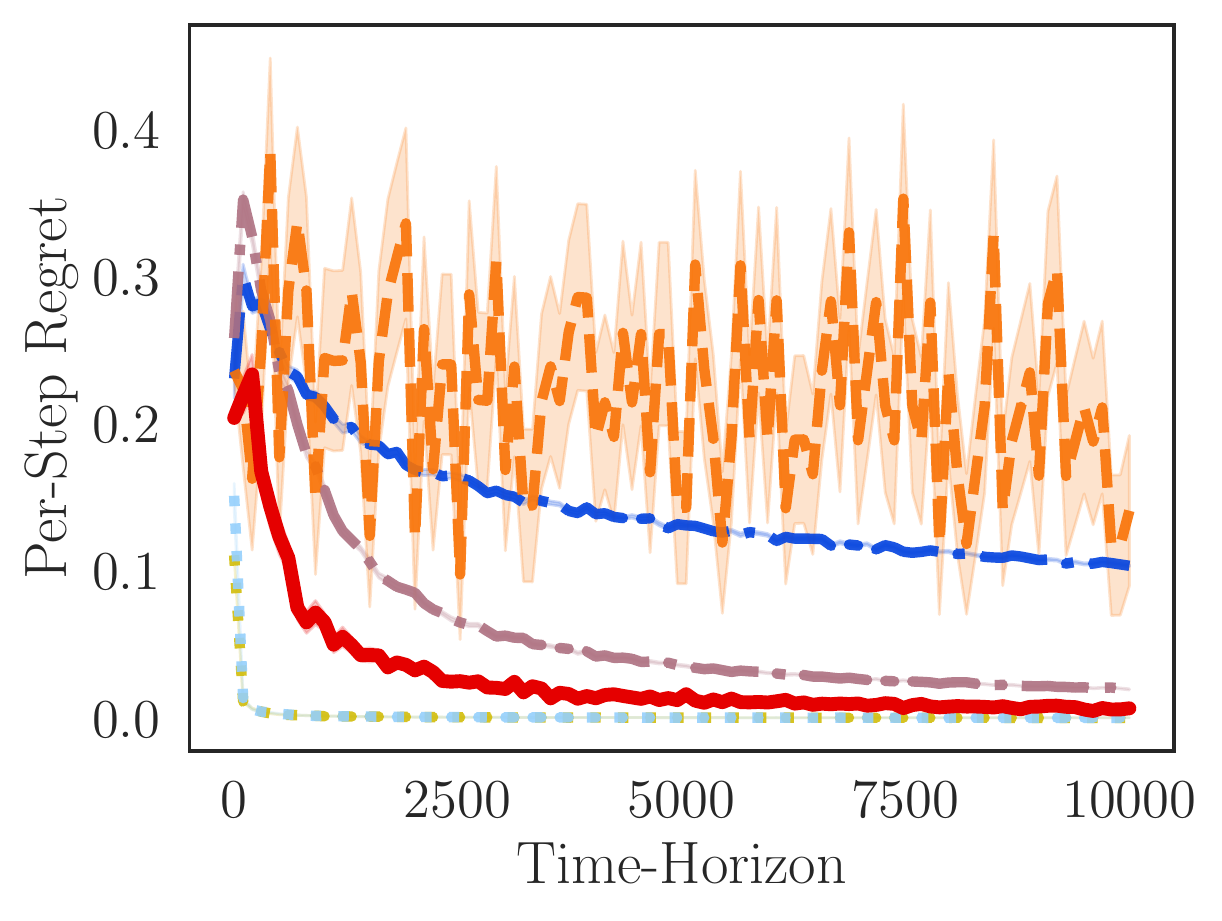}\hfill
      \includegraphics[width=0.32\linewidth]{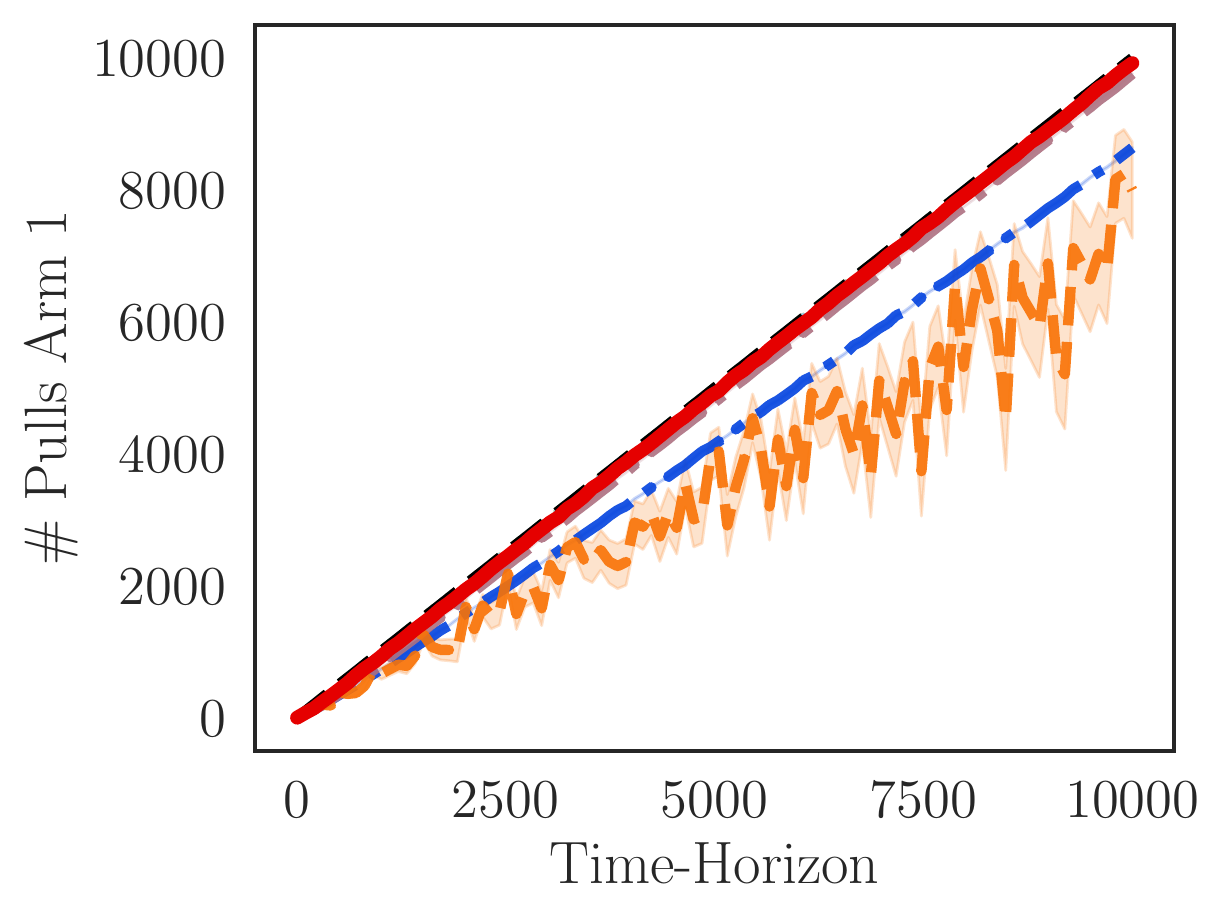}
      \caption{Noise with $\sigma = 0.01$}
   \end{subfigure}\vspace{1em}
   \begin{subfigure}[c]{\linewidth}
      \includegraphics[width=0.32\linewidth]{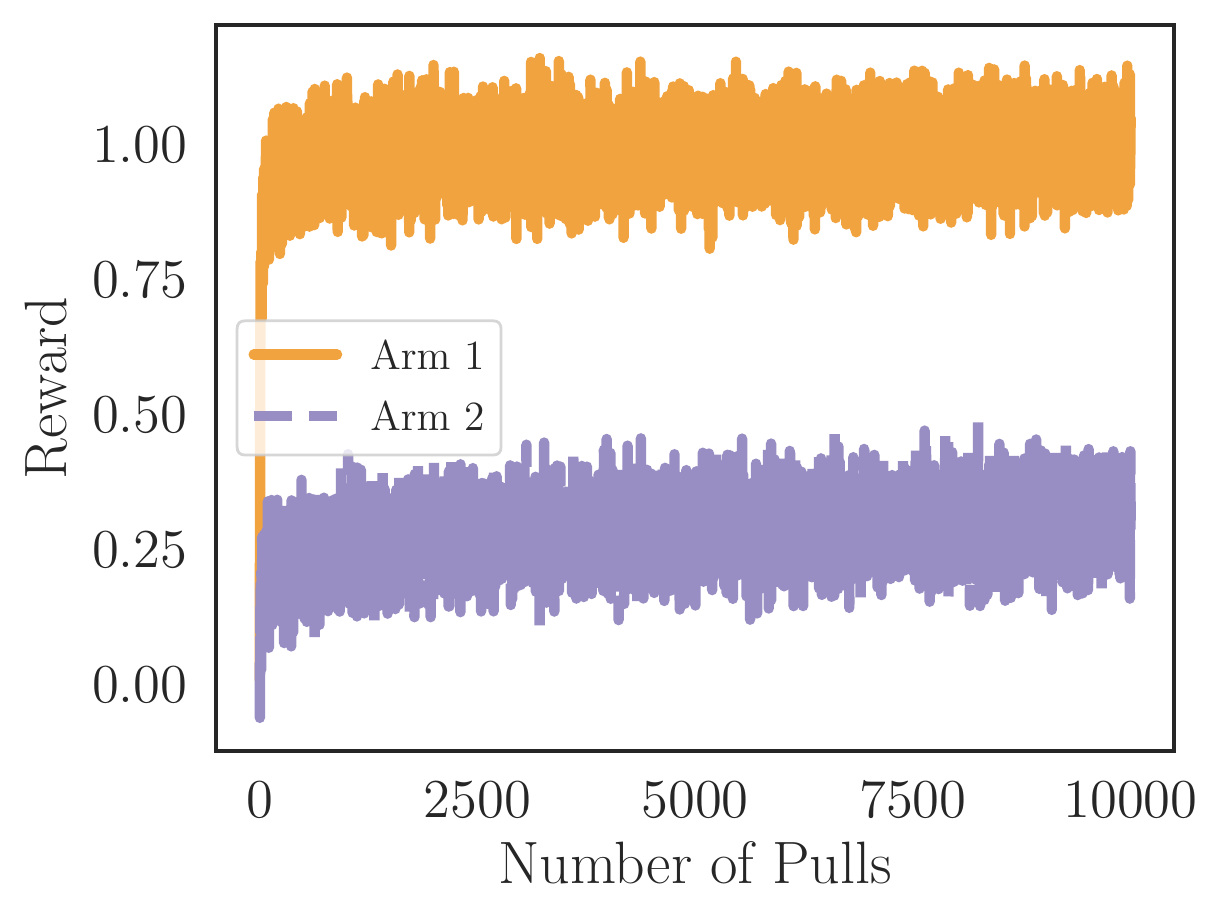}\hfill
      \includegraphics[width=0.32\linewidth]{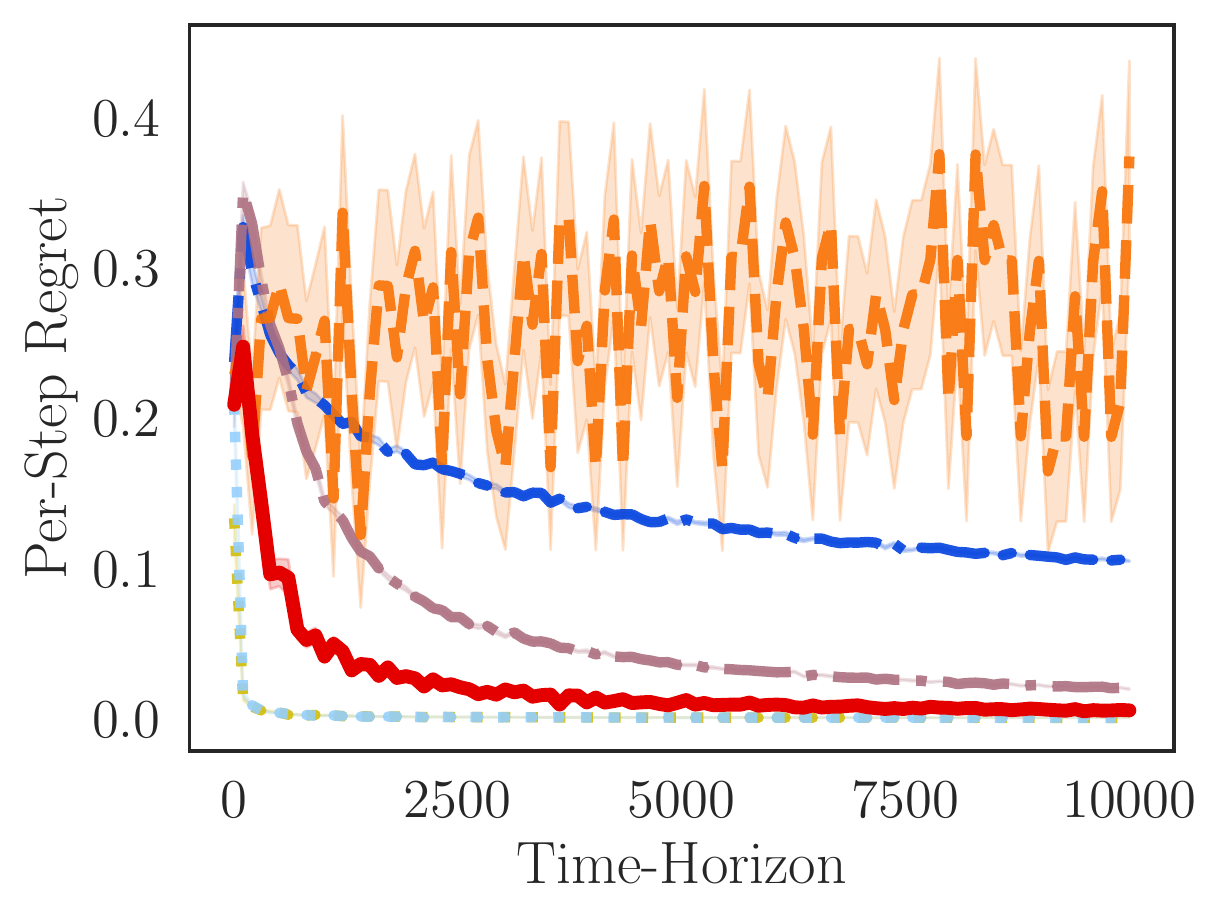}\hfill
      \includegraphics[width=0.32\linewidth]{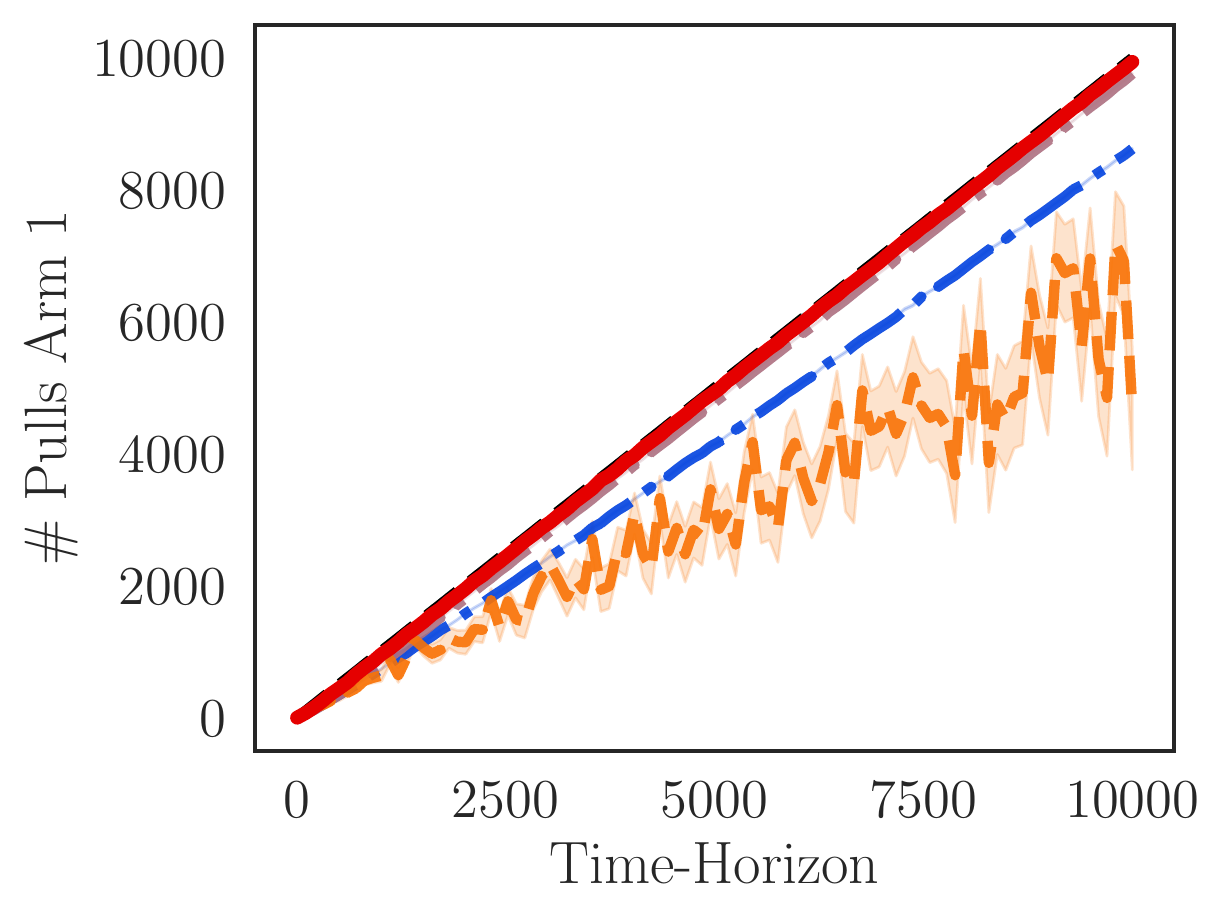}
      \caption{Noise with $\sigma = 0.05$}
   \end{subfigure}\vspace{1em}
   \begin{subfigure}[c]{\linewidth}
      \includegraphics[width=0.32\linewidth]{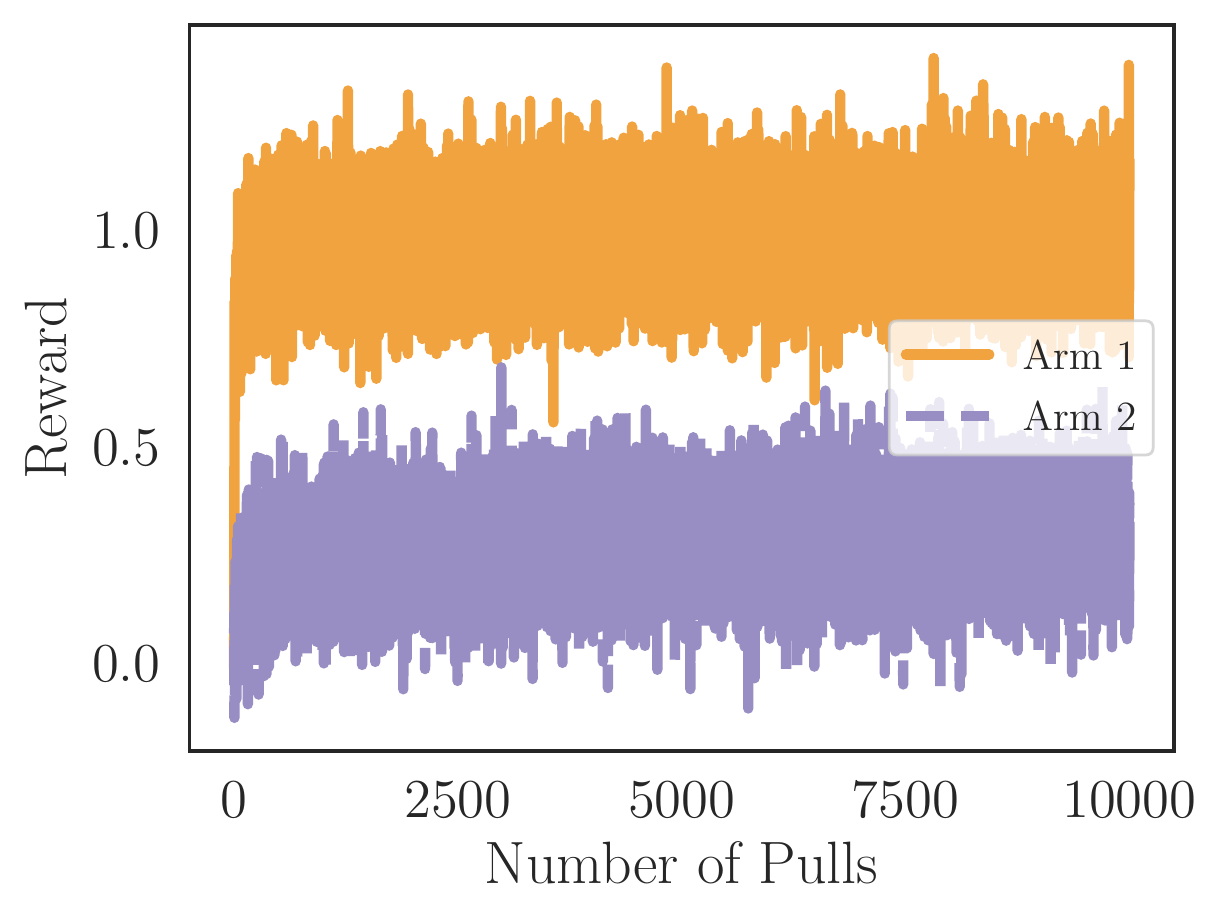}\hfill
      \includegraphics[width=0.32\linewidth]{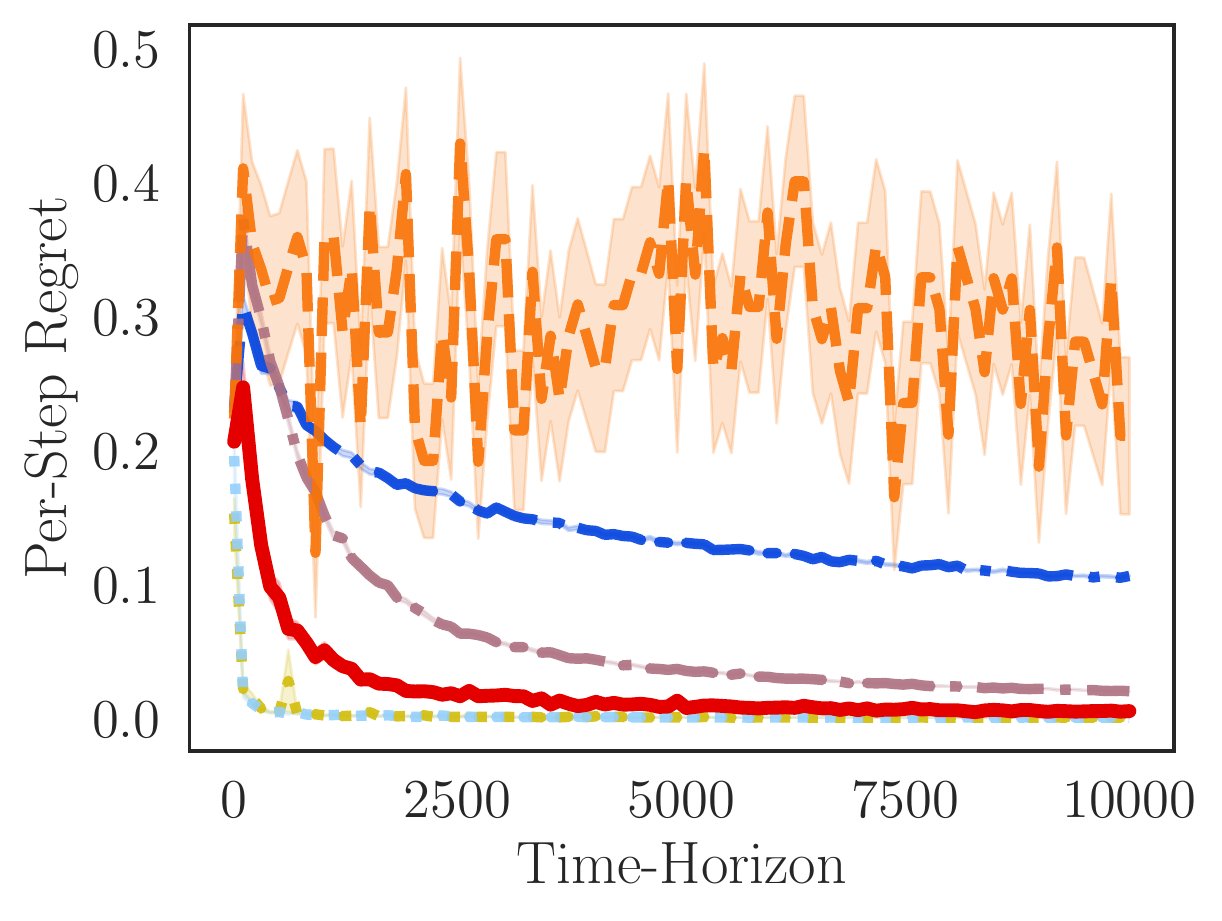}\hfill
      \includegraphics[width=0.32\linewidth]{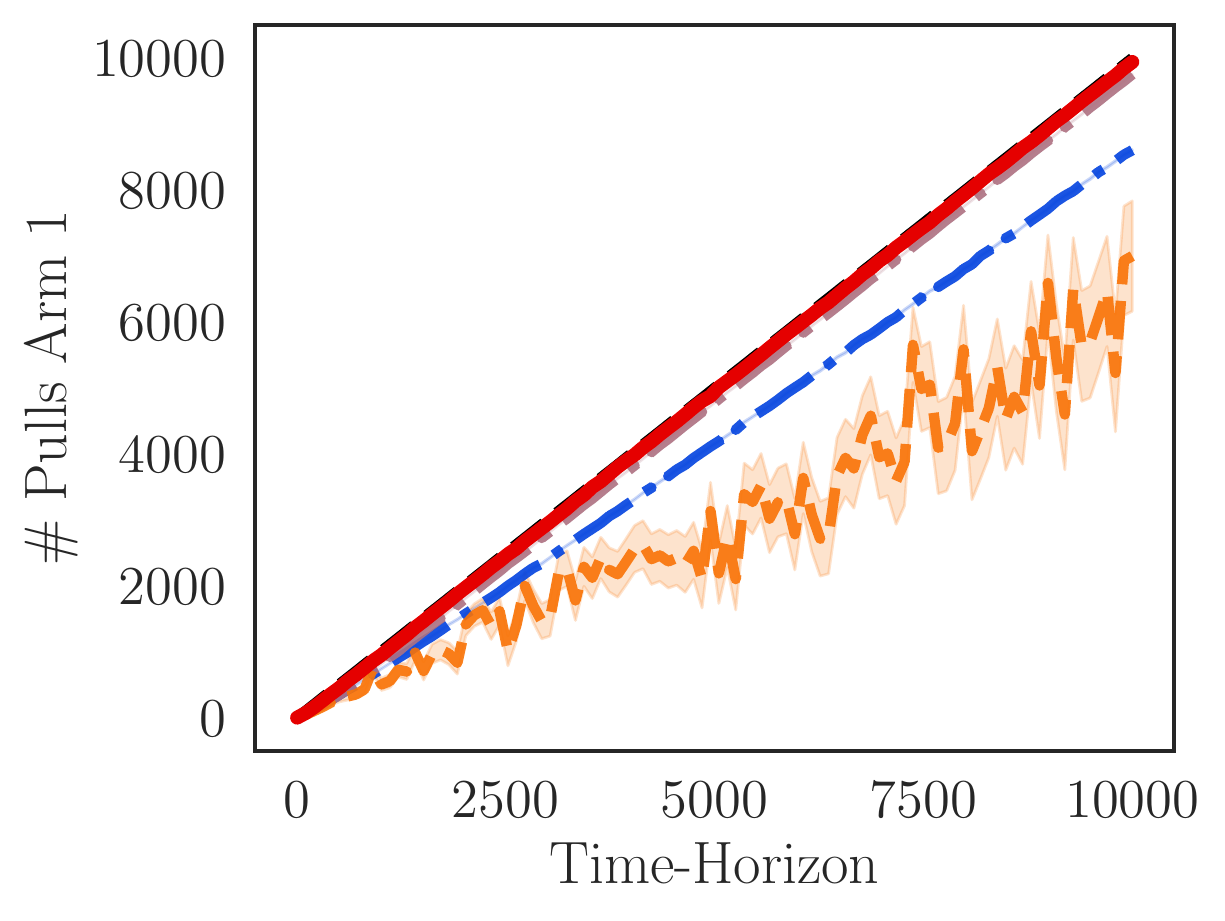}
      \caption{Noise with $\sigma = 0.1$}
   \end{subfigure}\vspace{1em}
   \caption{The left-hand plots show the increasing reward functions defined by $f_1(t) = 1 - t^{-0.5}$, $f_2(t) = 0.5 - 0.5 t^{-\alpha}$ where $\alpha = 0.1$. We add Gaussian noise to the observations, with $\sigma = 0, 0.01, 0.05, 0.1$. The middle plots show the per-step policy regret achieved by \algshort ({\protect\legendSPO}) compared to the algorithm proposed by \citet{Heidari2016} for increasing rewards ({\protect\legendINCREASING}), EXP3 ({\protect\legendEXP}), R-EXP3 ({\protect\legendREXP}), D-UCB ({\protect\legendDUCB}), and SW-UCB ({\protect\legendSWUCB}). The right-hand plots show the policies these algorithms choose in comparison to the optimal policy ({\protect\legendOPTIMAL}).}
   \label{fig:experiment_inc_noise}
\end{minipage}
\end{figure*}

\begin{figure*}[p]
\centering
\begin{minipage}[b]{.48\textwidth}
   \centering
  \begin{subfigure}[c]{\linewidth}
     \includegraphics[width=0.32\linewidth]{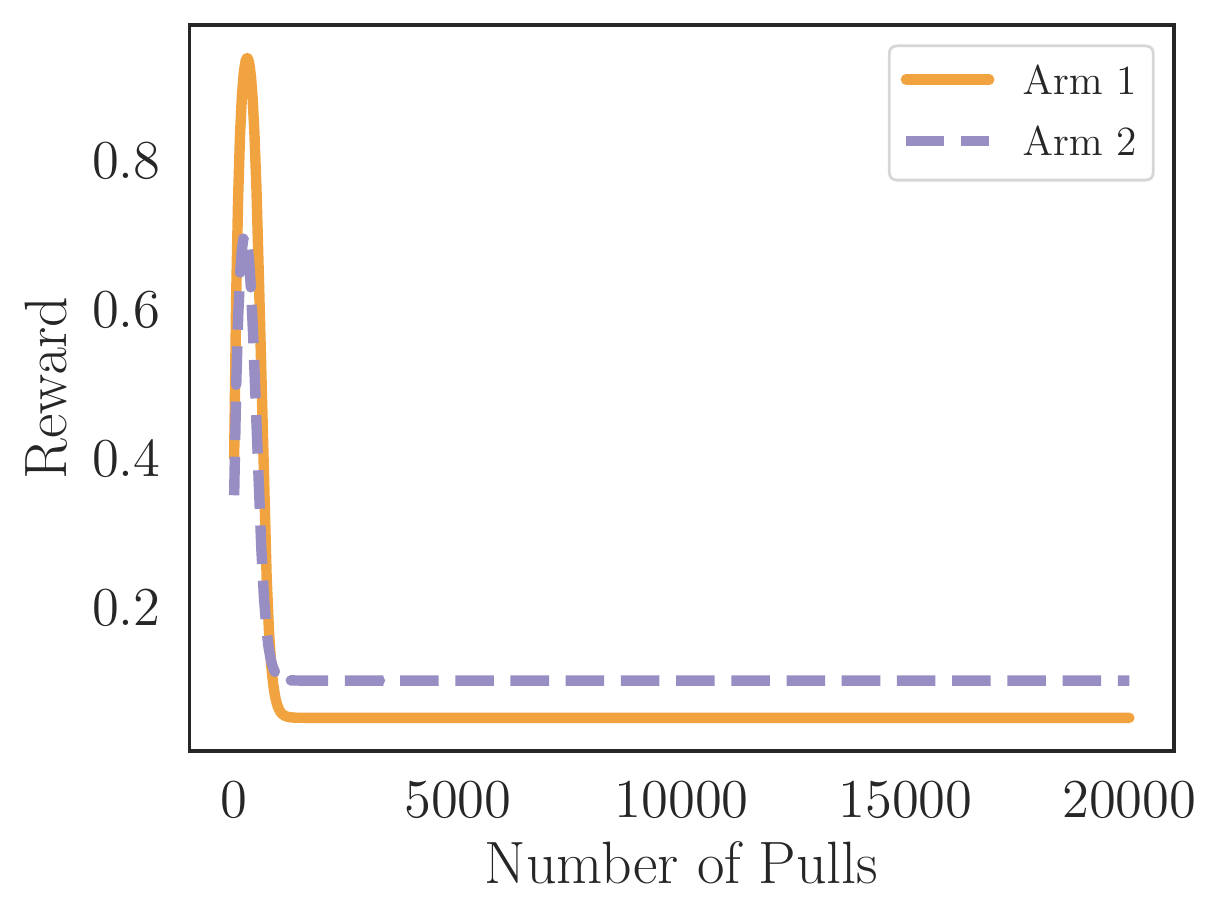}\hfill
     \includegraphics[width=0.32\linewidth]{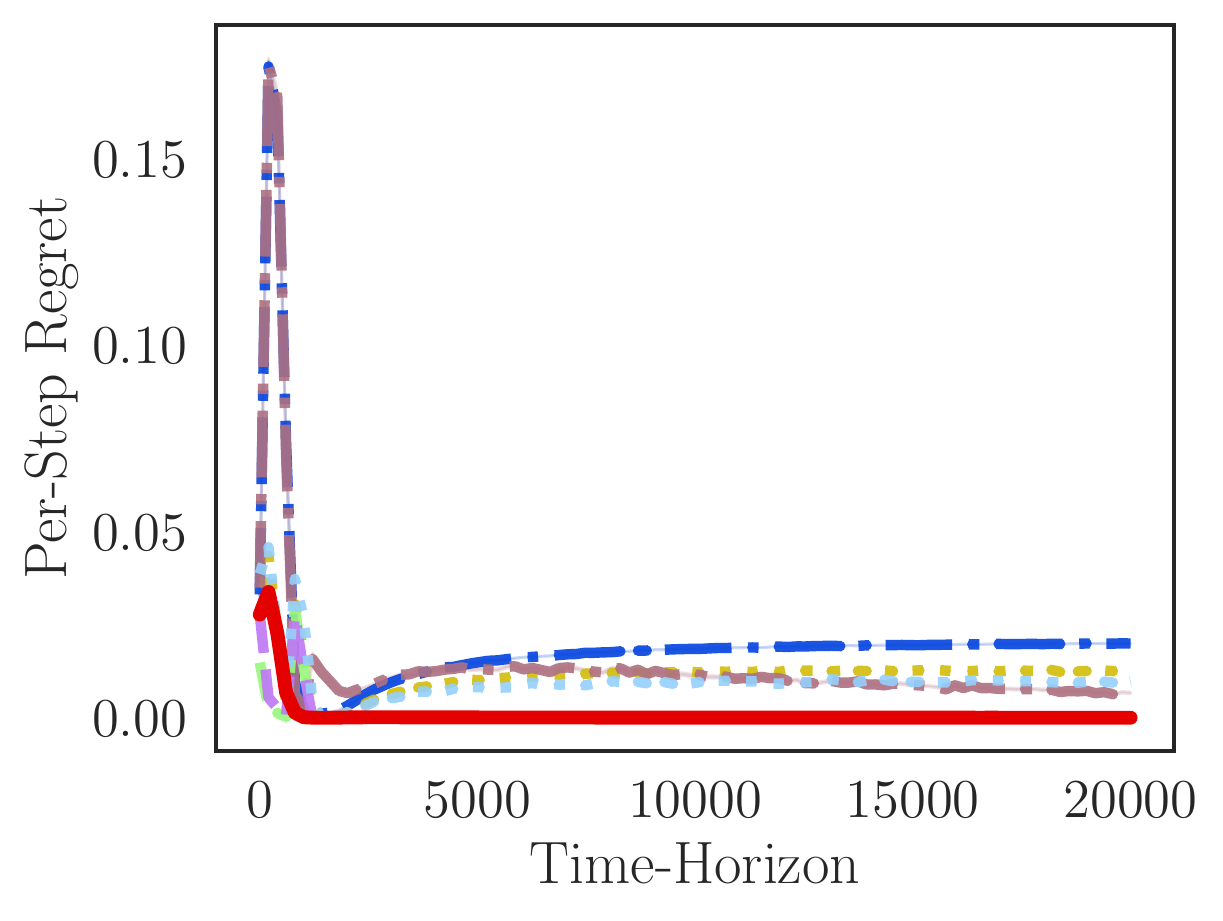}\hfill
     \includegraphics[width=0.32\linewidth]{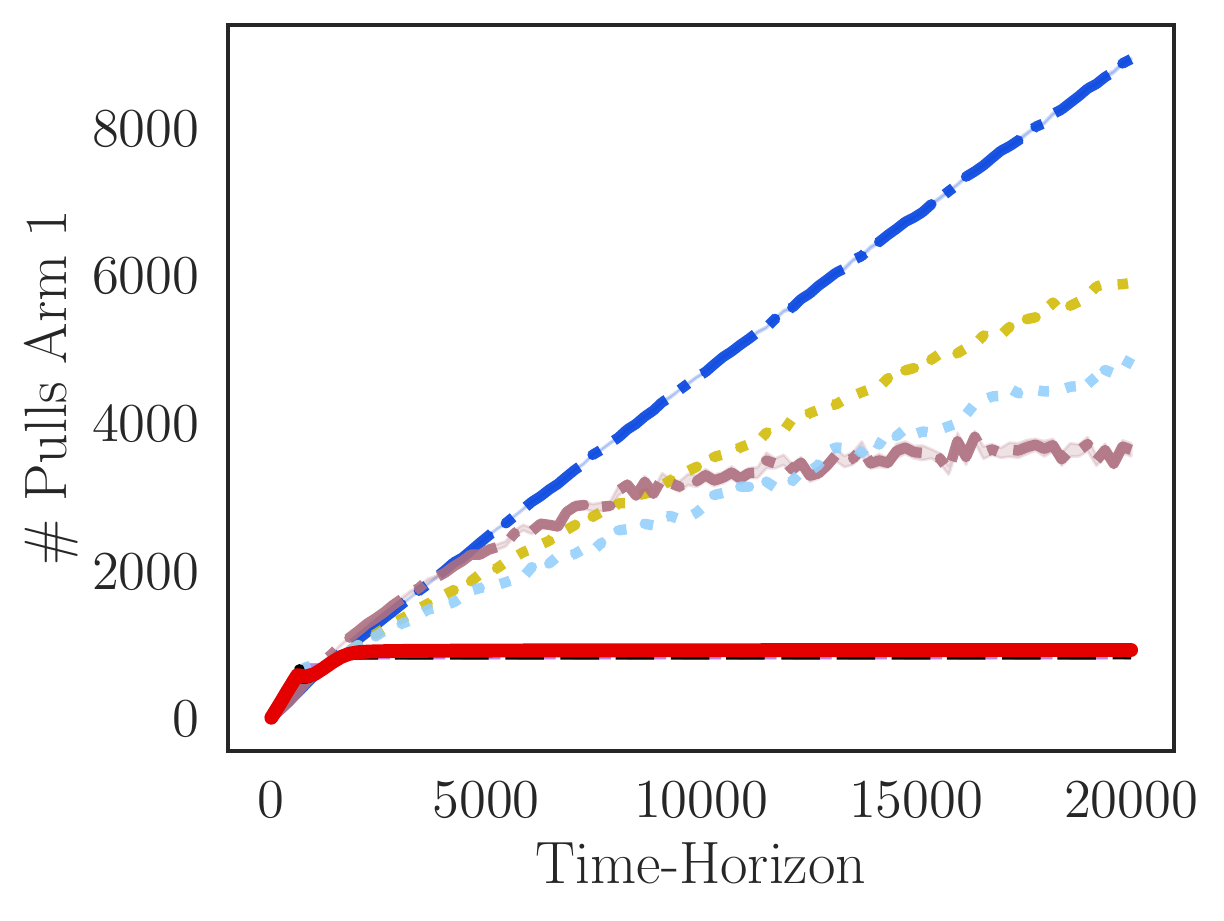}
     \caption{Noise-free observations}
  \end{subfigure}\vspace{1em}
  \begin{subfigure}[c]{\linewidth}
     \includegraphics[width=0.32\linewidth]{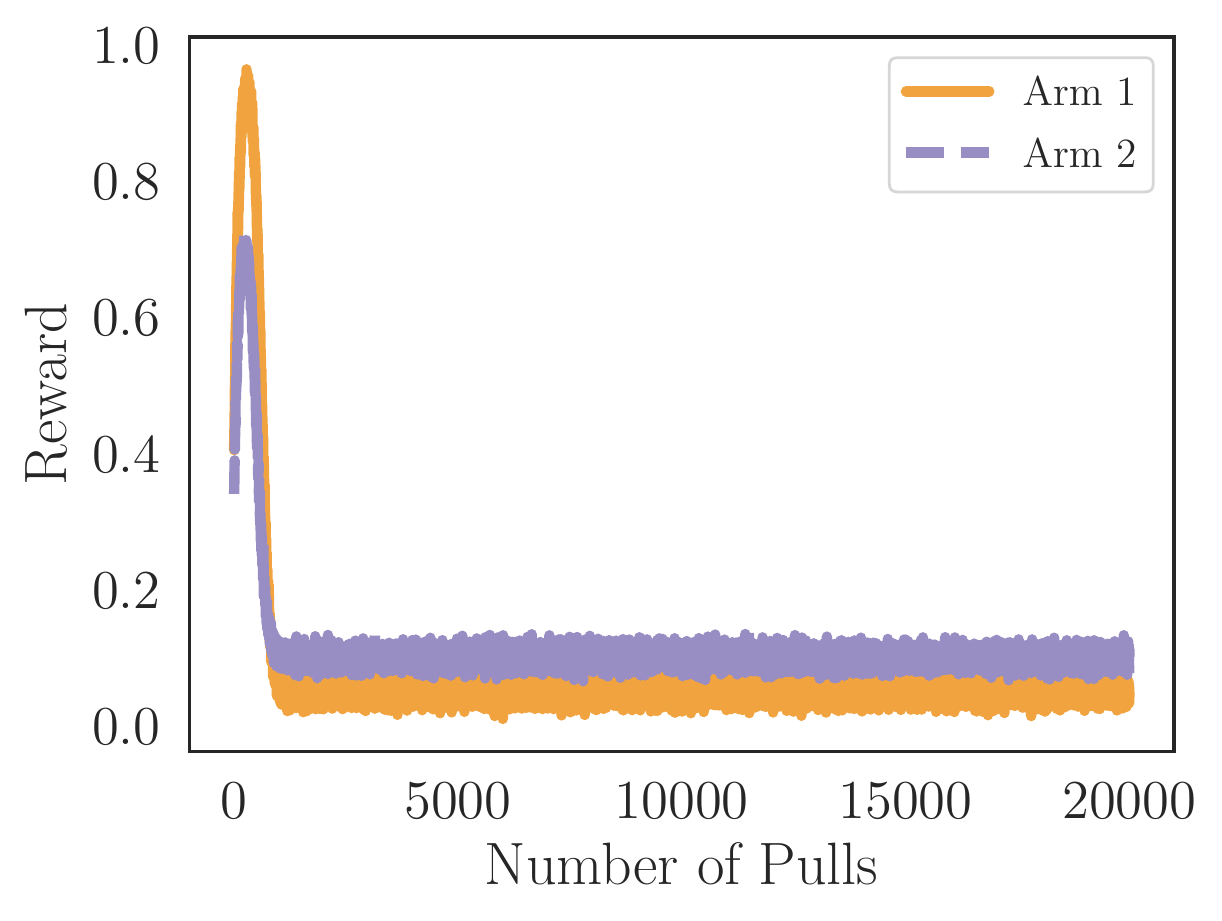}\hfill
     \includegraphics[width=0.32\linewidth]{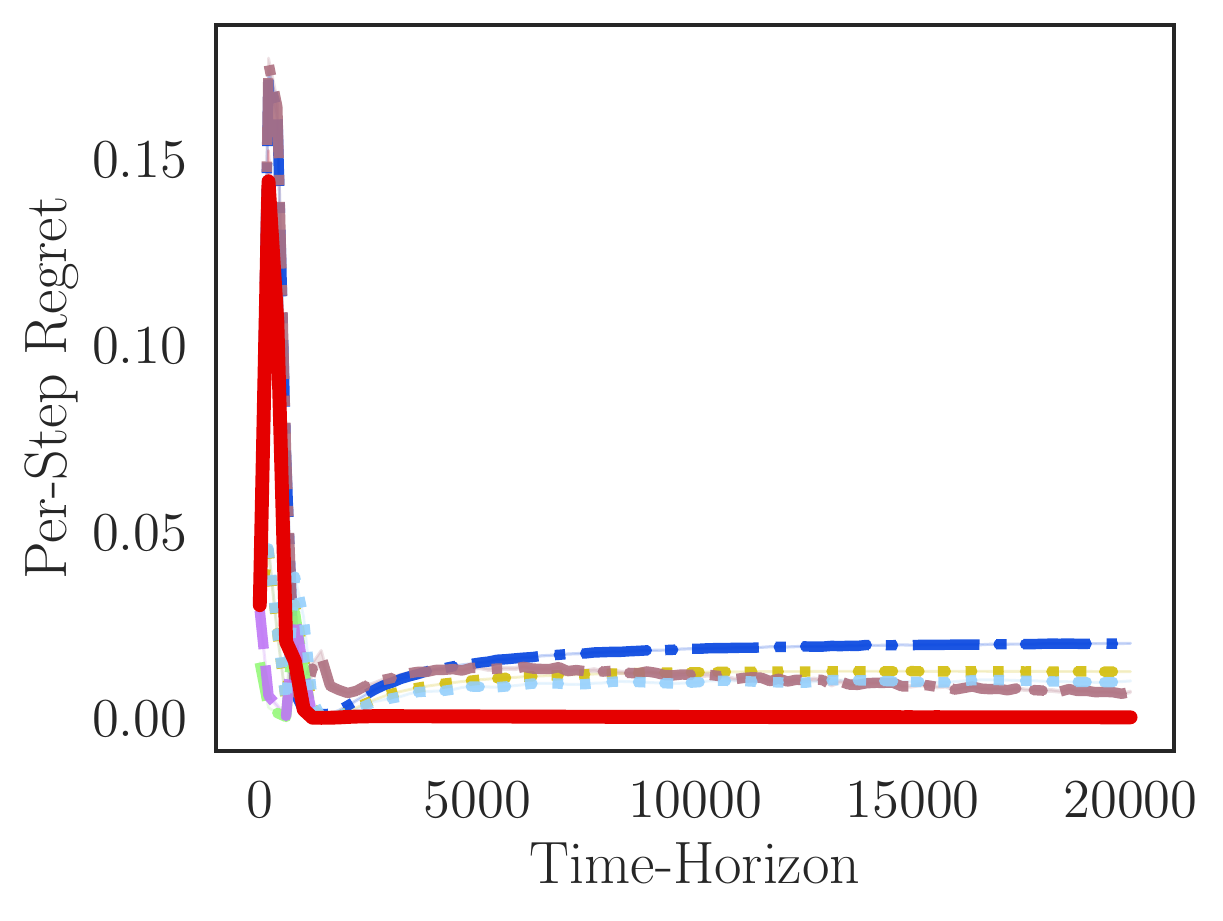}\hfill
     \includegraphics[width=0.32\linewidth]{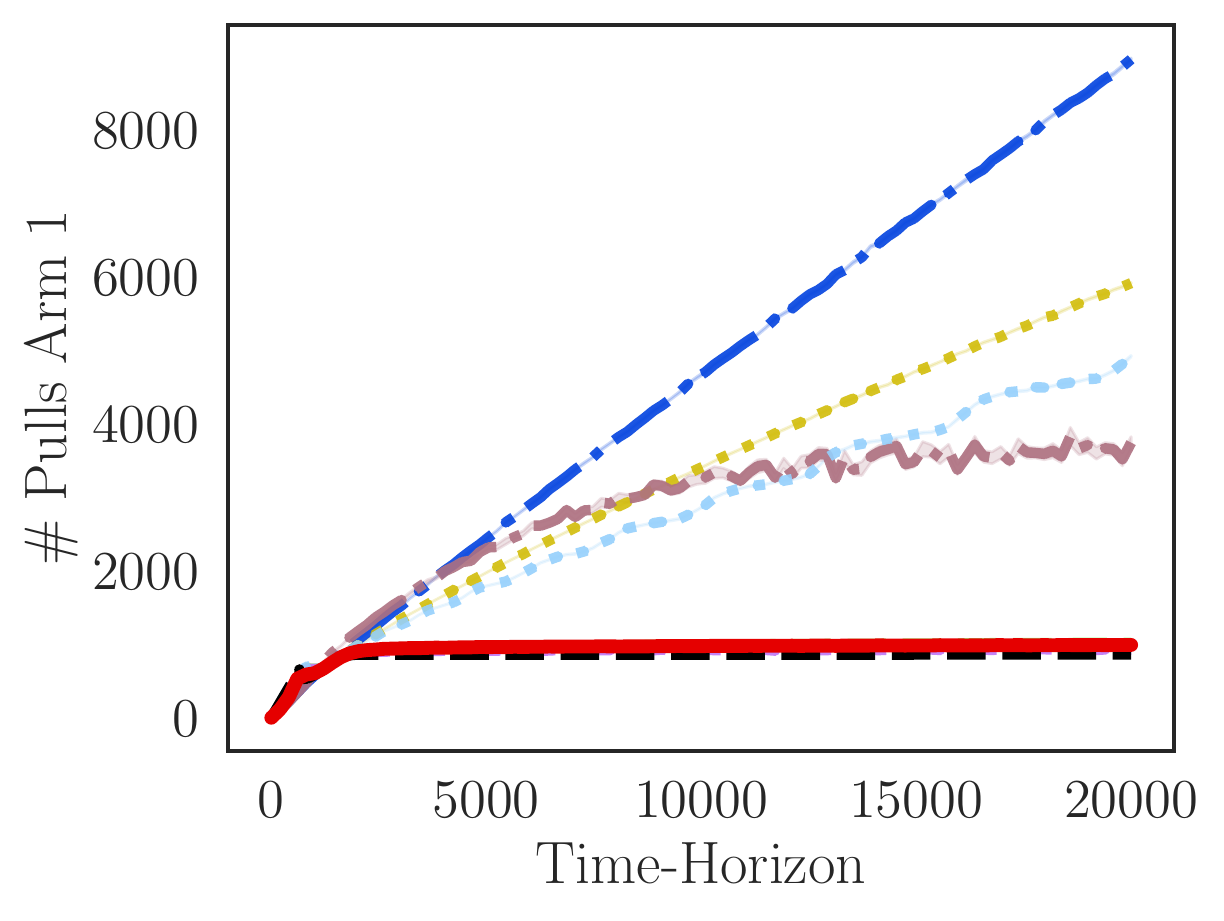}
     \caption{Noise with $\sigma = 0.01$}
  \end{subfigure}\vspace{1em}
  \begin{subfigure}[c]{\linewidth}
     \includegraphics[width=0.32\linewidth]{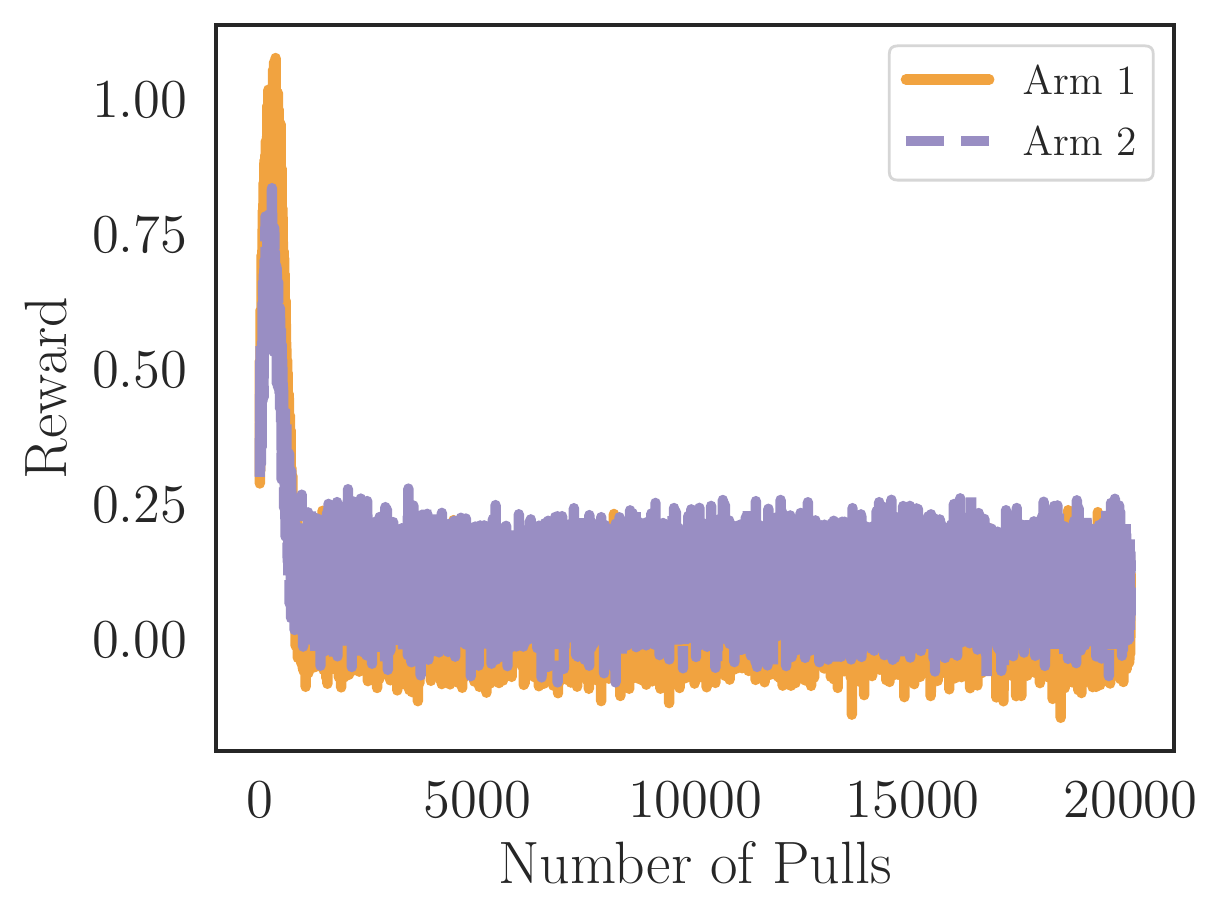}\hfill
     \includegraphics[width=0.32\linewidth]{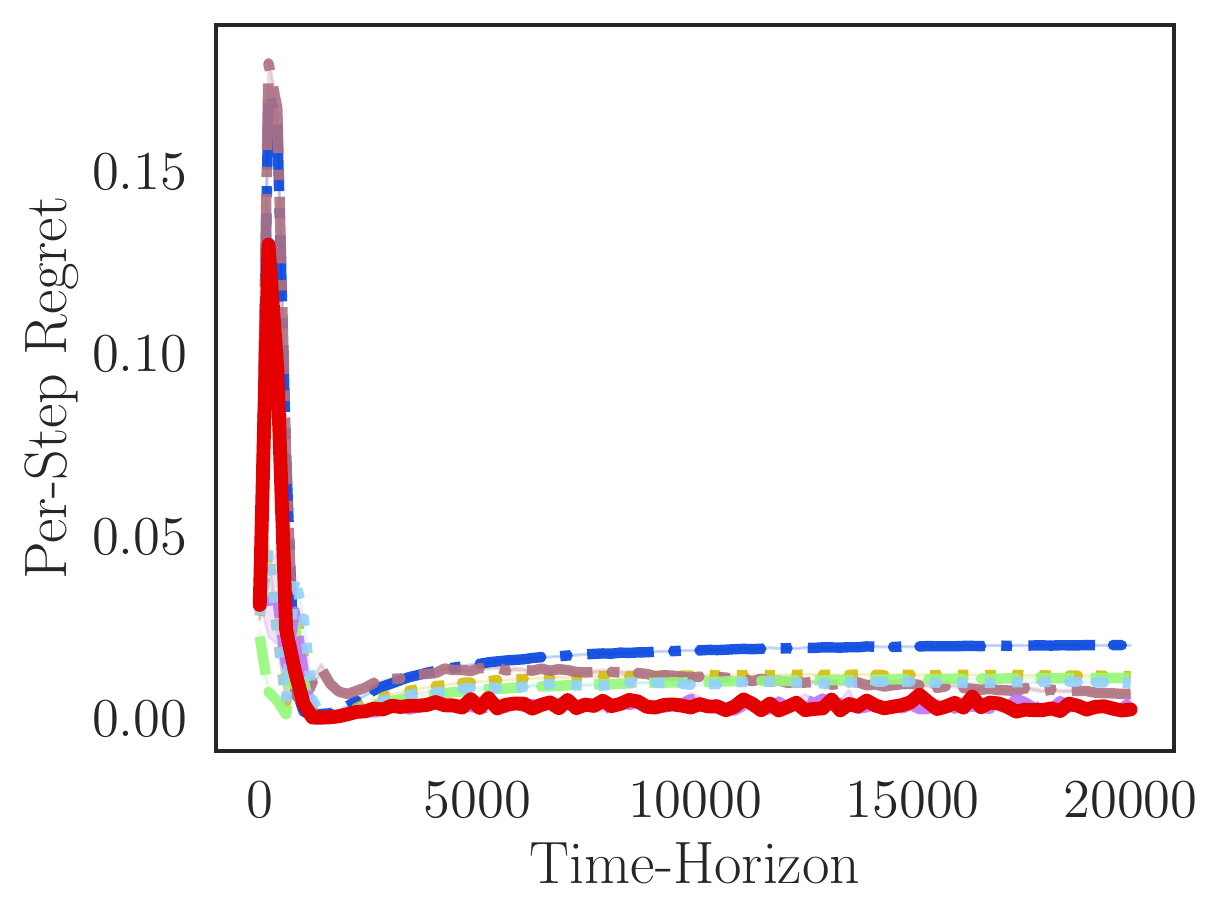}\hfill
     \includegraphics[width=0.32\linewidth]{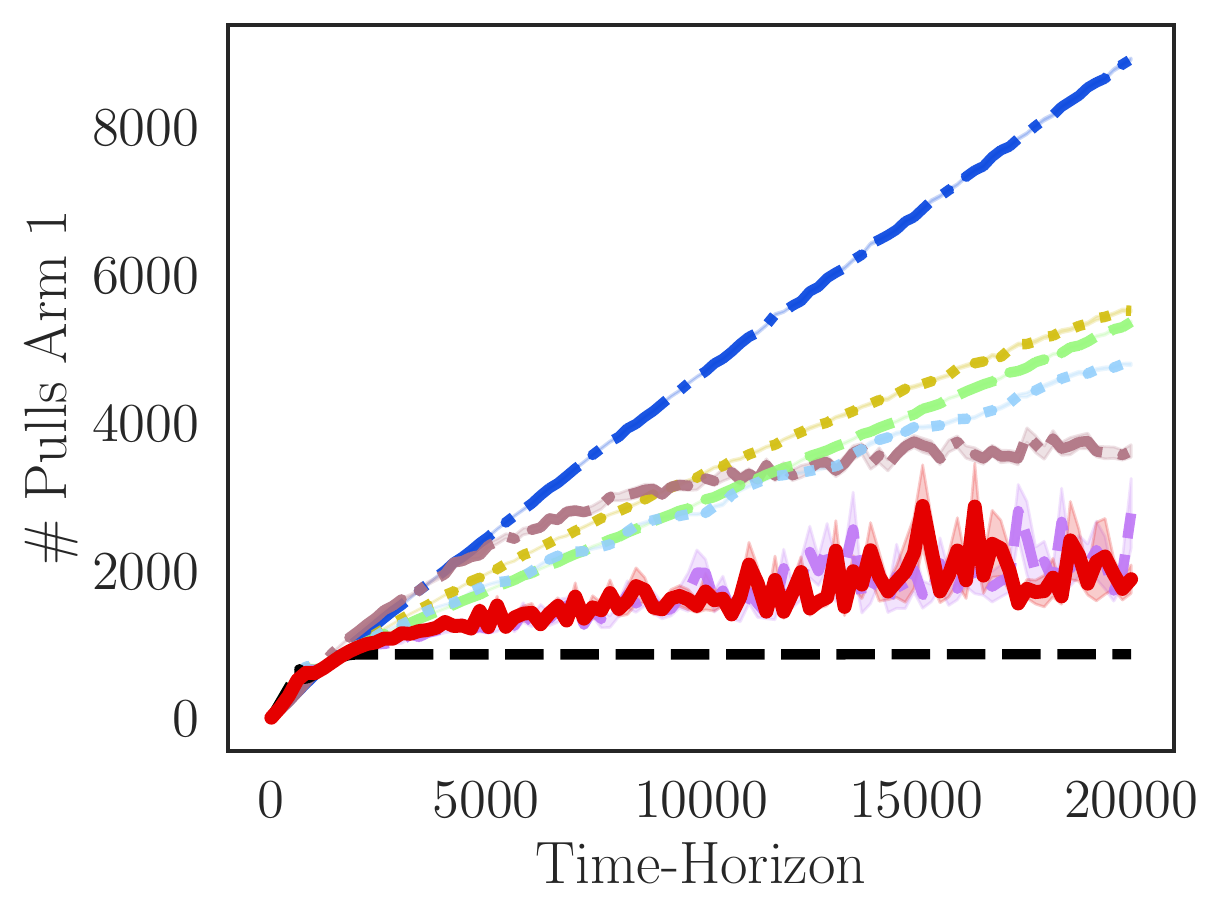}
     \caption{Noise with $\sigma = 0.05$}
  \end{subfigure}\vspace{1em}
  \resizebox{\linewidth}{!}{\begin{minipage}{\linewidth}
     \begin{align*}
     f_1(t) = &-0.0015 \cdot e^{-0.01 \cdot (t-600)} +
              &\frac{-0.95}{e^{-(0.011) \cdot (t-600)} + 1} + 1 \\
     f_2(t) = &-0.005 \cdot e^{-0.009 \cdot (t-500)} +
              &\frac{-0.7}{e^{-(0.0099) \cdot (t-500)} + 1} + 0.8
     \end{align*}
  \end{minipage}}
   \caption{The left-hand plots show the single-peaked reward functions $f_1$ and $f_2$ with simulated Gaussian noise. The middle plots show the per-step policy regret achieved by \algshort ({\protect\legendSPO}) compared to EXP3 ({\protect\legendEXP}), R-EXP3 ({\protect\legendREXP}), D-UCB ({\protect\legendDUCB}), SW-UCB ({\protect\legendSWUCB}), a one-step-optimistic ({\protect\legendOSO}), and a greedy algorithm ({\protect\legendGREEDY}). The right-hand plots show the policies these algorithms choose in comparison to the optimal policy ({\protect\legendOPTIMAL}).}
   \label{fig:experiment_inc_dec_1}
\end{minipage}\hfill
\begin{minipage}[b]{.48\textwidth}
   \centering
  \begin{subfigure}[c]{\linewidth}
     \includegraphics[width=0.32\linewidth]{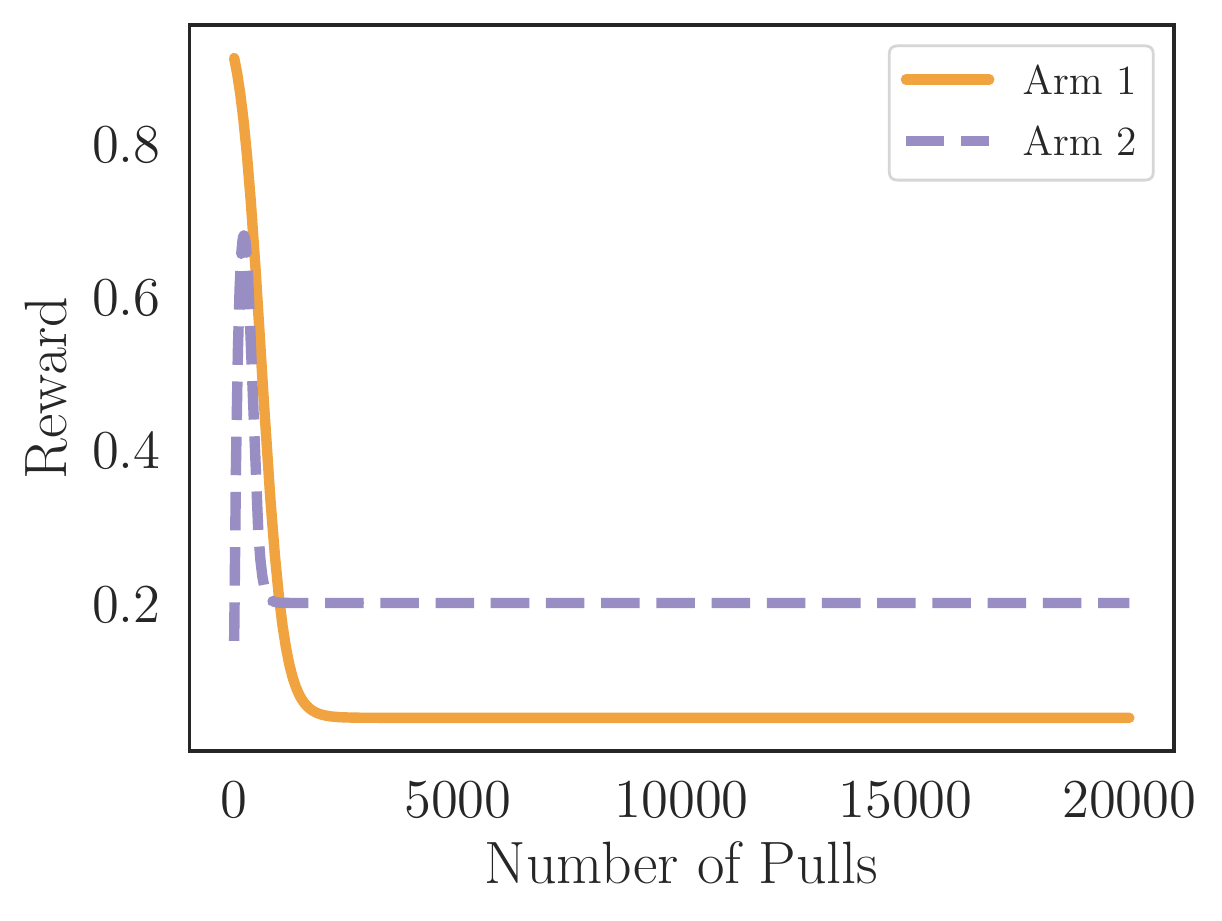}\hfill
     \includegraphics[width=0.32\linewidth]{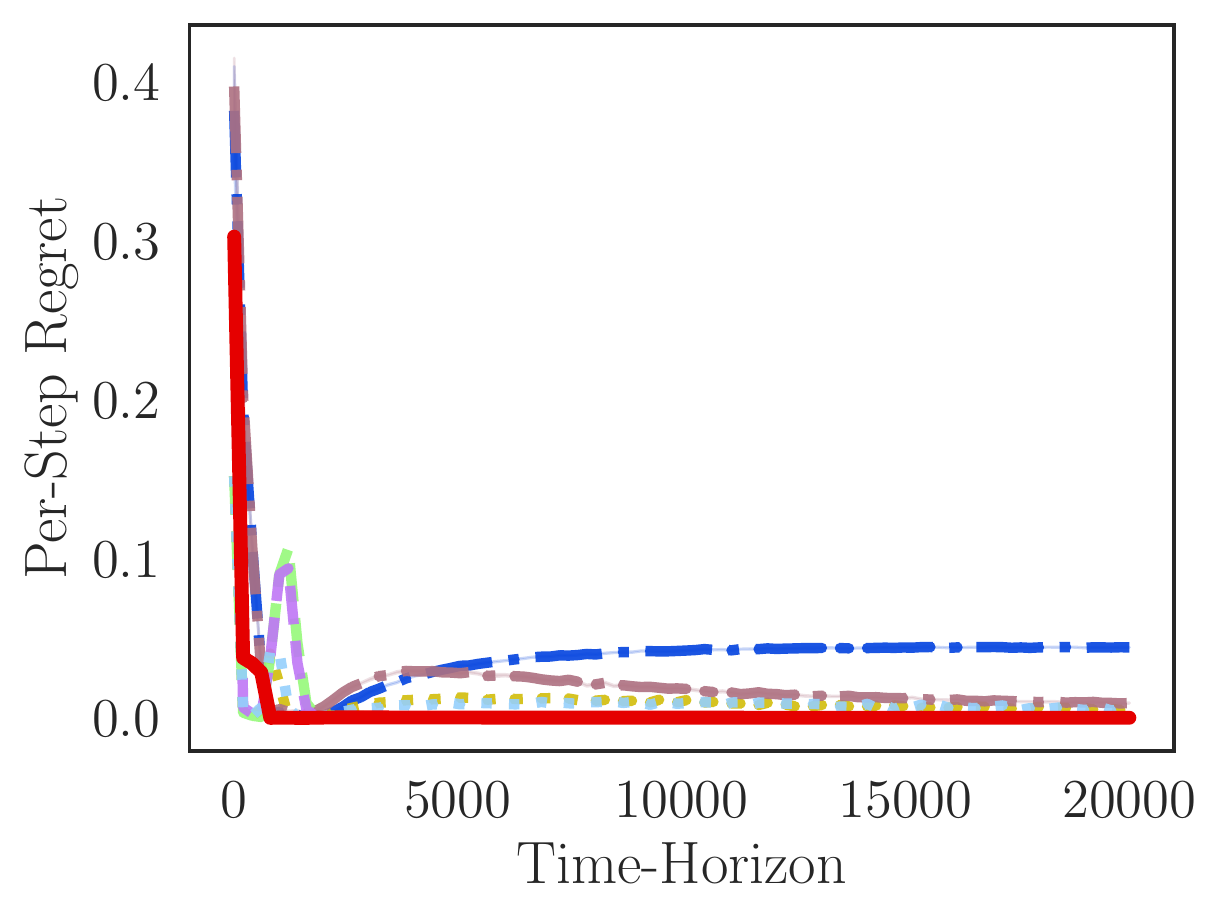}\hfill
     \includegraphics[width=0.32\linewidth]{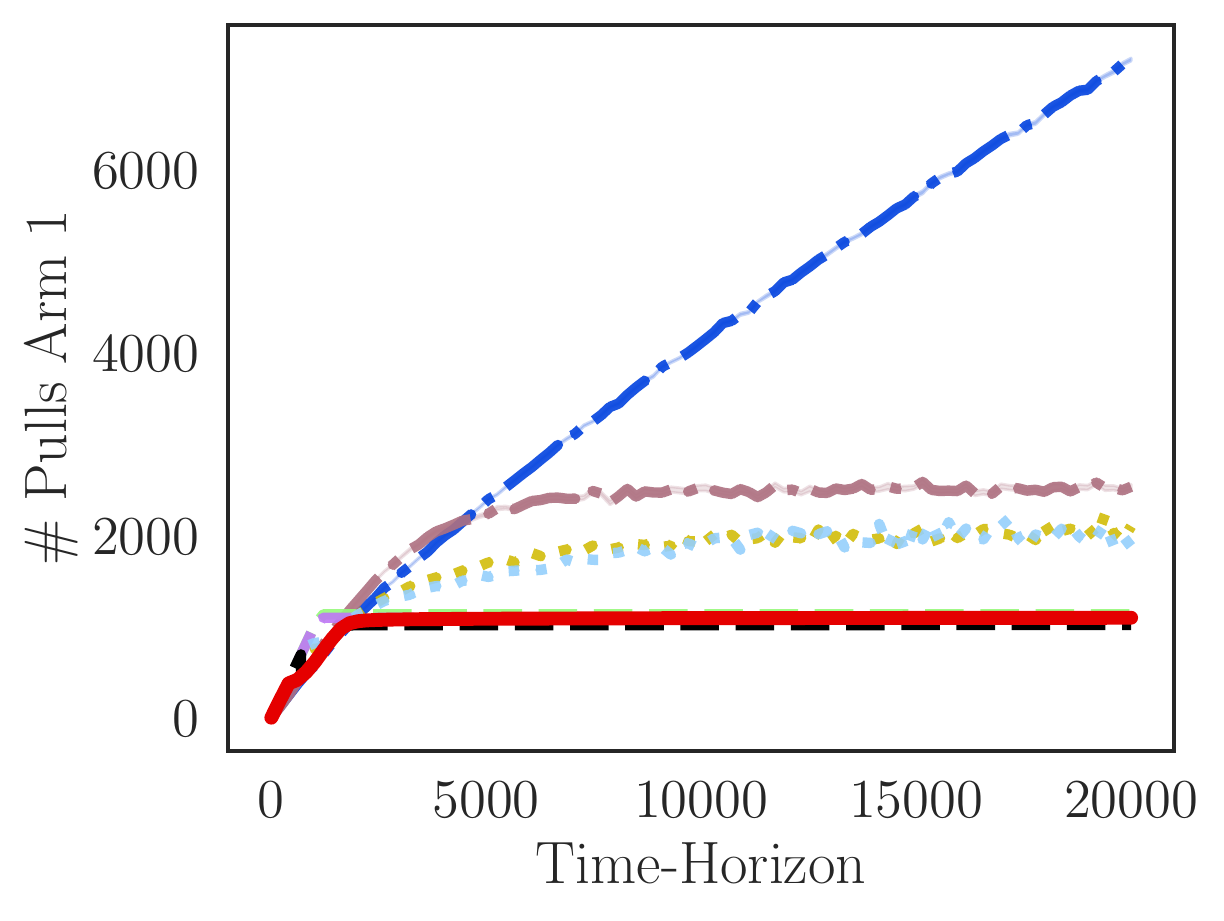}
     \caption{Noise-free observations}
  \end{subfigure}\vspace{1em}
  \begin{subfigure}[c]{\linewidth}
     \includegraphics[width=0.32\linewidth]{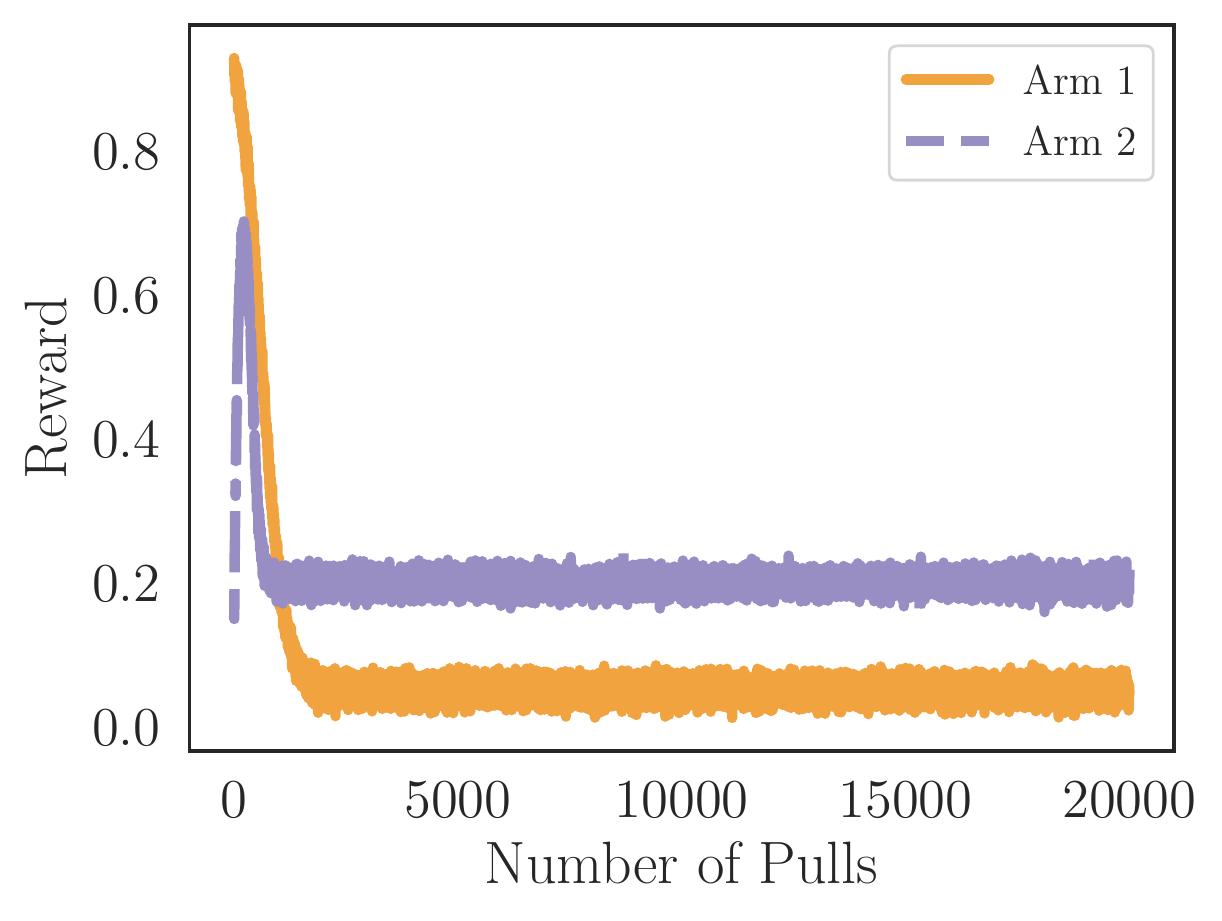}\hfill
     \includegraphics[width=0.32\linewidth]{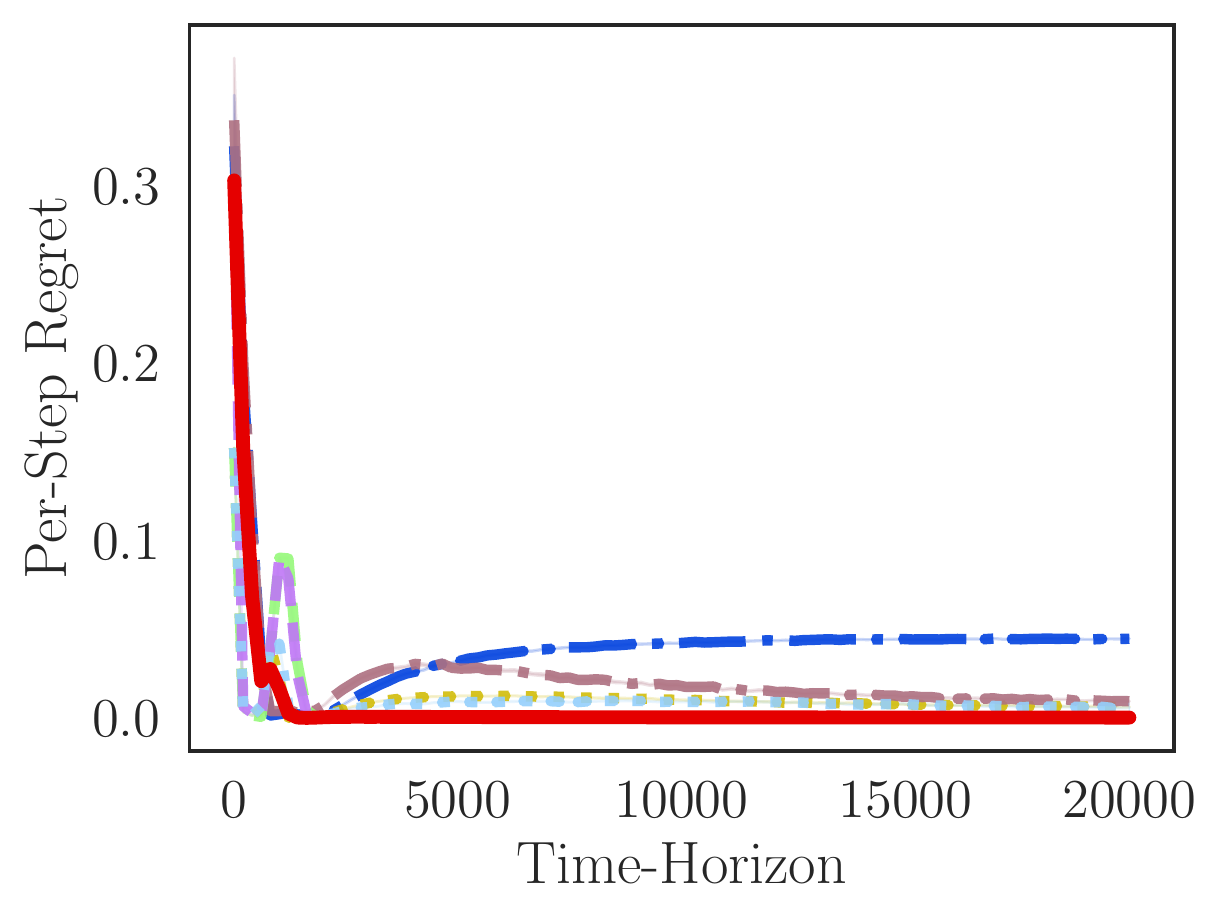}\hfill
     \includegraphics[width=0.32\linewidth]{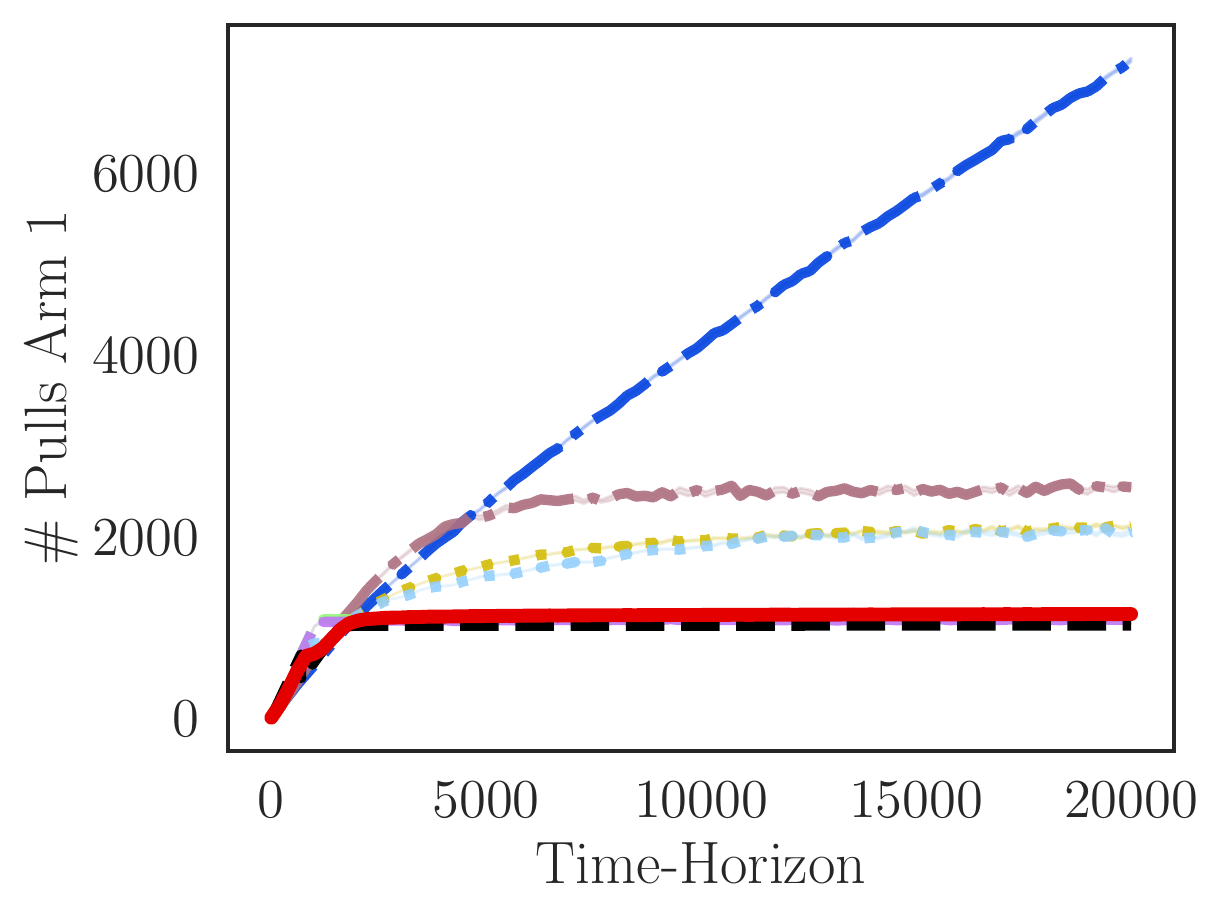}
     \caption{Noise with $\sigma = 0.01$}
  \end{subfigure}\vspace{1em}
  \begin{subfigure}[c]{\linewidth}
     \includegraphics[width=0.32\linewidth]{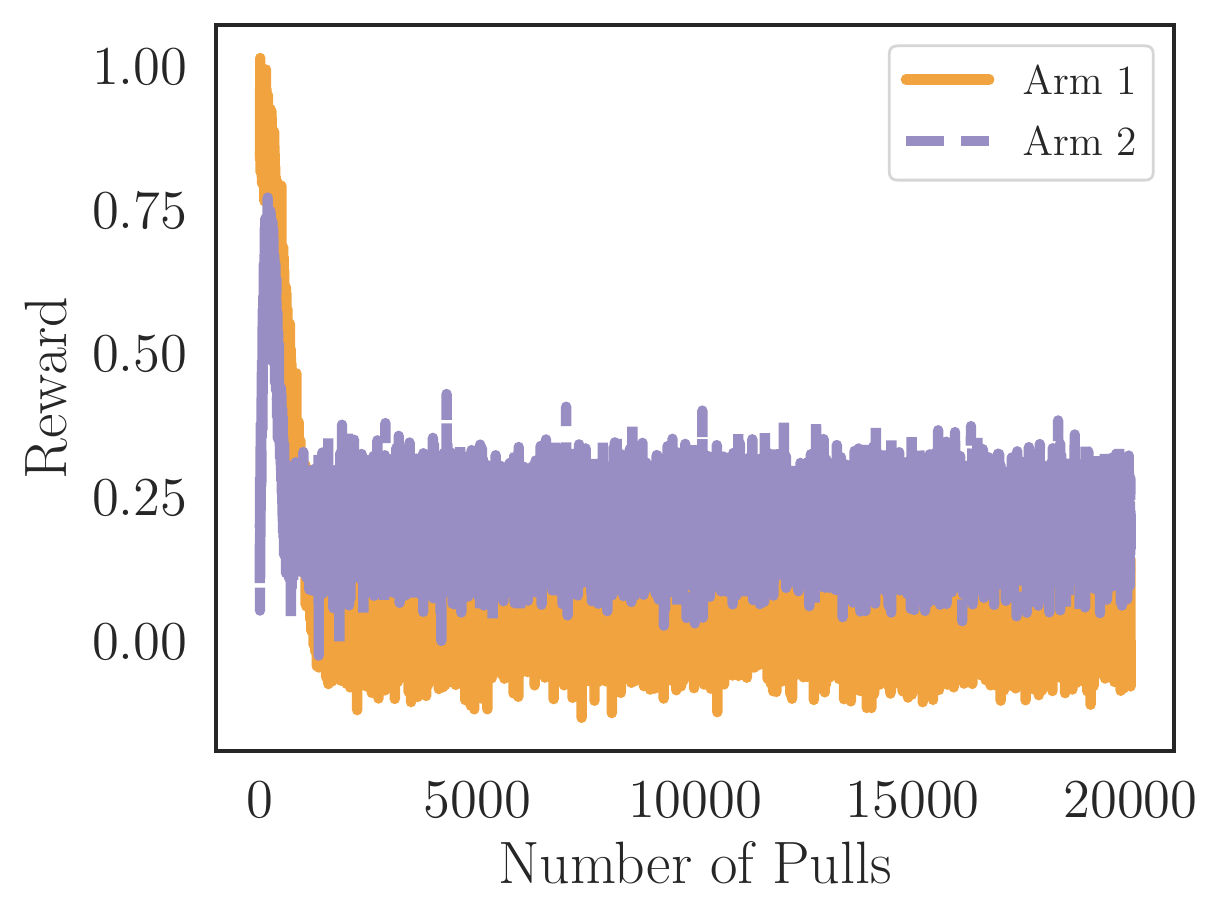}\hfill
     \includegraphics[width=0.32\linewidth]{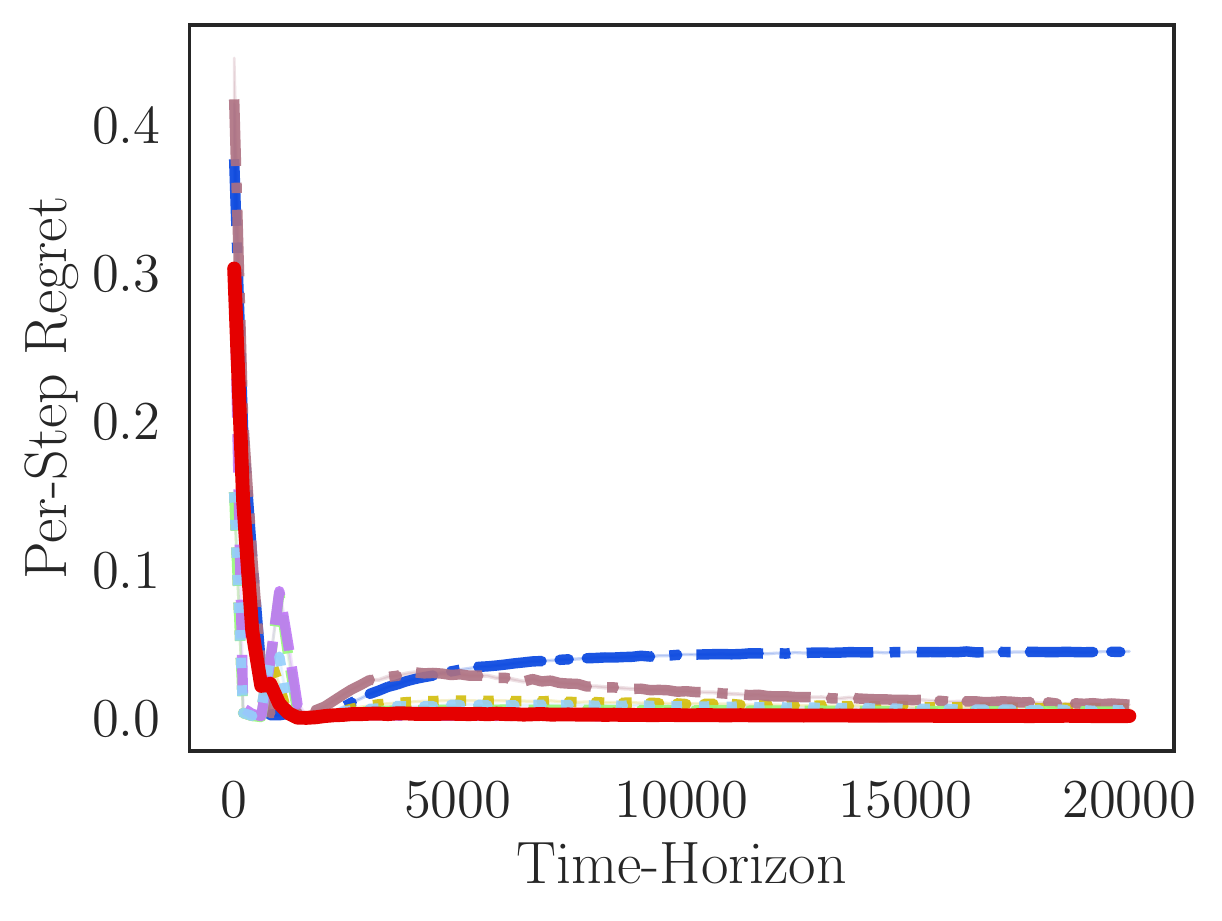}\hfill
     \includegraphics[width=0.32\linewidth]{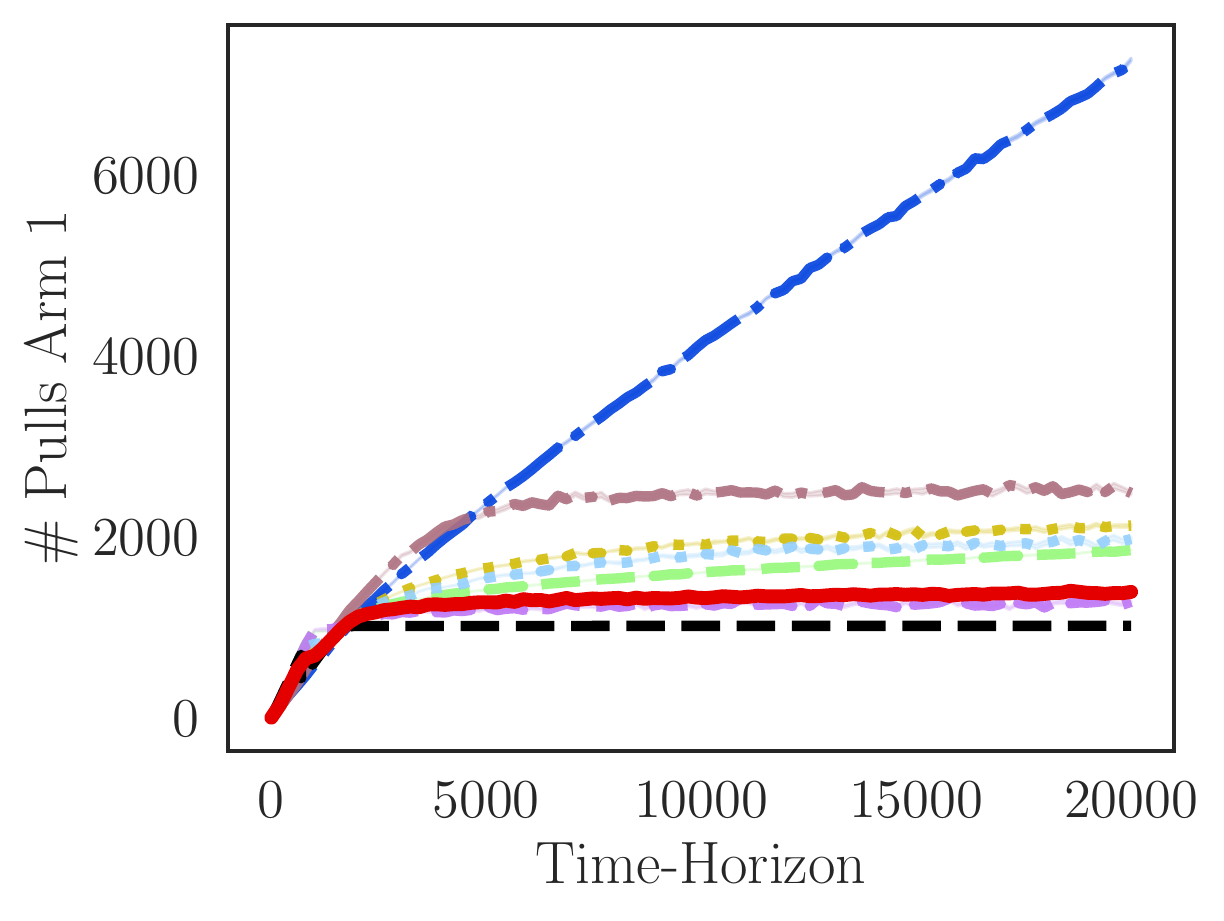}
     \caption{Noise with $\sigma = 0.05$}
  \end{subfigure}\vspace{1em}
  \resizebox{\linewidth}{!}{\begin{minipage}{\linewidth}
     \begin{align*}
        f_1(t) = &-0.0015 \cdot e^{-0.003 \cdot (t-600)} + \frac{-0.95}{e^{-(0.004) \cdot (t-600)} + 1} + 1 \\
        f_2(t) = &-0.008 \cdot e^{-0.011 \cdot (t-400)} + \frac{-0.6}{e^{-(0.012) \cdot (t-400)} + 1} + 0.8
     \end{align*}
  \end{minipage}}
   \caption{The left-hand plots show the single-peaked reward functions $f_1$ and $f_2$ with simulated Gaussian noise. The middle plots show the per-step policy regret achieved by \algshort ({\protect\legendSPO}) compared to EXP3 ({\protect\legendEXP}), R-EXP3 ({\protect\legendREXP}), D-UCB ({\protect\legendDUCB}), SW-UCB ({\protect\legendSWUCB}), a one-step-optimistic ({\protect\legendOSO}), and a greedy algorithm ({\protect\legendGREEDY}). The right-hand plots show the policies these algorithms choose in comparison to the optimal policy ({\protect\legendOPTIMAL}).}
   \label{fig:experiment_inc_dec_2}
\end{minipage}
\end{figure*}

\begin{figure*}[p]
\centering
\begin{minipage}[b]{.48\textwidth}
  \centering
  \begin{subfigure}[c]{\linewidth}
     \includegraphics[width=0.32\linewidth]{images/experiments/inc_dec_3_arms.pdf}\hfill
     \includegraphics[width=0.32\linewidth]{images/experiments/inc_dec_3_regret.pdf}\hfill
     \includegraphics[width=0.32\linewidth]{images/experiments/inc_dec_3_policy.pdf}
     \caption{Noise-free observations}
  \end{subfigure}\vspace{1em}
  \begin{subfigure}[c]{\linewidth}
     \includegraphics[width=0.32\linewidth]{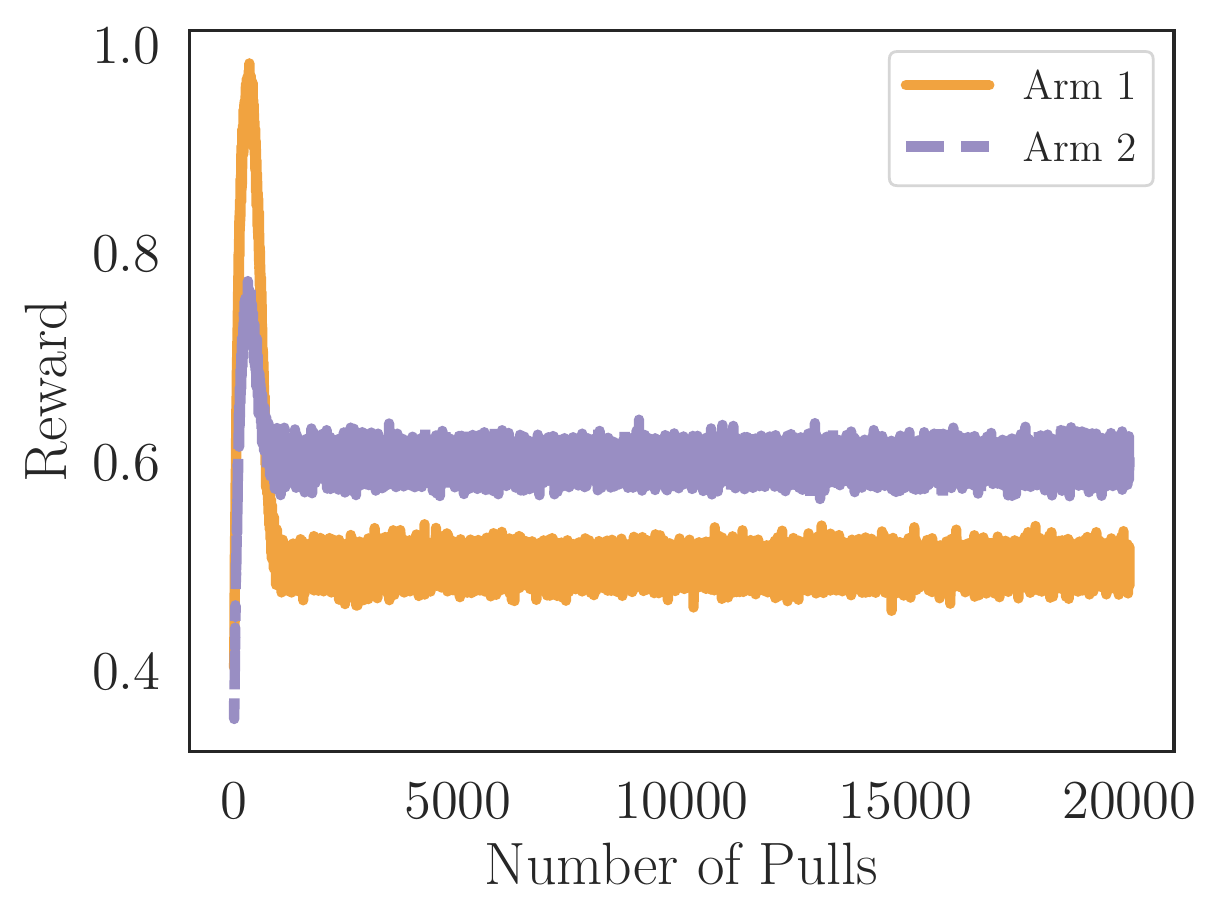}\hfill
     \includegraphics[width=0.32\linewidth]{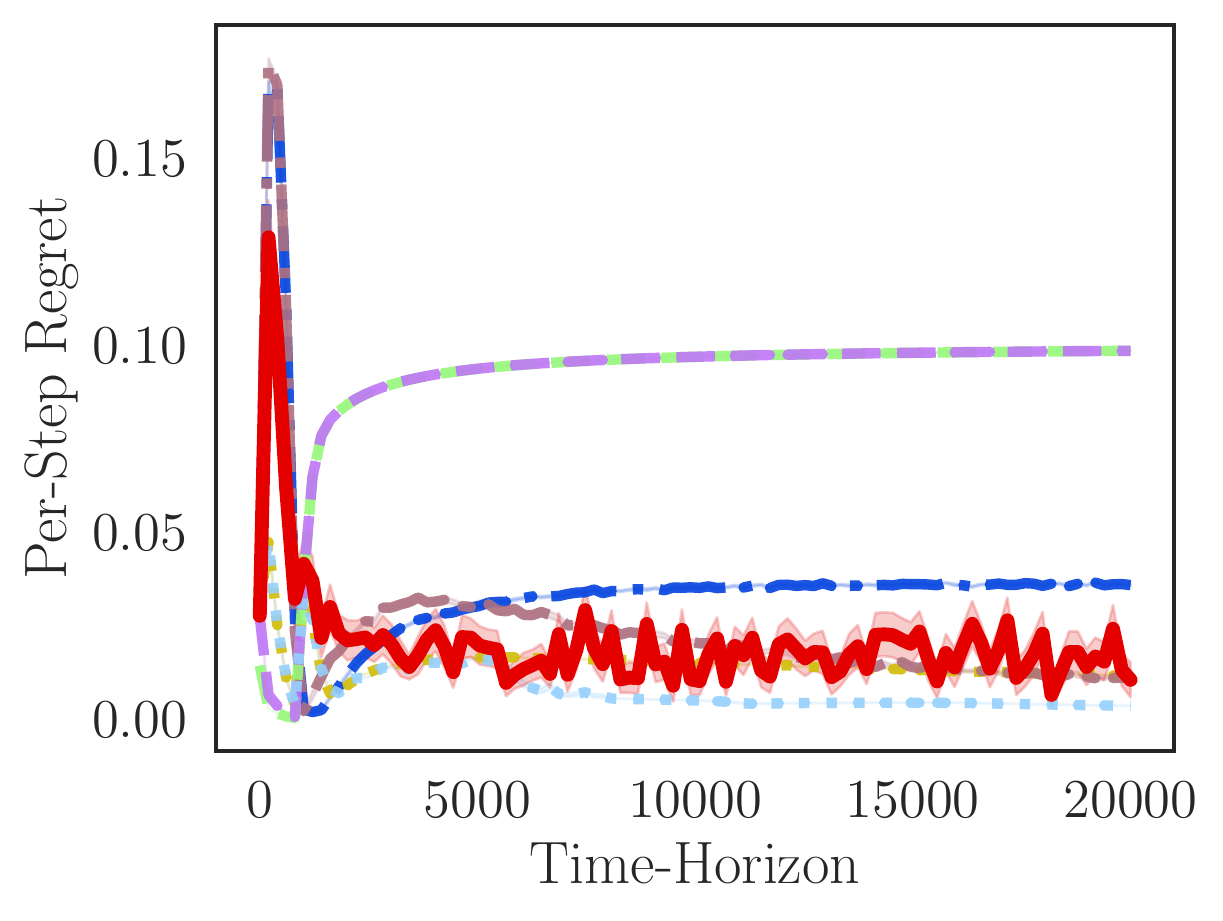}\hfill
     \includegraphics[width=0.32\linewidth]{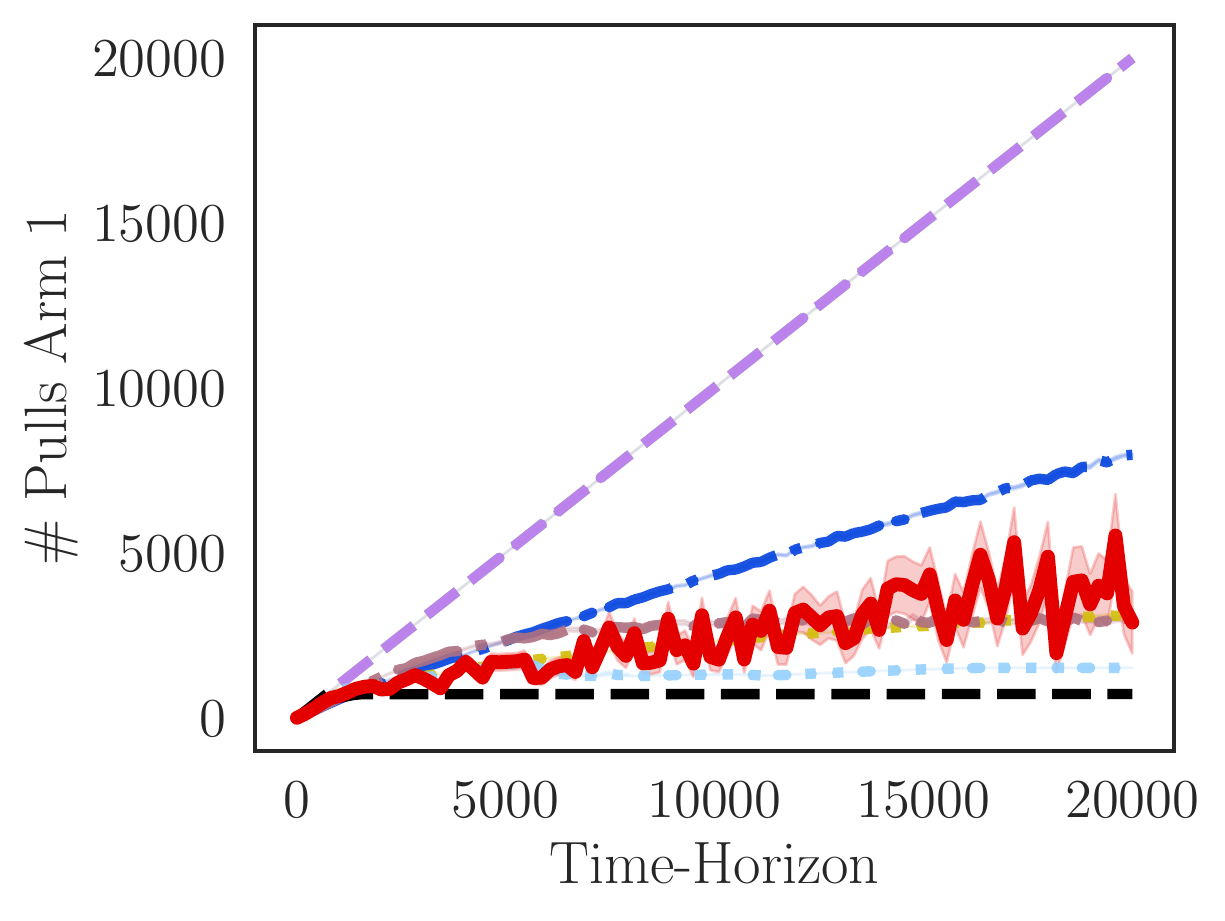}
     \caption{Noise with $\sigma = 0.01$}
  \end{subfigure}\vspace{1em}
  \begin{subfigure}[c]{\linewidth}
     \includegraphics[width=0.32\linewidth]{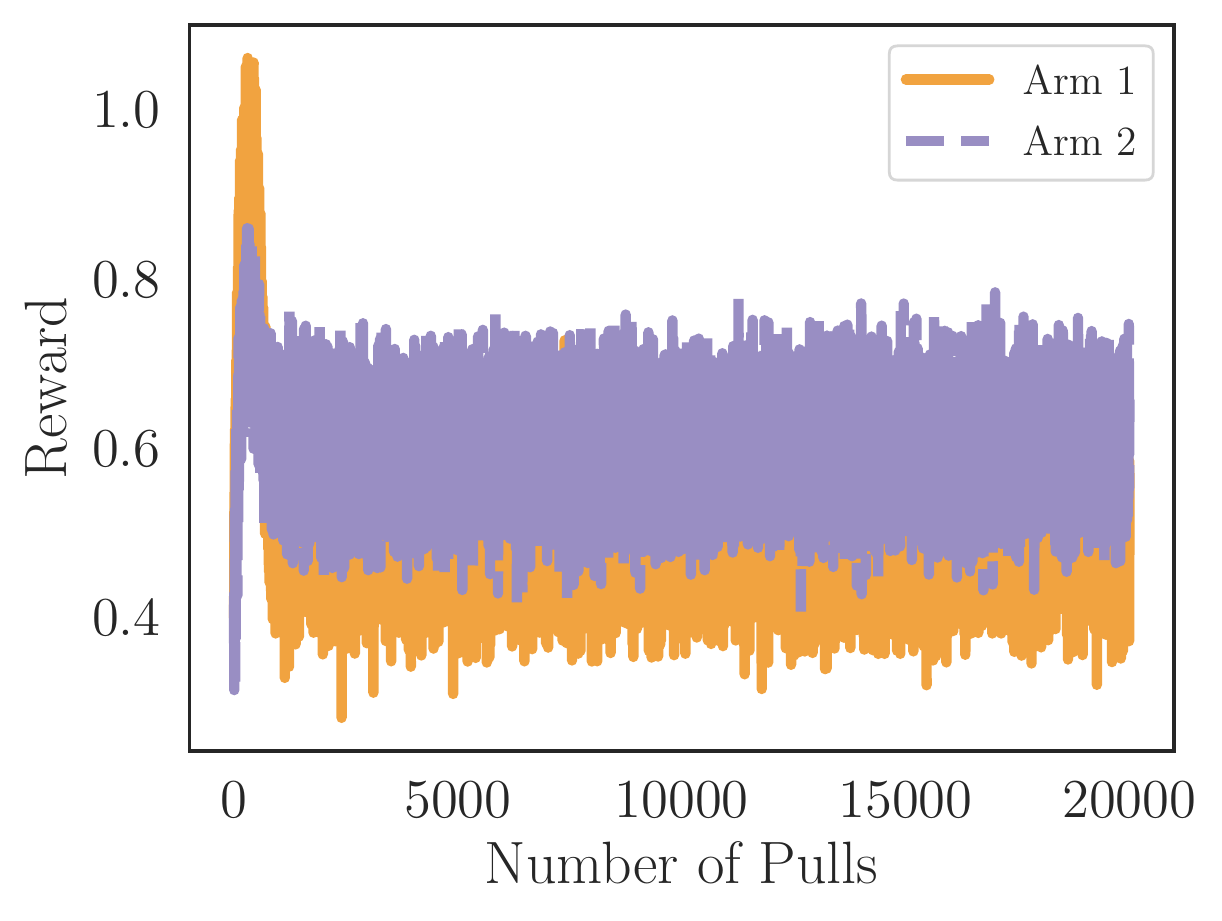}\hfill
     \includegraphics[width=0.32\linewidth]{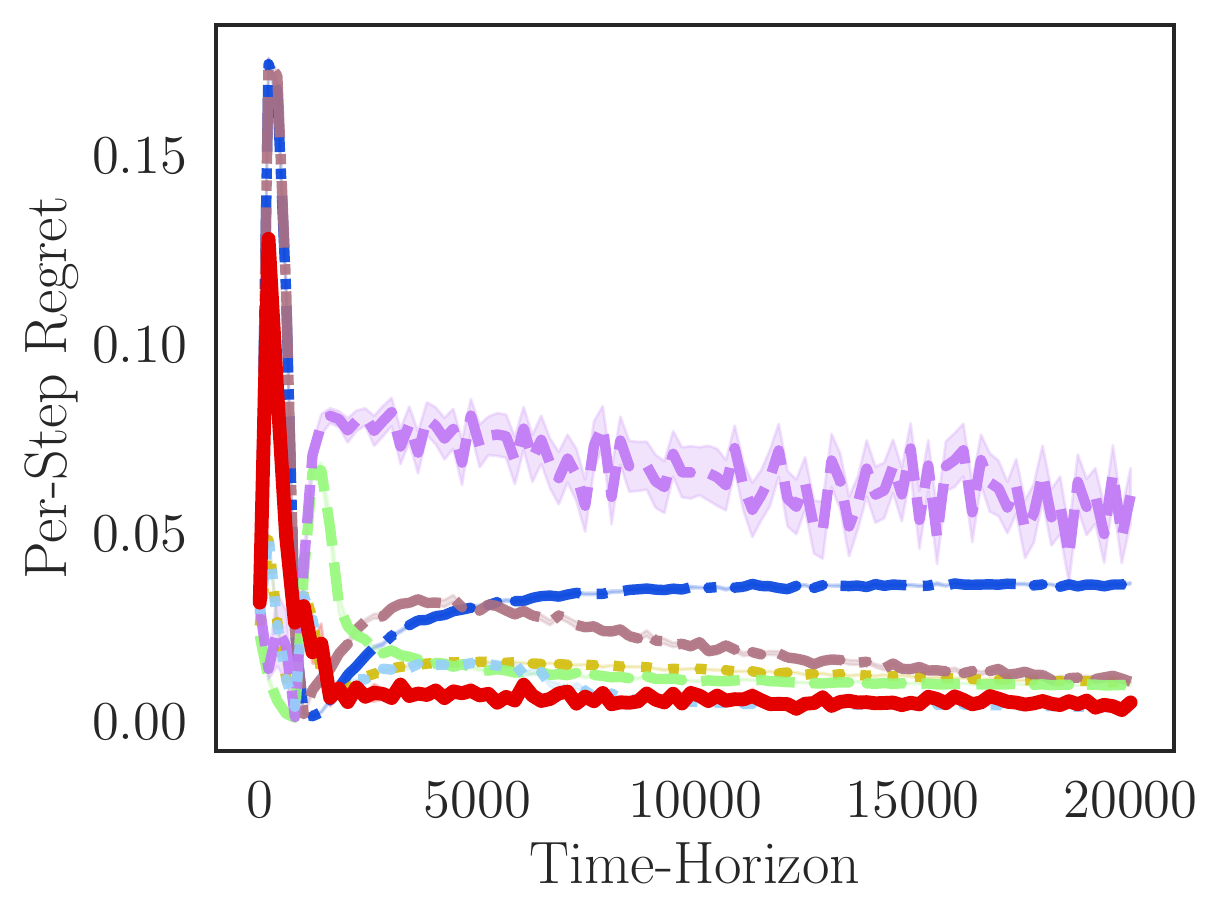}\hfill
     \includegraphics[width=0.32\linewidth]{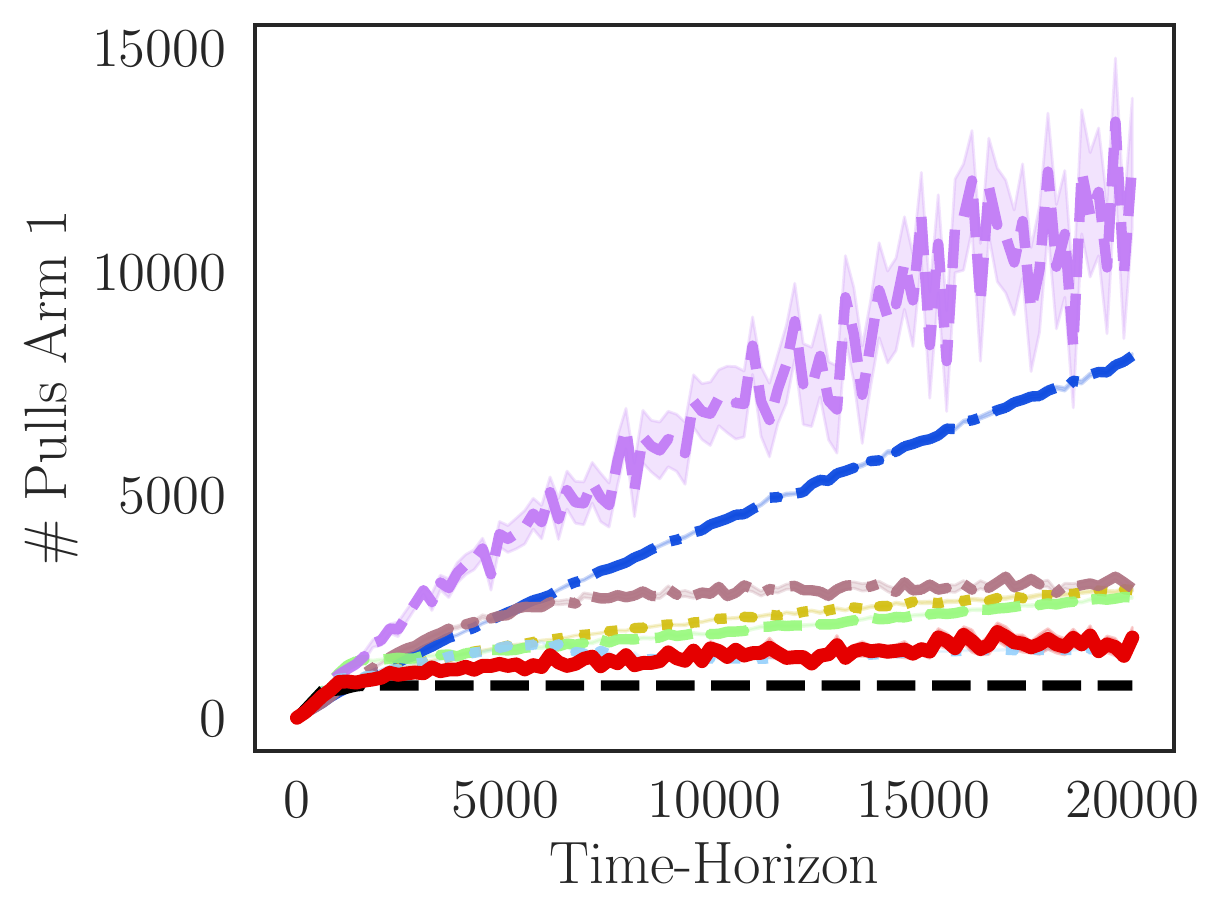}
     \caption{Noise with $\sigma = 0.05$}
  \end{subfigure}\vspace{1em}
  \resizebox{\linewidth}{!}{\begin{minipage}{\linewidth}
     \begin{align*}
        f_1(t) = &-0.0015 \cdot e^{-0.01 \cdot (t-600)} + \frac{-0.5}{e^{-(0.011) \cdot (t-600)} + 1} + 1 \\
        f_2(t) = &-0.005 \cdot e^{-0.009 \cdot (t-500)} + \frac{-0.2}{e^{-(0.0099) \cdot (t-500)} + 1} + 0.8
     \end{align*}
  \end{minipage}}
   \caption{The left-hand plots show the single-peaked reward functions $f_1$ and $f_2$ with simulated Gaussian noise. The middle plots show the per-step policy regret achieved by \algshort ({\protect\legendSPO}) compared to EXP3 ({\protect\legendEXP}), R-EXP3 ({\protect\legendREXP}), D-UCB ({\protect\legendDUCB}), SW-UCB ({\protect\legendSWUCB}), a one-step-optimistic ({\protect\legendOSO}), and a greedy algorithm ({\protect\legendGREEDY}). The right-hand plots show the policies these algorithms choose in comparison to the optimal policy ({\protect\legendOPTIMAL}).}
   \label{fig:experiment_inc_dec_3}
\end{minipage}\hfill
\begin{minipage}[b]{.48\textwidth}
  \centering
  \begin{subfigure}[c]{0.4\linewidth}
     \includegraphics[width=\linewidth]{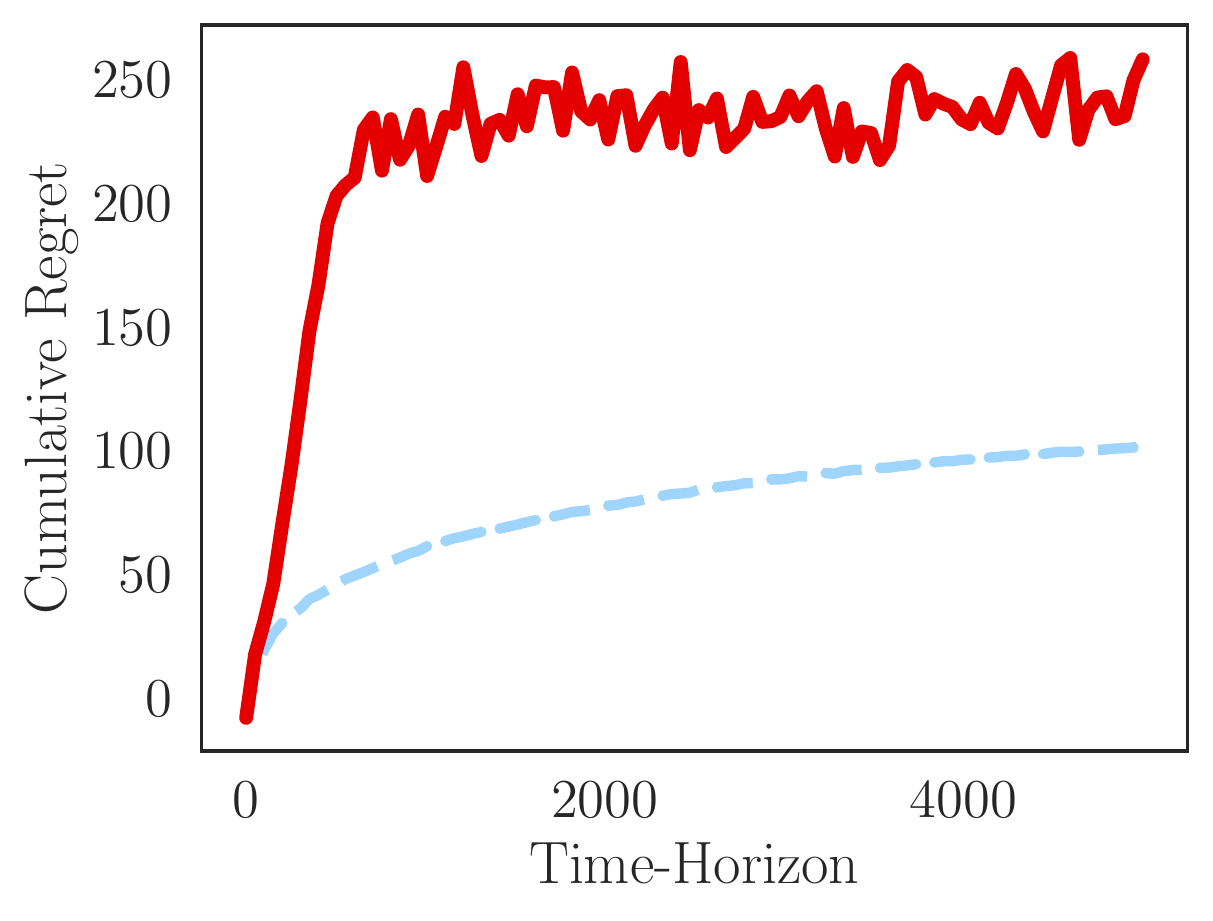}
     \caption{$\sigma = 0.01$}
  \end{subfigure}\vspace{1em}\\
  \begin{subfigure}[c]{0.4\linewidth}
     \includegraphics[width=\linewidth]{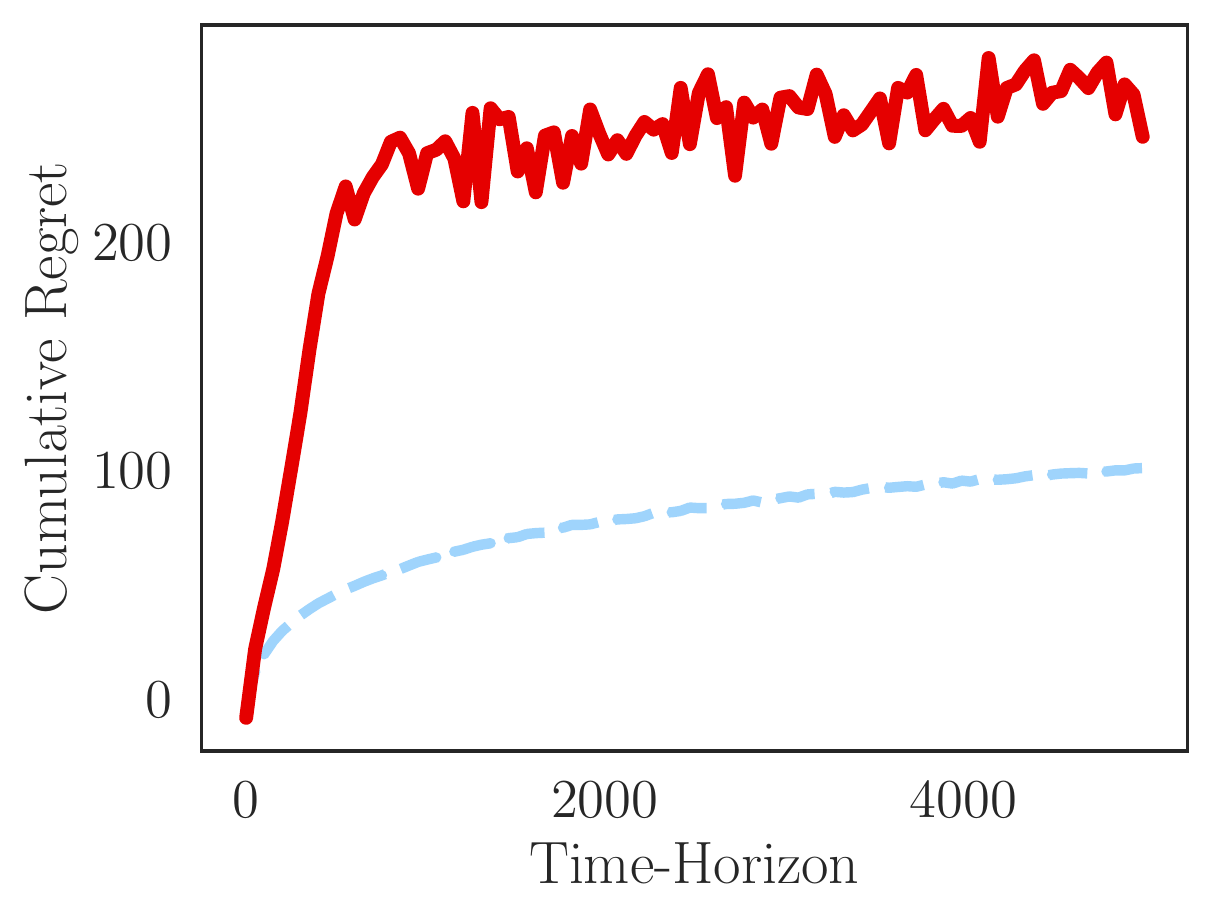}
     \caption{$\sigma = 0.05$}
  \end{subfigure}\vspace{1em}\\
  \begin{subfigure}[c]{0.4\linewidth}
     \includegraphics[width=\linewidth]{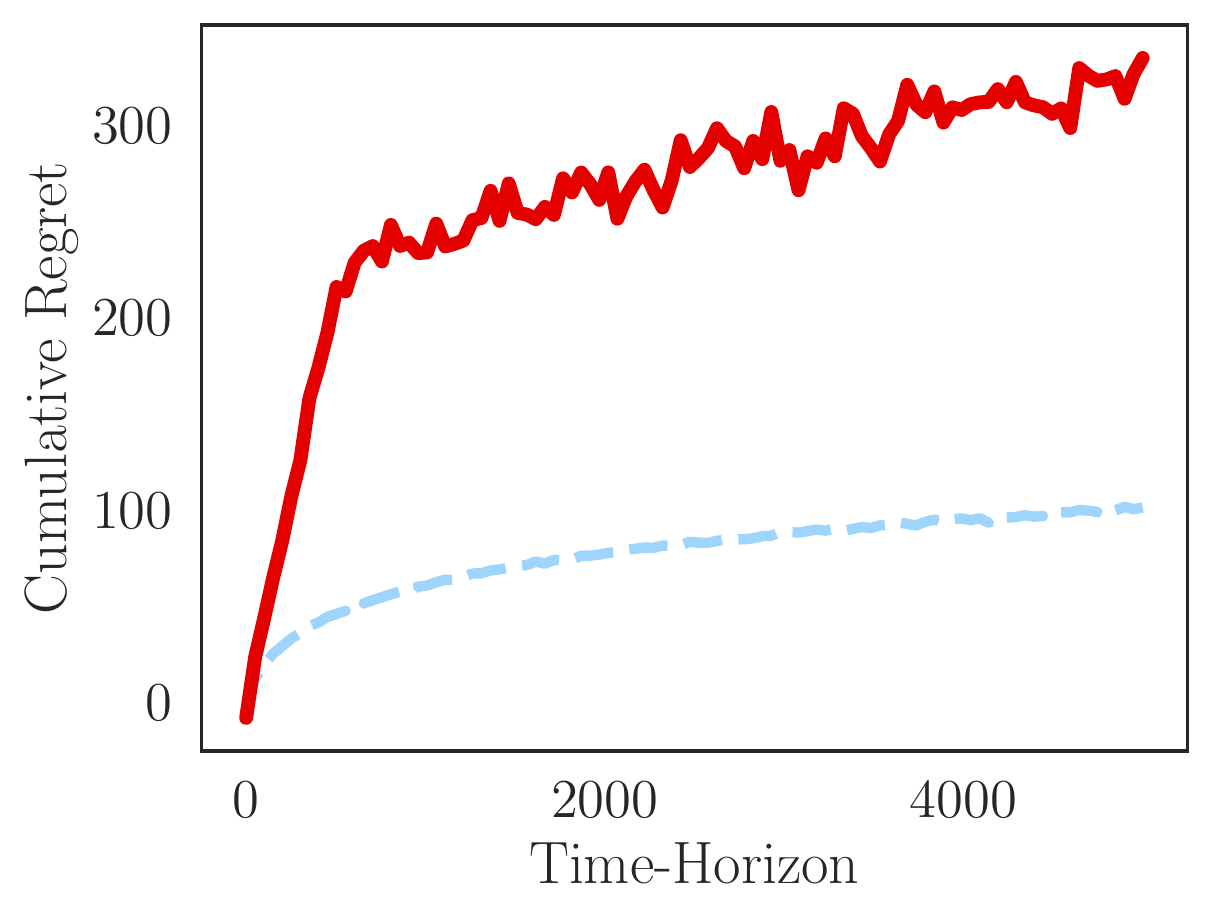}
     \caption{$\sigma = 0.1$}
  \end{subfigure}\vspace{1em}
   \caption{Regret of SPO (\legendSPO) compared to UCB (\legendUCB) on a Gaussian multi-armed bandit for different time horizons. We show the mean of the regret over 30 instances of MABs with 10 arms with means uniformly sampled between $0$ and $1$.}
   \label{fig:experiment_const_1}
\end{minipage}
\end{figure*}

\begin{figure*}[p]
\centering
\begin{minipage}[b]{.48\textwidth}
  \centering
  \begin{subfigure}[c]{\linewidth}
     \includegraphics[width=0.48\linewidth]{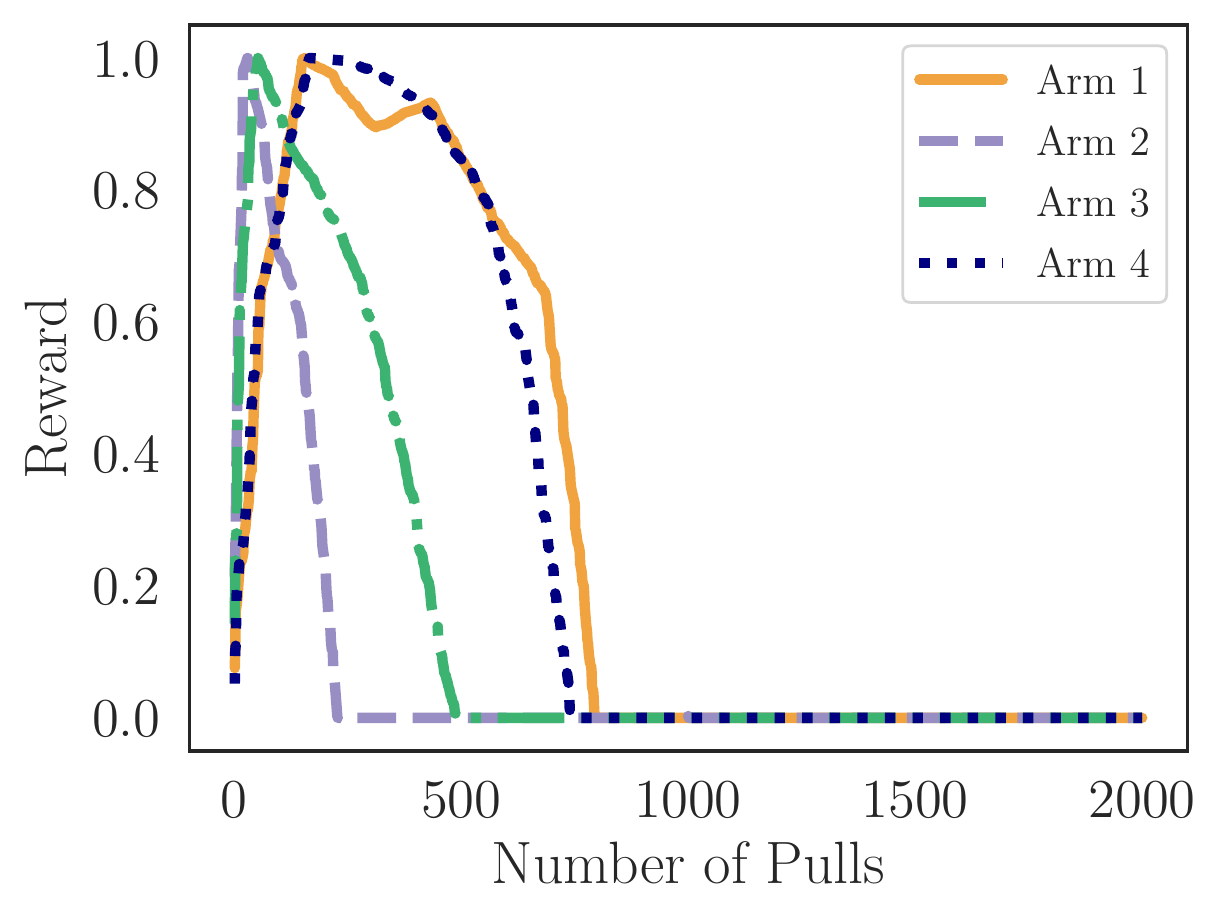}\hfill
     \includegraphics[width=0.48\linewidth]{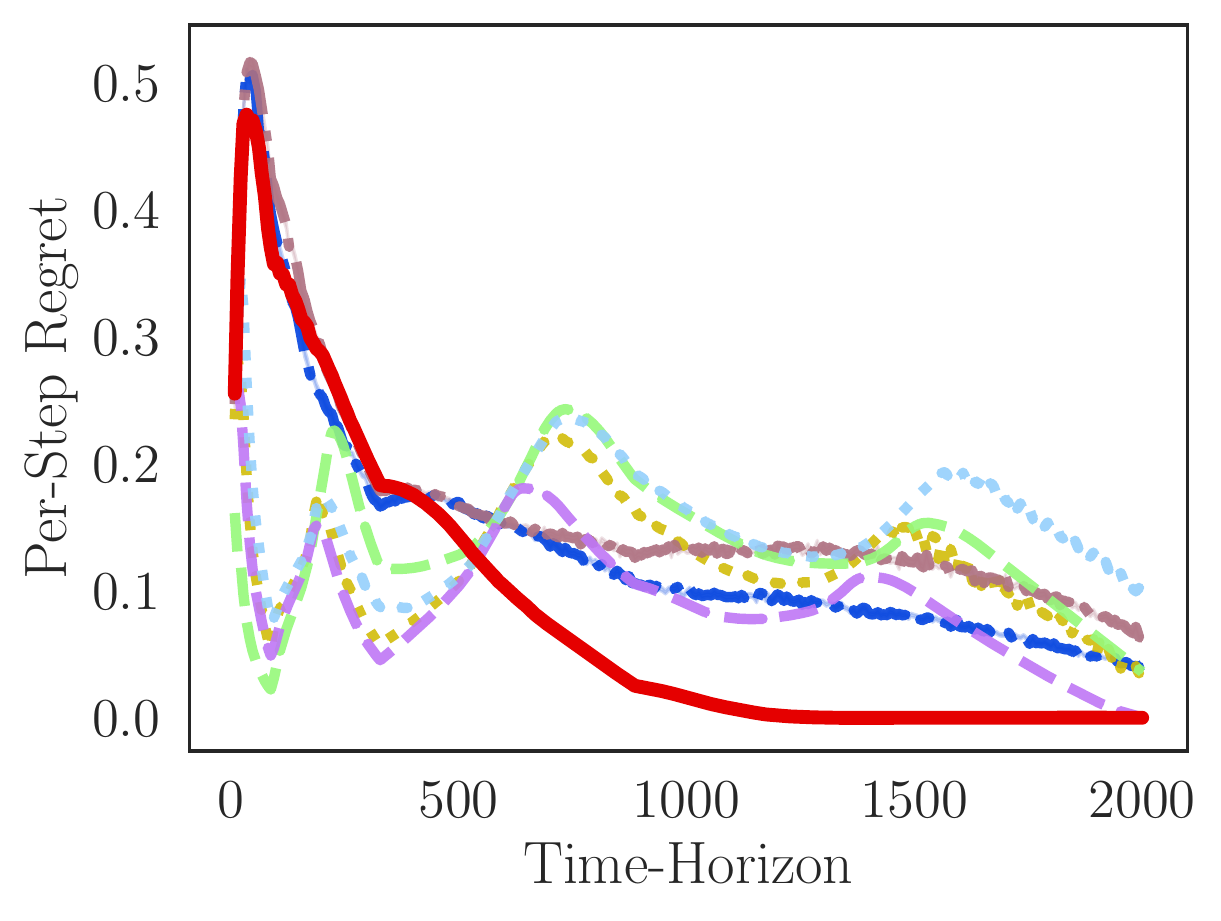}
     \caption{Noise-free observations}
  \end{subfigure}\vspace{1em}
  \begin{subfigure}[c]{\linewidth}
     \includegraphics[width=0.48\linewidth]{images/experiments/fico_all_score_change_gaussian_noise_0.01_arms.pdf}\hfill
     \includegraphics[width=0.48\linewidth]{images/experiments/fico_all_score_change_gaussian_noise_0.01_regret.pdf}
     \caption{Noise with $\sigma = 0.01$}
  \end{subfigure}\vspace{1em}
  \begin{subfigure}[c]{\linewidth}
     \includegraphics[width=0.48\linewidth]{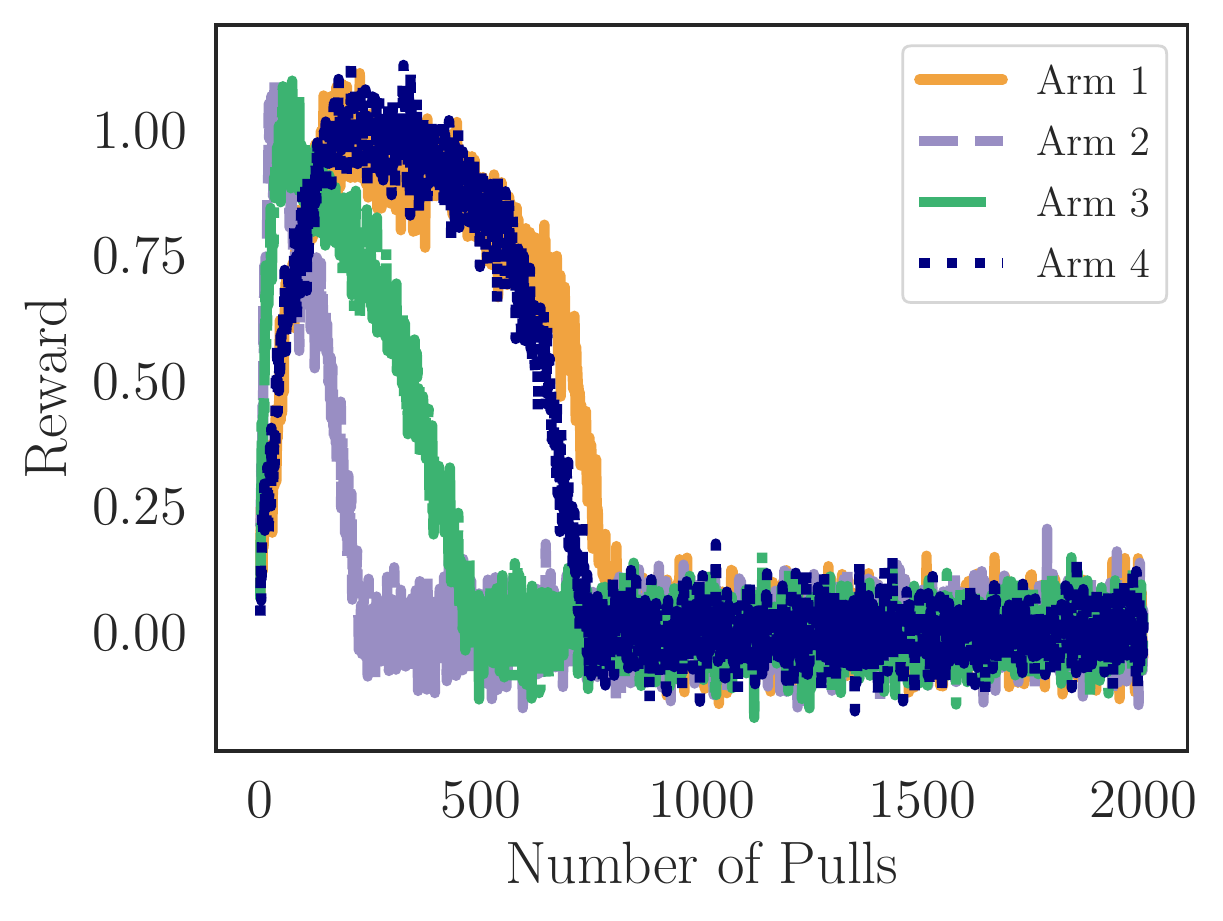}\hfill
     \includegraphics[width=0.48\linewidth]{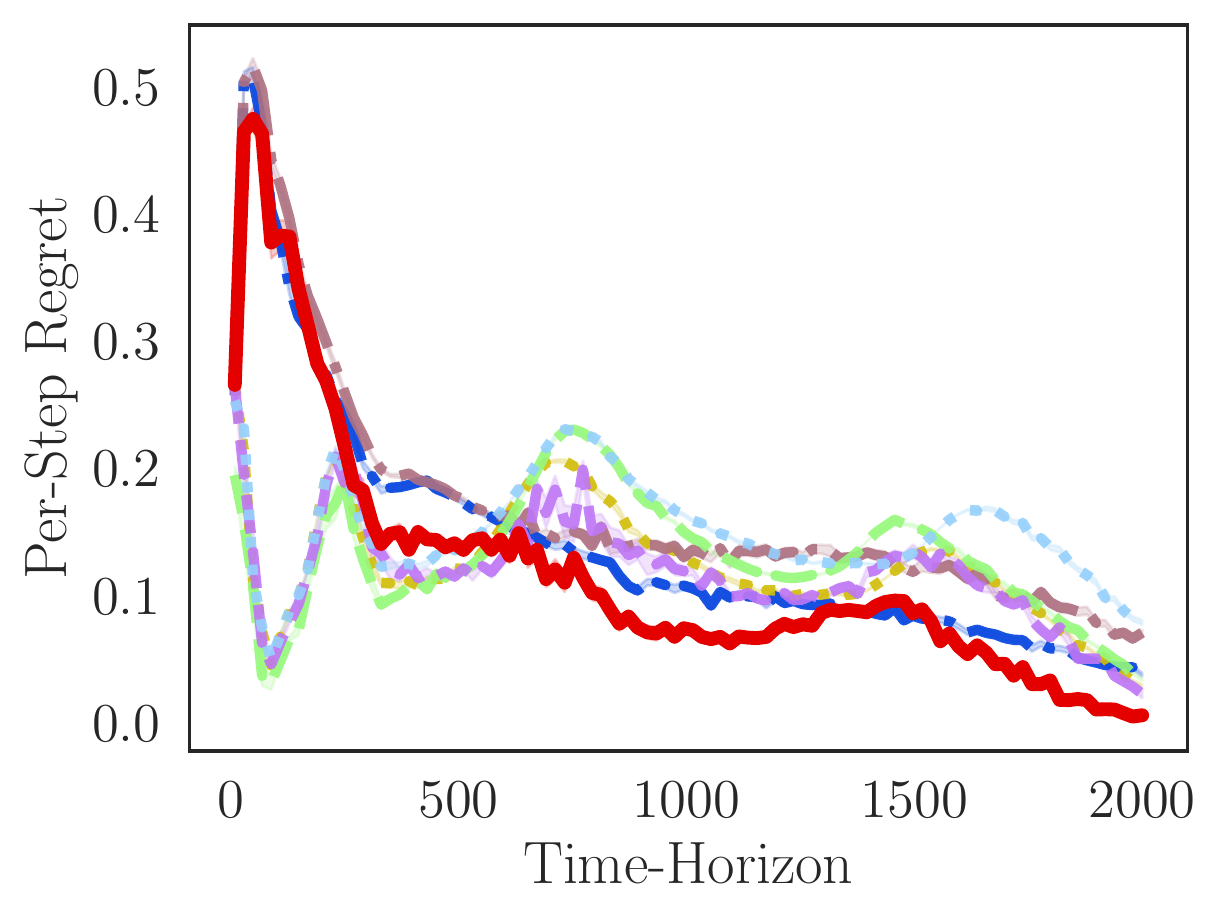}
     \caption{Noise with $\sigma = 0.05$}
  \end{subfigure}\vspace{1em}
   \caption{The left-hand plots show the reward functions defined as the utility of receiving a loan for different social groups using the FICO dataset, as described in the main text. We add simulated Gaussian noise with different variances. The right-hand plots show the per-step policy regret achieved by \algshort ({\protect\legendSPO}) compared to EXP3 ({\protect\legendEXP}), R-EXP3 ({\protect\legendREXP}), D-UCB ({\protect\legendDUCB}), SW-UCB ({\protect\legendSWUCB}), a one-step-optimistic ({\protect\legendOSO}), and a greedy algorithm ({\protect\legendGREEDY}).}
   \label{fig:experiment_fico_score_change}
\end{minipage}\hfill
\begin{minipage}[b]{.48\textwidth}
  \centering
  \begin{subfigure}[c]{\linewidth}
     \includegraphics[width=0.48\linewidth]{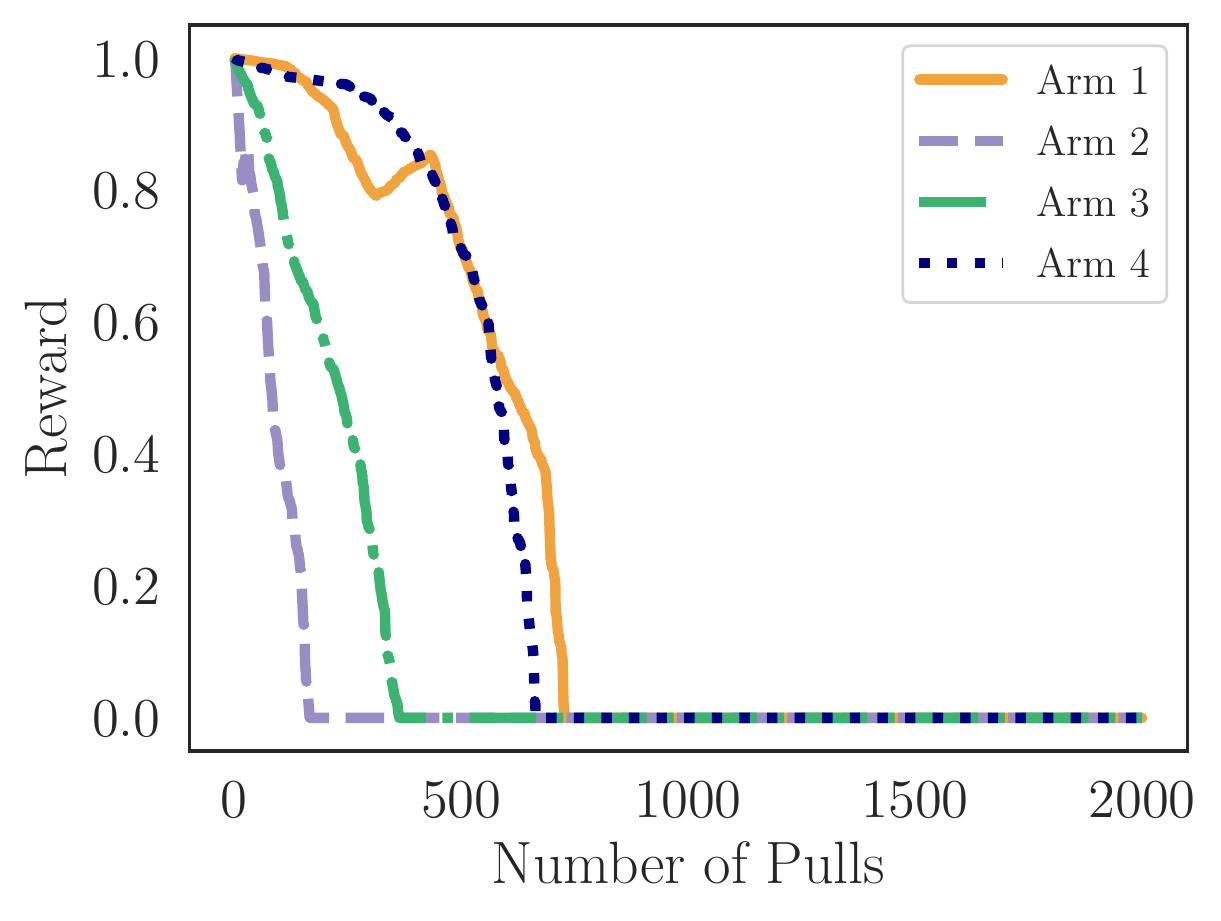}\hfill
     \includegraphics[width=0.48\linewidth]{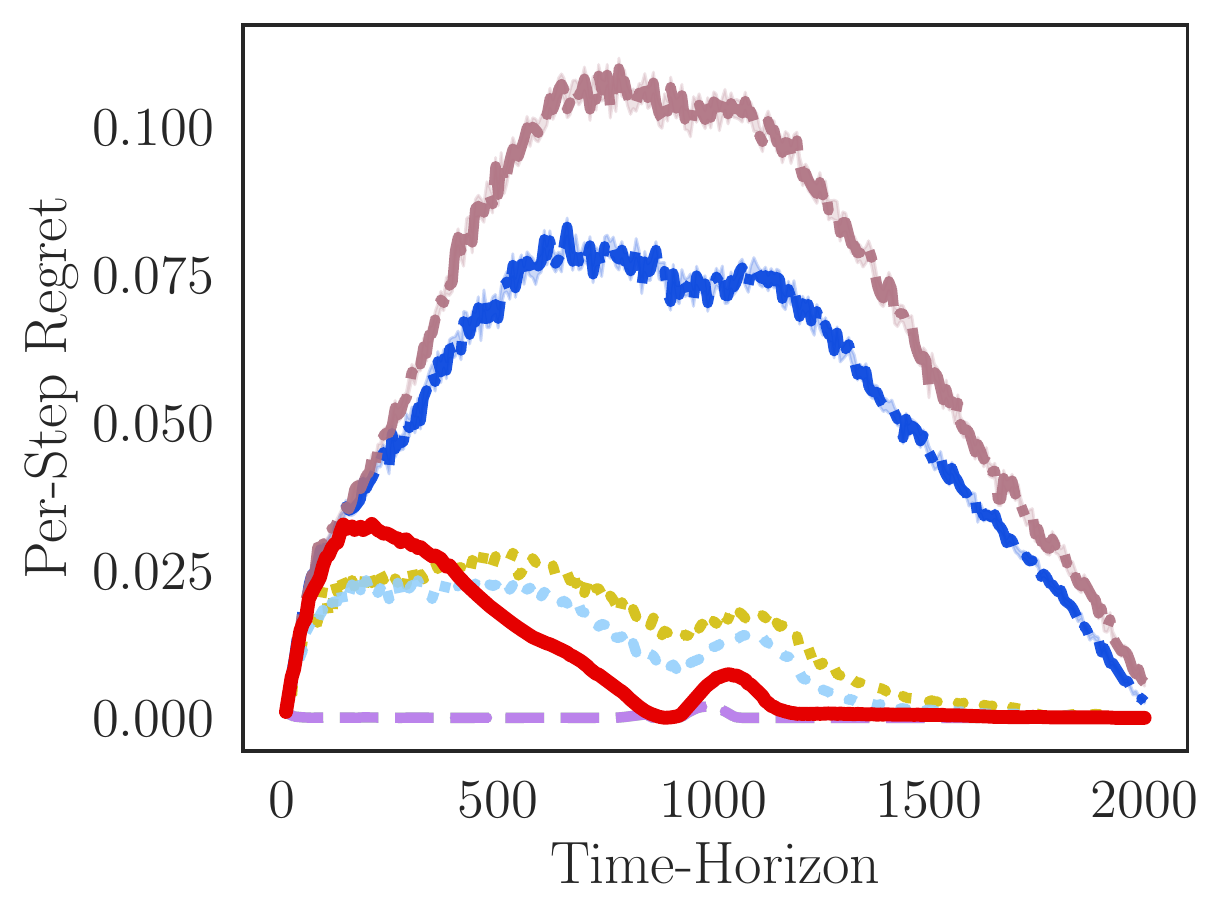}
     \caption{Noise-free observations}
  \end{subfigure}\vspace{1em}
  \begin{subfigure}[c]{\linewidth}
     \includegraphics[width=0.48\linewidth]{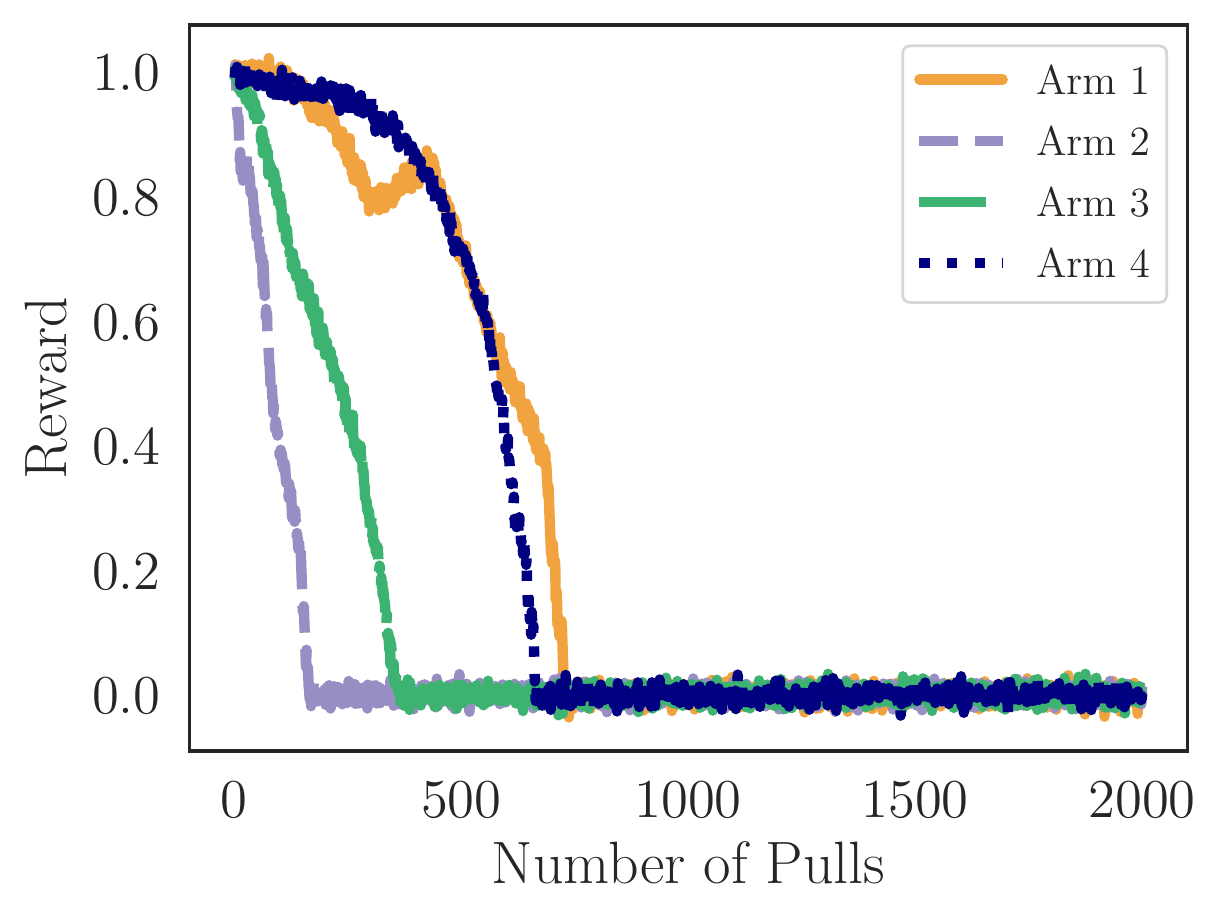}\hfill
     \includegraphics[width=0.48\linewidth]{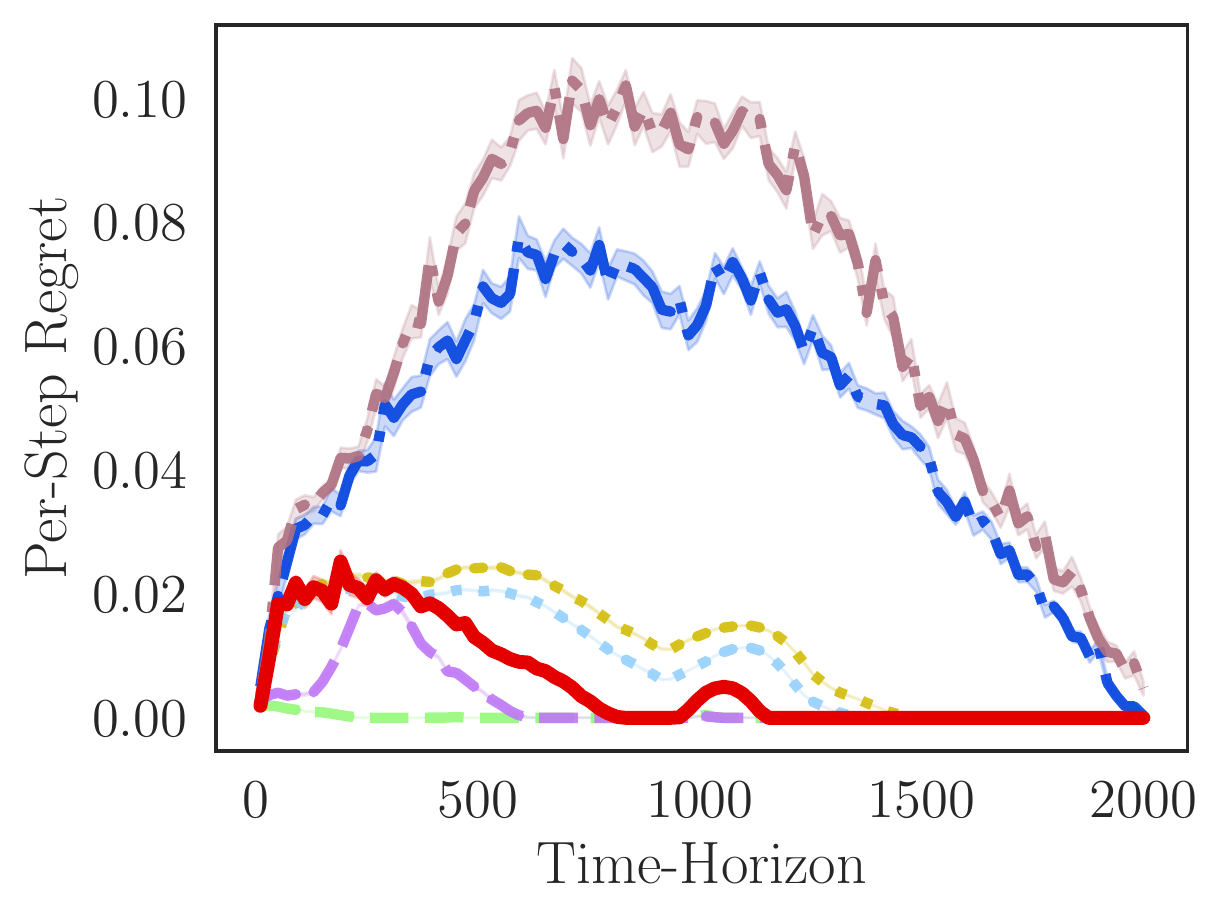}
     \caption{Noise with $\sigma = 0.01$}
  \end{subfigure}\vspace{1em}
  \begin{subfigure}[c]{\linewidth}
     \includegraphics[width=0.48\linewidth]{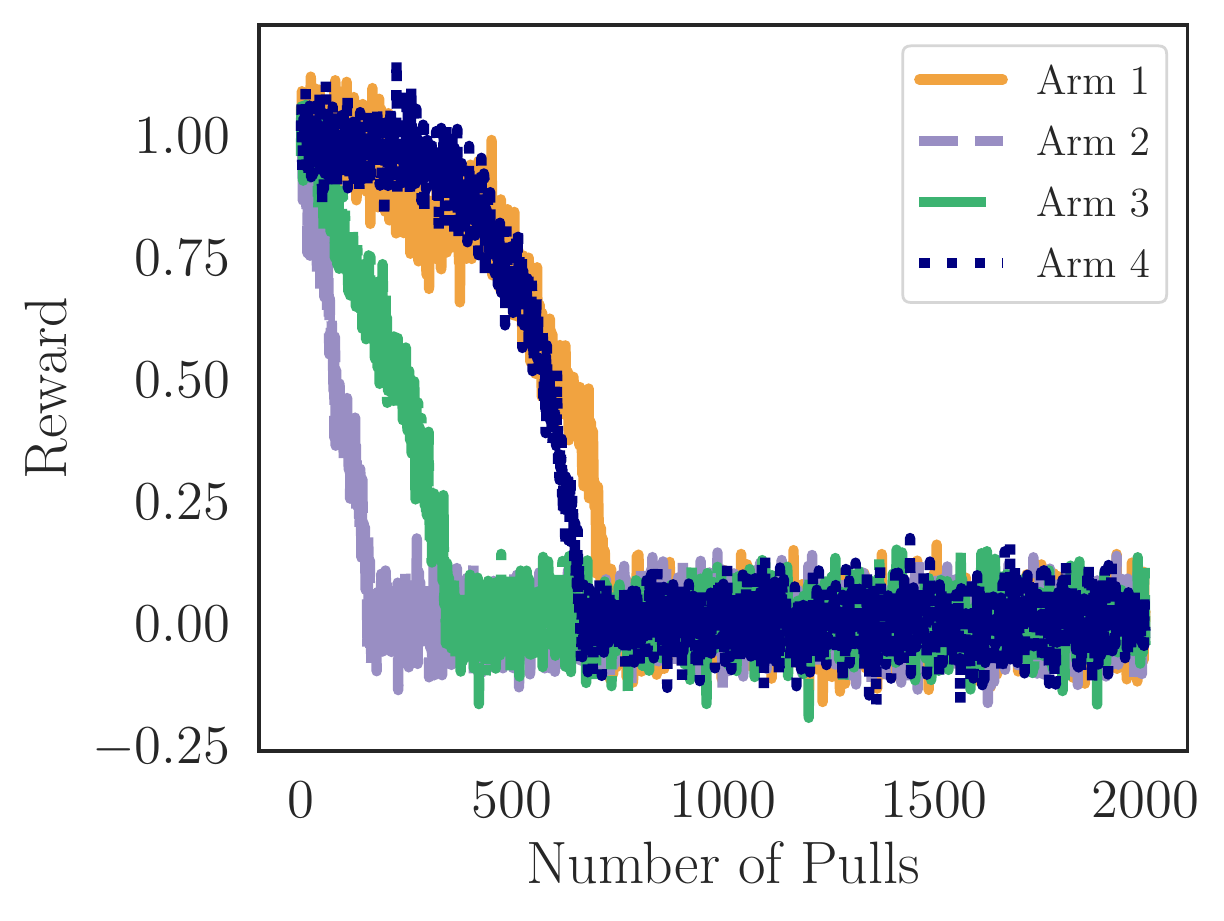}\hfill
     \includegraphics[width=0.48\linewidth]{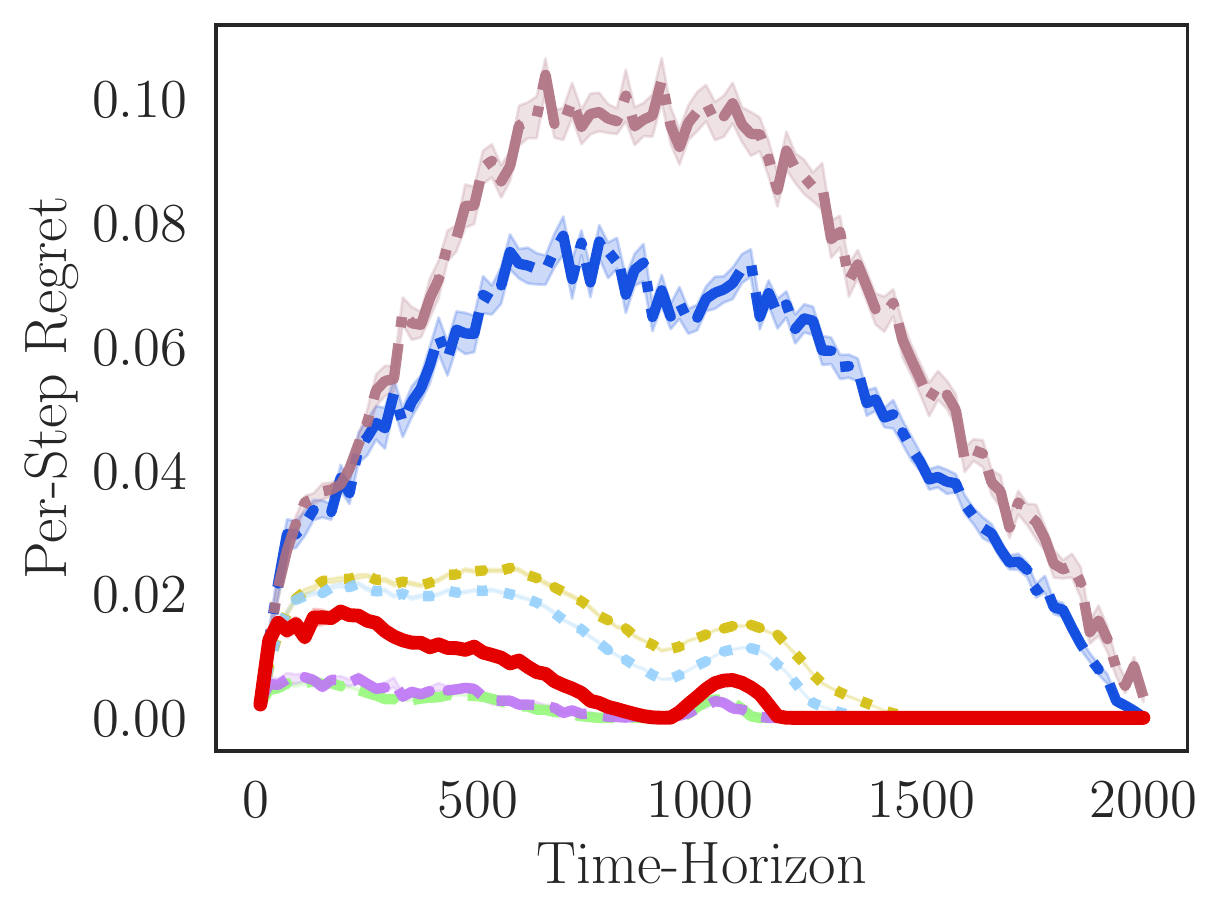}
     \caption{Noise with $\sigma = 0.05$}
  \end{subfigure}\vspace{1em}
   \caption{The left-hand plots show the reward functions defined as the bank's utility of giving a loan. We add simulated Gaussian noise with different variances. The right-hand plots show the per-step policy regret achieved by \algshort ({\protect\legendSPO}) compared to EXP3 ({\protect\legendEXP}), R-EXP3 ({\protect\legendREXP}), D-UCB ({\protect\legendDUCB}), SW-UCB ({\protect\legendSWUCB}), a one-step-optimistic ({\protect\legendOSO}), and a greedy algorithm ({\protect\legendGREEDY}).}
   \label{fig:experiment_fico_utility}
\end{minipage}
\end{figure*}

\begin{figure*}[p]
\centering
\begin{minipage}[b]{.48\textwidth}
  \centering
  \begin{subfigure}[c]{\linewidth}
     \includegraphics[width=0.48\linewidth]{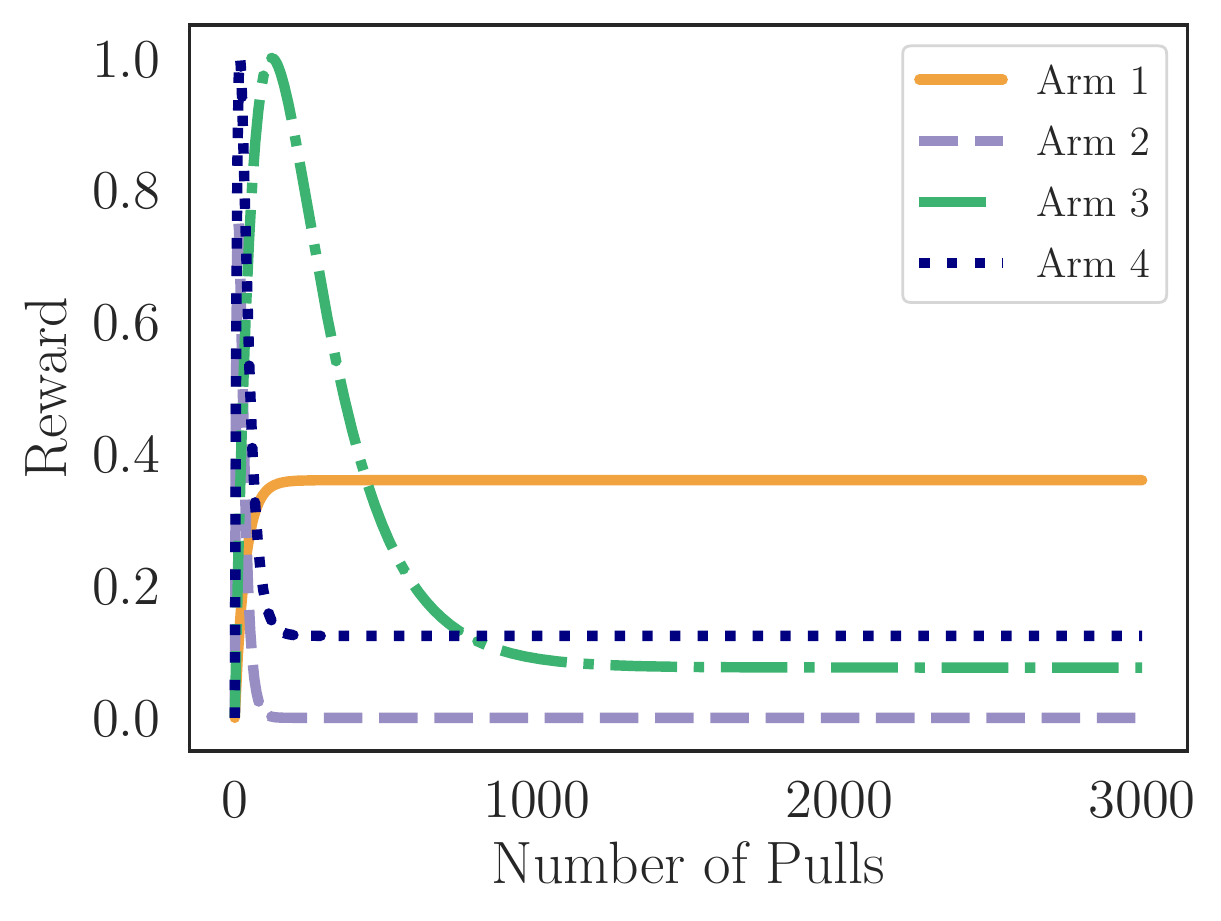}\hfill
     \includegraphics[width=0.48\linewidth]{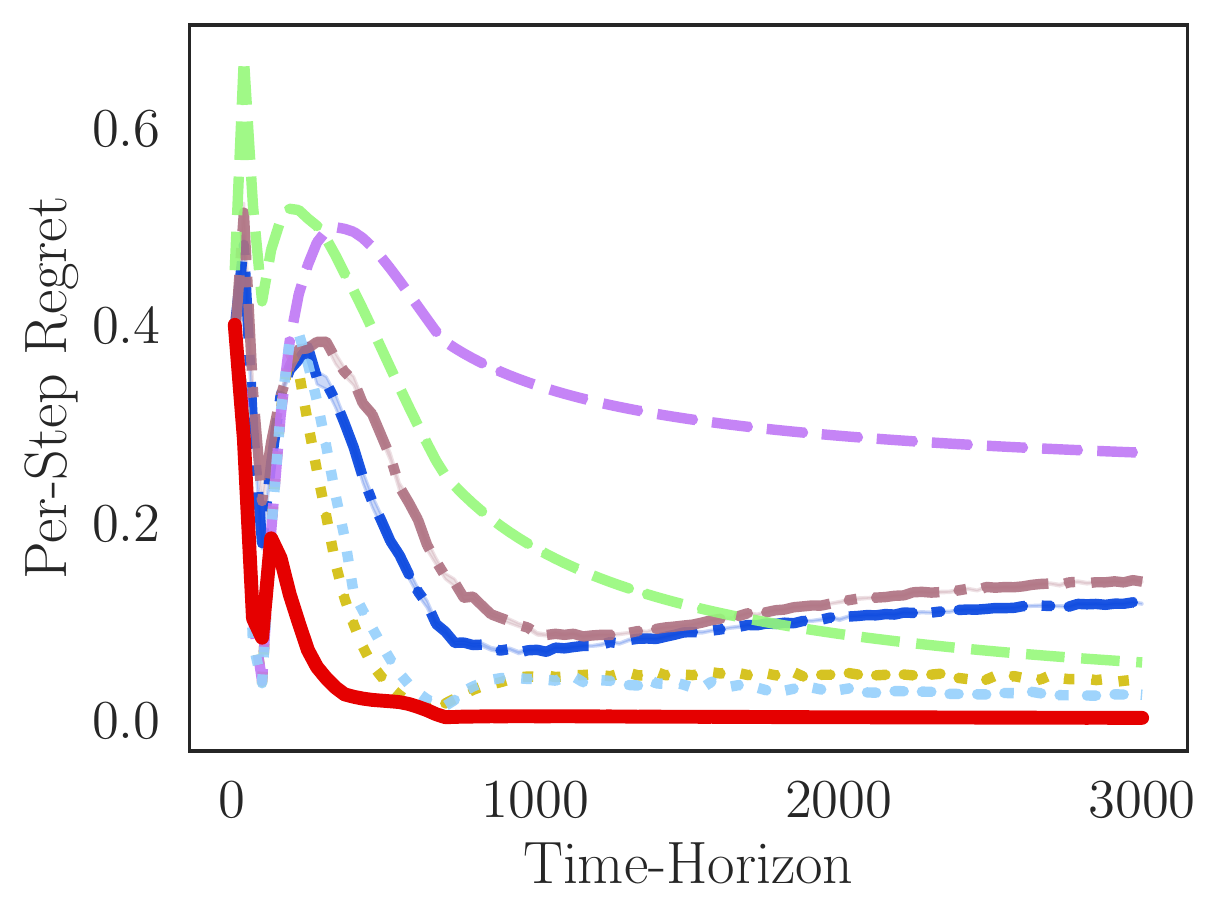}
     \caption{Noise-free observations}
  \end{subfigure}\vspace{1em}
  \begin{subfigure}[c]{\linewidth}
     \includegraphics[width=0.48\linewidth]{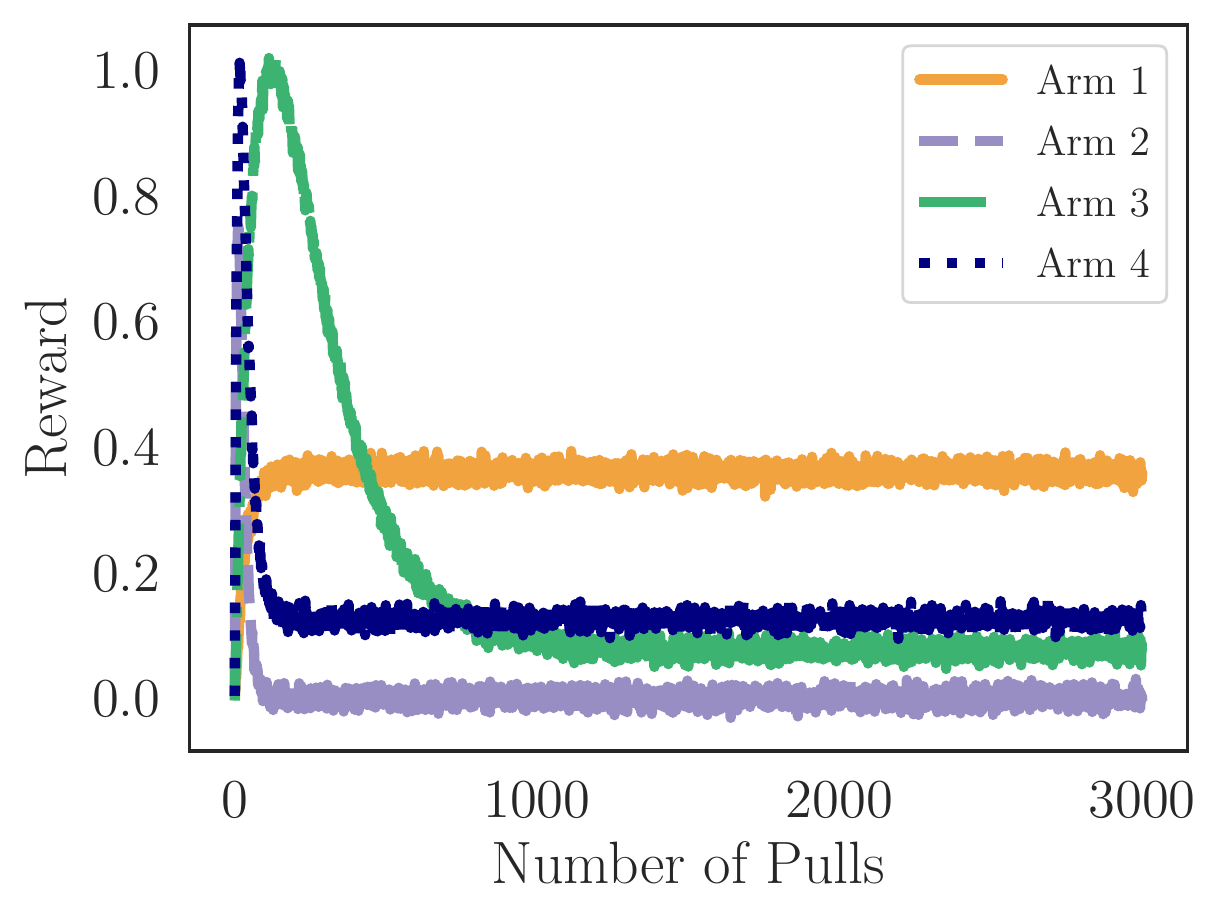}\hfill
     \includegraphics[width=0.48\linewidth]{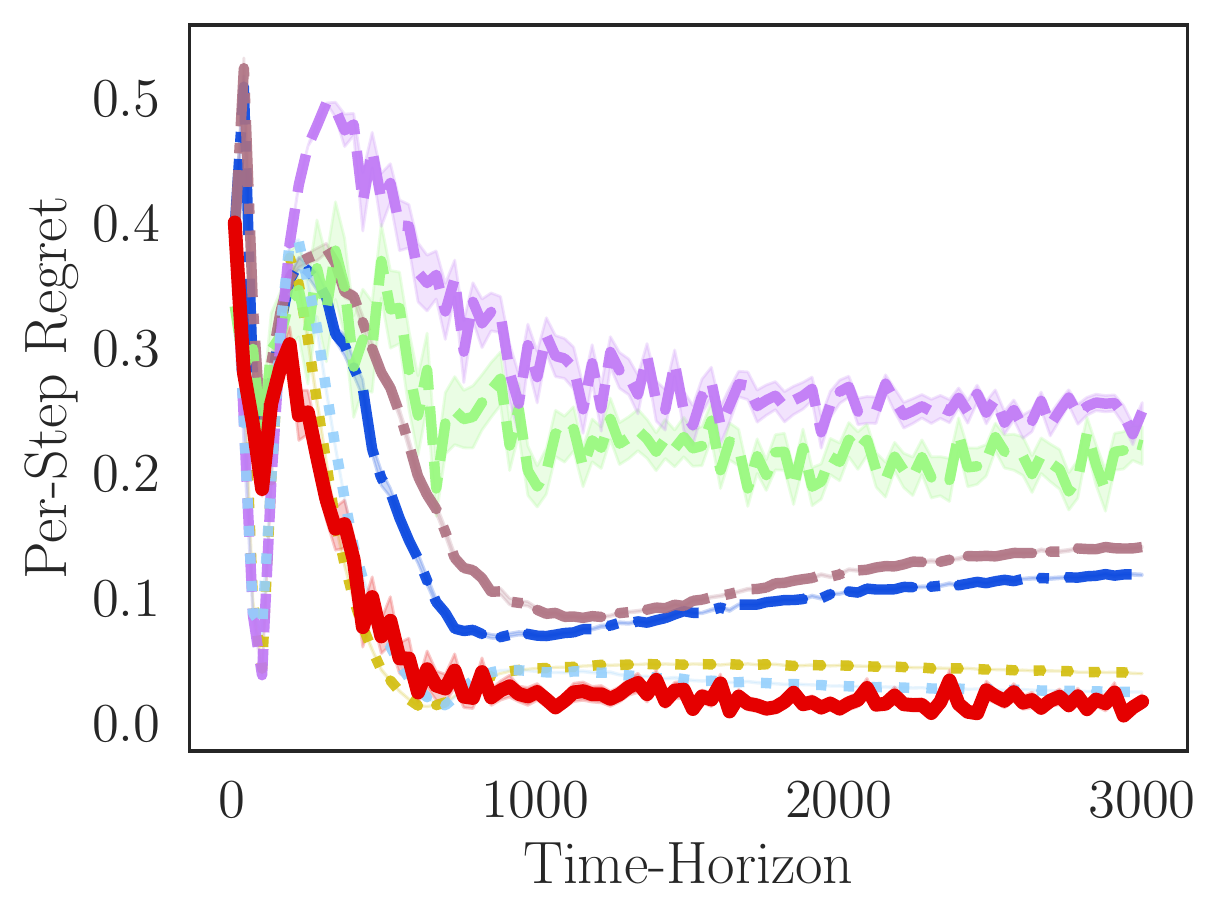}
     \caption{Noise with $\sigma = 0.01$}
  \end{subfigure}\vspace{1em}
  \begin{subfigure}[c]{\linewidth}
     \includegraphics[width=0.48\linewidth]{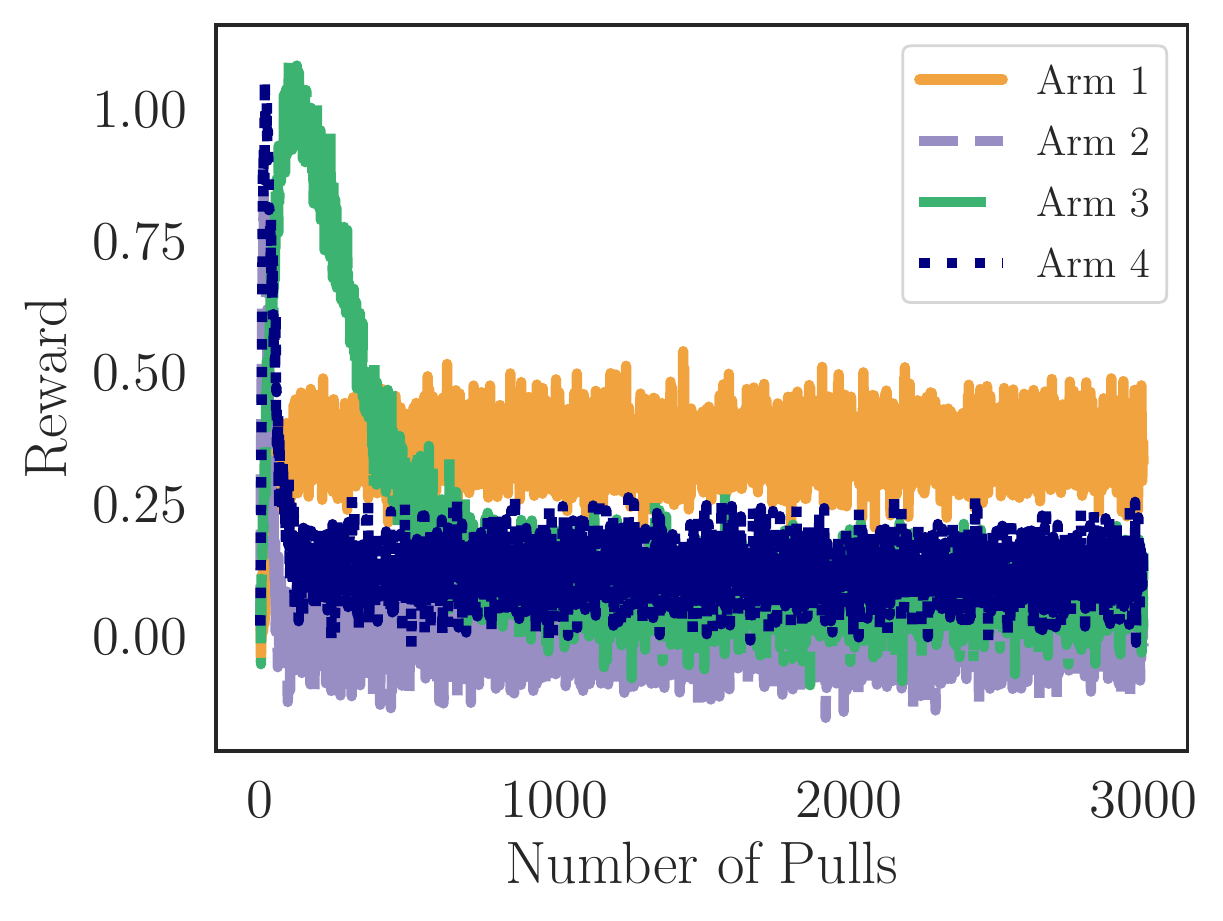}\hfill
     \includegraphics[width=0.48\linewidth]{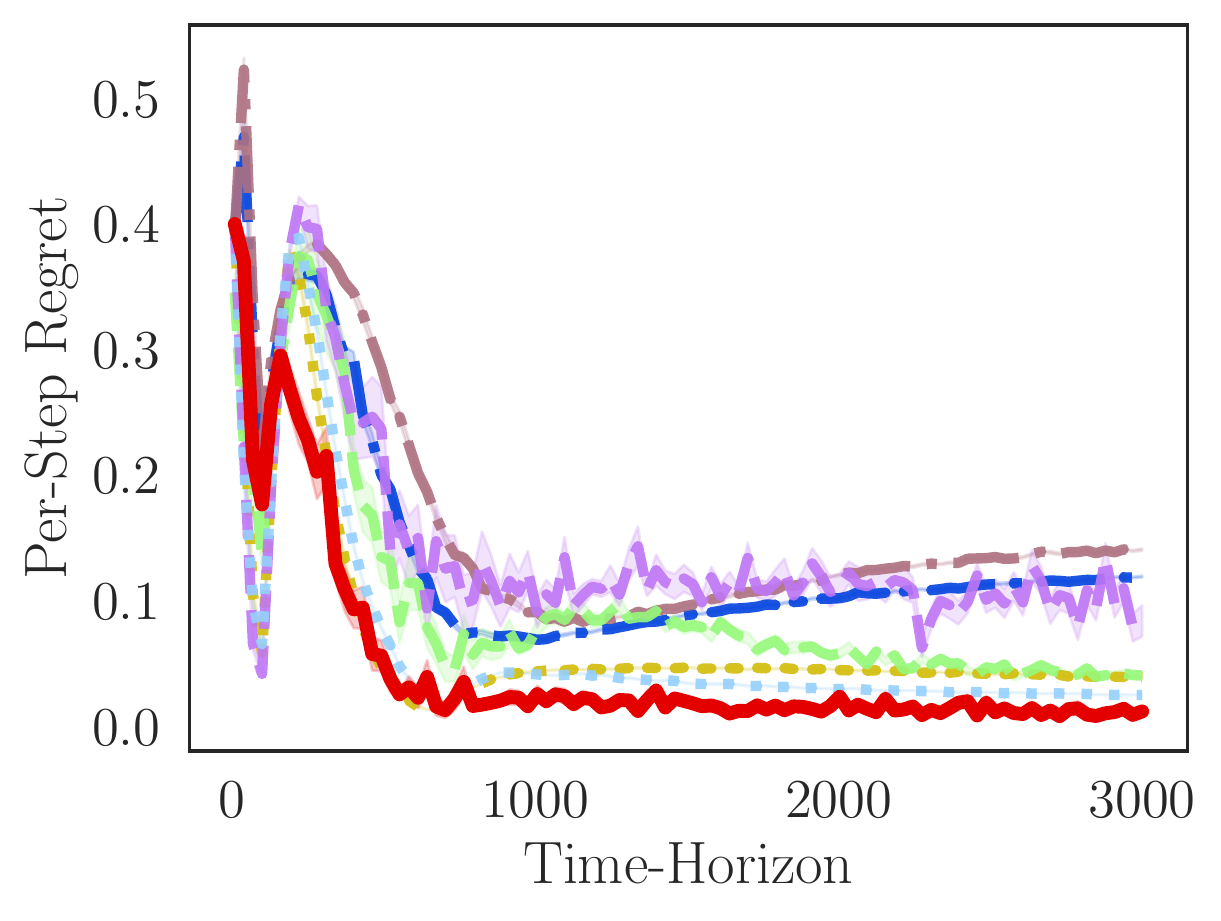}
     \caption{Noise with $\sigma = 0.05$}
      \end{subfigure}\vspace{1em}
   \caption{The left-hand plots show the reward function of instance A of the simulated recommender system with simulated Gaussian noise. The right-hand plots show the per-step policy regret achieved by \algshort ({\protect\legendSPO}) compared to EXP3 ({\protect\legendEXP}), R-EXP3 ({\protect\legendREXP}), D-UCB ({\protect\legendDUCB}), SW-UCB ({\protect\legendSWUCB}), a one-step-optimistic ({\protect\legendOSO}), and a greedy algorithm ({\protect\legendGREEDY}).}
   \label{fig:experiment_recommender_A}
\end{minipage}\hfill
\begin{minipage}[b]{.48\textwidth}
  \centering
  \begin{subfigure}[c]{\linewidth}
     \includegraphics[width=0.48\linewidth]{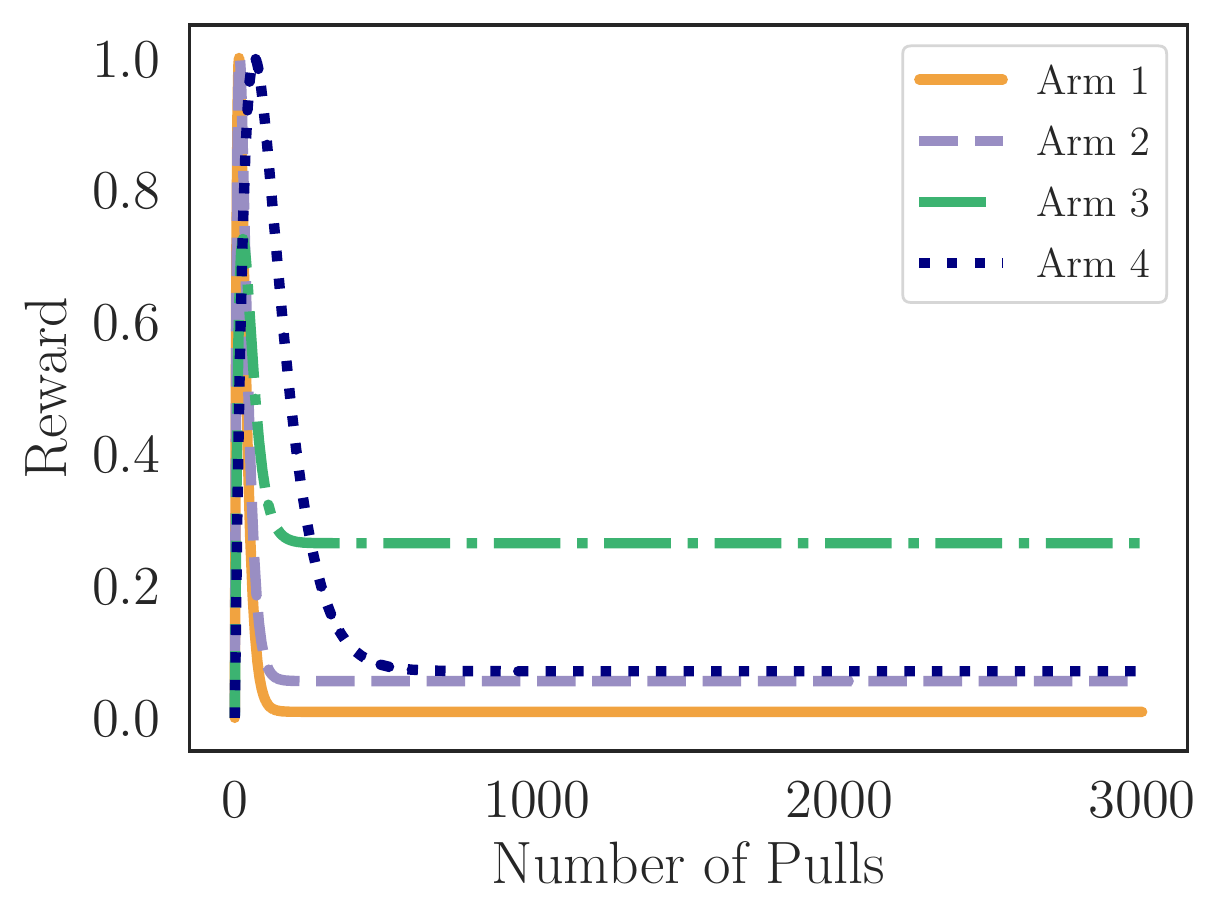}\hfill
     \includegraphics[width=0.48\linewidth]{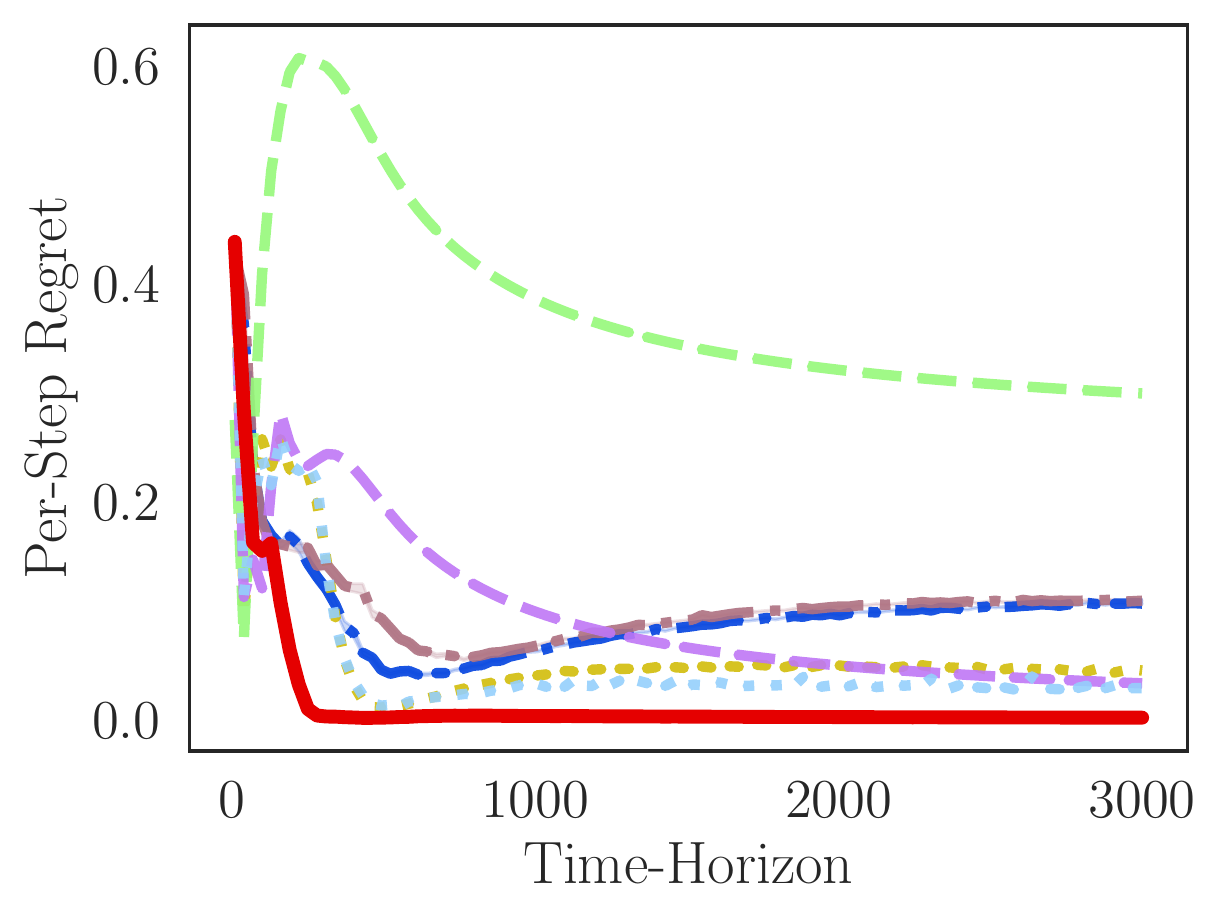}
     \caption{Noise-free observations}
  \end{subfigure}\vspace{1em}
  \begin{subfigure}[c]{\linewidth}
     \includegraphics[width=0.48\linewidth]{images/experiments/recommender_4_2_gaussian_noise_0.01_arms.pdf}\hfill
     \includegraphics[width=0.48\linewidth]{images/experiments/recommender_4_2_gaussian_noise_0.01_regret.pdf}
     \caption{Noise with $\sigma = 0.01$}
  \end{subfigure}\vspace{1em}
  \begin{subfigure}[c]{\linewidth}
     \includegraphics[width=0.48\linewidth]{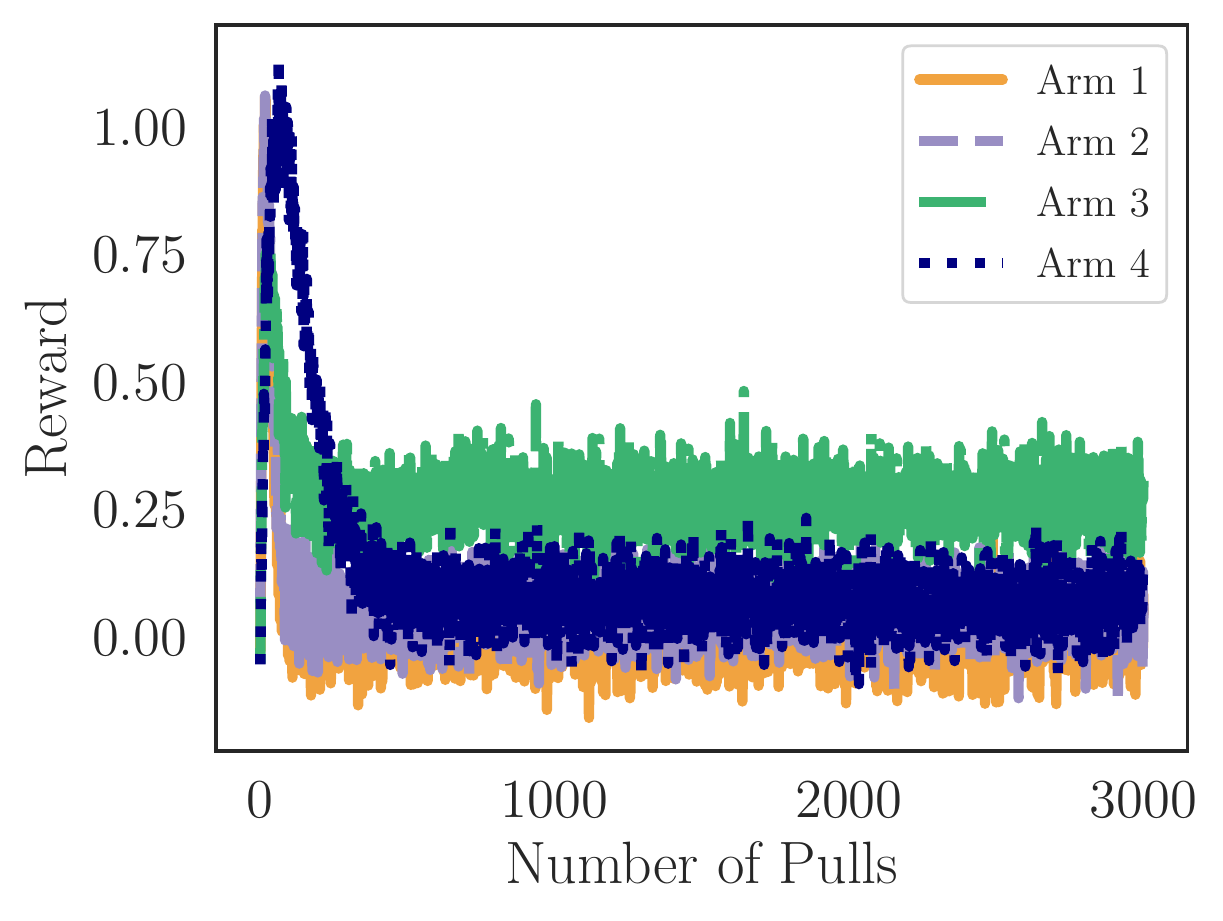}\hfill
     \includegraphics[width=0.48\linewidth]{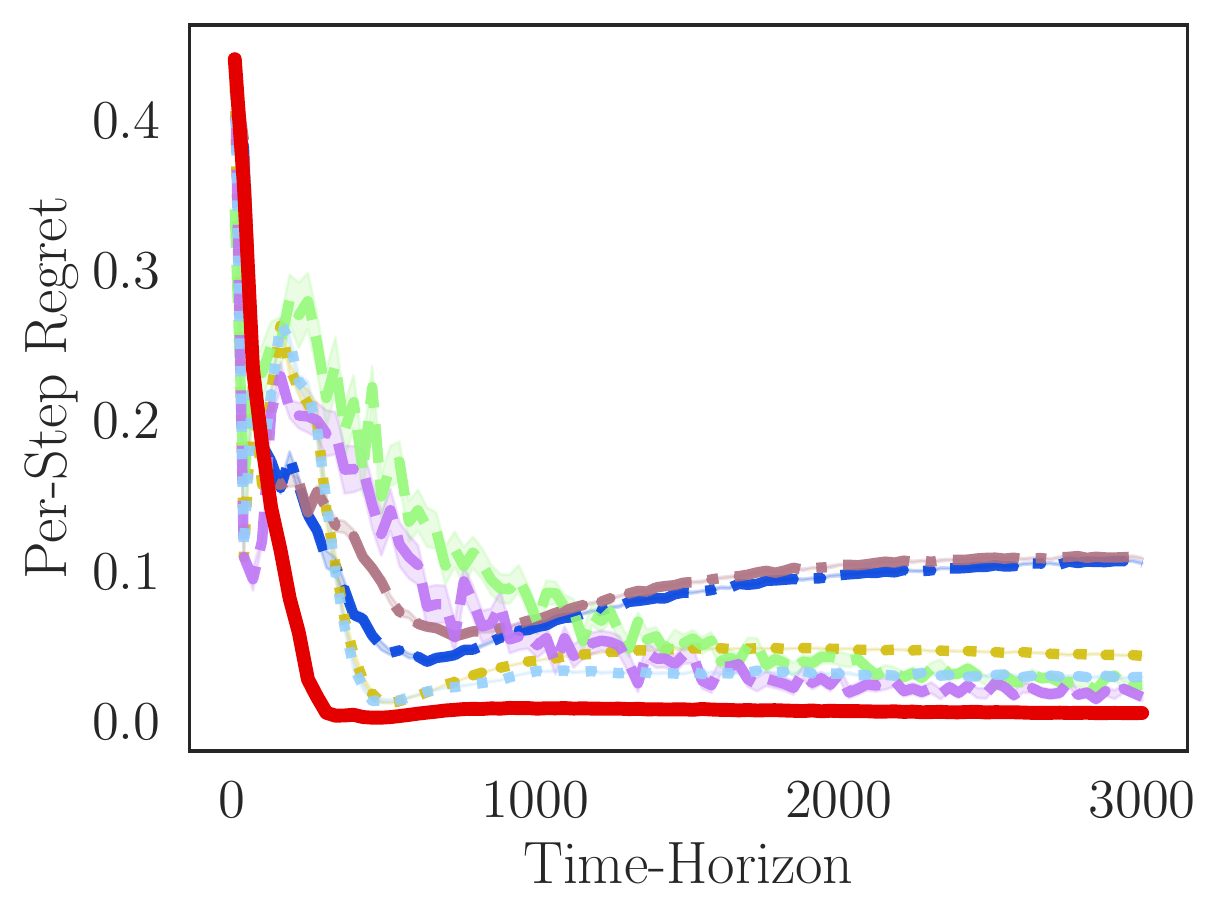}
     \caption{Noise with $\sigma = 0.05$}
  \end{subfigure}\vspace{1em}
   \caption{The left-hand plots show the reward function of instance B of the simulated recommender system with simulated Gaussian noise. The right-hand plots show the per-step policy regret achieved by \algshort ({\protect\legendSPO}) compared to EXP3 ({\protect\legendEXP}), R-EXP3 ({\protect\legendREXP}), D-UCB ({\protect\legendDUCB}), SW-UCB ({\protect\legendSWUCB}), a one-step-optimistic ({\protect\legendOSO}), and a greedy algorithm ({\protect\legendGREEDY}).}
   \label{fig:experiment_recommender_B}
\end{minipage}
\end{figure*}

\begin{figure*}[p]
\centering
\begin{minipage}[b]{.48\textwidth}
  \centering
  \begin{subfigure}[c]{\linewidth}
     \includegraphics[width=0.48\linewidth]{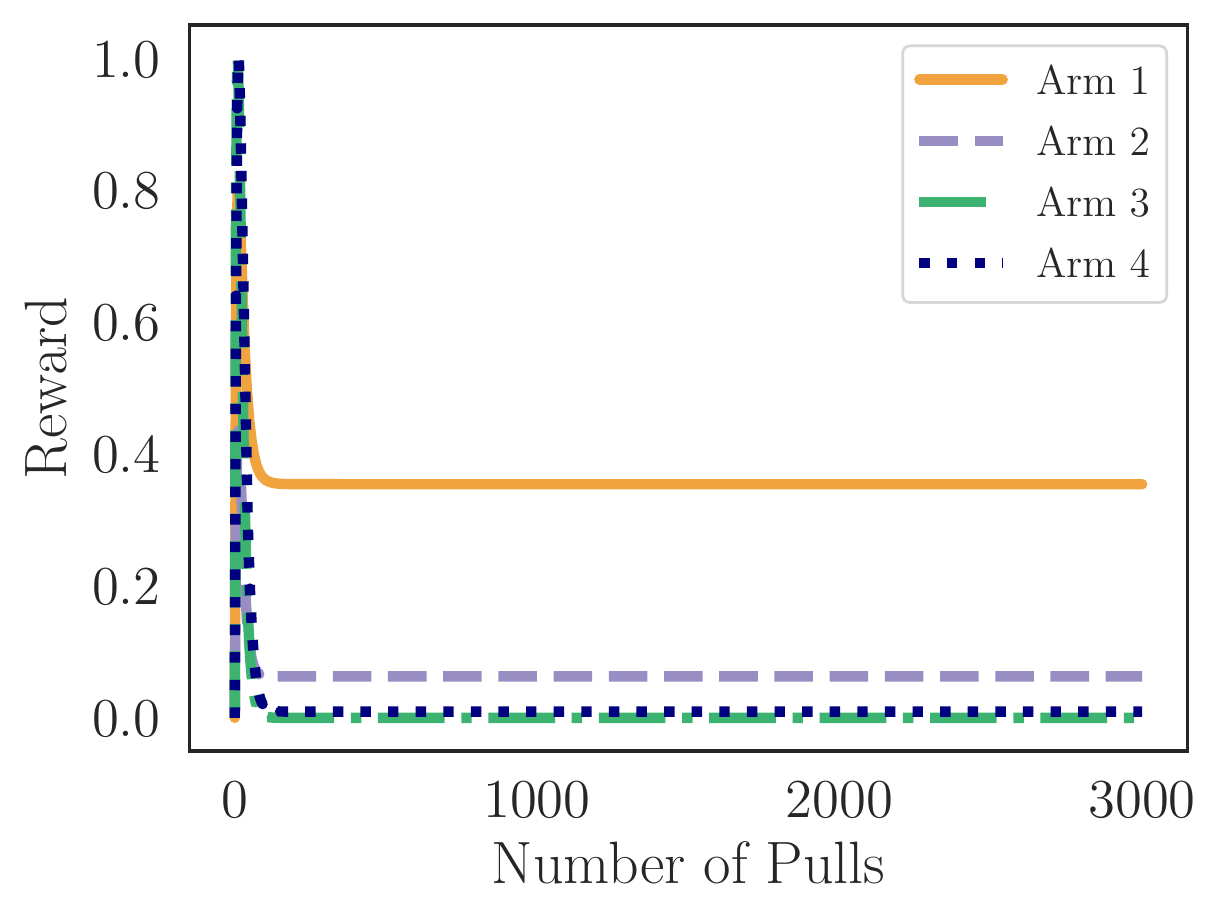}\hfill
     \includegraphics[width=0.48\linewidth]{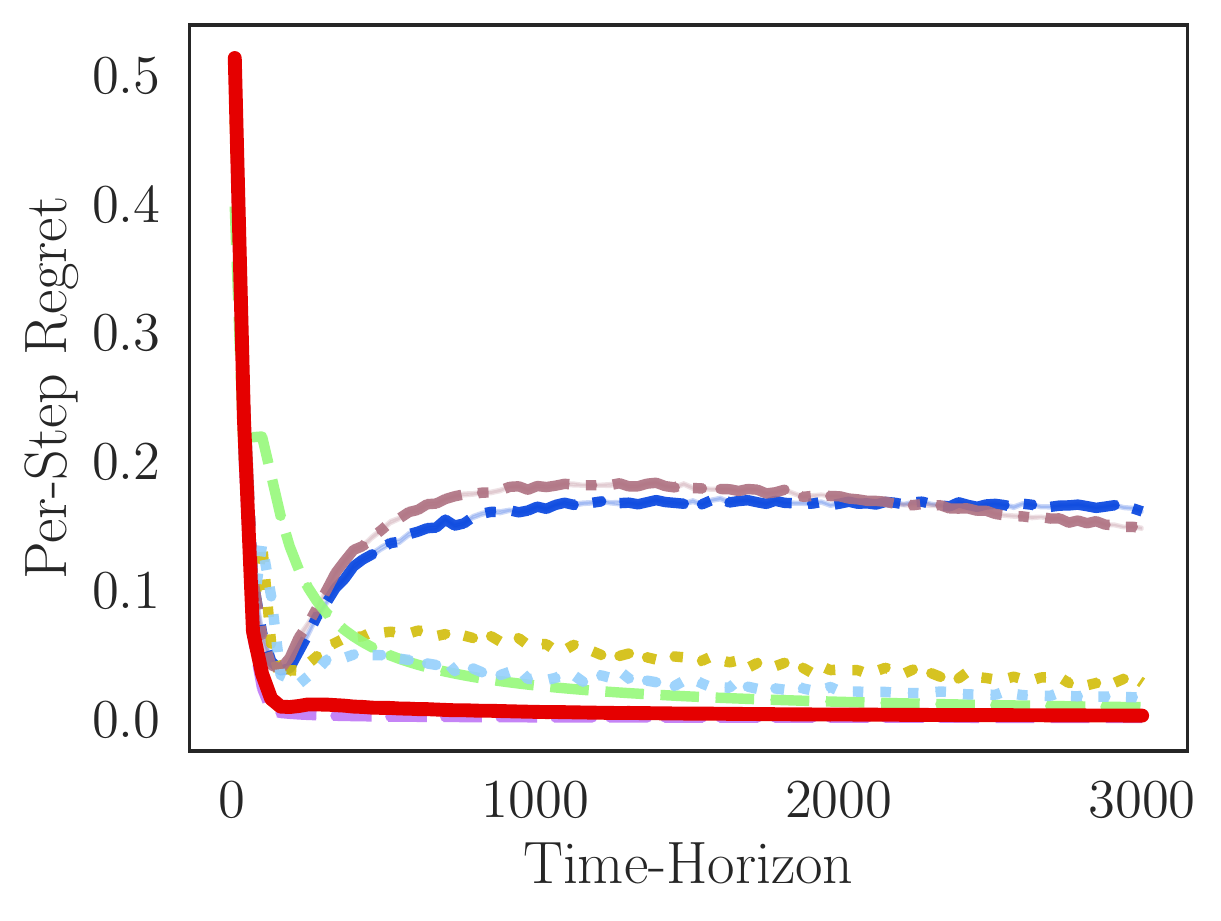}
     \caption{Noise-free observations}
  \end{subfigure}\vspace{1em}
  \begin{subfigure}[c]{\linewidth}
     \includegraphics[width=0.48\linewidth]{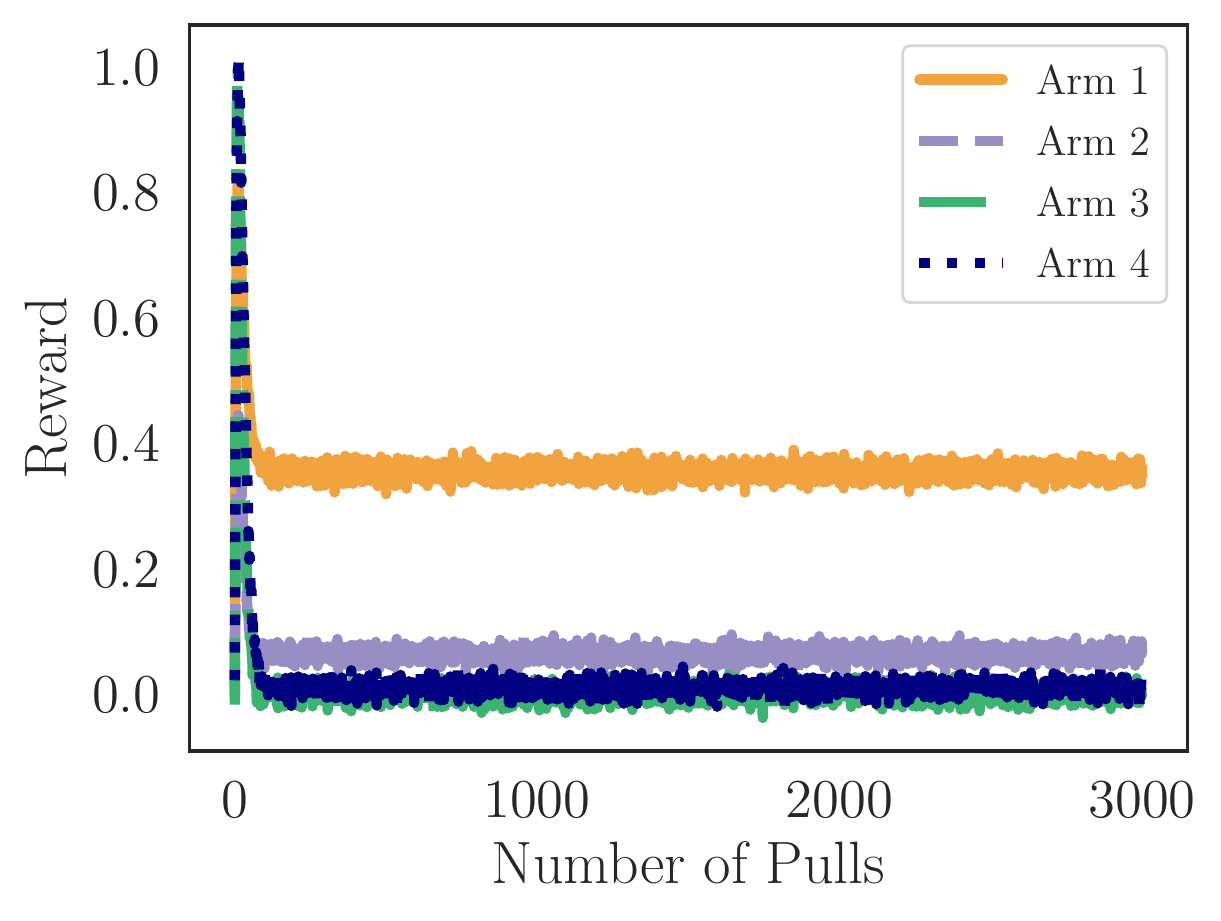}\hfill
     \includegraphics[width=0.48\linewidth]{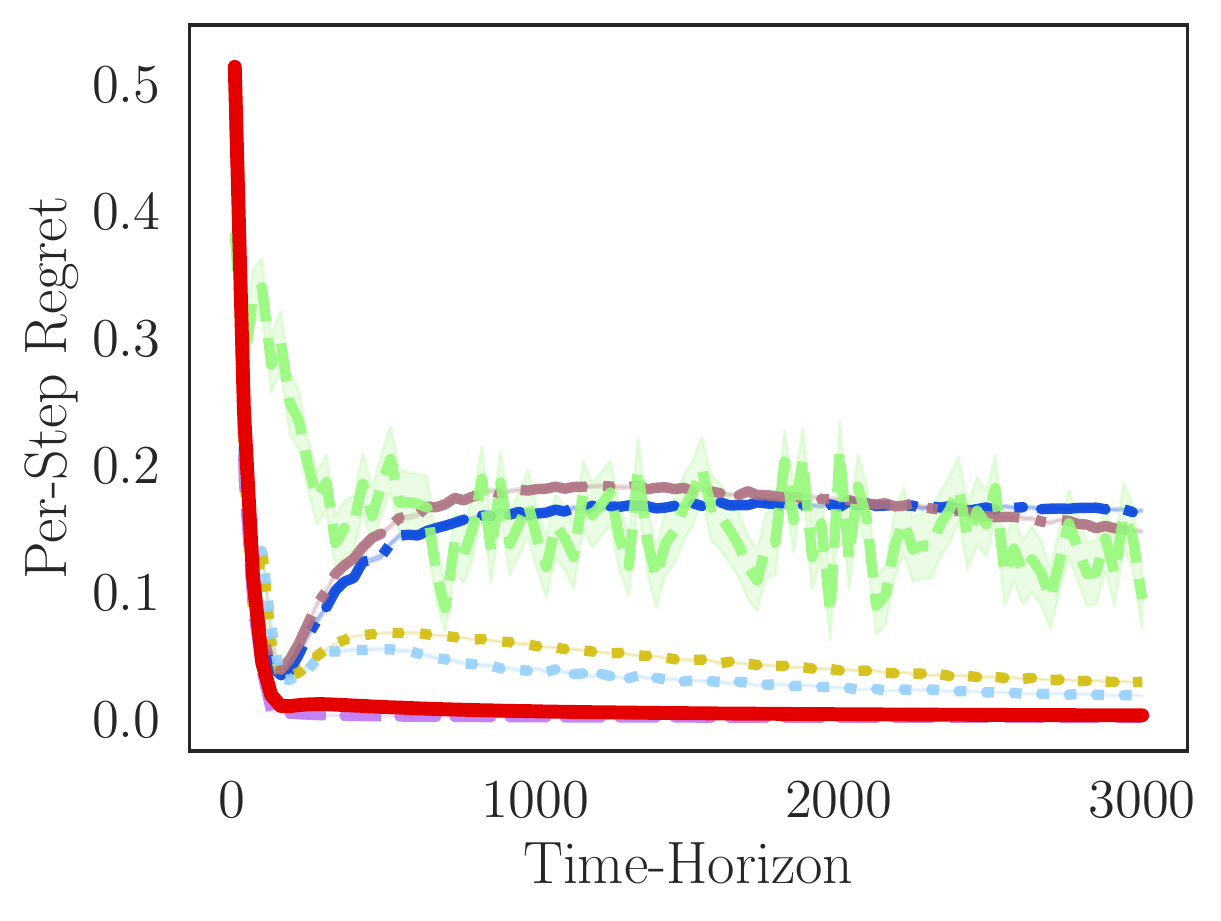}
     \caption{Noise with $\sigma = 0.01$}
  \end{subfigure}\vspace{1em}
  \begin{subfigure}[c]{\linewidth}
     \includegraphics[width=0.48\linewidth]{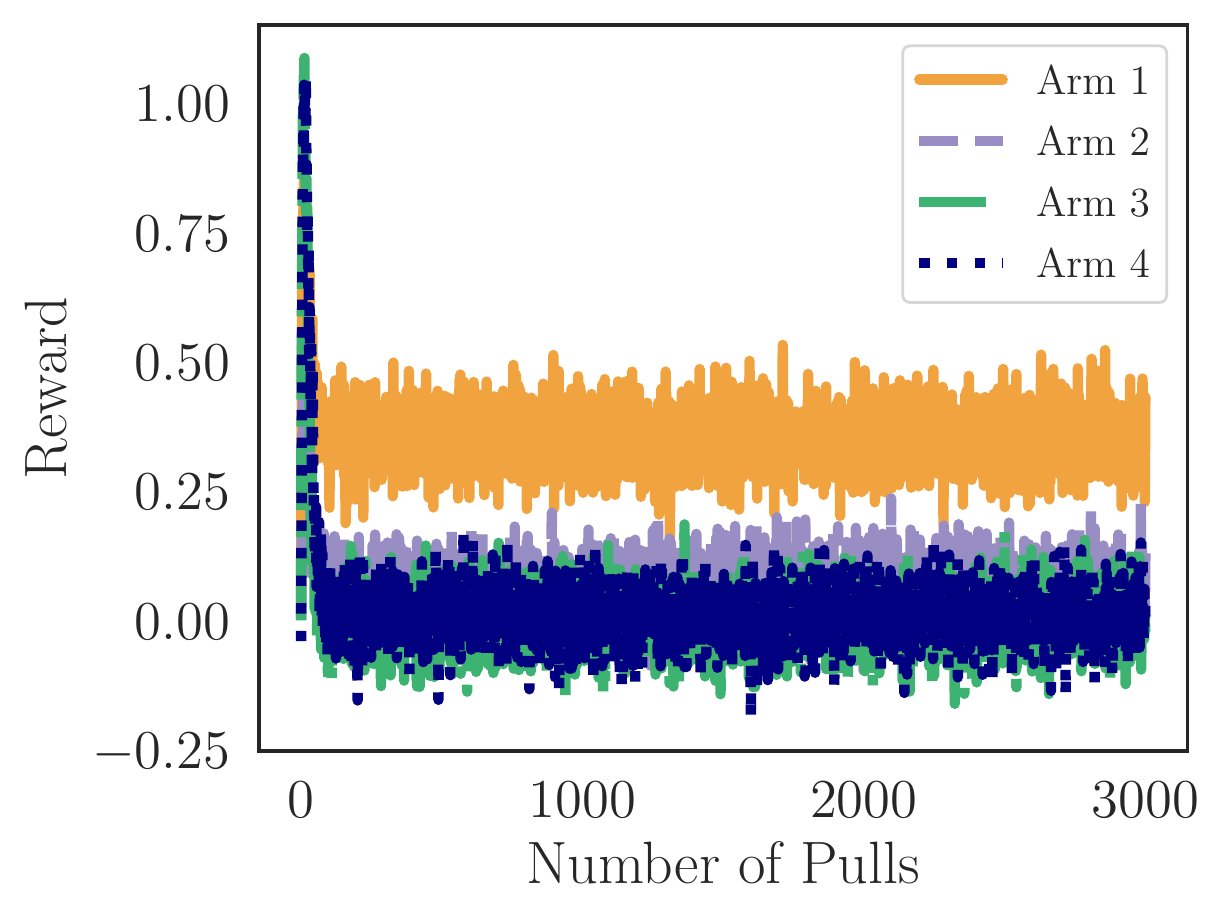}\hfill
     \includegraphics[width=0.48\linewidth]{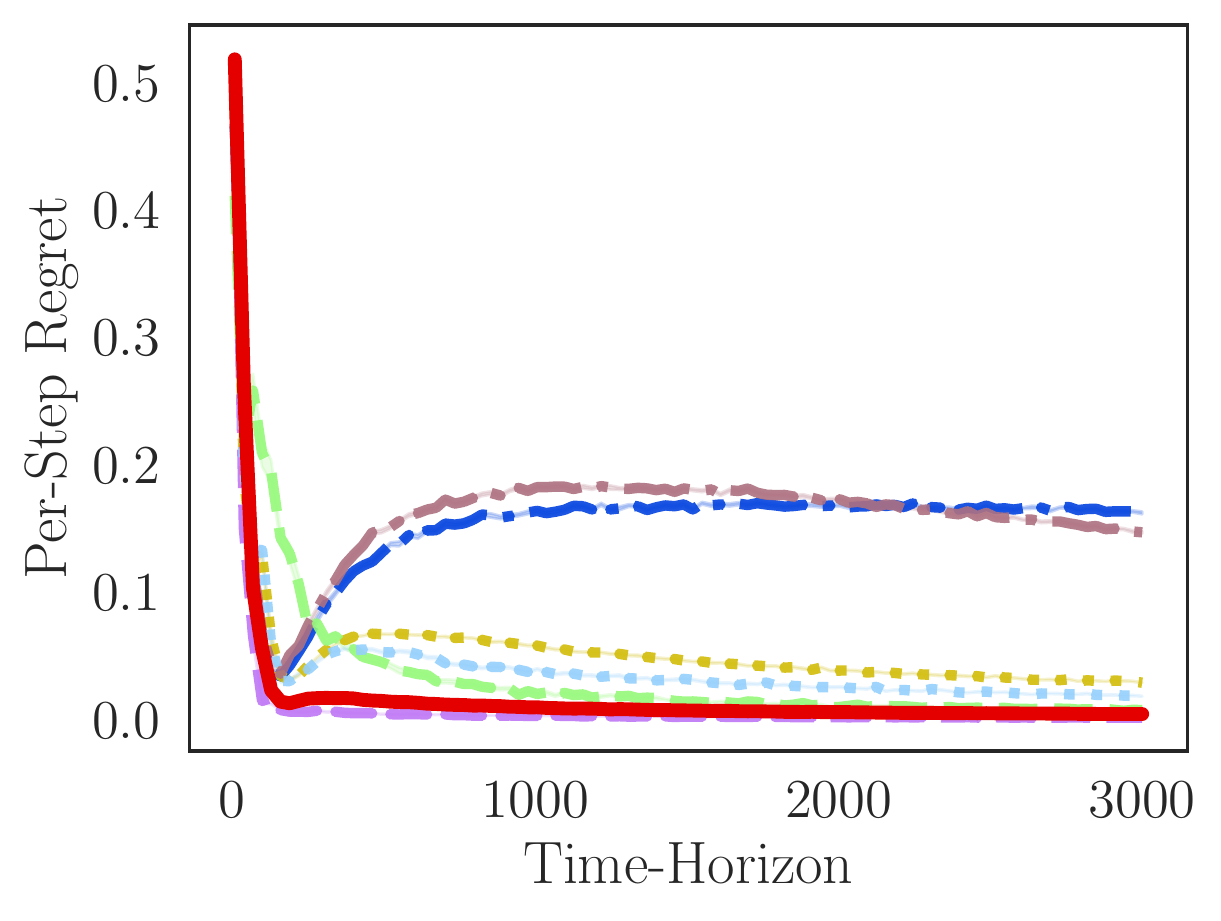}
     \caption{Noise with $\sigma = 0.05$}
  \end{subfigure}\vspace{1em}
\caption{The left-hand plots show the reward function of instance C of the simulated recommender system with simulated Gaussian noise. The right-hand plots show the per-step policy regret achieved by \algshort ({\protect\legendSPO}) compared to EXP3 ({\protect\legendEXP}), R-EXP3 ({\protect\legendREXP}), D-UCB ({\protect\legendDUCB}), SW-UCB ({\protect\legendSWUCB}), a one-step-optimistic ({\protect\legendOSO}), and a greedy algorithm ({\protect\legendGREEDY}).}
\label{fig:experiment_recommender_C}
\end{minipage}\hfill
\begin{minipage}[b]{.48\textwidth}
\end{minipage}
\end{figure*}
}

\end{document}